\documentclass{article}

\usepackage[utf8]{inputenc} %
\usepackage[T1]{fontenc}    %
\usepackage{url}            %
\usepackage{booktabs}       %
\usepackage{amsfonts}       %
\usepackage{nicefrac}       %
\usepackage{microtype}      %

\usepackage{enumitem}

\usepackage{algorithm}
\usepackage{verbatim}
\usepackage[noend]{algpseudocode}

\usepackage{setspace}

\usepackage{inconsolata}
\usepackage[scaled=.90]{helvet}
\usepackage{xspace}

\usepackage[letterpaper, left=1in, right=1in, top=1in, bottom=1in]{geometry}

\PassOptionsToPackage{hypertexnames=false}{hyperref}  %
\usepackage{parskip}

\usepackage[dvipsnames]{xcolor}
\usepackage[colorlinks=true, linkcolor=blue!70!black, citecolor=blue!70!black,urlcolor=black]{hyperref}
\usepackage{microtype}
\usepackage{hhline}

\usepackage{algorithm}

\usepackage{natbib}
\bibpunct{(}{)}{;}{a}{,}{,}

\usepackage{amsthm}
\usepackage{mathtools}
\usepackage{amsmath}
\usepackage{bbm}
\usepackage{amsfonts}
\usepackage{amssymb}

\usepackage{xpatch}

\theoremstyle{definition}  %

\newtheorem{lemma}{Lemma}[section]

\newtheorem{corollary}{Corollary}[section]
\newtheorem{proposition}{Proposition}[section]

\newtheorem{assumption}{Assumption}

\theoremstyle{plain}

\newtheorem{theorem}{Theorem}[section]
\newtheorem{definition}{Definition}[section]

\xpatchcmd{\proof}{\itshape}{\normalfont\proofnameformat}{}{}
\newcommand{\proofnameformat}{\bfseries}

\usepackage[nameinlink,capitalize]{cleveref}

\newcommand{\pref}[1]{\cref{#1}}
\newcommand{\pfref}[1]{Proof of \pref{#1}}

\crefformat{equation}{#2(#1)#3}
\Crefformat{equation}{#2(#1)#3}

\Crefformat{figure}{#2Figure #1#3}
\Crefformat{assumption}{#2Assumption #1#3}

\usepackage{xparse}

\ExplSyntaxOn
\DeclareDocumentCommand{\XDeclarePairedDelimiter}{mm}
 {
  \__egreg_delimiter_clear_keys: %
  \keys_set:nn { egreg/delimiters } { #2 }
  \use:x %
   {
    \exp_not:n {\NewDocumentCommand{#1}{sO{}m} }
     {
      \exp_not:n { \IfBooleanTF{##1} }
       {
        \exp_not:N \egreg_paired_delimiter_expand:nnnn
         { \exp_not:V \l_egreg_delimiter_left_tl }
         { \exp_not:V \l_egreg_delimiter_right_tl }
         { \exp_not:n { ##3 } }
         { \exp_not:V \l_egreg_delimiter_subscript_tl }
       }
       {
        \exp_not:N \egreg_paired_delimiter_fixed:nnnnn 
         { \exp_not:n { ##2 } }
         { \exp_not:V \l_egreg_delimiter_left_tl }
         { \exp_not:V \l_egreg_delimiter_right_tl }
         { \exp_not:n { ##3 } }
         { \exp_not:V \l_egreg_delimiter_subscript_tl }
       }
     }
   }
 }

\keys_define:nn { egreg/delimiters }
 {
  left      .tl_set:N = \l_egreg_delimiter_left_tl,
  right     .tl_set:N = \l_egreg_delimiter_right_tl,
  subscript .tl_set:N = \l_egreg_delimiter_subscript_tl,
 }

\cs_new_protected:Npn \__egreg_delimiter_clear_keys:
 {
  \keys_set:nn { egreg/delimiters } { left=.,right=.,subscript={} }
 }

\cs_new_protected:Npn \egreg_paired_delimiter_expand:nnnn #1 #2 #3 #4
 {%
  \mathopen{}
  \mathclose\c_group_begin_token
   \left#1
   #3
   \group_insert_after:N \c_group_end_token
   \right#2
   \tl_if_empty:nF {#4} { \c_math_subscript_token {#4} }
 }
\cs_new_protected:Npn \egreg_paired_delimiter_fixed:nnnnn #1 #2 #3 #4 #5
 {
  \mathopen{#1#2}#4\mathclose{#1#3}
  \tl_if_empty:nF {#5} { \c_math_subscript_token {#5} }
 }
\ExplSyntaxOff

\XDeclarePairedDelimiter{\supnorm}{
  left=\lVert,
  right=\rVert,
  subscript=\infty
  } %

\DeclarePairedDelimiter{\abs}{\lvert}{\rvert} %
\DeclarePairedDelimiter{\brk}{[}{]}
\DeclarePairedDelimiter{\crl}{\{}{\}}
\DeclarePairedDelimiter{\prn}{(}{)}
\DeclarePairedDelimiter{\nrm}{\|}{\|}
\DeclarePairedDelimiter{\tri}{\langle}{\rangle}

\DeclarePairedDelimiter{\ceil}{\lceil}{\rceil}
\DeclarePairedDelimiter{\floor}{\lfloor}{\rfloor}

\let\Pr\undefined

\DeclareMathOperator{\En}{\mathbb{E}}

\DeclareMathOperator{\Pr}{Pr}

\DeclareMathOperator*{\argmin}{arg\,min} %
\DeclareMathOperator*{\argmax}{arg\,max}

\newcommand{\mb}[1]{\boldsymbol{#1}}
\newcommand{\wt}[1]{\widetilde{#1}}
\newcommand{\wh}[1]{\widehat{#1}}
\newcommand{\wb}[1]{\widebar{#1}}

\def\ddefloop#1{\ifx\ddefloop#1\else\ddef{#1}\expandafter\ddefloop\fi}
\def\ddef#1{\expandafter\def\csname bb#1\endcsname{\ensuremath{\mathbb{#1}}}}
\ddefloop ABCDEFGHIJKLMNOPQRSTUVWXYZ\ddefloop
\def\ddefloop#1{\ifx\ddefloop#1\else\ddef{#1}\expandafter\ddefloop\fi}
\def\ddef#1{\expandafter\def\csname b#1\endcsname{\ensuremath{\mathbf{#1}}}}
\ddefloop ABCDEFGHIJKLMNOPQRSTUVWXYZ\ddefloop
\def\ddef#1{\expandafter\def\csname sf#1\endcsname{\ensuremath{\mathsf{#1}}}}
\ddefloop ABCDEFGHIJKLMNOPQRSTUVWXYZ\ddefloop
\def\ddef#1{\expandafter\def\csname c#1\endcsname{\ensuremath{\mathcal{#1}}}}
\ddefloop ABCDEFGHIJKLMNOPQRSTUVWXYZ\ddefloop
\def\ddef#1{\expandafter\def\csname h#1\endcsname{\ensuremath{\widehat{#1}}}}
\ddefloop ABCDEFGHIJKLMNOPQRSTUVWXYZ\ddefloop
\def\ddef#1{\expandafter\def\csname hc#1\endcsname{\ensuremath{\widehat{\mathcal{#1}}}}}
\ddefloop ABCDEFGHIJKLMNOPQRSTUVWXYZ\ddefloop
\def\ddef#1{\expandafter\def\csname t#1\endcsname{\ensuremath{\widetilde{#1}}}}
\ddefloop ABCDEFGHIJKLMNOPQRSTUVWXYZ\ddefloop
\def\ddef#1{\expandafter\def\csname tc#1\endcsname{\ensuremath{\widetilde{\mathcal{#1}}}}}
\ddefloop ABCDEFGHIJKLMNOPQRSTUVWXYZ\ddefloop

\newcommand{\ls}{\ell}
\newcommand{\ind}{\mathbbm{1}}    %
\newcommand{\pmo}{\crl*{\pm{}1}}

\newcommand{\veps}{\varepsilon}

\newcommand{\ldef}{\vcentcolon=}
\newcommand{\rdef}{=\vcentcolon}

\makeatletter
\newcommand{\DeclareMathActive}[2]{%
  \expandafter\edef\csname keep@#1@code\endcsname{\mathchar\the\mathcode`#1 }
  \begingroup\lccode`~=`#1\relax
  \lowercase{\endgroup\def~}{#2}%
  \AtBeginDocument{\mathcode`#1="8000 }%
}
\newcommand{\std}[1]{\csname keep@#1@code\endcsname}
\patchcmd{\newmcodes@}{\mathcode`\-\relax}{\std@minuscode\relax}{}{\ddt}
\AtBeginDocument{\edef\std@minuscode{\the\mathcode`-}}
\makeatother
\DeclareMathActive{*}{\ast}
\renewcommand{\ast}{\star} %

\newcommand{\starucblong}{star hull upper confidence bound\xspace}

\newcommand{\optionone}{\textsc{Option I}\xspace}
\newcommand{\optiontwo}{\textsc{Option II}\xspace}
\newcommand{\oracle}{\textsf{Oracle}\xspace}

\newcommand{\squarecb}{\textsf{SquareCB}\xspace}
\newcommand{\falcon}{\textsf{FALCON}\xspace}

\newcommand{\Dis}{\mathrm{Dis}}

\newcommand{\xdist}{\cD}
\newcommand{\rdist}{\bbP_{r}}

\newcommand{\algcomment}[1]{\textcolor{blue!70!black}{\small{\texttt{\textbf{//\hspace{2pt}#1}}}}}

\newcommand{\algcolor}[1]{\textcolor{blue!70!black}{#1}}

\newcommand{\RegExp}{\En\brk{\Reg}}
\newcommand{\Regbar}{\overline{\mathrm{Reg}}_T}
\newcommand{\Reg}{\mathrm{Reg}_T}

\newcommand{\Holder}{H{\"o}lder\xspace}

\newcommand{\trn}{\top}

\newcommand{\approxleq}{\lesssim}

\newcommand{\fhat}{\wh{f}}
\newcommand{\ahat}{\wh{a}}
\newcommand{\fbar}{\bar{f}}

\renewcommand{\ind}[1]{^{{\scriptscriptstyle(#1)}}}

\newcommand{\bigoh}{\cO}
\newcommand{\bigoht}{\wt{\cO}}
\newcommand{\bigom}{\Omega}
\newcommand{\bigomt}{\wt{\Omega}}
\newcommand{\bigthetat}{\wt{\Theta}}

\newcommand{\fstar}{f^{\star}}

\newcommand{\pistar}{\pi^{\star}}

\newcommand{\indic}{\mathbb{I}}

\newcommand{\gstar}{g^{\star}}

\newcommand{\overgeq}[1]{\overset{#1}{\geq{}}}

\newcommand{\overeq}[1]{\overset{#1}{=}}

\renewcommand{\Pr}{\bbP}

\newcommand{\aone}{a\ind{1}}
\newcommand{\atwo}{a\ind{2}}

\newcommand{\poly}{\mathrm{poly}}
\newcommand{\polylog}{\mathrm{polylog}}

\newcommand{\kl}[2]{D_{\mathrm{kl}}(#1\;\|\;#2)}
\newcommand{\Dphi}[2]{D_{\phi}(#1\;\|\;#2)}
\newcommand{\dphi}[2]{d_{\phi}(#1\;\|\;#2)}

\newcommand{\breg}[2]{D_{\cR}(#1\;\|\;#2)}

\newcommand{\Ber}{\mathrm{Ber}}
\newcommand{\dmid}{\;\|\;}

\newcommand{\conv}{\mathrm{conv}}

\newcommand{\comp}{\mathrm{c}}

\newcommand{\Picsc}{\Pi^{\mathrm{csc}}}

\newcommand{\mainalg}{\textsf{AdaCB}\xspace}

\newcommand{\Vf}{\mathbf{V}}
\newcommand{\Qf}{\mathbf{Q}}

\newcommand{\Qbar}{\widebar{\mathbf{Q}}}

\newcommand{\Vbar}{\widebar{\mathbf{V}}}

\newcommand{\Vstar}{\Vf^{\star}}
\newcommand{\Qstar}{\Qf^{\star}}

\newcommand{\Ssafe}{\cS_{\mathrm{safe}}}

\newcommand{\unif}{\mathrm{unif}}

\newcommand{\starhull}{\mathrm{star}}

\newcommand{\igw}{\textsf{IGW}}

\newcommand{\emi}{\psi}

\newcommand{\piunif}{\pi_{\mathrm{unif}}}
\newcommand{\cPhat}{\wh{\cP}}

\newcommand{\supp}{\mathrm{supp}}

\renewcommand{\Ssafe}[1][h]{\cS_{h;\mathrm{safe}}}

\newcommand{\blockalg}{\textsf{RegRL}\xspace}

\newcommand{\betaconf}{\beta_{\mathrm{conf}}}

\newcommand{\Pstar}{P^{\star}}

\newcommand{\Econf}{\cE_{\mathrm{conf}}}

\newcommand{\Ebar}{\widebar{\mathbf{E}}}

\newcommand{\clip}{\texttt{clip}}
\newcommand{\gap}{\Delta}
\newcommand{\gapcheck}{\check{\gap}}
\newcommand{\gapmin}{\gap_{\mathrm{min}}}

\newcommand{\gfilt}{\mathfrak{G}}
\newcommand{\hfilt}{\mathfrak{F}}

\renewcommand{\betaconf}[1][h]{\beta_{#1}}
\renewcommand{\betaconf}[1][h]{\beta_{#1}}
\newcommand{\cFhat}{\wh{\cF}}
\newcommand{\cFbar}{\widebar{\cF}}
\newcommand{\Vtil}{\wt{\Vf}}
\newcommand{\Qtil}{\wt{\Qf}}
\newcommand{\ftil}{\tilde{f}}
\newcommand{\betatil}{\tilde{\beta}}
\newcommand{\OP}{\textsf{OP}}
\newcommand{\ET}{\textsf{ET}}
\newcommand{\cveps}{c_{\veps,\delta;h+1}}

  \newcommand{\mathand}{\quad\text{and}\quad}

\newcommand{\betahat}{\hat{\beta}}

\def\multiset#1#2{\ensuremath{\left(\kern-.3em\left(\genfrac{}{}{0pt}{}{#1}{#2}\right)\kern-.3em\right)}}

\newcommand{\fbayes}{\fstar}
\newcommand{\nact}{A}
\newcommand{\K}{\nact}

\newcommand{\grad}{\nabla}
\newcommand{\diag}{\mathrm{diag}}

\newcommand{\alg}{\textsf{A}}

\newcommand{\dgraph}{d_{\cG}}

\newcommand{\iid}{i.i.d.\xspace}

\newcommand{\Fmax}{F_{\mathrm{max}}}

\renewcommand{\ls}{r}

  \newcommand{\pol}{\mathsf{pol}}
  \renewcommand{\csc}{\mathsf{csc}}
  \newcommand{\val}{\mathsf{val}}
  \newcommand{\PolicyDisL}[1]{\mb{\theta}^{\pol}_{\cD,\pistar}(\Pi,#1)}
  \newcommand{\PolicyDis}[1]{\mb{\theta}^{\pol}(\Pi,#1)}
    \newcommand{\PolicyDisA}{\mb{\theta}^{\pol}}
  
  \newcommand{\CostDis}[1]{\mb{\theta}^{\csc}(\Pi,#1)}
  \newcommand{\CostDisA}{\mb{\theta}^{\csc}
  }
  \newcommand{\ValueDisL}[2]{\mb{\theta}^{\val}_{\cD;\fstar}(\cF,#1,#2)}
    \newcommand{\ValueDisP}[2]{\mb{\theta}^{\val}_{\cD,p;\fstar}(\cF,#1,#2)}

  \newcommand{\ValueDis}[2]{\mb{\theta}^{\val}\prn*{\cF,#1,#2}}
  \newcommand{\ValueDisA}{\mb{\theta}^{\val}}
  \newcommand{\PolicyStarWL}{\underline{\mathfrak{s}}^{\pol}_{\pistar}(\Pi)}
  
  \newcommand{\PolicyStarL}{\mathfrak{s}^{\pol}_{\pistar}(\Pi)}
  
  \newcommand{\PolicyStarShort}{\mathfrak{s}^{\pol}}
  \newcommand{\ValueStarL}[1]{\mathfrak{s}^{\val}_{\fstar}(\cF,#1)}
  
  \newcommand{\ValueStarCL}[1]{\check{\mathfrak{s}}^{\val}_{\fstar}(\cF,#1)}
  
  \newcommand{\ValueStarWL}[2]{\underline{\mathfrak{s}}^{\val}_{\fstar}(\cF,#1,#2)}
  
  \newcommand{\PolicyEluderL}{\mathfrak{e}^{\pol}_{\pistar}(\Pi)}
  
  \newcommand{\ValueEluderL}[1]{\mathfrak{e}^{\val}_{\fstar}(\cF,#1)}
  \newcommand{\ValueEluder}[1]{\mathfrak{e}^{\val}(\cF,#1)}
  \newcommand{\ValueEluderCL}[1]{\check{\mathfrak{e}}^{\val}_{\fstar}(\cF,#1)}
  \newcommand{\ValueEluderC}[1]{\check{\mathfrak{e}}^{\val}(\cF,#1)}
  \newcommand{\ValueEluderWL}[2]{\underline{\mathfrak{e}}^{\val}_{\fstar}(\cF,#1,#2)}

  \newcommand{\btheta}{\mb{\theta}}

  \newcommand{\policydis}{policy disagreement coefficient\xspace}
  \newcommand{\cscdis}{cost-sensitive policy disagreement coefficient\xspace}
  \newcommand{\valuedis}{value function disagreement coefficient\xspace}  
  
  \newcommand{\policystar}{policy star number\xspace}
  \newcommand{\valuestar}{value function star number\xspace}

  \newcommand{\policyeluder}{policy eluder dimension\xspace}
  \newcommand{\valueeluder}{value function eluder dimension\xspace}

  \newcommand{\PolicyMinimax}{\mathfrak{M}^{\pol}(\cF,\veps,\theta)}
  \newcommand{\ValueMinimax}[1][\Delta]{\mathfrak{M}^{\val}(\cF,#1,\veps,\theta)}

\newcommand{\wcR}{\widehat{\mathcal{R}}}
\newcommand{\pReg}{{\mathrm{Reg}}}
\newcommand{\wReg}{\widehat{\mathrm{Reg}}}
\newcommand{\E}{{\mathbb{E}}}
\newcommand{\1}{\mathbb{I}}
\newcommand{\Prob}{\mathbb{P}}
\newcommand{\N}{\mathbb{N}}
\newcommand{\extq}[1]{{\varrho_{#1}}}
\newcommand{\exthq}[1]{{\widehat{\varrho}_{#1}}}
\newcommand{\crate}{c}

\newcommand{\uniformgc}{uniform Glivenko-Cantelli\xspace}
\newcommand{\setname}{candidate action set\xspace}

\newcommand{\cboracle}{\textsf{ConfBound}\xspace}
\newcommand{\dcoracle}{\textsf{ConfBoundDiff}\xspace}
\newcommand{\csoracle}{\textsf{CandidateSet}\xspace}
\newcommand{\cworacle}{\textsf{ConfWidth}\xspace}

\newcommand{\qhat}{\wh{q}}

 \let\underbar\undefined
\makeatletter
\let\save@mathaccent\mathaccent
\newcommand*\if@single[3]{%
  \setbox0\hbox{${\mathaccent"0362{#1}}^H$}%
  \setbox2\hbox{${\mathaccent"0362{\kern0pt#1}}^H$}%
  \ifdim\ht0=\ht2 #3\else #2\fi
  }
\newcommand*\rel@kern[1]{\kern#1\dimexpr\macc@kerna}
\newcommand*\widebar[1]{\@ifnextchar^{{\wide@bar{#1}{0}}}{\wide@bar{#1}{1}}}
\newcommand*\underbar[1]{\@ifnextchar_{{\under@bar{#1}{0}}}{\under@bar{#1}{1}}}
\newcommand*\wide@bar[2]{\if@single{#1}{\wide@bar@{#1}{#2}{1}}{\wide@bar@{#1}{#2}{2}}}
\newcommand*\under@bar[2]{\if@single{#1}{\under@bar@{#1}{#2}{1}}{\under@bar@{#1}{#2}{2}}}
\newcommand*\wide@bar@[3]{%
  \begingroup
  \def\mathaccent##1##2{%
    \let\mathaccent\save@mathaccent
    \if#32 \let\macc@nucleus\first@char \fi
    \setbox\z@\hbox{$\macc@style{\macc@nucleus}_{}$}%
    \setbox\tw@\hbox{$\macc@style{\macc@nucleus}{}_{}$}%
    \dimen@\wd\tw@
    \advance\dimen@-\wd\z@
    \divide\dimen@ 3
    \@tempdima\wd\tw@
    \advance\@tempdima-\scriptspace
    \divide\@tempdima 10
    \advance\dimen@-\@tempdima
    \ifdim\dimen@>\z@ \dimen@0pt\fi
    \rel@kern{0.6}\kern-\dimen@
    \if#31
      \overline{\rel@kern{-0.6}\kern\dimen@\macc@nucleus\rel@kern{0.4}\kern\dimen@}%
      \advance\dimen@0.4\dimexpr\macc@kerna
      \let\final@kern#2%
      \ifdim\dimen@<\z@ \let\final@kern1\fi
      \if\final@kern1 \kern-\dimen@\fi
    \else
      \overline{\rel@kern{-0.6}\kern\dimen@#1}%
    \fi
  }%
  \macc@depth\@ne
  \let\math@bgroup\@empty \let\math@egroup\macc@set@skewchar
  \mathsurround\z@ \frozen@everymath{\mathgroup\macc@group\relax}%
  \macc@set@skewchar\relax
  \let\mathaccentV\macc@nested@a
  \if#31
    \macc@nested@a\relax111{#1}%
  \else
    \def\gobble@till@marker##1\endmarker{}%
    \futurelet\first@char\gobble@till@marker#1\endmarker
    \ifcat\noexpand\first@char A\else
      \def\first@char{}%
    \fi
    \macc@nested@a\relax111{\first@char}%
  \fi
  \endgroup
}
\newcommand*\under@bar@[3]{%
  \begingroup
  \def\mathaccent##1##2{%
    \let\mathaccent\save@mathaccent
    \if#32 \let\macc@nucleus\first@char \fi
    \setbox\z@\hbox{$\macc@style{\macc@nucleus}_{}$}%
    \setbox\tw@\hbox{$\macc@style{\macc@nucleus}{}_{}$}%
    \dimen@\wd\tw@
    \advance\dimen@-\wd\z@
    \divide\dimen@ 3
    \@tempdima\wd\tw@
    \advance\@tempdima-\scriptspace
    \divide\@tempdima 10
    \advance\dimen@-\@tempdima
    \ifdim\dimen@>\z@ \dimen@0pt\fi
    \rel@kern{0.6}\kern-\dimen@
    \if#31
      \underline{\rel@kern{-0.6}\kern\dimen@\macc@nucleus\rel@kern{0.4}\kern\dimen@}%
      \advance\dimen@0.4\dimexpr\macc@kerna
      \let\final@kern#2%
      \ifdim\dimen@<\z@ \let\final@kern1\fi
      \if\final@kern1 \kern-\dimen@\fi
    \else
      \underline{\rel@kern{-0.6}\kern\dimen@#1}%
    \fi
  }%
  \macc@depth\@ne
  \let\math@bgroup\@empty \let\math@egroup\macc@set@skewchar
  \mathsurround\z@ \frozen@everymath{\mathgroup\macc@group\relax}%
  \macc@set@skewchar\relax
  \let\mathaccentV\macc@nested@a
  \if#31
    \macc@nested@a\relax111{#1}%
  \else
    \def\gobble@till@marker##1\endmarker{}%
    \futurelet\first@char\gobble@till@marker#1\endmarker
    \ifcat\noexpand\first@char A\else
      \def\first@char{}%
    \fi
    \macc@nested@a\relax111{\first@char}%
  \fi
  \endgroup
}
\makeatother

\let\oldparagraph\paragraph
\renewcommand{\paragraph}[1]{\oldparagraph{#1.}}

\title{Instance-Dependent Complexity of Contextual Bandits and
  Reinforcement Learning: A Disagreement-Based Perspective}

\author{%
  	Dylan J.\ Foster\\
        {\small\texttt{dylanf@mit.edu}}\\ 
	  \and
Alexander Rakhlin\\
{\small\texttt{rakhlin@mit.edu}}\\
	  \and
David Simchi-Levi\\
{\small\texttt{dslevi@mit.edu}}\\
	  \and
Yunzong Xu\\
{\small\texttt{yxu@mit.edu}}\\
\and
~\\
{\large Massachusetts Institute of Technology}
              }
\date{}

\begin{document}

\maketitle

\begin{abstract}
In the classical multi-armed bandit problem, \emph{instance-dependent}
algorithms attain improved performance on ``easy'' problems with a gap
between the best and second-best arm. Are similar guarantees possible
for contextual bandits? While positive results are known for certain
special cases, there is no general theory characterizing when and how
instance-dependent regret bounds for contextual bandits can be
achieved for rich, general classes of policies. We introduce
a family of complexity measures that are both sufficient and necessary to obtain instance-dependent regret bounds. We then introduce new oracle-efficient algorithms which adapt to the gap whenever possible, while also
attaining the minimax rate in the worst case. Finally, we provide structural
results that tie together a number of complexity measures previously
proposed throughout contextual bandits, reinforcement learning, and active learning and elucidate their role in determining the
optimal instance-dependent regret. In a large-scale empirical
evaluation, we find that our approach typically gives superior results for
challenging exploration problems.

Turning our focus to reinforcement learning with function
approximation, we develop new oracle-efficient algorithms for
reinforcement learning with rich observations that obtain optimal gap-dependent sample complexity.

\end{abstract}

{
\hypersetup{linkcolor=black}
\tableofcontents
}
\addtocontents{toc}{\protect\setcounter{tocdepth}{2}}

\section{Introduction}
\label{sec:intro}
How can we adaptively allocate measurements to exploit problem structure in the presence of rich,
  high-dimensional, and potentially stateful contextual information?
  In this paper, we investigate this question in the \emph{contextual bandit} problem
  and its stateful relative, the problem of \emph{reinforcement
    learning with rich observations}.

The contextual bandit is a fundamental problem in sequential decision
making. At each round, the learner receives a \emph{context}, selects
an \emph{action}, and receives a \emph{reward}; their goal is to
select actions so as to maximize the total long-term reward. This
model has been successfully deployed in news article recommendation
\citep{li2010contextual,agarwal2016making}, where actions represent
articles to display and rewards represent clicks, and healthcare \citep{tewari2017ads,bastani2020online}, where actions
represent treatments to prescribe and rewards represent the patient's
response. Reinforcement learning with rich observations \citep{krishnamurthy2016pac,jiang2017contextual}
is a substantially more challenging generalization in which the learner's actions
influence the evolution of the contexts, and serves as a stylized model for reinforcement learning with
function approximation.

For both settings, our aim is to develop \emph{instance-dependent}
algorithms that adapt to gaps between actions in the underlying reward function to
obtain improved regret. In the classical (non-contextual) multi-armed
bandit problem, this issue has enjoyed
extensive investigation beginning with the work of
\cite{lai1985asymptotically}. Here, it is well-understood that when
the mean reward function admits a constant gap between the best and
second-best action, well-designed algorithms can obtain logarithmic (in $T$, the number of
rounds) regret, which offers significant improvement over the
worst-case minimax rate of $\sqrt{T}$. Subsequent work has developed a
sharp understanding of optimal instance-dependent regret, both
asymptotically and with finite samples
\citep{burnetas1996optimal,garivier2016explore,kaufmann2016complexity,lattimore2018refining,garivier2019explore}. Beyond
the obvious appeal of lower regret, instance-dependent algorithms are
particularly compelling for applications such as clinical trials---where
excessive randomization may be undesirable or unethical---because they
identify and eliminate suboptimal actions more quickly than algorithms
that only aim for worst-case optimality.

We take the first step towards developing a similar theory for
contextual bandits and reinforcement learning with general function
approximation. We focus on the ``realizable'' or ``well-specified''
setting in
which the learner has access to a class of regression functions $\cF$ that
is flexible enough to capture the true reward function or value
function. Our aim is to develop \emph{learning-theoretic}
guarantees for rich, potentially nonparametric function classes that
1) scale only with the statistical capacity of the class, and 2) are
efficient in terms of basic computational primitives for the class.

For contextual bandits, instance-dependent regret bounds are not
well-understood. Positive results are known for simple classes of
functions such as linear classes
\citep{dani2008stochastic,abbasi2011improved,hao2019adaptive} or
nonparametric Lipschitz/\Holder classes \citep{rigollet2010nonparametric,perchet2013multi,hu2020smooth}. On the other hand, for arbitrary finite
function classes, it is known that gap-dependent regret bounds are not
possible in general \citep{foster2020beyond}. One line of work develops
algorithms which attain instance-dependent bounds for general classes under additional structural assumptions or distributional assumptions
\citep{russo2013eluder,bietti2018contextual,foster2018practical}, but
it is not clear whether these assumptions are fundamental (in
particular, they are not required to obtain minimax rates). For reinforcement learning, the situation is
more dire: while instance-dependent rates have been explored in the finite state/action
setting \citep{burnetas1996optimal,tewari2008optimistic,ok2018exploration,simchowitz2019non},
very little is known for the general setting with high
dimensional states and function approximation.

Beyond the basic issue of what instance-dependent rates can be
achieved for general function classes, an important question is
whether they can be achieved efficiently, using practical
algorithms. A recent line of work
\citep{foster2018practical,foster2020beyond,simchi2020bypassing,xu2020upper}
develops algorithms that are efficient in terms calls to an oracle for
(offline/online) \emph{supervised regression}. A secondary goal in
this work is to develop \emph{practical} instance-dependent algorithms
based on this primitive.

~\\
Altogether, our central
questions are:
\begin{enumerate}
\item For contextual bandits and reinforcement learning with rich observations, what properties of the function class
  enable us to adapt to the gap, and what are the fundamental limits?
\item Can we adapt to the gap \emph{efficiently}?
\item More ambitiously, can we get the \emph{best of both worlds}: Adapt to
  the gap and obtain the minimax rate simultaneously?
\end{enumerate}

For contextual bandits, we address each of these issues. We introduce
a family of new complexity measures which are both necessary (in a
certain sense) and sufficient to obtain fast gap-dependent regret bounds. We
introduce new oracle-efficient algorithms which adapt to the gap and
to these complexity measures whenever possible, while also
obtaining the minimax rate. We prove new structural results which---in
conjunction with our lower bounds---tie together a number of
complexity measures previously proposed in contextual bandits,
reinforcement learning, and active learning and provide new insight into their role in determining the
optimal instance-dependent regret. We then extend these complexity measures
to reinforcement learning with function approximation and give new
oracle-efficient algorithms that adapt to them. Overall, our results for RL are
somewhat less complete, but we believe they suggest a number of
exciting new directions for future research.

\subsection{Overview of Results: Contextual Bandits}
\subsubsection{Contextual Bandit Setup}

We consider the following stochastic contextual bandit protocol. At each round $t\in\brk*{T}$, the learner observes a context $x_t\in\cX$,
selects an action $a_t\in\cA$, then observes a reward
$\ls_t(a_t)\in\brk*{0,1}$. %
We assume that contexts are drawn
\iid from a fixed but unknown distribution $\xdist$, and that each reward function $\ls_t:\cA\to\brk*{0,1}$ is drawn independently from a fixed but unknown 
\emph{context-dependent} distribution
$\bbP_{\ls}(\cdot\mid{}x_t)$. We consider finite actions, with
$A\ldef{}\abs*{\cA}$.\footnote{We refer to each pair $(\xdist, \rdist)$ for the
        contextual bandit problem as an \emph{instance}.}

We assume that the learner has access to a class of value functions
$\cF\subset(\cX\times\cA\to\brk*{0,1})$ (e.g., regression trees or
neural networks) that is flexible enough to model the true reward
distribution. In particular, we make the following standard
\emph{realizability} assumption \citep{chu2011contextual,agarwal2012contextual,foster2018practical}.
\begin{assumption}[Realizability]
  \label{ass:realizability}
There exists a function $\fstar\in\cF$ such that
$\fstar(x,a) = \En\brk*{\ls(a)\mid{}x}$.
\end{assumption}
For each regression function $f\in\cF$,
let $\pi_f(x)=\argmax_{a\in\cA}f(x,a)$ denote the induced policy (with
ties broken arbitrarily, but consistently), and
let $\Pi=\crl*{\pi_f\mid{}f\in\cF}$ be the induced policy class. The goal
of the learner is to ensure low \emph{regret} to the optimal policy:
\begin{equation}
  \label{eq:regret}
\Reg=  
\sum_{t=1}^{T}\ls_t(\pistar(x_t))
- \sum_{t=1}^{T}\ls_t(a_t),
\end{equation}
where $\pistar\ldef{}\pi_{\fstar}$. For simplicity, we assume that $\argmax_{a\in\cA}\fstar(x,a)$ is unique
        for all $x$, but our results extend when this is not
        the case.

        \paragraph{Reward gaps and instance-dependent regret bounds}
        Consider the simple case where $\cF$ is finite. For general
        finite classes $\cF$ under \pref{ass:realizability}, the minimax
rate for contextual bandits is $\Theta\prn{\sqrt{\K{}T\log\abs{\cF}}}$
\citep{agarwal2012contextual}. The main question we investigate is to
what extent this rate can be improved when the instance has a
\emph{uniform gap}\footnote{This is sometimes referred to as
  the \emph{Massart noise condition}, which has been widely studied in statistical learning theory in the context of obtaining faster rates for classification.}
  in the sense that for all $x\in\cX$,
\begin{equation}
  \label{eq:gap}
  \fstar(x,\pistar(x))
  -\fstar(x,a)\geq\Delta\quad\forall{}a\neq\pistar(x).
\end{equation}
For multi-armed bandits, the minimax rate is
$\Theta\prn{\sqrt{\K{}T}}$, but instance-dependent algorithms can
achieve a \emph{logarithmic} regret bound of the form
  $\Reg\leq{}\bigoh\prn[\big]{\frac{\K\log{}T}{\Delta}}$
when the gap is $\Delta$, and this is optimal
\citep{garivier2019explore}. Moving to contextual bandits, a natural guess
would be that we can achieve\footnote{We use $\bigoht(\cdot)$ and
  $\bigomt(\cdot)$ to suppress $\polylog(T)$ factors; see
  \pref{sec:notation} for details.}
\begin{equation}
  \RegExp=\bigoht(1)\cdot{}\frac{\K\log\abs{\cF}}{\Delta}.\label{eq:log_regret}
\end{equation}
This is impossible in a fairly strong sense: \cite{foster2020beyond}
show that exist
function classes $\cF$ for which any algorithm must have\footnote{\cite{foster2020beyond}
  prove this lower bound for adversarial contexts. Our
  \pref{thm:disagreement_lb} implies an analogous lower bound for
  stochastic contexts.}
\[\RegExp=\bigom(1)\cdot\frac{\abs{\cF}}{\Delta}.\]
Since $\cF$ is exponentially large for most models, polynomial
dependence on $\abs{\cF}$ is unacceptable. The natural question then,
and the one we address, is what structural properties of $\cF$ allow
for bounds of the form \pref{eq:log_regret} that scale only
logarithmically with the size of the value function
class. We primarily present results on finite classes for simplicity, but our lower bounds
and structural results concern infinite classes, and our algorithms
make no assumption on the structure.

\subsubsection
{Disagreement-Based Guarantees}
We show that variants of the \emph{disagreement coefficient}, a key parameter
in empirical process theory and active learning \citep{alexander1987central,hanneke2015minimax}, play a fundamental
role in determining the optimal gap-dependent regret bounds for
contextual bandits with rich function classes.

Our most basic results concern a parameter we call the \emph{policy
  disagreement coefficient},\footnote{In fact, for binary actions the policy disagreement
  coefficient is the same as the usual disagreement coefficient from
  active learning \citep{hanneke2015minimax}; we adopt the name
  \emph{\policydis} only to distinguish from
  other parameters we introduce.} defined as 
\begin{equation}
  \label{eq:policy_disagreement}
  \PolicyDisL{\veps_0}= \sup_{\veps\geq{}\veps_0}\frac{\bbP_{\cD}\prn*{x:\exists\pi\in\Pi_{\veps}: \pi(x)\neq\pistar(x)}}{\veps},
\end{equation}
where $\Pi_{\veps}\ldef{}\crl*{\pi\in\Pi:
  \bbP_{\cD}\prn{\pi(x)\neq\pistar(x)}\leq{}\veps}$; when $\cD$ and $\pistar$ are
clear from context we abbreviate to $\PolicyDis{\veps_0}$.
This parameter, sometimes called \emph{Alexander's
  capacity function}, dates back to \cite{alexander1987central}, and
was rediscovered and termed the disagreement coefficient in the context of active learning by
\cite{hanneke2007bound,hanneke2011rates}. In empirical process theory
and statistical learning, the disagreement coefficient grants control over the fine-grained behavior of VC
classes \citep{gine2006concentration,raginsky2011lower,zhivotovskiy2016localization},
and primarily determines whether certain logarithmic terms can appear
in excess risk bounds for empirical risk minimization (ERM) and other
algorithms under low-noise conditions. In active learning, the
disagreement coefficient plays a more critical role, as it provides a
sufficient (and weakly necessary) condition under which one can
achieve label complexity logarithmic in the target precision
\citep{hanneke2007bound,hanneke2011rates,raginsky2011lower,hanneke2014theory,hanneke2015minimax}.

Informally, the \policydis
measures how likely we are to encounter a context on which \emph{some}
near-optimal policy disagrees with $\pistar$. Low disagreement
coefficient means that all the near-optimal policies deviate from
$\pistar$ only in a small, shared region of the context space, while
large disagreement coefficient means that the points on which
disagreement occurs are more prevalent throughout the context space
(w.r.t $\cD$), so that
many samples are required to rule out all of these policies.

We introduce a new contextual bandit algorithm, \mainalg, which adapts
to the gap whenever the \policydis is bounded. In particular, we show
the following.
\newtheorem*{thm:informal1}{Theorem \ref*{thm:disagreement_ub} (informal)}
\begin{thm:informal1}
For all instances, \mainalg ensures that
    \begin{equation}
      \label{eq:disagreement_ub_short}
      \RegExp=\bigoht(1)\cdot\min_{\varepsilon>0}\max\left\{\varepsilon\Delta T, \frac{\PolicyDis{\veps}\cdot{}\K\log|\cF|}{\Delta}\right\}
    \end{equation}
    with no prior knowledge of $\Delta$ or $\PolicyDis{\veps}$.
  \end{thm:informal1}
\pref{thm:disagreement_ub} is a best-of-both-worlds guarantee. In the
worst case, we have $\PolicyDis{\veps}\leq{}1/\veps$, so that
\pref{eq:disagreement_ub_short} becomes
$\bigoht\prn{\sqrt{\K{}T\log\abs{\cF}}}$, the minimax rate. However,
if $\PolicyDis{\veps}=\polylog(1/\veps)$, then \pref{eq:disagreement_ub_short}
ensures that \[
  \RegExp = \bigoht(1)\cdot{}\frac{A\log\abs{\cF}}{\Delta},
\]
so that \mainalg enjoys logarithmic regret. We emphasize that while
\pref{thm:disagreement_ub} concerns finite classes, this is only a
stylistic choice: \mainalg places no assumption on the structure of
$\cF$, and the analysis trivially generalizes by replacing
$\log\abs*{\cF}$ with standard learning-theoretic complexity measures
such as the pseudodimension.

While this is certainly encouraging, it is not immediately clear
whether the rate in \pref{eq:disagreement_ub_short} is fundamental. To this end, we prove that dependence on the disagreement
coefficient is qualitatively necessary.

\newtheorem*{thm:informal2}{Theorem \ref*{thm:disagreement_lb} (informal)}
\begin{thm:informal2}
    For any $\K{}\in\bbN$, $\Delta>0$, $\veps>0$, and functional
    $\PolicyDis{\veps}$, there exists a function class $\cF$ with $\K{}$ actions and a
  distribution over realizable instances with uniform gap $\Delta$ such that
  any algorithm has
  \[
    \RegExp=\bigomt(1)\cdot{}\min_{\veps>0}\max\crl*{\veps\Delta{}T, \frac{\PolicyDis{\veps}\cdot{}\K\log\abs{\cF}}{\Delta}}.
  \]
\end{thm:informal2}
\pref{thm:disagreement_lb} shows that the regret bound
\pref{eq:disagreement_ub_short} attained by \mainalg cannot be improved
without further assumptions on $\cF$. However, it leaves the
possibility of more refined complexity measures that are tighter than
$\PolicyDis{\veps}$ for most instances, yet coincide on the
construction that realizes the lower bound in
\pref{thm:disagreement_lb}. To this end, we introduce a second complexity measure, the \emph{\valuedis},
which can exploit the scale-sensitive nature of the value function class
$\cF$ to provide tighter bounds. The \valuedis is defined as
\begin{equation}
  \label{eq:value_disagreement}
  \ValueDisL{\Delta_0}{\veps_0} = \sup_{\Delta>\Delta_0,\veps>\veps_0}\sup_{p:\cX\to\Delta(\cD)}\frac{\Delta^{2}}{\veps^{2}}\bbP_{\cD,p}\prn*{
    \exists{}f\in\cF: \abs*{f(x,a)-\fstar(x,a)}>\Delta,\;\;\nrm*{f-\fstar}_{\cD,p}\leq\veps
    },
  \end{equation}
  where
  $\nrm*{f}_{\cD,p}^{2}\ldef{}\En_{x\sim\cD,a\sim{}p(x)}\brk{f^{2}(x)}$.
  We abbreviate
  $\ValueDis{\Delta_0}{\veps_0}\equiv\ValueDisL{\Delta_0}{\veps_0}$
  when the context is clear. The
  key difference from the \policydis is that rather than using a
  binary property ($\pi(x)\neq{}\pistar(x)$) to measure disagreement,
  we use a more refined scale-sensitive notion: Two functions $f$ and
  $\fstar$ are said to $\Delta$-disagree on $(x,a)$ if
  $\abs*{f(x,a)-\fstar(x,a)}>\Delta$, and the \valuedis simply measures how likely we are to encounter a context for which a value
  function that is $\veps$-close to $\fstar$ in $L_2$ distance
  $\Delta$-disagrees from it (for a worst-case action distribution). This refined view leads to
  tighter guarantees for common function classes. For example, when
  $\cF$ is a linear function class, i.e.
  $\cF = \crl*{(x,a)\mapsto\tri*{w,\phi(x,a)}\mid{}w\in\bbR^{d}}$ for
  a fixed feature map $\phi(x,a)$, the \policydis is only bounded for
  sufficiently regular distributions, whereas the \valuedis is always bounded by $d$.

  We show that \mainalg, with a slightly different parameter configuration, can adapt to \valuedis in a
  best-of-both-worlds fashion.

\newtheorem*{thm:informal3}{Theorem \ref*{thm:value_disagreement_ub} (informal)}
\begin{thm:informal3}
  For all instances, \mainalg ensures that
  \[
    \RegExp=\bigoht(1)\cdot{}\min\crl*{
      \sqrt{\K{}T\log\abs{\cF}}, \frac{\ValueDis{\Delta/2}{\veps_T}\cdot{}A\log\abs{\cF}}{\Delta}
      },
  \]
where $\veps_T\propto\sqrt{\log\abs{\cF}/T}$.
\end{thm:informal3}
We show
(\pref{thm:value_disagreement_lb}) that this dependence on
$\ValueDis{\Delta}{\veps}$ is qualitatively necessary, meaning that
\mainalg is adapts near-optimally without additional assumptions.

Beyond contextual bandits, our scale-sensitive generalization of
the disagreement coefficient is new to both empirical process theory
and active learning to our knowledge, and may be of
independent interest.

\subsubsection{Distribution-Free Guarantees and Structural Results}

  \begin{figure}[tp]
  \centering
  \includegraphics[width=.85\textwidth]{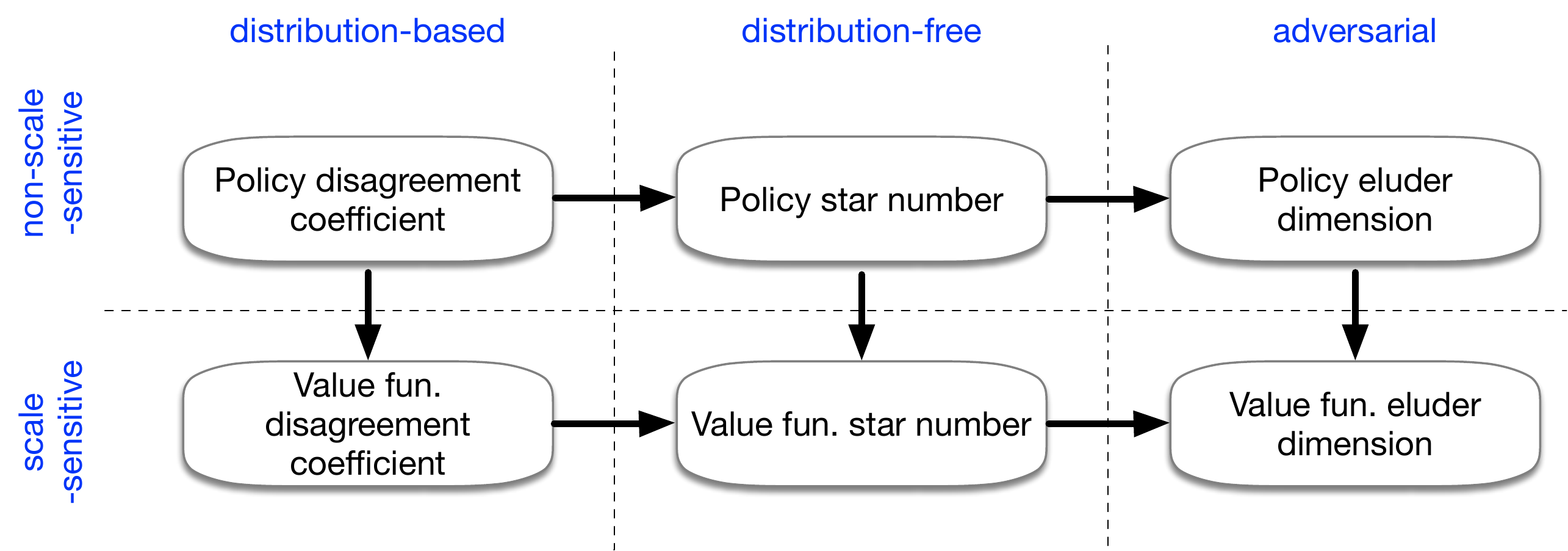}
  \caption{Relationship between complexity measures.}
  \label{fig:disagreement}
\end{figure}

While the distribution-dependent nature of our disagreement-based
upper bounds can lead to tight guarantees for benign distributions,
it is natural to ask: \emph{For what classes $\Pi$ (resp. $\cF$) can we ensure the
policy (resp. value) disagreement coefficient is bounded for any
distribution $\cD$?} \cite{hanneke2015minimax} show that the \policydis
is always bounded by a combinatorial parameter for $\Pi$ called the (policy) \emph{star number}.\footnote{In fact, the
  star number exactly coincides with the worst-case value of the
  disagreement coefficient over all possible distributions and scale
  parameters.} An immediate consequence (via
\pref{thm:disagreement_ub}) is that \mainalg enjoys logarithmic regret
even in the \emph{distribution-free} setting for classes with bounded
\policystar. More interestingly, we show (\pref{thm:star_lb_simple}) that for any class
$\Pi$, bounded \policystar is \emph{necessary} to obtain logarithmic
regret in the worst-case (with
respect to both $\cD$ and the class $\cF$ realizing $\Pi$). Thus, we
have the following characterization.
\newtheorem*{thm:informal4}{Theorem (informal)}
\begin{thm:informal4}
For any policy class $\Pi$, bounded \policystar is necessary and sufficient to obtain logarithmic
regret.
\end{thm:informal4}
Compared to our disagreement-based lower bounds, which rely on
specially designed function classes, this lower bound holds for
\emph{any} policy class.

This characterization motivates us to define a scale-sensitive
analogue of the star number called the \emph{\valuestar}. The
\valuestar is a new combinatorial parameter even within the broader
literature on active learning and empirical process theory, and we
show (\pref{thm:disagreement_to_star}) that it bounds the \valuedis for all choices of the context
distribution $\cD$ and scale parameter $\veps$. We then show
(\pref{thm:ss_star_lb}) that a weak version of the \valuedis is
\emph{necessary} to obtain logarithmic regret for worst-case context
distributions, leading to the following characterization.
\newtheorem*{thm:informal5}{Theorem (informal)}
\begin{thm:informal5}
  For any value function class $\cF$, bounded \valuestar
  is (weakly)
  necessary and sufficient to obtain logarithmic regret.
\end{thm:informal5}

The \valuestar is closely related to---and in particular always upper
bounded by---the (value function) \emph{eluder dimension} of
\cite{russo2013eluder}. The eluder dimension was introduced to prove
regret bounds for the generalized UCB algorithm and Thompson sampling
for contextual bandits with adversarial contexts, and more recently
has been used to analyze algorithms for reinforcement
learning with function approximation
\citep{osband2014model,wen2017efficient,ayoub2020model,wang2020provably}. An immediate consequence of the (disagreement coefficient)$\;\leq\;$(star
number)$\;\leq\;$(eluder dimension) connection is that boundedness of
the eluder dimension suffices to obtain logarithmic regret with
\mainalg. Unlike the star number though, bounded eluder dimension is not
required for the stochastic setting we consider. However, building on
our previous lower bounds, we show (\pref{thm:eluder_lb}) that a weak
version of the eluder dimension is \emph{necessary} to obtain
logarithmic regret under adversarial contexts, and give a tighter
analysis of the generalized UCB algorithm to show that it attains this
rate (this is not a best-of-both-worlds guarantee). This
result places the eluder dimension on more solid footing and shows that
while it is not required for minimax rates, it plays a fundamental
role for instance-dependent rates.

The relationship between all of our complexity measures,
old and new, is summarized in \pref{fig:disagreement}. Beyond
expanding the scope of settings for which logarithmic regret is
achievable, we hope our structural results and lower bounds provide a
new lens through which to understand existing algorithms and
instance-dependent rates, and provide new clarity.

As a disclaimer, we mention that the primary goal of this work is to understand how \emph{contextual information} shapes the optimal
instance-dependent rates for contextual bandits. We believe that this
question is challenging and interesting even in the finite-action
regime (in fact, even when $\K=2$!) and as such, we do not focus on obtaining
optimal dependence on $\K$ in our upper or lower bounds, nor do we
handle infinite actions. Fully understanding the interplay between
contexts and actions is a fascinating open problem, and we hope to
see this addressed in future work.

\subsubsection{Efficiency and Practical Peformance}
\paragraph{Computational efficiency}
Our main algorithm, \mainalg, is \emph{oracle-efficient}. That is, it
accesses the value function class $\cF$ only through a \emph{weighted
  least squares regression oracle} capable of solving problems of the
form
\begin{equation}
  \tag{RO}
  \oracle(\cH) = \argmin_{f\in\cF}\sum_{(w,x,a,y)\in\cH}w\prn*{f(x,a)-y}^{2}
\end{equation}
for a given set $\cH$ %
of examples $(w,x,a,y)$ where $w\in\bbR_{+}$ specifies the example weight. This makes
the algorithm highly practical, as it can be combined with any
out-of-the-box algorithm for supervised regression for the model of
interest. To achieve this guarantee we build on recent work
\citep{foster2020beyond,simchi2020bypassing} which uses a probability
selection scheme we call \emph{inverse gap weighting}
\citep{abe1999associative} to build oracle-efficient algorithms that
attain the minimax rate $\sqrt{\K{}T\log\abs{\cF}}$ using online and
offline oracles, respectively. We build on the analysis of
\cite{simchi2020bypassing}, and our algorithm carefully
combines the inverse gap weighting 
technique with another reduction \citep{krishnamurthy2017active,foster2018practical} which leverages the regression oracle to
compute confidence bounds that can be used to eliminate
actions.%

\paragraph{Empirical performance}
We replicated the large-scale empirical contextual bandit evaluation
setup of
\cite{bietti2018contextual}, which compares a number of state-of-the
art general-purpose contextual bandit algorithms across more than 500
datasets. We found that our new algorithm, \mainalg, typically gives
comparable or superior results to existing baselines, particularly on challenging
datasets with many actions.

\subsection{Overview of Results: Reinforcement Learning}

\subsubsection{Block MDP Setup}\label{sec:rl-setup}

Building on our contextual bandit results, we provide disagreement-based guarantees for episodic reinforcement learning
with function approximation in a model called the \emph{block MDP}
\citep{krishnamurthy2016pac,du2019latent}, which is an important
type of \emph{contextual decision process} \citep{jiang2017contextual}.

The block MDP may be thought of as a generalization of the contextual
bandit problem. Each round of interaction is
replaced by an \emph{episode} of length $H$. While the initial context
$x_1$ (now referred to as a \emph{state}) in each episode is drawn
i.i.d. as in the contextual bandit, the evolution of the subsequent states
$x_2,\ldots,x_H$ is influenced by the learner's actions. Now, without
further assumptions, this is simply a general MDP, and function approximation
provides no benefits in the worst-case. To allow for sample-efficient
learning guarantees, the \emph{block MDP} model assumes there is an
unobserved latent MDP with $S$ states, and that each observed state $x_h$
is drawn from an \emph{emission distribution} for the current latent
state $s_h$. When $H=1$ and $S=1$, this recovers the contextual bandit, and
in general the goal is to use an appropriate value function class $\cF$ to attain
sample complexity guarantees that are polynomial in $S$, but not $\abs*{\cX}$ (which, as in
the contextual bandit, is typically infinite and high-dimensional).

More formally, the block MDP setup we consider is a layered episodic Markov decision
process with horizon $H$, state space $\cX=\cX_1\cup\cdots\cup\cX_H$
(with $\cX_i\cap\cX_j=\emptyset$),
and action space $\cA$ with $\abs*{\cA}=A$. We proceed in $K$ episodes. Within each episode we observe rewards and observations through the
following protocol, beginning with $x_1\sim\mu$.
\begin{itemize}
\item For $h=1,\ldots,H$:
  \begin{itemize}
  \item Choose action $a_h$.
  \item Observe reward $r_h$ and next state $x_{h+1}\sim\Pstar_h(\cdot\mid{}x_h,a_h)$.
  \end{itemize}
\end{itemize}
Note that for this setting we use the subscript $h$ on e.g., $x_h$, to refer
to the layer within a fixed episode, whereas for contextual bandits we
use the subscript $t$ to refer to the round/episode itself. We always
use $h$ for the former setting and $t$ for the latter to distinguish.

As mentioned above, the state space is potentially rich and high-dimensional, and
dependence on $\abs*{\cX}$ is unacceptable. Hence, to enable sample-efficient reinforcement learning guarantees with function
approximation, the block MDP model assumes the existence of a
\emph{latent state space} $\cS=\cS_1\cup\cdots\cup\cS_{H}$, and assumes
that each state $x\in\cX$ can be uniquely attributed to a latent state
$s\in\cS$. More precisely, we assume that for each $h$, $\Pstar_h$
factorizes, so that we can view $x_{h+1}$ as generated by the
process $s_{h+1}\sim\Pstar_{h}(\cdot\mid{}x_h,a_h)$, $x_{h+1}\sim\emi(s_{h+1})$,
where $\emi:\cS\to\Delta(\cX)$ is an (unknown) \emph{emission
  distribution}, and $s_{h+1}$ is the latent state for layer
$h+1$. We make the following standard decodability assumption \citep{krishnamurthy2016pac,jiang2017contextual,du2019latent}.
\begin{assumption}[Decodability]
For all $s\neq{}s'$, $\supp(\emi(s))\cap\supp(\emi(s'))=\emptyset$.
\end{assumption}
This assumption implies that the optimal policy $\pistar$ depends only on the
current context $x_h$. We write the optimal $Q$-function for layer $h$ as $\Qstar_h(x,a)$ and let
$\Vstar_h(x)=\max_{a\in\cA}\Qstar_h(x,a)$ be the optimal value function.

\paragraph{Function approximation and gaps}
As in the contextual bandit setting, take as a given class of 
functions $\cF$ that attempts to model the
optimal value function. We let
$\cF_h\subseteq(\cX\times\cA\to\brk*{0,H})$ be the value function
class for layer $h$ (with $\cF=\cF_1\times\cdots\times\cF_h$), and we make the following optimistic completeness
assumption \citep{jin2020provably,wang2019optimism,wang2020provably}.
\begin{assumption}
  \label{ass:completeness}
  For all $h$ and all functions $V:\cX_{h+1}\to\brk*{0,H}$, we have that
  \[
    (x,a)\mapsto{} \En\brk*{r_h+V(x_{h+1})\mid{}x_h=x,a_h=a} \in\cF_h.
  \]
\end{assumption}
\pref{ass:completeness} implies that $\Qstar_h\in\cF_h$, generalizing
the realizability assumption (\pref{ass:realizability}) but it is
significantly stronger, as it requires that the function class
contains Bellman backups for arbitrary functions.

\subsubsection{An Efficient, Instance-Dependent Algorithm
}
We develop a new instance-dependent algorithm that adapts to the gap in the optimal
value function $\Qstar$ to attain improved sample complexity.  Define $\gap(x,a) = \Vstar_h(x) -
\Qstar_h(x,a)$, and define the worst-case gap as
\[
  \Delta = \min_{s}\inf_{x\in\supp(\emi(s))}\min_{a}\crl*{\gap(x,a)\mid{}\gap(x,a)>0}.
\]
Our main result is an oracle-efficient algorithm, \blockalg, which
attains a tight gap-dependent PAC-RL guarantee whenever an appropriate
generalization of the \valuedis $\ValueDisA$ is bounded.
\newtheorem*{thm:informal6}{Theorem \ref*{thm:block_mdp} (informal)}
\begin{thm:informal6}
  For all instances, \blockalg finds an $\veps$-suboptimal policy
  using
  $\poly(S,A,H,\ValueDisA)\cdot\frac{\log\abs*{\cF}}{\veps\cdot\Delta}$ episodes.
\end{thm:informal6}
This theorem has two key features. First, when
$\ValueDisA=\bigoht(1)$, the scaling of
$\veps$ and $\Delta$ in the term
$\frac{\log\abs*{\cF}}{\veps\cdot\Delta}$ is optimal even for in the
special case of contextual bandits, and improves over the minimax rate,
which scales as $\frac{1}{\veps^{2}}$. Second, and perhaps more
importantly, \blockalg is computationally efficient, and only requires a
regression oracle for the value function class. Previous works require
stronger oracles and typically do not attain optimal dependence on
$\veps$, but are not fully comparable in terms of statistical
assumptions
\citep{krishnamurthy2016pac,jiang2017contextual,dann2018oracle,du2019provably,du2019latent,misra2019kinematic,feng2020provably,agarwal2020flambe};
see \pref{sec:rl} for a detailed comparison. At a conceptual level,
the design and analysis of \blockalg use several new techniques that leverage our disagreement-based perspective, and we hope
that they will find broader use.

\subsection{Additional Notation}
\label{sec:notation}
	We adopt non-asymptotic big-oh notation: For functions
	$f,g:\cX\to\bbR_{+}$, we write $f=\bigoh(g)$ (resp. $f=\bigom(g)$) if there exists some constant
	$C>0$ such that $f(x)\leq{}Cg(x)$ (resp. $f(x)\geq{}Cg(x)$)
        for all $x\in\cX$. We write $f=\bigoht(g)$ if
        $f=\bigoh(g\cdot\mathrm{polylog}(T))$, $f=\bigomt(g)$ if $f=\bigom(g/\polylog(T))$, and
        $f=\bigthetat(g)$ if $f=\bigoht(g)$ and $f=\bigomt(g)$. %
	We use $f\propto g$ as shorthand for $f=\bigthetat(g)$.

	For a vector $x\in\bbR^{d}$, we let $\nrm*{x}_{2}$ denote the euclidean
	norm and $\nrm*{x}_{\infty}$ denote the element-wise $\ell_{\infty}$
	norm.    %
	For an integer $n\in\N$, we let $[n]$ denote the set
        $\{1,\dots,n\}$. For a set or a sequence $S$, we let
        $\unif(S)$ denote the uniform distribution over all the
        elements in $S$ (note that a sequence allows identical
        elements to appear multiple times). For a set $\cX$, we let
        $\Delta(\cX)$ denote the set of all probability distributions
        over $\cX$. Given a policy $\pi:\cX\to\cA$, we occasionally overload notation and write $\pi(x,a)=\indic\crl*{\pi(x)=a}$.

\subsection{Organization}
\pref{sec:cb} contains our main contextual bandit
results. \pref{sec:alg,sec:policy_disagreement,sec:value_disagreement} contain our main algorithm, \mainalg, and 
disagreement-based best-of-both-worlds guarantees and lower bounds. \pref{sec:star} and \pref{sec:eluder} contain structural results
and guarantees for worst-case distributions and adversarial
contexts. \pref{sec:rl} contains disagreement-based guarantees for
reinforcement learning with function approximation. In \pref{sec:oracle} we show in detail how to implement all of our
algorithms for contextual bandits and reinforcement learning using
regression oracles for the value function
class. \pref{sec:experiments} contains experiments with \mainalg, and
we conclude in \pref{sec:discussion} with discussion and open
problems. Proofs are deferred to the appendix.

\section{Contextual Bandits}
\label{sec:cb}
We now introduce our contextual bandit algorithm, \mainalg, and
give regret bounds based on the policy and value function disagreement
coefficients, as well as matching lower bounds. We then show how to relate
these quantities to other structural parameters for the
distribution-free and adversarial settings, and instantiate our bounds
for concrete settings of interest.

\subsection{The Algorithm}\label{sec:alg}

Our main algorithm, \mainalg, is presented in \pref{alg:main}. %
Exploration in \mainalg is based on
a probability selection strategy introduced by \cite{abe1999associative} (see also
\cite{abe2003reinforcement}) and extended to contextual bandits with
general function classes by \cite{foster2020beyond} and
\cite{simchi2020bypassing} for online and offline regression oracles,
respectively. We utilize a general version of the Abe-Long strategy
which we refer to by the more descriptive name ``inverse gap
weighting'' (\igw). The strategy is parameterized by a
learning rate $\gamma$ and a subset $\cA'\subseteq\cA$ of
actions. Given a context $x$ and reward predictor $\fhat\in\cF$, we
define a probability distribution $\igw_{\cA',\gamma}(x;\fhat)\in\Delta(\cA)$ by
\begin{equation}
  \label{eq:igw}
  \prn*{\igw_{\cA',\gamma}(x;\fhat)}_{a} = \begin{cases}
        \frac{1}{|\cA'|+\gamma\left(\fhat(x,\widehat{a})-\fhat(x,a)\right)},&\text{for all }a\in \cA'/\{\widehat{a}\},\\
        1-\sum_{a\in \cA'/\{\widehat{a}\}}p_t(a),&\text{for }a=\widehat{a},\\
        0,&\text{for }a\notin \cA',
        \end{cases}
      \end{equation}
      where $\ahat\ldef{}\argmax_{a\in\cA'}\fhat(x,a)$. Both
      \cite{foster2020beyond} and \cite{simchi2020bypassing} apply
      this strategy with $\cA'=\cA$, and with the learning rate
      $\gamma$ selected either constant or following a fixed
      non-adaptive schedule. Building on this approach, \mainalg follows the same general template as the \falcon algorithm of
\cite{simchi2020bypassing}, but with two key
differences. First, rather than applying the \igw{} scheme to all
actions, we restrict only to actions $a$ which are ``plausible'' in
the sense that
they are induced by a version space $\cF_m$ maintained (implicitly) by
the algorithm. Second, we choose the learning rate $\gamma_m$ in a
data-driven fashion.

\begin{algorithm}[htpb]
\caption{\mainalg (Adaptive Contextual Bandits)}
\label{alg:main}
\textbf{input:}  Function class $\cF$. Number of rounds $T$.
\\
\textbf{initialization:}
\begin{itemize}[leftmargin=*]
\item[-] $M=\lceil\log_2 T\rceil$. \algcomment{Number of epochs.}
\item[-] Define $\tau_m=2^m, t_m=(\tau_m+\tau_{m-1})/2$ and
  $n_m=\tau_m-\tau_{m-1}$ for $m\in[M]$ \algcomment{Epoch schedule.}
\item[] and $\tau_0=0$, $t_0=0$, and $n_0=1/2$.
\item[-] Set  $\delta=1/T$. \algcomment{Failure probability.}
\item[-] $\beta_m=16(M-m+1)\log(2|\cF|T^2/\delta)$ for $m\in\brk*{M}$ \algcomment{Confidence
    radius.}
\item[-] $\mu_m=64\log(4M/\delta)/n_{m-1}$ for $m\in\brk*{M}$ \algcomment{Smoothing parameter.}
\end{itemize}
\textbf{notation:}
  \begin{itemize}[leftmargin=*]
\item[-] $\sum_{t=1}^{0}[\ldots]\ldef0$ and
  $\E_{x\sim\cD_1}[\ldots]\ldef1$. %
  \item[-]For $\cF'\subset\cF$, define
    \begin{align*}
\cA(x;\cF')&=\{a\in\cA:\pi_f(x)=a\text{ for some
             }f\in\cF'\},\quad\text{\algcomment{Candidate action set.}}\\
w(x;\cF')&=\mathbb{I}\{|\cA(x;\cF')|>1\}\cdot\max_{a\in
           \cA(x;\cF')}\sup_{f,f'\in\cF'}
           \left|f(x,a)-f'(x,a)\right|.\quad\text{\algcomment{Confidence width.}}
    \end{align*}
\end{itemize}

\textbf{algorithm:}
\begin{algorithmic}[1]
\For{epoch $m=1,2,\dots,M$}
    \State Compute the predictor $\widehat{f}_m=\arg\min_{f\in\cF}\sum_{t=1}^{\tau_{m-1}}(f(x_t,a_t)-r_t(a_t))^2$.\label{line:erm}
    \State Define \label{line:version_space}
    \[\cF_m=\left\{f\in\cF \;\Big{|}\; \sum_{t=1}^{t_{m-1}}(f(x_t,a_t)-r_t(a_t))^2\le\inf_{f'\in\cF}\sum_{t=1}^{t_{m-1}}(f'(x_t,a_t)-r_t(a_t))^2+\beta_m\right\}.\]
    \State{}Compute the {\it instance-dependent scale
      factor:}\label{line:factor} if $m>1$, %
    \[
    \lambda_m=\begin{cases}
\frac{\E_{x\sim\cD_m}[\mathbb{I}\{|\cA(x;\cF_m)|>1\}]+\mu_m}{\sqrt{\mathbb{E}_{x\sim\cD_{m-1}}[\mathbb{I}\{|\cA(x;\cF_{m-1})|>1\}]+\mu_{m-1}}}, &\text{\algcolor{\optionone:} Policy-based exploration,}\\
    \mathbb{I}\left\{\mathbb{E}_{x\sim\cD_m}\left[w(x;\cF_m)\right]\ge\frac{\sqrt{\K
          T\log(|\cF|/\delta)}}{n_{m-1}}\right\},
    &\text{\algcolor{\optiontwo:} Value-based exploration,
    }
    \end{cases}
    \]
    \Statex{}~~~~
     where $\cD_m=\unif(x_{t_{m-1}+1},\dots,x_{\tau_{m-1}})$; else, $\lambda_1=1$ (\optionone) or $0$ (\optiontwo). %
    \State{}Compute the learning rate:
    \[\gamma_m=\lambda_m\cdot \sqrt{\frac{\K n_{m-1}}{\log(2|\cF|T^2/\delta)}}.\] \label{line:adacb_learning_rate}
    \For{round $t=\tau_{m-1}+1,\cdots,\tau_m$}
        \State{Observe context $x_t\in\cX$}.
        \State Compute the {\it \setname} \label{line:disagreement}\[\cA_t=\cA(x_t;\cF_m).\] %
        \State{Compute $\widehat{f}_m(x_t,a)$ for each action $a\in
          \cA_t$. Let $\widehat{a}_t=\max_{a\in
            \cA_t}\widehat{f}_m(x_t,a)$}. Define
        \begin{equation*}
          ~~~~~~~~~~~~~~~~~~~~~~~~~~~~~~~~~~~~~~~~~~~~
          p_t =
          \igw_{\cA_t,\gamma_m}(x_t;\wh{f}_m).\quad\text{\algcomment{Inverse
              gap weighting; see
              Eq. \pref{eq:igw}.}}
        \end{equation*}
      \label{line:igw}
        \State{Sample $a_t\sim p_t$ and observe reward $r_t(a_t)$.}
    \EndFor
\EndFor
\end{algorithmic}
\end{algorithm}

In more detail, we operate in a doubling epoch schedule. Letting
$\tau_m=2^m$ with $\tau_0=0$, each epoch $m\geq{}1$ consists of rounds $\tau_{m-1}+1,\dots, \tau_{m}$, and there are $M=\lceil\log_2 T\rceil$ epochs in total. 
At the beginning of each epoch $m$, we compute an
estimator $\fhat_m$ for the Bayes regression function $\fstar$ by performing
least-squares regression on data collected so far
(\pref{line:erm}). We also maintain a version space $\cF_m$, which is the set of all plausible predictors that cannot yet be eliminated based on square loss confidence bounds (\pref{line:version_space}). Based on $\cF_m$, we select the learning rate $\gamma_m$ for the
current epoch adaptively by estimating a parameter called the
\emph{instance-dependent scale factor} ($\lambda_m$) which is closely
related to the \policydis ({\optionone}) and the \valuedis
({\optiontwo}). Then, when a context $x_t$ in epoch $m$ arrives,
\mainalg first computes the \emph{\setname}
$\cA_t\ldef{}\cA(x_t;\cF_m)$ (\pref{line:disagreement}), which is the
set of actions that are optimal for some predictor $f\in\cF_m$, and thus could
plausibly be equal to $\pistar(x_t)$. The algorithm then sets
$p_t=\igw_{\cA_t,\gamma_m}(x_t;\fhat_m)$ (\pref{line:igw}), samples
$a_t\sim{}p_t$, and proceeds to the next round.

The adaptive learning rate $\gamma_m$ balances the algorithm's efforts between exploration and exploitation: 
a larger learning rate leads to more aggressive exploitation (following the least-squares predictor $\widehat{f}_m$), 
while a smaller learning rate leads to more conservative exploration over the \setname{}. 
\mainalg's learning rate $\gamma_m$ (\pref{line:adacb_learning_rate}) has two components: 
the instance-dependent scale factor $\lambda_m$, which is adaptively determined by the collected data; 
and a non-adaptive component propositional to $\sqrt{A n_{m-1}/\log|\cF|}$, where $n_{m-1}$ is the length of the epoch $m-1$. 
While the non-adaptive component is the same as the learning rate in \falcon and is sufficient if one only aims to achieve the minimax regret, 
the adaptive factor $\lambda_m$, combined with the action elimination
procedure above, is essential for \mainalg to achieve near-optimal instance-dependent regret. 
We offer two different schemes to select $\lambda_m$: The first adapts
to the \policydis, while the second adapts to the \valuedis.
\begin{itemize}
    \item \optionone (policy-based exploration). This option selects
      $\lambda_m$ as a sample-based approximation to the quantity $$\bbP_{\cD}(|\cA(x,\cF_m)|>1)/\sqrt{\bbP_{\cD}(|\cA(x,\cF_{m-1})|>1)},$$ where $\bbP_{\cD}(|\cA(x,\cF_m)|>1)$ and $\bbP_{\cD}(|\cA(x,\cF_{m-1})|>1)$ are \emph{disagreement probabilities} (i.e., the probability that we encounter a context on which we cannot yet determine the true optimal action) for epoch $m$ and epoch $m-1$, respectively. 
    Intuitively, this configuration asserts that we should adaptively
    discount the learning rate if either 1) the current disagreement
    probability is small, or 2) the disagreement probability is
    decreasing sufficiently quickly across epochs.
    This scheme is natural because if we expect that no exploration is
    required for a large portion future contexts, %
     then we have flexibility to perform more thorough exploration on other contexts where the true optimal action cannot yet be determined.
    {This accelerates \mainalg's exploration of more effective policies.}
    \item \optiontwo (value-based exploration). While the disagreement
      probability used in \optionone is a useful quantity that provides
      information on the hardness of the problem instance, it does not
      fully utilize the value function structure. In particular, it is
      only sensitive to the {occurrence} of disagreement on each
      context, but is not sensitive to the \textit{scale} of disagreement (i.e., how much it would cost if we chose a disagreeing action) on each context. 
    This motivates \optiontwo, which is based on a refined
    \emph{confidence width} $w(x;\cF_m)$ that accounts for both the
    occurrence and the scale of disagreement. Specifically,
    $w(x;\cF_m)$ measures the worst-case cost of exploring a
    sub-optimal action in the \setname{} for $x$, and \optiontwo
    selects $\lambda_m$ as a sample-based approximation to the quantity \[\1\{\bbE_\cD[w(x;\cF_m)]\ge\sqrt{AT\log|\cF|}/n_{m-1}\}.\]
    In other words, we adaptively zero out the learning rate and
    perform uniform exploration if
    $\bbE_\cD[w(x;\cF_m)]$ is smaller than an epoch-varying
    threshold. This is reasonable because if $\bbE_\cD[w(x;\cF_m)]$ is small,
    then the average cost of exploration is small for the
    underlying instance, so we should take advantage of this and
    explore as much as possible.
  \end{itemize}

Finally, since $\cD$ is unknown, to obtain $\lambda_m$ we compute an
empirical approximation to $\bbE_{\cD}[\cdot]$ using sample
splitting. That is, we use separate sample to compute $\cF_m$ and to
approximate $\cD$ to ensure independence; this is reflected in the sample splitting schedule
$\{t_m\}_{m=1}^M$ in \mainalg. The smoothing parameter $\mu_m$ is
designed to correct the approximation error incurred by this procedure.

We make a few additional remarks. First, the learning rate and
confidence width parameters in \pref{alg:main} (and consequently our
main theorems) consider
a general finite class $\cF$. This is only a stylistic choice: \mainalg works
as-is for general function classes, with the dependence on $\log|\cF|$
in these parameters replaced by standard learning-theoretic complexity
measures such as the pseudodimension; see \pref{sec:cb_extensions}. Second, \pref{alg:main} takes $T$ as input. One can straightforwardly extend \pref{alg:main} to work with unknown $T$ using the standard \emph{doubling trick}. 
Finally, we emphasize that \optionone and \optiontwo are designed based on
different techniques and lead to different instance-dependent
guarantees. Designing a single option that simultaneously achieving
the goals of \optionone and \optiontwo is an interesting future direction.

\paragraph{Oracle efficiency}
\mainalg can be implemented efficiently with a weighted least squares
regression oracle \oracle (see \pref{eq:oracle}) as follows.
\begin{itemize}
\item At each epoch $m$, call \oracle to compute the square loss
  empirical risk minimizer $\fhat_m$.
\item For any given context $x$, the \setname $\cA(x;\cF_m)$
  can be computed using either $\bigoht(\K)$ oracle calls when $\cF$ is
  convex or $\bigoht(\K T^2)$ oracle calls for general (in
  particular, finite) classes.
\item For \optiontwo, the function $w(x;\cF_m)$ can be computed
  in a similar fashion to $\cA(x;\cF_m)$ using $\bigoht(\K)$ or
  $\bigoht(\K{}T^2)$ oracle calls in the convex and general case, respectively.
\end{itemize}
Altogether, since $\cA(x;\cF_m)$ and $w(x;\cF_m)$ are computed
for $\bigoh(1)$ different contexts per round amortized, the algorithm
requires $\bigoh(\K T)$ calls to \oracle overall when $\cF$ is convex. The reduction is described in full in \pref{sec:oracle}.

\subsection{Disagreement-Based Guarantees}
\label{sec:policy_disagreement}
We are now ready to state our first main regret guarantee for
\mainalg, which is based on the \policydis
\pref{eq:policy_disagreement}. The theorem also includes a more
general result in terms of an intermediate quantity we call the \cscdis, which we
define by
\begin{equation}
  \label{eq:csc_disagreement}
  \CostDis{\veps_0}= \sup_{\veps\geq{}\veps_0}\frac{\bbP_{\cD}\prn*{x:\exists\pi\in\Picsc_{\veps}: \pi(x)\neq{}\pistar(x)}}{\veps},
\end{equation}
where $\Picsc_{\veps}=\crl*{\pi\in\Pi: R(\pistar) - R(\pi) \leq{}
  \veps}$ for $R(\pi)\ldef{}\En\brk{r(\pi(x))}$.\footnote{The acronym
  CSC refers to \emph{cost-sensitive classification}.} The \cscdis
grants finer control over the cost-sensitive structure of the problem
and---beyond leading to our main gap-based result---leads to instance-dependent guarantees even when
the instance does not have uniform gap.

\begin{theorem}[Instance-dependent regret for \mainalg (policy version)]
  \label{thm:disagreement_ub}
  For any instance with uniform gap $\Delta$, \pref{alg:main} with
\optionone ensures that
\begin{equation}
\En\brk{\Reg}=\bigoht(1)\cdot\min_{\varepsilon>0}\max\left\{\varepsilon\Delta
  T, \frac{\PolicyDis{\veps}{\K\log|\cF|}}{\Delta}\right\}+\bigoht(1).
\label{eq:disagreement_ub}
\end{equation}
More generally, \pref{alg:main} with \optionone ensures that for every instance, without any gap assumption,
\begin{equation}
\En\brk{\Reg}=\bigoht(1)\cdot\min_{\varepsilon>0}\max\left\{\varepsilon
  T, \CostDis{\veps}{\K\log|\cF|}\right\}+\bigoht(1).
\label{eq:disagreement_ub_csc}
\end{equation}
\end{theorem}
Let us describe some key features of \pref{thm:disagreement_ub}.
\newcommand{\vepsopt}{\veps_T}
\begin{itemize}
\item Whenever $\PolicyDis{\veps}\leq\polylog(1/\veps)$, we may choose
  $\veps\propto{}1/T$ in \pref{eq:disagreement_ub} so that
  \[\En\brk*{\Reg}=\bigoht(1)\cdot{}\frac{\K\log\abs*{\cF}}{\Delta}.\] For
  example, for the classical multi-armed bandit setup where $\cX$ is a
  singleton, we have
  $\PolicyDis{\veps}=1$, recovering the usual instance-dependent rate
  (up to logarithmic factors). We give some more examples where
  logarithmic regret can be attained in a moment.
  \item More generally, since the function
    $\veps\mapsto\veps\Delta{}T$ is increasing in $\veps$ and
    $\PolicyDis{\veps}$ is decreasing, the best choice for the bound \pref{eq:disagreement_ub} (up to
    constant factors) is the critical radius
    $\vepsopt$ that satisfies the balance
    \begin{equation}
      \vepsopt\Delta{}T \propto \frac{\PolicyDis{\vepsopt}\K\log\abs{\cF}}{\Delta}.
    \end{equation}
    For example, if $\PolicyDis{\veps}\propto\veps^{-\rho}$ for some
    $\rho\in(0,1)$, then choosing
    $\veps_T\propto(\K\log\abs{\cF}(\Delta^{2}T)^{-1})^{\frac{1}{1+\rho}}$,
    leads to
    \[
      \En\brk*{\Reg}= \bigoht(1)\cdot{}\frac{ (A\log\abs*{\cF})^{\frac{1}{1+\rho}}\cdot T^{\frac{\rho}{1+\rho}}}{\Delta^{\frac{1-\rho}{1+\rho}}}.
    \]
    The critical radius also plays an
    important role in the \emph{proof} of \pref{thm:disagreement_ub}.
\item With no assumption on the gap or $\PolicyDisA$, we may always
  take $\CostDis{\veps}\leq{}1/\veps$, so that
  \pref{eq:disagreement_ub_csc} implies the minimax rate $\sqrt{\K{}T\log\abs*{\cF}}$.
\end{itemize}
The general bound \pref{eq:disagreement_ub_csc} can be seen to imply
\pref{eq:disagreement_ub}, since
$\Picsc_{\veps}\subseteq{}\Pi_{\veps/\Delta}$ whenever the gap is
$\Delta$. More generally, the \cscdis can also lead to instance-dependent
regret bounds under other standard assumptions which go beyond the
uniform gap; these are discussed at the end of the section.

\paragraph{Optimality}
We now show that the regret bound attained by \mainalg in
\pref{thm:disagreement_ub} is near-optimal, in the sense that it cannot be
improved beyond log factors without making additional assumptions on the class $\cF$ or
the contextual bandit instance.

Formally, we model a contextual bandit algorithm $\alg$ as a sequence
of mappings $\alg_{t}:
(\cX\times\cA\times{}\brk*{0,1})^{t-1}\times{}\cX\to\Delta(\cA)$, so that
\begin{equation}
  \label{eq:cb_alg}
\alg_t(x_t; (x_1,a_1,\ls_1(a_1)),\ldots,(x_{t-1},a_{t-1},\ls_{t-1}(a_{t-1})))
\end{equation}
is the algorithm's action distribution after observing context $x_t$
at round $t$.

For a given function class $\cF$, we define
\begin{equation}
  \label{eq:policy_minimax}
\PolicyMinimax=
\inf_{\alg}\sup_{(\xdist,\rdist)}\crl*{\En\brk*{\Reg}\mid{} \fstar\in\cF,\; \PolicyDis{\veps}\leq{}\theta}
\end{equation}
to be the \emph{constrained minimax complexity}, which measures the
worst-case performance of any algorithm \pref{eq:cb_alg} across all
instances realizable by $\cF$ for which the \policydis at scale
$\veps$ is at most $\theta$.\footnote{We leave implicit that rewards
  are restricted to the range $\brk*{0,1}$. In fact, for our lower
  bound it suffices to only consider reward distributions $\bbP_r$
  for which $r(a)$ is Bernoulli with mean $\fstar(x,a)$ given $x$ in \pref{eq:policy_minimax}.}

Our main lower bound shows that there exists a function class $\cF$
for which the constrained minimax complexity matches the upper bound \pref{eq:disagreement_ub}.
\begin{theorem}[Tight lower bound for specific function class]
  \label{thm:disagreement_lb}
  Let parameters $\K{},F\in\bbN$ and $\Delta\in(0,1/4)$ be given. For any $\veps\in(0,1)$ and $1\leq{}\theta\leq{}\min\crl{1/\veps,e^{-2}A/F}$,  there exists a
  function class $\cF\subseteq(\cX\to\cA)$ with $\K$ actions and $\abs*{\cF}\leq{}F$ such
  that:
  \begin{itemize}
  \item All $f\in\cF$ have uniform gap $\Delta$.
  \item The constrained minimax complexity is lower bounded by
    \[
\PolicyMinimax=\bigomt(1)\cdot{}\min\crl*{\veps\Delta{}T, \frac{\theta\K{}\log{}F}{\Delta}},
  \]
  where $\bigomt(\cdot)$ hides factors logarithmic in $A$ and $\veps^{-1}$.
  \end{itemize}
\end{theorem}
This lower bound has a simple interpretation: The term
$\veps\Delta{}T$ is the regret incurred if we commit to playing a
particular policy $\pi\in\Pi_{\veps}$ for any ``simple'' instance in
which the gap is no larger than $\bigoh(\Delta)$ for
all actions, while the term $\frac{\PolicyDis{\veps}\K\log{}\abs*{\cF}}{\Delta}$ is
the cost of exploration to find such a policy.

The first implication of this lower bound is that without an
assumption such as the disagreement coefficient, logarithmic regret is
impossible even when the gap is constant; this alone is not surprising
since \cite{foster2020beyond} already showed a similar impossibility for
non-stochastic contexts, but \pref{thm:disagreement_lb} strengthens
this result since it holds for stochastic contexts. More
importantly, the lower bound shows that the tradeoff in
\pref{thm:disagreement_ub} is tight as a function of $\Delta, \veps,
\K, \log\abs*{\cF}$, and $\PolicyDisA$, so additional assumptions are
required to attain stronger instance-dependent regret bounds for
specific classes. We explore such assumptions in the sequel.

We
mention one important caveat: Compared to instance-dependent lower
bounds for multi-armed bandits (e.g., \cite{garivier2019explore}), the
quantification for \pref{thm:disagreement_lb} is slightly weaker.
Rather than lower bounding the regret for \emph{any} particular
instance (assuming uniformly good performance in a neighborhood), we
only show existence of a \emph{particular} realizable instance with
gap for which the regret lower bound holds. We suspect that
strengthening the lower bound in this regard will be difficult unless
one is willing to sacrifice dependence on $\log{}F$.

\paragraph{Examples}
\newcommand{\policygap}{\Delta_{\mathrm{pol}}}
The (policy) disagreement coefficient has been studied extensively in active
learning, and many bounds are known for different function classes and
distributions of interest. We refer to \cite{hanneke2014theory} for a
comprehensive survey and summarize some notable examples here
(restricting to the binary/two-action case, which has been the main
focus of active learning literature).
\begin{itemize}
\item When $\cF$ is a $d$-dimensional linear function class,
  $\PolicyDis{\veps}\leq{}\bigoht(d^{1/2}\log(1/\veps))$ whenever
  $\xdist$ is isotropic log-concave \citep{balcan2013active}. More
  generally, $\PolicyDis{\veps}=o(1/\veps)$ as long as $\cD$ admits a
  density \citep{hanneke2014theory}.
\item $\PolicyDis{\veps}=\polylog(1/\veps)$ whenever $\cF$ is smoothly
  parameterized by a subset of euclidean space, subject to certain
  regularity conditions \citep{friedman2009active}. This includes, for
  example, axis-aligned rectangles.
\item When $\Pi$ is a class of depth-limited decision trees, we have
  $\PolicyDis{\veps}=\polylog(1/\veps)$ \citep{balcan2010true}.
\end{itemize}
For an example which leverages the more general parameter $\CostDisA$, \cite{langford2008epoch} give logarithmic regret bounds for finite-class contextual bandits
based on a different notion of gap called the \emph{policy gap}
defined by
$\policygap=R(\pistar)-\max_{\pi\neq{}\pistar}R(\pi)$, where $R(\pi)=\En_{x,r}\brk*{r(\pi(x))}$. It is simple to
see that $\CostDis{\veps}\leq{}\policygap^{-1}$, so that
\pref{thm:disagreement_ub} gives
$\En\brk*{\Reg}\leq{}\bigoht\prn*{\frac{\K\log\abs{\cF}}{\policygap}}$,
which improves upon the gap dependence of their result.

\subsection{Scale-Sensitive Guarantees}
\label{sec:value_disagreement}

We now give instance-dependent regret guarantees based on the
\valuedis, which is defined via
  \begin{equation}
    \label{eq:value_disagreement_full}
      \ValueDis{\Delta_0}{\veps_0} = \sup_{\Delta>\Delta_0,\veps>\veps_0}\sup_{p:\cX\to\Delta(\cD)}\frac{\Delta^{2}}{\veps^{2}}\bbP_{\cD,p}\prn*{
    \exists{}f\in\cF: \abs*{f(x,a)-\fstar(x,a)}>\Delta,\;\nrm*{f-\fstar}_{\cD,p}\leq\veps
    }.
  \end{equation}

Compared to the \policydis, the \valuedis is somewhat easier to bound
directly when the value function class $\cF$ has simple structure. For
example, when $\cF$ is linear, we can bound $\ValueDisA$ in terms of
the dimension for \emph{any} distribution with a simple linear
algebraic calculation.
\begin{proposition}
  \label{prop:value_disagreement_linear}
  Let $\phi(x,a)\in\bbR^{d}$ be a fixed feature map, and let
  $\cF=\crl*{(x,a)\mapsto{}\tri*{w,\phi(x,a)}\mid{}w\in\cW}$, where
  $\cW\subseteq\bbR^{d}$ is any fixed set. Then for all $\xdist$,
  $\Delta$, $\veps$, 
  \[
    \ValueDis{\Delta}{\veps}\leq{}d.
  \]
Furthermore, if
  $\cF=\crl*{(x,a)\mapsto{}\sigma(\tri*{w,\phi(x,a)})\mid{}w\in\cW}$,
  where $\sigma:\bbR\to\bbR$ is any fixed link function with
  $0<c_{l}\leq\sigma'\leq{}c_{u}$ almost surely, we have
    \[
    \ValueDis{\Delta}{\veps}\leq{}\prn*{\frac{c_u}{c_l}}^{2}\cdot{}d.
  \]

\end{proposition}
More generally---as we show in the next section---the \valuedis is
always bounded by the so-called eluder dimension for $\cF$,
allowing us to leverage existing results for this parameter
\citep{russo2013eluder}. However, the \valuedis can be significantly
tighter because---among other reasons---it can leverage benign distributional
structure. %

We now show that \mainalg can simultaneously attain the minimax regret
bound and adapt to the \valuedis.
\begin{theorem}
  \label{thm:value_disagreement_ub}
  For any instance, \pref{alg:main} with
  \optiontwo ensures that
  \begin{equation}
    \label{eq:value_ub_main}
\RegExp=\bigoht(1)\cdot\min\crl*{
      \sqrt{\K{}T\log\abs{\cF}}, \frac{\ValueDis{\Delta/2}{\veps_T}A\log\abs{\cF}}{\Delta}}+\bigoh(1),
\end{equation}
where $\veps_T\propto\sqrt{{\log(\abs*{\cF}T)}/{T}}$.
\end{theorem}
This rate improves over the minimax rate asymptotically whenever
$\ValueDis{\Delta/2}{\veps}=o(1/\veps)$, and is logarithmic whenever
$\ValueDis{\Delta/2}{\veps}=\polylog(1/\veps)$.

As with our policy disagreement-based result, we complement \pref{thm:value_disagreement_ub}
with a lower bound. To state the result, we define
\begin{equation}
  \label{eq:value_minimax}
\ValueMinimax=
\inf_{\alg}\sup_{(\xdist,\rdist)}\crl*{\En\brk*{\Reg}\mid{} \fstar\in\cF,\; \ValueDis{\Delta}{\veps}\leq{}\theta},
\end{equation}
which is the value-based analogue of the constrained minimax
complexity \pref{eq:policy_minimax}. Our main lower bound is as follows.
\begin{theorem}
  \label{thm:value_disagreement_lb}
    Let parameters $\K{},F\in\bbN$ and $\Delta\in(0,1/4)$ be given. For any $\veps\in(\Delta,1)$ and $0\leq{}\theta\leq{}\min\crl*{\Delta^{2}/\veps^{2},e^{-2}F/A}$,  there exists a
  function class $\cF:\cX\to\cA$ with $\K$ actions and $\abs*{\cF}\leq{}F$ such
  that:
  \begin{itemize}
  \item All $f\in\cF$ have uniform gap $\Delta$.
  \item The constrained minimax complexity is lower bounded by

    \begin{equation}
        \ValueMinimax[\Delta/2]
        =\bigomt(1)\cdot{}\min\crl*{\frac{\veps^{2}}{\Delta{}}T,
          \frac{\theta\K{}\log{}F}{\Delta}},\label{eq:value_disagreement_lb_full}
      \end{equation}
      where $\bigomt(\cdot)$ hides factors logarithmic in $A$ and $\Delta/\veps$.
  \end{itemize}
\end{theorem}
As with \pref{thm:disagreement_lb}, the lower bound
\pref{eq:value_disagreement_lb_full} has a simple interpretation: The
term $\frac{\veps^{2}}{\Delta{}}T$ is an upper bound on the regret of any policy
$\pi_f$ for which the predictor $f$ is within $L_2$-radius $\veps$ of
$\fstar$ (under gap $\Delta$), and the term
$\frac{\ValueDis{\Delta/2}{\veps}\K{}\log{}\abs*{\cF}}{\Delta}$ is
the exploration cost to find such a predictor.

The most important implication of \pref{thm:value_disagreement_lb} is as
follows: Suppose that
$\ValueDis{\Delta/2}{\veps}=\polylog(1/\veps)$. Then by taking
$\veps_T\propto({A\log{}\abs*{\cF}}/{T})^{\frac{1}{2}+\rho}$ for
any $\rho>0$, we conclude that for sufficiently
large $T$, any algorithm on the lower bound instance must have
 \[
   \En\brk{\Reg} = \bigomt(1)\cdot
      \frac{\ValueDis{\Delta/2}{\veps_T}\K{}\log{}\abs*{\cF}}{\Delta}.
    \]
    This implies that the instance-dependent term in
        \pref{eq:value_ub_main} is nearly optimal in this regime, in that the
        parameter $\veps_T=\sqrt{{\log\abs*{\cF}}/{T}}$ used by
        the algorithm can at most be increased by a
        sub-polynomial factor. In general, however,
        \pref{eq:value_ub_main} does not exactly match the tradeoff in
        \pref{eq:value_disagreement_lb_full}, but we suspect that
        \mainalg can be improved to close the gap.\footnote{With
          a-priori knowledge of $\ValueDisA$, this is fairly straightforward.}

\subsection{Distribution-Free Guarantees}
\label{sec:star}
The disagreement coefficients introduced in the previous section
depend strongly on the context distribution $\xdist$. On one hand,
this is a desirable feature, since it means we may pay very little to
adapt to the gap $\Delta$ for benign distributions. On the other hand,
in practical applications, we may not have prior knowledge of how favorable
$\cD$ is, or whether we should expect to do any better than the minimax
rate. A natural question then is for what function classes we can
guarantee logarithmic regret for \emph{any} distribution $\cD$. An
important result of \cite{hanneke2015minimax} shows that in the binary
setting, the \policydis is always bounded by a
combinatorial parameter called the (policy) \emph{star number}. We give
distribution-free results based on two multiclass generalizations of
this parameter
\begin{definition}[Policy star number (weak)]
  For any policy $\pistar$ and policy class $\Pi$, let the weak \policystar
  $\PolicyStarWL$ denote the largest number $m$ such that there
  exist contexts $x\ind{1},\ldots,x\ind{m}$ and policies
  $\pi\ind{1},\ldots,\pi\ind{m}$ such that for all $i$,
  \[
    \pi\ind{i}(x\ind{i})\neq{}\pistar(x\ind{i}),\mathand\pi\ind{i}(x\ind{j})=\pistar(x\ind{j})\quad\forall{}j\neq{}i.
  \]
\end{definition}
\begin{definition}[Policy star number (strong)]
\label{def:policy_star_strong}
    For any policy $\pistar$ and policy class $\Pi$, let the strong
    policy star number
  $\PolicyStarL$ denote the largest number $m$ such that there
  exist context-action pairs $(x\ind{1},a\ind{1}),\ldots,(x\ind{m},a\ind{m})$ and policies
  $\pi\ind{1},\ldots,\pi\ind{m}$ such that for all $i$,
  \[
\pi\ind{i}(x\ind{i}) = a\ind{i}\neq{}\pistar(x\ind{i}),\mathand\pi\ind{i}(x\ind{j})=\pistar(x\ind{j})\quad\forall{}j\neq{}i:x\ind{j}\neq{}x\ind{i}.
  \]
\end{definition}
These definitions are closely related: It is simple to see that
\begin{equation}
  \PolicyStarWL \leq{} \PolicyStarL \leq{}
  (\K-1)\cdot\PolicyStarWL,\label{eq:policy_star_weak_strong}
\end{equation}
and that both of these inequalities can be tight in the worst case. To
obtain distribution-free bounds based on the star number, we
recall the following key result of \cite{hanneke2015minimax}.\footnote{Technically, the original theorem in \cite{hanneke2015minimax}
    only holds the binary case, but the multiclass case here follows
    immediately by applying the theorem with the collection of binary
    classifiers $\cH=\crl*{x\mapsto\indic\crl*{\pi(x)\neq{}\pistar(x)}\mid{}\pi\in\Pi}$.}
\begin{theorem}[Star number bounds disagreement coefficient \citep{hanneke2015minimax}]
  \label{thm:disagreement_star}
For all policies $\pistar$,
\begin{equation}
  \label{eq:disagreement_star}
    \sup_{\cD}\sup_{\veps>0}\PolicyDisL{\veps} \leq{} \PolicyStarWL.
  \end{equation}
\end{theorem}
This result immediately implies that \mainalg enjoys logarithmic
regret for any function class with bounded \policystar.
\begin{corollary}[Distribution-free bound for \mainalg]
  For any function class $\cF$, \mainalg with \optionone has
  \begin{equation}
    \En\brk*{\Reg} =
    \bigoht\prn*{\frac{\PolicyStarWL\cdot{}\K{}\log\abs*{\cF}}{\Delta}}.\label{eq:star_ub_main}
  \end{equation}
\end{corollary}
One slightly unsatisfying feature of our lower bounds based on the
disagreement coefficient
(\pref{thm:disagreement_lb}/\pref{thm:value_disagreement_lb}) is that they are worst-case in nature, and
rely on an adversarially constructed policy class. Our next theorem shows
that \pref{eq:star_ub_main} is near-optimal for
\emph{any} policy class $\Pi$ (albeit, in the worst case over all
value function classes $\cF$ inducing $\Pi$). This means that if we
take the policy class $\Pi$ as a given rather than the value function
class $\cF$, bounded \policystar is both necessary and sufficient for
logarithmic regret.
\begin{theorem}
  \label{thm:star_lb_simple}
Let a policy class $\Pi$, $\pistar\in\Pi$,  and gap $\Delta\in(0,1/8)$ be given. Then
there exists a value function class $\cF$ such that
\begin{enumerate}
\item $\Pi=\crl*{\pi_f\mid{}f\in\cF}$, and in particular some $\fstar\in\cF$ has
$\pistar=\pi_{\fstar}$.
  \item Each $f\in\cF$ has uniform
  gap $\Delta$.
\item For any algorithm with
  $\En\brk*{\Reg}\leq{}\frac{\Delta{}T}{16\PolicyStarL}$ for all
  instances realizable by $\cF$, there exists an instance with
  $\fstar$ as the Bayes reward function such that
  \begin{equation}
    \label{eq:star_lb}
    \En\brk*{\Reg} =
    \bigom\prn*{\frac{\PolicyStarL}{\Delta}}.
  \end{equation}
\end{enumerate}
\end{theorem}
This bound scales with the strong variant of the \policydis rather
than the (smaller) weak variant, but does not directly scale with the
number of actions. Hence, the dependence matches the upper bound of
\mainalg in \pref{eq:star_ub_main} whenever the second inequality in
\pref{eq:policy_star_weak_strong} saturates (since $\PolicyStarL$ can
itself scale with the number of actions). We suspect that the lower
bound is tight and that the upper bound can be improved to scale with
$\PolicyStarL$, with no explicit dependence on the number of actions.

Unlike the upper bound \pref{eq:star_ub_main}, the lower bound
\pref{eq:star_lb} does not scale with $\log\abs*{\cF}$. This does not
appear to be possible to resolve without additional assumptions, as
there are classes for which \pref{eq:star_lb} is tight (consider $d$
independent multi-armed bandit problems), as well as classes for which
\pref{eq:star_ub_main} is tight (cf. \pref{thm:disagreement_lb}). Similar issues arise in lower bounds for
active learning \citep{hanneke2015minimax}. However in the full version of \pref{thm:star_lb_simple} (\pref{sec:star_lb}), we are
able to strengthen the lower bound to roughly
$\Omega\prn[\Big]{\frac{\PolicyStarL+\log\abs*{\cF}}{\Delta}}$
for Natarajan classes.

\subsubsection{Scale-Sensitive Guarantees for the Distribution-Free
  Setting}
We now extend our development based on the star number to
give distribution-free upper bounds on the \valuedis. Compared to the
policy-based setting, where we were able to simply appeal to upper
bounds from \cite{hanneke2015minimax}, scale-sensitive analogues of
the star number have not been studied in the literature to our
knowledge. This leads us to introduce the following definition.
\begin{definition}[Value function star number]
  \label{def:value_star}
  Let $\ValueStarCL{\Delta}
  $ be the length of the longest sequence
  of context-action pairs $(x\ind{1},a\ind{1}),\ldots,(x\ind{m},a\ind{m})$ such that for all $i$, there exists $f\ind{i}\in\cF$
  such that
  \[
\abs*{f\ind{i}(x\ind{i},a\ind{i})-\fstar(x\ind{i},a\ind{i})}>\Delta,\quad\text{and}\quad\sum_{j\neq{}i}\prn{f\ind{i}(x\ind{j},a\ind{j})-\fstar(x\ind{j},a\ind{j})}^{2}\leq{}\Delta^{2}.
  \]
  The \valuestar is defined as $\ValueStarL{\Delta_0} \ldef
  \sup_{\Delta>\Delta_0}\ValueStarCL{\Delta}$.
\end{definition}
When the function class $\cF$ is $\crl*{0,1}$-valued, the \valuestar
coincides with the \policystar, i.e. $\ValueStarL{1} =
\PolicyStarShort_{\fstar}(\cF)$. In general though, for a given class
$\cF$, the \policystar for the induced class can be arbitrarily large
compared to the \valuestar.\footnote{Interestingly, this construction
  also shows that in general, the \valuestar for $\cF$ can be
  arbitrarily small compared to the fat-shattering dimension. This is
  somewhat counterintuitive because the star number for a policy class always upper bounds its
  VC dimension.}
\begin{proposition}
  \label{prop:value_policy_star_separation}
  For every $d\in\bbN$, there exists a class $\cF$ and
  $\fstar\in\cF$ such that $\sup_{\Delta}\ValueStarL{\Delta}\leq{}5$
  and $\PolicyStarL\geq{}d$.
\end{proposition}
Generalizing the result of \cite{hanneke2015minimax}, we show that the
\valuestar bounds the \valuedis for all distributions and all scale levels.
\begin{theorem}[Value 
function star number bounds disagreement coefficient]
  \label{thm:disagreement_to_star}
  For any \uniformgc class $\cF$ and $\fstar:\cX\times\cA\to\brk*{0,1}$,
\begin{equation}
  \label{eq:disagreement_to_star}
  \sup_{\cD}\sup_{\veps>0}\ValueDisL{\Delta}{\veps} \leq{} 4
  (\ValueStarL{\Delta})^{2},\quad\forall{}\Delta>0.
\end{equation}
\end{theorem}
Compared to the bound for the policy star number (\pref{thm:disagreement_star}),
\pref{thm:disagreement_to_star} is worse by a quadratic factor when
specialized to discrete function classes. Improving
\pref{eq:disagreement_to_star} to be linear in the star number is an
interesting technical question. The assumption that $\cF$ is
\uniformgc is quite weak and arises for technical reasons: compared to
the \policystar, which always bounds the VC/Natarajan dimension, boundedness of the \valuestar is not sufficient to
ensure that $\cF$ enjoys uniform convergence.

The main takeaway from \pref{thm:disagreement_to_star} is that
\mainalg with \optiontwo guarantees
  \begin{equation}
    \En\brk*{\Reg} =
    \bigoht\prn*{\frac{(\ValueStarL{\Delta/2})^{2}\cdot{}\K{}\log\abs*{\cF}}{\Delta}},\label{eq:value_star_ub_main}
  \end{equation}
for any distribution. Following our development for the \policystar,
we now turn our attention to establishing the necessity of the
\valuestar for gap-dependent regret bounds. Our lower bound depends on
the following ``weak'' variant of the parameter.
\begin{definition}[Value function star number (weak variant)]
  \label{def:value_star_weak}
  For any $\Delta\in(0,1)$ and $\veps\in(0,\Delta/2)$, define $\ValueStarWL{\Delta}{\veps}$ be the length of the largest sequence of points $x\ind{1},\ldots,x\ind{m}$ such that for
  all $i$, there exists $f\ind{i}\in\cF$, such that
  \begin{enumerate}
  \item\label{item:weak_star1}
    $f\ind{i}(x\ind{i},\pi_{f\ind{i}}(x\ind{i}))\geq{}\max_{a\neq{}\pi_{f\ind{i}}(x\ind{i})
    }f\ind{i}(x\ind{i},a)+\Delta$ and $\pi_{f\ind{i}}(x\ind{i})\neq{}\pistar(x\ind{i})$.
    \item\label{item:weak_star2} $\max_{a}\abs*{f\ind{i}(x\ind{i},a)-\fstar(x\ind{i},a)}\leq{}2\Delta$
    \item\label{item:weak_star3} $\sum_{j\neq{}i}\max_{a}\abs*{f\ind{i}(x\ind{j},a)-\fstar(x\ind{j},a)}^{2}<\veps^{2}$.
  \end{enumerate}
\end{definition}
Relative to the basic \valuestar, the key difference above is that we
allow a separate scale parameter to control the sum constraint in
\pref{item:weak_star3} above. This is important to prevent passive
information leakage in our lower bound construction, but we suspect
this condition can be relaxed to more closely match
\pref{def:value_star}. Our main lower bound is as
follows.\footnote{To avoid technical conditions involving the boundary
  of the interval $\brk*{0,1}$, we allow for unit Gaussian rewards
  with means in $\brk*{0,1}$ for this lower bound.}
\begin{theorem}
  \label{thm:ss_star_lb}
Let a function class $\cF$ and $\fstar\in\cF$ with uniform gap
$\Delta$ be given. Let $\veps_T\in(0,\Delta/4)$ be the largest
solution to the equation\footnote{There is always at least one
  solution to \pref{eq:star_fixed_point}, since we can take $\veps_T=0$.}
\begin{equation}
  \label{eq:star_fixed_point}
\veps_T^{2}T\leq{}\ValueStarWL{\Delta/2}{\veps_T}.
\end{equation}
Then there exists a distribution $\cD$ such that for any algorithm with
$\En\brk*{\Reg}\leq{}2^{-6}\frac{\Delta{}T}{\ValueStarWL{\Delta/2}{\veps_T}}$
on all instances realizable by $\cF$, there exists an instance with
$\fstar$ as the Bayes reward function such that
\begin{equation}
  \label{eq:value_star_lb}
\En\brk*{\Reg} = \bigom\prn*{\frac{\ValueStarWL{\Delta/2}{\veps_T}}{\Delta}}.
\end{equation}
\end{theorem}
As mentioned before, we suspect that the linear scaling in
\pref{eq:value_star_lb} is correct and that
\pref{eq:value_star_ub_main} can be improved to match. The dependence
on the additional scale parameter $\veps_T$ is more
subtle, and requires further investigation.

\subsection{Adversarial Contexts and the Eluder Dimension}
\label{sec:eluder}
The \emph{eluder dimension} \citep{russo2013eluder} is another
combinatorial parameter which was introduced to analyze the regret for
general function class variants of the UCB algorithm and Thompson
sampling for contextual bandits with adversarial contexts.\footnote{In
  the adversarial context setting we allow the contexts
  $x_1,\ldots,x_T$ to be chosen by an adaptive adversary, but we still
  assume that $r_t\sim{}\rdist(\cdot\mid{}x_t)$ at each round.} We recall the definition
here.\footnote{This definition differs slightly from that of
  \cite{russo2013eluder}, and in fact is always smaller (yet still
  sufficient to analyze UCB and Thompson sampling). The original
  definition allows $\Delta$ in \pref{eq:eluder1} to vary as a
  function of the index $i$.}
\begin{definition}[Value function eluder dimension]
  \label{def:value_eluder}
  Let $\ValueEluderCL{\Delta}
  $ be the length of the longest sequence
  of context-action pairs $(x\ind{1},a\ind{1}),\ldots,(x\ind{m},a\ind{m})$ such that for all $i$, there exists $f\ind{i}\in\cF$
  such that
  \begin{equation}
    \label{eq:eluder1}
\abs*{f\ind{i}(x\ind{i},a\ind{i})-\fstar(x\ind{i},a\ind{i})}>\Delta,\quad\text{and}\quad\sum_{j<i}\prn{f\ind{i}(x\ind{j},a\ind{j})-\fstar(x\ind{j},a\ind{j})}^{2}\leq{}\Delta^{2}.
  \end{equation}
  The \valueeluder is defined as $\ValueEluderL{\Delta_0} =
  \sup_{\Delta>\Delta_0}\ValueEluderC{\Delta}$.
\end{definition}
The only difference between the \valuestar and the \valueeluder is
whether the sum in \pref{eq:eluder1} takes the form
``$\sum_{j\neq{}i}$'' or ``$\sum_{j<i}$''; the latter reflects the
stronger sequential structure present when contexts are
adversarial. It is immediate that
\begin{equation}
  \label{eq:eluder_star}
\ValueStarL{\Delta} \leq{} \ValueEluderL{\Delta}.
\end{equation}
However, the separation between the two parameters can be arbitrarily
large in general.
\begin{proposition}
  \label{prop:eluder_star_separation}
    For every $d\in\bbN$ and $\Delta\in(0,1)$ there exists $\cF$ and
  $\fstar\in\cF$ such that $\sup_{\Delta'}\ValueStarL{\Delta'}\leq{}2$
  and $\ValueEluderL{\Delta/2}\geq{}d$.
  \end{proposition}
While \pref{eq:eluder_star} shows that boundedness of the eluder
dimension is sufficient for \mainalg achieve logarithmic regret for stochastic
contexts (via \pref{eq:value_star_ub_main}),
\pref{prop:eluder_star_separation}, shows that it may lead to rather
pessimistic upper bounds. This is not surprising, since the eluder
dimension was designed to accomodate adversarially chosen
contexts. The next result, which is a small refinement of the analysis
of \cite{russo2013eluder}, shows that bounded eluder dimension
indeed suffices to guarantee logarithmic regret for the adversarial setting; we
defer a precise description of the algorithm to the proof.
\begin{proposition}
  \label{prop:eluder_refined}
For the adversarial context setting, the general function class UCB algorithm---when configured
appropriately---guarantees that   
\begin{equation}
  \label{eq:eluder_ub}
    \En\brk*{\Reg} = \bigoht\prn*{
      \frac{\ValueEluderL{\Delta/2}\cdot{}\log\abs*{\cF}}{\Delta}
      }
    \end{equation}
    for any instance with uniform gap $\Delta$.%
\end{proposition}

Paralleling our results for the \valuestar, we show that boundedness
of a weak variant of the \valueeluder is required for logarithmic
regret with adversarial contexts.
\begin{definition}[Value function eluder dimension (weak variant)]
  \label{def:value_eluder_weak}
  For any $\Delta\in(0,1)$ and $\veps\in(0,\Delta/4)$, define $\ValueEluderWL{\Delta}{\veps}$ be the length of the largest sequence of contexts $x\ind{1},\ldots,x\ind{m}$ such that for
  all $i$, there exists $f\ind{i}\in\cF$, such that
  \begin{enumerate}
  \item\label{item:weak_eluder1}
    $f\ind{i}(x\ind{i},\pi_{f\ind{i}}(x\ind{i}))\geq{}\max_{a\neq{}\pi_{f\ind{i}}(x\ind{i})
    }f\ind{i}(x\ind{i},a)+\Delta$ and $\pi_{f\ind{i}}(x\ind{i})\neq{}\pistar(x\ind{i})$.    
    \item\label{item:weak_eluder2} $\max_{a}\abs*{f\ind{i}(x\ind{i},a)-\fstar(x\ind{i},a)}\leq{}2\Delta$
    \item\label{item:weak_eluder3} $\sum_{j<i}\max_{a}\abs*{f\ind{i}(x\ind{j},a)-\fstar(x\ind{j},a)}^{2}<\veps^{2}$.
  \end{enumerate}
\end{definition}
Our main lower bound here shows that---with the same caveats as
\pref{thm:ss_star_lb}---the scaling in \pref{eq:eluder_ub} is near-optimal.\footnote{As with \pref{thm:ss_star_lb}, we allow for unit Gaussian rewards
  with means in $\brk*{0,1}$ for this lower bound.}

\begin{theorem}
  \label{thm:eluder_lb}
Let a function class $\cF$ and $\fstar\in\cF$ with uniform gap
$\Delta$ be given. Let $\veps_T\in(0,\Delta/4)$ be the largest
solution to the equation
\begin{equation}
  \label{eq:eluder_fixed_point}
\veps_T^{2}T\leq{}\ValueEluderWL{\Delta/2}{\veps_T}.
\end{equation}
Then there exists a distribution $\cD$ such that for any algorithm with
$\En\brk*{\Reg}\leq{}2^{-6}\frac{\Delta{}T}{\ValueEluderWL{\Delta/2}{\veps_T}}$
on all instances realizable by $\cF$, there exists an a sequence
$\crl*{x_t}_{t=1}^{T}$ and instance with
$\fstar$ as the Bayes reward function such that
\begin{equation}
  \label{eq:value_eluder_lb}
\En\brk*{\Reg} = \bigom\prn*{\frac{\ValueEluderWL{\Delta/2}{\veps_T}}{\Delta}}.
\end{equation}  
\end{theorem}

\paragraph{Relating the eluder dimension to the disagreement
  coefficient}
An immediate consequence of \pref{thm:disagreement_to_star} and
\pref{eq:eluder_star} is that we always have
$\ValueDis{\Delta}{\veps}\leq\bigoh(\ValueEluder{\Delta}^2)$. While
this bound scales quadratically, we can show
through a more direct argument that the \valuedis grows at most
linearly with the eluder dimension.
\begin{theorem}[Value 
function eluder dimension bounds disagreement coefficient]
  \label{thm:disagreement_to_eluder}
  For any \uniformgc class $\cF$ and $\fstar:\cX\times\cA\to\brk*{0,1}$,
\begin{equation}
  \label{eq:disagreement_to_eluder}
  \sup_{\cD}\sup_{\veps>0}\ValueDisL{\Delta}{\veps} \leq{} 4
  \ValueEluderL{\Delta},\quad\forall{}\Delta>0.
\end{equation}
\end{theorem}
This result strongly suggests that the quadratic dependence on the \valuestar
in \pref{thm:disagreement_to_star} can be improved.

\subsubsection{The Policy Eluder Dimension}

Previous work which uses the eluder dimension to analyze algorithms
for contextual bandits and reinforcement learning
\citep{russo2013eluder,osband2014model,ayoub2020model,wang2020provably}
only works with the value function-based formulation in 
\pref{def:value_eluder}. In light of our results for the
disagreement coefficient and star number, we propose the following
policy-based variant of the eluder dimension.
\begin{definition}[Policy eluder dimension]
    For any policy $\pistar$ and policy class $\Pi$, let the \policyeluder
  $\PolicyEluderL$ denote the largest number $m$ such that there
  exist context-action pairs $(x\ind{1},a\ind{1}),\ldots,(x\ind{m},a\ind{m})$ and policies
  $\pi\ind{1},\ldots,\pi\ind{m}$ such that for all $i$,
  \[
\pi\ind{i}(x\ind{i})=a\ind{i}\neq{}\pistar(x\ind{i}),\mathand\pi\ind{i}(x\ind{j})=\pistar(x\ind{j})\quad\forall{}j<i
: x\ind{j}\neq{}x\ind{i}
  \]
\end{definition}
Of course, we immediately have that
\begin{equation}
  \label{eq:policy_star_eluder}
\PolicyStarL \leq{} \PolicyEluderL.  
\end{equation}
We are not yet aware of any upper bounds based on the \policyeluder,
but we can show that boundedness of this parameter is indeed necessary for logarithmic
regret in the adversarial context setting (in a worst-case sense).
  \begin{theorem}
    \label{thm:policy_eluder_lb}
    Consider the adversarial context setting. Let a policy class $\Pi$, $\pistar\in\Pi$,  and gap $\Delta\in(0,1/8)$ be given. Then
there exists a value function class $\cF$ such that:
\begin{enumerate}
\item $\Pi=\crl*{\pi_f\mid{}f\in\cF}$, and in particular some $\fstar\in\cF$ has
$\pistar=\pi_{\fstar}$.
  \item Each $f\in\cF$ has uniform gap $\Delta$.
\item For any algorithm with
  $\En\brk*{\Reg}\leq{}\frac{\Delta{}T}{32\PolicyEluderL}$ for all
  instances realizable by $\cF$, there exists a sequence
  $\crl*{x_t}_{t=1}^{T}$ and instance with
  $\fstar$ as the Bayes reward function such that
  \begin{equation}
    \label{eq:eluder_lb}
    \En\brk*{\Reg} =
    \bigom\prn*{\frac{\PolicyEluderL}{\Delta}}.
  \end{equation}
\end{enumerate}
\end{theorem}

\subsection{Discussion}
\subsubsection{Proof Techniques}\label{sec:cb_prooftech}
\paragraph{Policy disagreement-based upper bound
  (\pref{thm:disagreement_ub})}
Our proof of \pref{thm:disagreement_ub} builds on the regret analysis
framework established in \cite{simchi2020bypassing}, which interprets
\igw{} as maintaining a distribution over policies in the
\emph{universal policy space} $\cA^{\cX}$, and shows that the induced
distribution of policies is a solution to an \emph{implicit
  optimization problem} which (when configured appropriately) provides
a sufficient condition for minimax contextual bandit
learning. Following this framework, we also view \mainalg's sequential
\igw{}  procedure as implicitly maintaining a sequence of
distributions over  policies, but with an additional key property: the
support of the implicit distribution over policies is
\textit{adaptively shrinking}. This is enabled by \mainalg's
elimination procedure and is essential to our instance-dependent
analysis. We show that the implicit distribution over policies given
by \mainalg is a solution to a novel \emph{data-driven} implicit
optimization problem (\pref{lm:iop-i}), which, when configured
appropriately by adaptively selecting the learning rate with
\optionone, provides a sufficient condition for optimal policy
disagreement-based instance-dependent contextual bandit learning. Our
proof introduces several new techniques to instance-dependent analysis
of contextual bandits, including using disagreement-based indicators
and disagreement probability to obtain faster policy convergence rates
(\pref{lm:reward-estimates,lm:regret-estimates,lm:policy-convergence}).
We also remark that the selection of the adaptive learning rate is
non-trivial, and we derive the schedule \optionone by carefully balancing key quantities appearing in our analysis.

\paragraph{Value function disagreement-based upper bound (\pref{thm:value_disagreement_ub})} 
The proof of \pref{thm:value_disagreement_ub} %
consists of two steps. %
In the first step, we build on the minimax analysis framework of
\cite{simchi2020bypassing} and show that \mainalg with \optiontwo
always guarantees the minimax rate. A new trick that we use here is to
carefully track the adaptive value of $\lambda_m$ and use it to infer
the exploration cost under the current instance. In the second step,
we establish the $\frac{\ValueDisA\cdot{}A\log|\cF|}{\Delta}$-type
instance-dependent upper bound for regret. The analysis is driven by a
key inequality (\pref{lm:expected-delta}), which provides a sharp
upper bound on $\E_{\cD}[w(x;\cF_m)]$ in terms of the ratio
$\frac{\ValueDisA}{\Delta}$. Beyond giving a means to bound the
(expected) instantaneous regret in terms of $\ValueDisA$ and $\Delta$,
this allows us to adaptively maintain an estimated lower bound for
$\frac{\ValueDisA}{\Delta}$ based on empirical data. We then use an
induction argument to show that the specification of $\gamma_m$ in \optiontwo enables \mainalg to enjoy a near-optimal instance-dependent guarantee.

\paragraph{Lower bounds}
Our lower bounds build on the work of \cite{raginsky2011lower}, which
provides information-theoretic lower bounds for passive and active
learning in terms of the disagreement coefficient. As in this work, we
rely on a specialized application of the Fano method using the
\emph{reverse} KL-divergence, but with some refinements to make the
technique more suited for \emph{regret} lower bounds. For
\pref{thm:disagreement_lb}, we also incorporate improvements to the
method suggested by
\cite{hanneke2014theory} to obtain the
correct dependence on $\log\abs{\cF}$.

\paragraph{Value function star number bounds disagreement coefficient
  (\pref{{thm:disagreement_to_star}})}
The proof of \pref{thm:disagreement_to_star} is somewhat different
from the proof of the analogous policy-based result by
\cite{hanneke2015minimax}. The key step toward proving \pref{{thm:disagreement_to_star}} is
to prove an empirical analogue of the result that holds whenever $\cD$
is uniform over a finite sequence of examples. This result is given in
\pref{lem:star}, and is motivated by a property of the eluder
dimension established in Proposition 3 of \cite{russo2013eluder}, with
their ``$\sum_{j< i}$''-based definition changed to our ``$\sum_{j\ne
  i}$''-based definition. The proof of \pref{lem:star} trickier, however, 
as our ``$\sum_{j\ne i}$''-based definition 
breaks several combinatorial properties utilized in the proof of  \cite{russo2013eluder}.
We address this challenge by proving a new combinatorial lemma (\pref{lem:comb}), 
which is fairly general and may be interesting on its own right.
Nevertheless, our upper bound is quadratic in $\ValueStarL{\Delta}$
rather than linear, and
we hope that this dependence can be improved in future work.

\subsubsection{Related Work}
Gap-dependent regret bounds for contextual bandits have not been
systematically studied at the level of generality we consider
here, and we are not aware of any prior lower bounds beyond the linear
setting. Most prior work has focused on structured function classes
such as linear \citep{dani2008stochastic,abbasi2011improved,hao2019adaptive} and
nonparametric Lipschitz/\Holder classes
\citep{rigollet2010nonparametric,perchet2013multi,hu2020smooth}. 

Our work draws inspiration from \cite{krishnamurthy2017active}, who defined
variants of the disagreement coefficient which depend on
scale-sensitive properties of the class $\cF$ in the context of cost-sensitive multiclass active
learning. Compared to these results, the key difference is that our
\valuedis is defined in terms of the $L_2$ ball for the class $\cF$
rather than the excess risk ball for the induced policy class. This change
is critical to ensure that the \valuedis is bounded by the \valuestar,
and in particular that it is always bounded for linear classes.

Our work also builds on \cite{foster2018practical}, who give
instance-dependent guarantees for
the generalized UCB algorithm and an action elimination variant for
general function classes based on the cost-sensitive multiclass disagreement coefficients
introduced in \cite{krishnamurthy2017active}. We improve upon this
result on several fronts: 1) As mentioned above, our notion of
\valuedis is tighter, and is always bounded by the \valuestar and
\valueeluder 2) we attain optimal dependence on the gap, 3)
our algorithms are guaranteed to attain the minimax rate in the worst
case, and 4) we complement these results with lower bounds.

Lastly, we mention that while there are no prior lower bounds for contextual bandits based on the
eluder dimension, \cite{wen2017efficient} give an eluder-based lower
bound for reinforcement learning with deterministic transitions and
known rewards. This result is closer in spirit to our
disagreement-based lower bounds
(\pref{thm:disagreement_lb,thm:value_disagreement_lb}), is it applies
to a carefully constructed function class rather than holding for all
function classes, and mainly serves to demonstrate the worst-case
tightness of a particular upper bound.

\subsection{Extensions}
\label{sec:cb_extensions}
We conclude this section by presenting some basic extensions of our
contextual bandit results, including extensions of our regret bounds
to handle infinite classes and weaker
noise conditions.

\paragraph{Infinite function classes}
As we have mentioned, \pref{alg:main}, \pref{thm:disagreement_ub} and
\pref{thm:value_disagreement_ub} trivially extend to infinite $\cF$,
with the dependence on $\log|\cF|$ in the algorithm's parameters and
the regret bounds  replaced by standard learning-theoretic complexity
measures such as the pseudodimension, (localized) Rademacher
complexity, or metric entropy. This is because the analysis of
\mainalg (see \pref{app:cb_upper}) does not rely on any complexity assumptions
for $\cF$, except for \pref{lm:high-prob-event}, which uses a standard
uniform martingale concentration bound for the square loss to show
that the empirical risk minimizer $\fhat_m$ has low excess risk at each
epoch. Therefore, to extend our results to infinite $\cF$, one only
needs to replace \pref{lm:high-prob-event} with an analogous uniform
martingale concentration inequality for infinite classes. Such results
have already been established in the literature, see, e.g., \cite{krishnamurthy2017active} and \cite{foster2018practical}.

\paragraph{Alternative noise conditions}
Beyond uniform gap, \mainalg can also adapt to the Tsybakov
noise condition \citep{mammen1999smooth,tsybakov2004optimal,audibert2007fast,rigollet2010nonparametric,hu2020smooth}, as the following
proposition shows.
\begin{proposition}[Regret under the Tsybakov noise condition]
  \label{prop:tsybakov}
  Suppose there exist constants $\alpha
  ,\beta\ge0$ such
  that
  \[
    \bbP_{\cD}\prn*{f^*(x,\pistar(x))-\max_{a\neq{}\pistar(x)}f^*(x,a)\le\gamma}\le\beta\gamma^{\alpha},~~\forall
    \gamma\ge0.
  \]Then \pref{alg:main} with \optionone
  ensures that
\[
\En\brk{\Reg}=\bigoht(1)\cdot\min_{\varepsilon>0}\max\left\{\varepsilon
  T,
  \left(\PolicyDis{\veps}{\K\log|\cF|}\right)^{\frac{1+\alpha}{2+\alpha}}T^{\frac{1}{2+\alpha}}
\right\}+ \bigoht(1).
\]
\end{proposition}

\paragraph{Tightening the \valuedis}
The supremum over the action distribution $p$ in the definition
\pref{eq:value_disagreement_full} of the
\valuedis is more pessimistic than
what is actually required to analyze
\mainalg. Consider the following action distribution-dependent
definition:
\begin{equation}
  \label{eq:value_disagreement_p}
  \ValueDisP{\Delta_0}{\veps_0} = \sup_{\Delta>\Delta_0,\veps>\veps_0}\frac{\Delta^{2}}{\veps^{2}}\bbP_{\cD,p}\prn*{
    \exists{}f\in\cF: \abs*{f(x,a)-\fstar(x,a)}>\Delta,\;\nrm*{f-\fstar}_{\cD,p}\leq\veps
    }.
\end{equation}
The regret bound in \pref{thm:value_disagreement_ub} can be tightened to
depend on $\sup_{p\in\mathcal{P}}\ValueDisP{\Delta/2}{\veps_T}$, 
where $\cP$ is a set of action
distributions with favorable properties that can lead to tighter bounds. In particular, the proof of \pref{thm:value_disagreement_ub} implies that for any instance with uniform gap $\Delta$, if $\ValueDisP{\Delta/2}{\veps_T}\le\theta$ for all $p$ such that $p(\pi^*(x)|x)\ge A^{-1}$ for all $x$, then \mainalg with \optiontwo ensures that
\[
\E[\Reg]=\bigoht(1)\cdot\min\crl*{
      \sqrt{\K{}T\log\abs{\cF}}, \frac{\theta A\log\abs{\cF}}{\Delta}}+\bigoh(1).
yx\]
 The following result shows that this property leads to
dimension-independent bounds for sparse linear function classes.
\begin{proposition}
  \label{prop:sparse_linear}
  Consider the function class
  $\cF=\crl*{(x,a)\mapsto\tri*{w,\phi(x,a)}\mid{}w\in\bbR^{d},\nrm*{w}_{0}\leq{}s}$,
  where $\nrm*{\phi(x,a)}_{\infty}\leq{}1$. Define
  $\Sigma^{\star}=\En_{\cD}\brk*{\phi(x,\pistar(x))\phi(x,\pistar(x))^{\trn}}$,
  and let
  $\lambda_{\mathrm{re}}=\inf_{w\neq{}0,\nrm*{w}_{0}\leq{}2s}\tri*{w,\Sigma^{\star}w}/\nrm*{w}^{2}_2$
  be the restricted eigenvalue. Then $\forall \Delta,\veps>0$,
  \[
    \ValueDisP{\Delta}{\veps}\leq{}2\alpha^{-1}\lambda_{\mathrm{re}}^{-1}s
  \]
  for all $p$ such that $p(\pistar(x)|x)\geq{}\alpha$ for all $x$.
\end{proposition}
As a concrete example, if $\phi(x,\pistar(x))\sim\unif(\pmo^{d})$ we
have $\lambda_{\mathrm{re}}=1$, so that the bound is indeed dimension-independent.

\paragraph{Handling multiple optimal actions}
For simplicity, we assume that $\arg\max_{a\in\cA} f^*(x,a)$ is unique for all $x$ in the main body of the paper. When such assumption does not hold, we keep the original definition of $\pi^*(x)$ (which makes $\pi^*(x)$ unique for each $x\in\cX$), while defining 
\[\pi^*_{\rm set}(x)\ldef{}\{a\in\cA\mid f^*(x,a)=\max_{a'\in\cA}f^*(x,a')\},~~\forall x\in\cX.\]
We then make the following modifications to our framework. First, we
modify the uniform gap condition \pref{eq:gap} to require that for all
$x\in\cX$, 
\[
\fstar(x,\pistar(x)) -\fstar(x,a)\geq\Delta\quad\forall{}a\notin\pi^*_{\rm set}(x).
\]
Second, we modify the definition of \policydis to
\begin{equation*}
  \PolicyDisL{\veps_0}= \sup_{\veps\geq{}\veps_0}\frac{\bbP_{\cD}\prn*{x:\exists\pi\in\Pi_{\veps}: \pi(x)\neq\pistar(x)}}{\veps},
\end{equation*}
where $\Pi_{\veps}\ldef{}\crl*{\pi\in\Pi:
  \bbP_{\cD}\prn{\pi(x)\notin\pistar_{\rm set}(x)}\leq{}\veps}$. By doing so, all our guarantees for \mainalg extend to the general setting where $\arg\max_{a\in\cA} f^*(x,a)$ may not be unique for some $x\in\cX$.

\section{Reinforcement Learning}
\label{sec:rl}
We now give disagreement-based guarantees for reinforcement learning
with function approximation in the
block MDP setting (cf. \pref{sec:rl-setup}). Before proceeding, let us introduce some additional notation.

\paragraph{Additional notation}
For any Markov policy $\pi(x)$, let
$\Qf^{\pi}_h(x,a)=\En\brk*{\sum_{h'\geq{}h}^{H}r_{h'}\mid{}x_h=x,a_h=a}$ be
the corresponding Q-function. We likewise define
$\Vf^{\pi}_h(x)=\max_{a\in\cA}\Qf^{\pi}(x,a)$, as well as
$\Vf^{\pi}=\En_{x_1}\brk*{\Vf_1^{\pi}(x_1)}$ and
$\Vstar=\En_{x_1}\brk*{\Vstar_1(x_1)}$. Next, for any function
$V:\cX\to\bbR$ we define the transition operator by
\[
\brk*{\Pstar_h{}V}(x,a)=\En\brk*{V(x_{h+1})\mid{}x_h=x,a_h=a}.
\]
We also define the Bayes reward function as
\[
\fbayes(x,a) = \En\brk*{r_h\mid{}x_h=x,a_h=a}
\]
for each $x\in\cX_h$. Finally, let $\cF=\cF_1\times\cF_2\times\cdots\times\cF_H$ be the full
regression function class, and define $\Fmax=\max_{h}\abs*{\cF_h}$.

\subsection{The Algorithm}\label{sec:blockalg}
Our main reinforcement learning algorithm, \blockalg, is presented in
\pref{alg:blockalg}. The algorithm follows the optimistic least-squares
value iteration framework
\citep{jin2020provably,wang2019optimism,wang2020provably}, with few key changes that allow us to prove guarantees
based on a suitable notion of \valuedis rather than stronger
complexity measures such as the eluder dimension. The most interesting
aspect of the algorithm is a feature we call the \emph{star hull upper confidence
  bound}: Compared to the classical UCB approach, which computes an
optimistic $Q$-function by taking largest predicted reward amongst all value function in an $L_2$ ball
around an empirical risk minimizer, we add an additional step which
first ``lightly convexifies'' this set. This step is based on techniques from the literature on aggregation in least squares
\citep{audibert2008progressive,liang2015learning}, and leads to more
stable predictions.

In more detail, the algorithm proceeds in $K$ iterations.\footnote{We
  use the term ``iteration'' distinctly from the term ``episode'', as
  each iteration consists of multiple episodes.} In each iteration $k$, we
compute an optimistic Q-function $\Qbar\ind{k}$ such that
\begin{equation}
\Qbar_h\ind{k}(x,a) \geq \Qstar_h(x,a)\quad\text{for all $x$, $a$,
  $h$.}\label{eq:optimism}
\end{equation}
We then take the greedy argmax policy defined by
$\pi\ind{k}(x)=\argmax_{a\in\cA}\Qbar_h\ind{k}(x,a)$ for $x\in\cX_h$,
and gather $H$ trajectories as follows: For each $h$, we roll in to layer $h$ with
$\pi\ind{k}$, then choose actions uniformly at random for the rest of
the episode. These
trajectories are used to refine our value function estimates for
subsequent iterations, with the $h$th trajectory used for estimation at
layer $h$. Choosing actions uniformly ensures that the data gathered
from these trajectories is useful regardless of the action
distribution in subsequent iterations.

Let us now elaborate on the upper confidence bound computation. Let
iteration $k$ and layer $h$ be fixed, and suppose we have already
computed $\Qbar_{h+1}\ind{k}$ and
$\Vbar_{h+1}\ind{k}(x)\ldef{}\max_{a\in\cA}\Qbar_{h+1}\ind{k}(x,a)$. The
first step, following the usual optimistic LSVI schema, is to estimate
a value function for layer $h$ by regressing onto the empirical
Bellman backups from the next layer (\pref{line:erm-rl}):
\begin{align}
\fhat_{h}\ind{k}=\argmin_{f\in\cF_h}\sum_{j<k}\prn*{f(x_h\ind{j,h},a_h\ind{j,h})
- \prn*{r_h\ind{j,h} + \Vbar_{h+1}\ind{k}(x_{h+1}\ind{j,h})
}}^{2};\label{eq:erm-rl}
\end{align}
here the $(j,h)$ superscript on
$(x_h\ind{j,h},a_h\ind{j,h},r_h\ind{j,h},x_{h+1}\ind{j,h})$ indicates that the example was collected in
the $h$th trajectory at iteration $j$.
\pref{ass:completeness} ensures that this regression problem is
well-specified. Let $\cZ_h\ind{k}=\crl*{(x_h\ind{j,h},a_h\ind{j,h})}_{j<k}$, and
define
\begin{equation}
\nrm*{f-f'}_{\cZ}^{2} =
\sum_{(x,a)\in\cZ}\prn*{f(x,a)-f'(x,a)}^{2}.\label{eq:znorm}
\end{equation}
At this point, the usual optimistic value function for layer $h$
(cf. \cite{russo2013eluder,foster2018practical} for contextual
bandits and \cite{jin2020provably,wang2019optimism,wang2020provably}
for RL) is defined as
\[
\Qbar_h(x,a) = \sup\crl*{f(x,a) \mid{} f\in\cF_h, \nrm[\big]{f-\fhat_h\ind{k}}_{\cZ_h\ind{k}}\leq{}\beta_h},
\]
where $\beta_h$ is a confidence parameter. As observed in
\cite{jin2020provably,wang2020provably}, however, this UCB function
can be unstable, leading to issues with generalization when we use it
as a target for least squares at layer $h-1$. Our approach to address
this problem is to expand the supremum above to include the \emph{star
  hull} of $\cF_h$ centered at $\fhat_h\ind{k}$. Define the star hull of $\cF_h$ centered at $f\in\cF_h$ by
\begin{equation}
  \starhull(\cF,f) =\bigcup_{f'\in\cF}\conv(\crl{f', f})
  = \crl*{t(f'-f) + f\mid{}f'\in\cF, t\in\brk*{0,1}}.\label{eq:star_hull}
\end{equation}
We define the \emph{\starucblong} (\pref{line:star-ucb}) by
\begin{equation}
\Qbar_h\ind{k}(x,a) = \sup\crl*{f(x,a) \mid{}
  f\in\starhull(\cF_h,\fhat_h\ind{k}),
  \nrm[\big]{f-\fhat_h\ind{k}}_{\cZ_h\ind{k}}\leq{}\beta_h}.\label{eq:star_ucb}
\end{equation}
When $\cF_h$ is convex this coincides with the usual upper confidence bound, but in the
star hull operation convexifies $\cF_h$ along rays emanating from
$\fhat_h\ind{k}$. This small amount of convexification (note that
$\starhull(\cF_h,f)$ is still non-convex if, e.g., $\cF_h$ is a finite
class), ensures that $\Qbar_h\ind{k}(x,a)$ is Lipschitz as a function
of the confidence radius $\beta_h$, which stabilizes the predictions and
facilitates a tight generalization analysis.
\paragraph{Oracle efficiency}
\blockalg is oracle-efficient, and can be implemented using an offline
regression oracle as follows.
\begin{itemize}
\item At each iteration, the empirical risk minimizer in
  \pref{line:erm-rl} can be computed with a single oracle call.
\item For any $(x,a)$ pair, the star hull UCB function in
  \pref{line:star-ucb} can be computed by reduction to a
  regression oracle. In particular, to compute an $\veps$-approximate UCB:
  \begin{itemize}
  \item For convex function classes, $\bigoh(\log(1/\veps))$ calls are required.
  \item For general (in particular, finite) classes,
    $\bigoht(\veps^{-3})$ oracle calls are required. The key idea here
    is that we can reduce ERM over the star hull to ERM over the
    original class.
  \end{itemize}
\end{itemize}
See \pref{sec:oracle} for more details.
\begin{algorithm}[ht]
  \setstretch{1.1}
  \textbf{input:} 
  Value function classes $\cF_1,\ldots,\cF_{H}$. Number of iterations $K$. 

  \textbf{initialization:}
  \begin{itemize}[leftmargin=*]
  \item[-] Let $\delta=1/KH$. \algcomment{Failure probability.}
  \item[-] Let $\beta^2_{H}=400H^{2}\log\prn{\Fmax{}HK\delta^{-1}}$
    \algcomment{Confidence radius.}
  \item[] and
    $\beta^{2}_{h} = \frac{1}{2}\beta_{h+1}^{2}+
    60^4H^{2}A^{2}\btheta^{\val}_{h+1}(\cF_{h+1},\beta_{h+1}K^{-1/2})^{2}\log^2(HKe)\log(2\Fmax{}HK\delta^{-1})
    + 700H^{2}S\log(2eK)$ for all $1\leq{}h\leq{}H-1$.
  \end{itemize}

  \textbf{algorithm:}
  \begin{algorithmic}[1]
    \For{iteration $k=1,\ldots,K$}
    \State{}Set $\Vbar_{H+1}\ind{k}(x)=0$.
    \State{}Define $\cZ_h\ind{k}=\crl*{(x_h\ind{j,h},a_h\ind{j,h})}_{j<k}$.
      \For{$h=H,\ldots,1$}
            \State Set
$\fhat_{h}\ind{k}=\argmin_{f\in\cF_h}\sum_{j<k}\prn*{f(x_h\ind{j,h},a_h\ind{j,h})
- \prn*{r_h\ind{j,h} + \Vbar_{h+1}\ind{k}(x_{h+1}\ind{j,h})
}}^{2}$.\label{line:erm-rl}
\Statex{}~~~~~~~~~\algcomment{Compute optimistic value function via star-hull upper confidence bound.}
            \State Define
            \[
              \Qbar_{h}\ind{k}(x,a)
              =\sup\crl*{f(x,a)\mid{}f\in\starhull(\cF_h,\fhat_h\ind{k}), \nrm[\big]{f-\fhat_h\ind{k}}_{\cZ_h\ind{k}}\leq{}\beta_h}.
            \]\label{line:star-ucb}
            \State{}$\pi\ind{k}(x)\ldef{}\argmax_{a\in\cA}\Qbar_h\ind{k}(x,a)$
            for all $x\in\cX_h$.
            \State{}$\Vbar_h\ind{k}(x)\ldef{}\max_{a\in\cA}\Qbar_h\ind{k}(x,a)$. %
      \EndFor
      \For{$h=1,\ldots,H$}
      \State{}Gather trajectory
      $(x_1\ind{k,h},a_1\ind{k,h},r_1\ind{k,h}), \ldots,
      (x_H\ind{k,h},a_H\ind{k,h},r_H\ind{k,h})$ by rolling in with
      $\pi\ind{k}$\Statex{}~~~~~~~~~for layers $1,\ldots,h-1$ and selecting actions
      uniformly for layers $h,\ldots,H$.
      \EndFor
    \EndFor
    \State \textbf{return} $\pi\ind{k}$ for $k\sim\unif(\brk*{K})$.
  \end{algorithmic}
  \caption{\blockalg}
  \label{alg:blockalg}
\end{algorithm}

\subsection{Main Result}
We now state the main guarantee for \blockalg. Our guarantee depends on the following
``per-state'' gap and worst-case gap:
\begin{align*}
  &\gap(s) =
    \min_{a}\inf_{x\in\supp(\emi(s))}\crl*{\gap(x,a)\mid{}\gap
    (x,a)>0},\\
  &\gapmin=\min_{s}\gap(s),
\end{align*}
where we recall that $\gap(x,a)\ldef \Vstar_h(x) - \Qstar_h(x,a)$.
%
\begin{comment}
  We adapt the \valuedis to the block MDP setting as follows. Let
  $\piunif$ be the policy that selects actions uniformly from
  $\cA$. For each latent state $s\in\cS_h$, we define
  \begin{align}
    \btheta^{\val}_{s}(\cF_h,\Delta,\veps_0) =
    \sup_{\fstar\in\cF_h}\sup_{\Delta>\Delta_0,\veps\geq{}\veps_0}\frac{\Delta^{2}}{\veps^{2}}\bbP_{x\sim\emi(s),a\sim\piunif}\prn[\Big]{\exists{}f\in\cF_h
    : \abs*{f(x,a)-\fstar(x,a)}>\Delta,
    \nrm*{f-\fstar}_{s}\leq{}\veps},\label{eq:rl_disagreement}
  \end{align}
  where
  $\nrm*{f}^{2}_s\ldef{}\En_{x\sim\emi(s),a\sim\piunif}\brk{f^{2}(x,a)}$. Beyond
  minor technical differences (e.g., the definition only uses a fixed
  action distribution) this is the same as the notion of \valuedis
  \pref{eq:value_disagreement_full} for the contextual bandit setting,
  with the context distribution $\cD$ replaced by the latent state's
  emmission distribution $\emi(s)$.
\end{comment}
  We adapt the \valuedis to the block MDP setting as follows. Let
  $\piunif$ be the policy that selects actions uniformly from
  $\cA$. For each latent state $s\in\cS_h$, we define
  \begin{align}
    \btheta^{\val}_{s}(\cF_h,\veps_0) =
    \sup_{\fstar\in\cF_h}\sup_{\veps\geq{}\veps_0}\frac{1}{\veps^{2}}\En_{x\sim\emi(s),a\sim\piunif}\sup\crl[\Big]{
     \abs*{f(x,a)-\fstar(x,a)}^{2}\mid f\in\cF_h,
    \nrm*{f-\fstar}_{s}\leq{}\veps},\label{eq:rl_disagreement}
  \end{align}
  where
  $\nrm*{f}^{2}_s\ldef{}\En_{x\sim\emi(s),a\sim\piunif}\brk{f^{2}(x,a)}$.
  This notion is closely related to the \valuedis
  \pref{eq:value_disagreement_full} for the contextual bandit setting
  via Markov's inequality,
  with the context distribution $\cD$ replaced by the latent state's
  emmission distribution $\emi(s)$. We
additionally define the total disagreement for layer $h$ by
\[
\btheta^{\val}_h(\cF_h,\veps) = \sum_{s\in\cS_h}\btheta^{\val}_s(\cF_h,\veps),
\]
and define $\btheta^{\val}_{\max}(\cF,\veps)=\max_{h}\max_{s\in\cS_h}\btheta^{\val}_s(\cF_h,\veps)$.

Our main theorem bounding the error of \blockalg is as
follows. As with our contextual bandit results, we focus on finite
classes $\cF$ for simplicity, but the result trivially extends to
general function classes.
\begin{theorem}
  \label{thm:block_mdp}
  \pref{alg:blockalg} guarantees that
    \begin{align*}
\Vstar-\En\brk{\Vf^{\pi}}   &=
                                            \bigoht\prn*{
                                            \frac{\btheta^{\val}_{\max}(\cF,\beta_HK^{-1/2})^3
                              \cdot{}H^{5}A^{3}S^{2}\log\abs*{\cF}}{\gapmin{}K}},
    \end{align*}
                                  and does so using at most $HK$ trajectories.
More generally, the algorithm guarantees that
  \begin{align*}
\Vstar-\En\brk{\Vf^{\pi}}   &=
                                            \bigoht\prn*{
                              C_{\cM}\cdot\frac{H^{2}A^{3}\max_{h}\btheta^{\val}_h(\cF,\beta_hK^{-1/2})^2\log\abs*{\cF}
                                            + H^{3}SA}{K}},
  \end{align*}
  where $C_{\cM}\ldef\sum_{h=1}^{H}\sum_{s\in\cS_h}\frac{\btheta_s^{\val}(\cF_h,\beta_hK^{-1/2})}{\gap(s)}$.

\end{theorem}
\newcommand{\thetamax}{\btheta^{\val}_{\mathrm{max}}}
\newcommand{\Deltamin}{\Delta_{\min}}
Let us describe a few key features of this theorem and interpret the result.
\begin{itemize}
\item First, if $\thetamax(\cF,\veps)\propto\polylog(1/\veps)$
  (e.g., for a linear function class), then---ignoring other parameters---we can attain an
  $\veps$-optimal policy using $\frac{1}{\Deltamin\veps}$
  trajectories. This \emph{fast rate} improves over the minimax
  optimal $\veps^{-2}$ rate, and is optimal even for bandits. This is
  the first fast rate result we are aware of for reinforcement learning in
  block MDPs.
\item In light of the results in \pref{sec:cb} this implies that one
  can attain the fast $\frac{1}{\Deltamin\veps}$ rate whenever the
  \valuestar for $\cF$ is bounded.
\item More generally, if
  $\thetamax(\cF,\veps)\propto\veps^{-\rho}$ for $\rho<2/3$,
  then $(\Delta\veps)^{-\frac{2}{2-3\rho}}$ trajectories suffice for
  an $\veps$-optimal policy. However, the guarantee becomes vacuous
  once $\rho\geq{}2/3$. In other words, bounded disagreement
  coefficient is essentially for the algorithm to have low error, and
  it does not necessarily attain the minimax rate if this fails to
  hold. Achieving a best-of-both-worlds guarantee similar to our
  results for contextual bandits is an interesting direction for
  future work.
\item To the best of our knowledge this is the first oracle-efficient
  algorithm that attains near-optimal statistical performance in terms of
  $\veps$ for block MDPs. Of course, this is only achieved in the
  low-noise regime where $\gapmin>0$, and when $\thetamax(\cF,\veps)\propto\polylog(1/\veps)$.
\end{itemize}
We emphasize that while the dependence on all of the
parameters in \pref{thm:block_mdp} can almost certainly be improved,
we hope this result will open the door for further disagreement-based
algorithms and analysis techniques in reinforcement learning.

\subsection{Discussion}
\subsubsection{Proof Techniques}
The proof of \pref{thm:block_mdp} has two main components. The first
part of the proof shows that with high probability, for all iterations
$k$ and layers $h$, the set $\cF_h\ind{k}\ldef\crl[\big]{f\in\cF_h\mid{}\nrm{f-\fhat_h\ind{k}}_{\cZ_h\ind{k}}\leq{}\beta_h}$
contains the Bellman backup
$\brk*{\Pstar_h\Vbar_{h+1}\ind{k}}(x,a)+\fstar(x,a)$ of the value
function from the next layer, which ensures that $\Qbar_h\ind{k}$ is
optimistic in the sense of \pref{eq:optimism} and leads to exploration. Then, in the second
part, we prove a regret decomposition which shows that whenever the
optimistic property holds, the suboptimality of $\pi\ind{k}$ is controlled by
the gap $\Delta$ and the \valuedis $\ValueDisA$.

The first part of the proof (\pref{app:confidence_radius}) boils down to showing that the empirical
risk minimizer in $\fhat_h\ind{k}$ in \pref{eq:erm-rl} has favorable
concentration properties. This is highly non-trivial because the
targets $\Vbar_{h+1}\ind{k}$ in \pref{eq:erm-rl} depend on the entire dataset, which
breaks the independence assumptions
required to apply standard generalization bounds for least
squares. Instead, following \cite{jin2020provably,wang2020provably},
we opt for a \emph{uniform} generalization bound which holds
uniformly over all possible choices of $\Vbar_{h+1}\ind{k}$. To do so,
we must show that $\Vbar_{h+1}\ind{k}$ is approximated by a relatively
low complexity function class, which we accomplish as follows. First, we
show that---thanks to a certain Lipschitz property granted by the star
hull---$\Qbar_{h+1}\ind{k}$ is well approximated by a function
\[
  \Qtil_{h+1}\ind{k}(x,a)
  \ldef{}\sup\crl*{f(x,a)\mid{}f\in\starhull(\cF_{h+1},\fhat_{h+1}\ind{k}),
    \nrm[\big]{f-\fhat_{h+1}\ind{k}}_{\cL_{h+1}\ind{k}}\leq{}\wt{\beta}_{h+1}},
\]
where $\wt{\beta}_{h+1}\approx\beta_{h+1}$, and where
\[
\nrm*{f}^{2}_{\cL_{h+1}\ind{k}}\ldef{}\sum_{j<k}\En_{x_{h+1}\sim\psi(s_{h+1}\ind{j,h+1}),a\sim\piunif}\brk{f^{2}(x,a)}
\]
is the \emph{latent state norm}, which measures the
expected squared error conditioned on the sequence of latent states
$\cL_{h+1}\ind{k}\ldef(s\ind{1,h+1},\ldots,s\ind{k-1,h+1})$
encountered in the trajectories gathered
for layer $h+1$. This
approximation argument is rather non-trivial, and involves a recursion
across all layers that we manage using the disagreement
coefficient. With this taken care of, the next step is to use the
block MDP structure to argue that $\Qtil_{h+1}\ind{k}$ has low
complexity. To see this, observe that $\Qtil_{h+1}\ind{k}$ is completely determined by the center $\fhat\ind{k}_{h+1}$ and the latent state sequence
above. Since the latent state constraint does not
depend on the ordering of the latent states, we can use a counting argument to show that there are
at most $\abs*{\cF}K^{\bigoh(S)}$ possible choices for
$\Qtil_{h+1}\ind{k}$ overall. This suffices to prove the desired
concentration guarantee.

The second part of the proof (\pref{app:rl_main}) proceeds as
follows. Define the Bellman surplus as 
\[
\Ebar\ind{k}_h(x,a)=\Qbar\ind{k}_h(x,a) -
\prn*{\fbayes(x,a) + \brk{\Pstar_h\Vbar\ind{k}_{h+1}}(x,a)},
\]
which measures the width for our upper confidence bound. We use a ``clipped'' regret decomposition from
\cite{simchowitz2019non} to show that whenever the concentration event
from the first part of the proof holds, the suboptimality of $\pi\ind{k}$ is
controlled by the confidence widths:
\begin{align*}
\Vstar-\Vf^{\pi\ind{k}} \approxleq{}
  \sum_{h=1}^{H}\sum_{s\in\cS_h}\bbP_{\pi\ind{k}}(s_h=s)  \cdot \frac{\En_{x_h\sim\emi(s_h),a_h\sim\piunif}\brk*{\Ebar_h\ind{k}(x_h,a_h)^{2}}}{\Delta(s)}.
\end{align*}
In particular, let $n\ind{k,h}(s)$ denote the number of times the
latent state $s$ was encountered in the layer $h$ trajectories prior
to iteration $k$. Our key observation is that bounded disagreement coefficient implies that for each state $s$, 
\[
\En_{x_h\sim\emi(s),a_h\sim\piunif}\brk*{\Ebar_h\ind{k}(x_h,a_h)^{2}}
\approxleq \ValueDisA_{s}\cdot\frac{\beta_{h}^{2}}{n\ind{k,h}(s)}.
\]
In other words, the disagreement coefficient controls the rate
  at which the confidence width shrinks. Moreover, since the width
  for latent state $s$ is proportional to the number of times we have
  visited the state (even though the algorithm cannot observe this
  quantity), we can bound the overall suboptimality across all iterations
  using similar arguments to those employed in the tabular setting \citep{azar2017minimax,simchowitz2019non}.

\subsubsection{Related Work}
\label{sec:rl_related_work}

Our result is closely related to that of \cite{wang2020provably}, who
gave regret bounds for a variant of optimistic LSVI based on the
eluder dimension of $\cF$. Compared to this result, we require the
additional block MDP assumption and finite actions, but our bounds scale with the \valuedis, which can be arbitrarily small compared to
the eluder dimension (\pref{prop:eluder_star_separation}). On the technical side, their algorithm stabilizes the upper confidence
bounds using a sensitivity sampling procedure, whereas we address this
issue using the star hull. \cite{ayoub2020model} give similar eluder
dimension-based guarantees for a model-based algorithm, though the
notion of eluder dimension is somewhat stronger, and it is not clear
whether this algorithm can be made oracle-efficient.

Reinforcement learning with function approximation in block MDPs has
been the subject of extensive recent investigation
\citep{krishnamurthy2016pac,jiang2017contextual,dann2018oracle,du2019provably,du2019latent,misra2019kinematic,feng2020provably,agarwal2020flambe}. In
terms of assumptions, we require the rather strong optimistic
completeness condition, but do not require any reachability conditions
or any clusterability-type assumptions that facilitate the use of
unsupervised learning. The main advantages of our results are 1) we
require only a basic regression oracle for the value function class,
and 2) we attain the optimal $\veps^{-1}$ fast rate in the presence of
the gap and bounded disagreement coefficient.

We should also mention that the gap for $\Qstar$ has been used in a
number of recent results on reinforcement learning with function
approximation \citep{du2019provably,du2019good,du2020agnostic}, albeit
for a somewhat different purpose. These results use the gap to prove
that certain ``non-optimistic'' algorithms succeed, whereas we use it
to beat the minimax rate.

Lastly, we note that the \valuedis is similar to the ``low variance'' parameter used in
\cite{du2019provably} to give guarantees for reinforcement learning
with linear function approximation, but can be considerably smaller
when applied to block MDPs. For example, in the trivial case in
which each emission distribution $\psi(s)$ is a singleton, the
\valuedis is automatically bounded by $1$, while the low variance
assumption may not be satisfied unless the latent MDP is
near-deterministic.

\section{Implementing the Algorithms with Regression Oracles}
\label{sec:oracle}

In \pref{sec:alg} and \pref{sec:blockalg}, we mentioned that both \mainalg (\pref{alg:main}) and \blockalg (\pref{alg:blockalg}) can be efficiently implemented with an offline regression oracle. In this section, we provide more details on this implementation, and on the overall computational complexity of our algorithms. %
Throughout this section, we deal with general (possibly infinite) function classes.

To start with, we introduce the regression oracle that we assume. Let us first consider the contextual bandit setup where $\cF$ is the value function class. Given  $\cF$, we assume a \textit{weighted least squares regression oracle}, which is an {offline} optimization oracle capable of solving  problems of the form
\begin{equation}\label{eq:oracle}\tag{RO}
\oracle(\cH) = \argmin_{f\in\cF}\sum_{(w,x,a,y)\in\cH}w\prn*{f(x,a)-y}^{2},
\end{equation}
where $\cH$  (the input to the oracle) is a set of  examples $(w,x,a,y)$, where $w\in\bbR_{+}$ specifies a weight, $x\in\cX$ specifies a context, $a\in\cA$ specifies an action, and $y$ specifies a \textit{target}. The above weighted least squares problem is very well-studied in optimization and supervised regression literature. In particular, it can be solved in closed form for many simple (e.g., linear) classes, and is amenable to gradient-based methods.

When one solves the regression problem \pref{eq:oracle} using gradient-based methods like stochastic gradient descent (SGD), the convergence rate typically depends on the Lipschitz constants of gradients, which further depend on the range of weights and targets. Motivated by this fact, we further define a range parameter $b$ which describes the range of weights and targets that our oracle accepts. Formally, for all $b>0$, we define
\begin{equation*}
    \oracle_b(\cH) = \argmin_{f\in\cF}\sum_{(w,x,a,y)\in\cH}w\prn*{f(x,a)-y}^{2},\quad\forall \cH\text{ consisting of }\quad(w,x,a,y)\in[0,b]\times\cX\times\cA\times[-b,b].
\end{equation*}
Clearly $\oracle_b$ becomes a stronger as $b$ increases. 
While access to $\oracle_b$ is generally a very mild assumption even when $b$ is large, in order to achieve better computational efficiency, in this section we will be precise about $b$ and aim to invoke $\oracle_b$ with $b$ as small as possible. %

In the block MDP setup, there are multiple value function classes $\cF_1,\cdots,\cF_H$. For this setting, we assume access to the regression oracle \pref{eq:oracle} for each of the classes $\cF_1,\ldots,\cF_H$.\footnote{For notational convenience, in this section, when we use the notation $\cF$ in the block MDP setting, it can stand for any one of $\cF_1,\dots,\cF_H$, rather than the full function class $\cF_1\times\cdots\times\cF_H$ defined in \pref{sec:rl}.}
Again, we use $\oracle_b$ to denote our oracle, where $b$ denotes that the oracle accepts $[-b,b]$-valued weights and targets. %

\subsection{Implementation Details}
\paragraph{Implementing \mainalg}
We first show how to implement \mainalg with the regression oracle. There are three computational tasks in the algorithm which require us to invoke the oracle.
\begin{enumerate}
    \item \pref{line:erm}  requires computing the empirical risk minimizer $\widehat{f}_m$.
    \item \pref{line:disagreement} and \optionone (\pref{line:factor}) require computing the \setname $\cA(x;\cF_m)$ for any given $x$.
    \item \optiontwo (\pref{line:factor}) requires computing the confidence width $w(x;\cF_m)$ for any given $x$.
\end{enumerate}

The first task is exactly a least squares problem, and we directly solve it using $\oracle_1$. The second and third tasks are more complicated, and we need to design additional subroutines to reduce them to weighted least squares regression. Lying at the heart of the reductions are two basic computational subroutines: \textsf{ConfBound} and \dcoracle, which we present in \pref{alg:range} and \pref{alg:range1} respectively. Specifically,
\begin{itemize}
\item \textsf{ConfBound} is designed to efficiently compute the \emph{upper confidence bound}
\begin{equation}\tag{UCB}\label{eq:ucb-compute}
    \sup_{f\in\cF}\crl[\Bigg]{f(x,a) \mid \sum_{(x',a',y')\in \cH} (f(x',a')-y')^2\le\inf_{f'\in\cF}\sum_{(x',a',y')\in \cH}(f'(x',a')-y')^2+\beta}
\end{equation}
and the \emph{lower confidence bound} 
\begin{equation}\tag{LCB}
    \inf_{f\in\cF}\crl[\Bigg]{f(x,a) \mid \sum_{(x',a',y')\in \cH} (f(x',a')-y')^2\le\inf_{f'\in\cF}\sum_{(x',a',y')\in \cH}(f'(x',a')-y')^2+\beta}
\end{equation}
for any given context-action pair $(x,a)\in\cX\times\cA$, based on a sample history $\cH$ and confidence radius $\beta$. It can be efficiently implemented with $\oracle_{\beta/\alpha}$ given precision $\alpha$, as pointed out in \cite{krishnamurthy2017active} and \cite{foster2018practical}.
\item \dcoracle is designed to efficiently compute the \emph{action difference lower confidence bound}
\begin{equation*}
    \inf_{f\in\cF}\crl[\Bigg]{f(x,a_1)-f(x,a_2) \mid \sum_{(x',a',y')\in H} (f(x',a')-y')^2\le\inf_{f'\in\cF}\sum_{(x',a',y')\in H}(f'(x',a')-y')^2+\beta}
\end{equation*}
for any given context $x\in\cX$ and actions $a_1,a_2\in\cA$, based on a sample history $\cH$ and confidence radius $\beta$. It can be used to identify whether an action $a_1$ is guaranteed to dominate another action $a_2$ on a context $x$, and can be efficiently implemented with $\oracle_{\beta/\alpha}$ given precision $\alpha$.
\end{itemize}
Building on \cboracle and \dcoracle, we design a subroutine \csoracle (see \pref{alg:set}) that accomplishes the the second task above, i.e., (approximately) computing $\cA(x;\cF_m)$ for any $x$. Then, building further on \csoracle and \cboracle, we design a subroutine \cworacle (see \pref{alg:width}) that accomplishes the third task above, i.e., computing $w(x;\cF_m)$ for any $x$. Therefore, by applying \cboracle, \dcoracle, \csoracle and \cworacle as subroutines, we can efficiently implement \mainalg  using $\oracle_{\beta/\alpha}$.

\begin{algorithm}
\caption{$\cboracle({\rm type},x,a,\cH,\beta,\alpha)$}\label{alg:range}
\textbf{input:}  type $\in\{{\rm High, Low}\}$, context $x$, action $a$, sample history $\cH$, radius $\beta>1$, precision $\alpha>0$.\\
\textbf{algorithm:} See Algorithm 3 of \cite{foster2018practical} for convex function classes; see Algorithm 2 and 3 of \cite{krishnamurthy2017active} for non-convex function classes.
\end{algorithm}

\begin{algorithm}
\caption{$\dcoracle(x,a_1,a_2,\cH,\beta,\alpha)$}\label{alg:range1}
\textbf{input:}  context $x$, actions $a_1,a_2$, sample history $\cH$, radius $\beta>1$, precision $\alpha>0$.\\
\textbf{algorithm:}
\begin{algorithmic}[0]
\State If $a_1=a_2$ then return 0. Otherwise, use the binary search procedure in \cite{foster2018practical} (if $\cF$ is convex) or \cite{krishnamurthy2017active} (if $\cF$ is non-convex)  to solve
\begin{equation*}
\begin{aligned}
&\underset{f\in\cF}{\mathrm{minimize}}%
&\quad &\frac{\alpha}{2}\prn[\bigg]{f(x,a_1)+\frac{1}{\alpha}}^2+\frac{\alpha}{2}\prn[\bigg]{f(x,a_2)-\frac{1}{\alpha}}^2\\
&\text{such that}& &\sum_{(x',a',y')\in \cH} (f(x',a')-y')^2\le\inf_{f'\in\cF}\sum_{(x',a',y')\in \cH}(f'(x',a')-y')^2+\beta,
\end{aligned}
\end{equation*}
and return the optimal objective value.
\end{algorithmic}
\end{algorithm}

\paragraph{Implementing \blockalg}
We now show how to implement \blockalg with the regression oracle. There are two computational tasks in \blockalg which require us to invoke the oracle.
\begin{enumerate}
\item \pref{line:erm-rl}  requires computing the empirical risk minimizer $\widehat{f}_m$.
\item \pref{line:star-ucb}  requires computing the star hull upper confidence bound $\Qbar_h^{(k)}(x,a)$ for any given $(x,a)$ pair.
\end{enumerate}
The first task is exactly a least squares problem, and we directly solve it using $\oracle_1$. The second task can be accomplished by the subroutine \cboracle (\pref{alg:range}), as long as we can solve weighted least squares regression over the star hull of each $\cF_h$
\begin{equation}\label{eq:star-ro}\tag{Star-RO}
\argmin_{f\in\starhull(\cF_h,\widehat{f})}\sum_{(w,x,a,y)\in\cH}w\prn*{f(x,a)-y}^{2}
\end{equation}
for any $\widehat{f}\in\cF_h$ (recall that the star hull upper confidence bound is just a modification of  \pref{eq:ucb-compute} where we replace $\cF$ with $\starhull(\cF,\widehat{f})$). Luckily, as \pref{lem:star_erm} shows, we can always reduce \pref{eq:star-ro} by \pref{eq:oracle}.
\begin{lemma}
  \label{lem:star_erm}
  Suppose weights in \pref{eq:star-ro} are bounded by $W$ and targets are bounded by $B$. Then for any $\alpha>0$, we can find an $\alpha$-approximate solution to \pref{eq:star-ro} using $\bigoh(BW/\alpha)$ calls to $\oracle_{\bigoh(BW/\alpha)}$ for \pref{eq:oracle}.
\end{lemma}
Therefore, by applying \cboracle with the reduction above as a subroutine, \blockalg can be efficiently implemented with $\oracle_{\beta/\alpha^2}$.

\begin{algorithm}
\caption{$\csoracle(x,\cH,\beta,\alpha)$}\label{alg:set}
\textbf{input:}  context $x$, sample history $\cH$, radius $\beta>1$, precision $\alpha>0$.\\
\textbf{algorithm:}
\begin{algorithmic}[0]
\If{$\cF$ is a product function class}
    \State Compute $$\cA_{\rm cs}=\left\{a\in\cA:\textsf{ConfBound}({\rm High},x,a,\cH,\beta,\alpha)\ge\max_{a'\in\cA}\textsf{ConfBound}({\rm Low},x,a',\cH,\beta,\alpha)\right\}.$$
\Else
    \State Compute $\widetilde{a}=\arg\max_{a\in\cA}\textsf{ConfBound}({\rm High},x,a,\cH,\beta,\alpha)$.
    \State Compute $$\cA_{\rm cs}=\left\{a\in\cA: \textsf{ConfBoundDif}(x,\widetilde{a},a,\cH,\beta,\alpha)\le0\right\}.$$
\EndIf
\Return{$\cA_{\rm cs}$}.
\end{algorithmic}
\end{algorithm}

\begin{algorithm}
\caption{$\cworacle(x,\cH,\beta,\alpha)$}\label{alg:width}
\textbf{input:}  context $x$, sample history $\cH$, radius $\beta>1$, precision $\alpha>0$.\\
\textbf{algorithm:}
\begin{algorithmic}[0]
\Statex{} Compute $$w=\1\left\{|\csoracle(x,\cH,\beta,\alpha)|>1\right\}\max_{a\in\cA}\left|\textsf{ConfBound}({\rm High},x,a,\cH,\beta,\alpha)-\textsf{ConfBound}({\rm Low},x,a,\cH,\beta,\alpha)\right|.$$
\Return{$w$}.
\end{algorithmic}
\end{algorithm}

\subsection{Computational Complexity}

Theoretical guarantees for the subroutines \cboracle, \dcoracle, \csoracle and \cworacle are deferred to \pref{app:cg}. Here we summarize the total computational complexity for our algorithms based on these reductions.

\paragraph{Computational complexity of \mainalg}
For \mainalg, we set the precision $\alpha=\bigoh(1/T)$ for \cboracle, \dcoracle, \csoracle and \cworacle so that the error does not degrade the algorithm's instance-dependent performance beyond additive constants. In total, when $\cF$ is convex, \mainalg calls $\oracle_{\bigoht(T)}$ for $\bigoht(AT)$ times over $T$ rounds, and when $\cF$ is non-convex, \mainalg calls $\oracle_{\bigoht(T)}$ for $\bigoht(AT^3)$ times over $T$ rounds. We remark that the total time spent by \mainalg outside of these regression oracle calls is $O(\K)$ in each round.

\paragraph{Computational complexity of \blockalg}
For \blockalg, when the target accuracy is $\veps$, we should set $\alpha=c\veps$ for a problem-dependent constant $c$. This requires that we solve $\oracle_{\bigoh(\beta_h/\veps)}$ for the star hull, which reduces to $\oracle_{\bigoh(\beta_h/\veps^{2})}$ for the original class $\cF_h$. Altogether, we conclude that $\mathrm{poly}(1/\veps)$ oracle calls suffice.

%

\iffalse
\begin{table}[htbp]
\centering
\caption{\pref{alg:main}'s Number of Oracle Calls over $T$ Rounds}
\vspace{0.5em}
\label{tb:calls}
\begin{tabular}{ |c||c|c| } 
 \hhline{|-||-|-|}
 Properties of $\cF$ & Product & Non-Product \\\hhline{=::=:=}
 Convex & $\bigoht(\K T)$ calls to $\oracle_1$ & $\bigoht(\K T)$ calls to $\oracle_T$\\\hhline{|-||-|-|}
 Non-Convex & $\bigoht(\K T^3)$ calls to $\oracle_1$ & $\bigoht(\K T^3)$ calls to $\oracle_T$\\ 
 \hhline{|-||-|-|}
\end{tabular}
\centering
\end{table}

\yxcomment{The RL algo is designed for the sample complexity setting. Is it possible to have a table to represent the computational complexity? If not, you can just write down something about the RL algo's computational complexity.}

\dfcomment{I don't think it's important to have a table. let's keep it in text.}
\fi

%
%
%
%
 
\section{Experiments}
\label{sec:experiments}
\newcommand{\regcb}{\textsf{RegCB}\xspace}
\newcommand{\regcbopt}{\textsf{RegCB-Opt}\xspace}
\newcommand{\regcbelim}{\textsf{RegCB-Elim}\xspace}
\newcommand{\linucb}{\textsf{LinUCB}\xspace}
\newcommand{\ucb}{\textsf{UCB}\xspace}
\newcommand{\thompson}{\textsf{TS}\xspace}
\newcommand{\bistro}{\textsf{BISTRO}\xspace}
\newcommand{\iltcb}{\textsf{ILTCB}\xspace}
\newcommand{\expfour}{\textsf{Exp4}\xspace}
\newcommand{\modcb}{\textsf{ModCB}\xspace}
\newcommand{\iltcblong}{\textsf{ILOVETOCONBANDITS}\xspace}
\newcommand{\egreedy}{$\veps$-\textsf{Greedy}\xspace}
\newcommand{\greedy}{\textsf{Greedy}\xspace}
\newcommand{\bagging}{\textsf{Bagging}\xspace}
\newcommand{\cover}{\textsf{Online Cover}\xspace}

To evaluate the empirical performance of \mainalg, we replicated a simplified version
of the large-scale contextual bandit evaluation setup of
\cite{bietti2018contextual}, which compares a number of contextual
bandit algorithms based on either cost-sensitive classification
oracles or regression oracles on over 500 multiclass, multi-label, and
cost-sensitive classification datasets. We
found that \mainalg typically enjoys superior performance, especially
on challenging datasets with many actions.

\paragraph{Datasets}
Following \cite{bietti2018contextual}, we use a collection of 516
multiclass classification datasets from the \url{openml.org}
platform.\footnote{We omit datasets $\{8, 189, 197, 209, 223, 227,
  287, 294, 298\}$ from the collection used in
  \cite{bietti2018contextual}, as these are regression
  datasets that were rounded to integer targets. This brings the count
  from 525 to 516.} This collection includes many standard datasets,
including the UCI datasets used in \cite{foster2018practical}; see \cite{bietti2018contextual} for
details. Beyond this collection, \cite{bietti2018contextual} also used 5 multilabel datasets and 3
cost-sensitive datasets; we omit these for simplicity.

For each dataset, we simulate bandit feedback by withholding
the true label. We work with losses rather than rewards, and provide a
loss of $0$ if the learner predicts the
correct label, and $1$ otherwise. We randomly shuffle the examples in
each dataset, but use the same fixed shuffle across all algorithms.

\paragraph{Baseline algorithms and oracle}
We compare to a collection of baseline algorithms implemented in the
Vowpal Wabbit (VW) online learning library.\footnote{\url{https://vowpalwabbit.org/}}
For the oracle, we use the default online learner in VW, which fits a
linear model (that is,
$\cF=\crl*{(x,a)\mapsto\tri*{w,\phi(x,a)}\mid{}w\in\bbR^{d}}$) with regression using an online gradient descent-type
algorithm which incorporates adaptive \citep{duchi2011adaptive},
normalized \citep{ross2013normalized}, and importance-weight aware
\citep{karampatziakis2011online} updates, and has a single tunable
step size which is treated as a hyperparameter. This
algorithm is an \emph{online regression oracle} in the sense of
\cite{foster2020beyond}.

All of our baseline algorithms are based on classification or regression
oracles. We use their implementations in VW, which incorporate
modifications to allow them to run in an online fashion with the  
oracle above. For the classification-based algorithms, VW uses an
additional reduction layer which reduces
classification to regression. Following \cite{bietti2018contextual},
we use the following algorithms.
\begin{itemize}
\item The standard \egreedy exploration strategy
  \citep{langford2008epoch}, as well as a purely greedy variant,
  \greedy.
\item \bagging, also known as bootstrap Thompson sampling
  \citep{agarwal2014taming,eckles2014thompson,osband2016deep}, which
  attempts to approximate the Thompson sampling algorithm.
\item \cover, a heuristic version of the \iltcb strategy of
  \cite{agarwal2014taming}. \iltcb itself provides an optimal and
  efficient reduction from stochastic contextual bandits to cost-sensitive classification, and represents the state of the art from
  that line of research.
\item \regcb \citep{russo2013eluder,foster2018practical}, an
  approximate version of the general function class UCB algorithm
  based on regression oracles. This algorithm was found to have the
  best overall performance in \cite{bietti2018contextual}, though it
  does not achieve the minimax rate for contextual bandits. We only
  evaluate the optimistic variant (\regcbopt), not the
  elimination-based variant (\regcbelim), as the former typically
  performs much better.
\end{itemize}
We refer to \cite{bietti2018contextual} for more details on the
algorithm configurations and hyperparameters.

Beyond these algorithms, we also evaluate against \squarecb
\citep{foster2020beyond}, which is the first optimal online regression
oracle-based contextual bandit algorithm.
\squarecb applies the inverse gap weighting strategy in \pref{eq:igw},
but uses all actions rather than adaptively narrowing to a smaller
candidate set as in \mainalg. Our implementation applies \igw{} at each
step $t$ with a learning rate $\gamma_t$. We set
$\gamma_t=\gamma_0t^{\rho}$, where $\gamma_0\in\crl{10, 50, 100, 400,
  700, 10^{3}}$ and $\rho\in\crl*{.25, .5}$ are hyperparameters.

\paragraph{\mainalg implementation}
We implemented a variant of \mainalg in VW, with a few practical simplifications.%
\footnote{The precise version of VW used to run the experiments may be found at   \url{https://github.com/canondetortugas/vowpal_wabbit/tree/}.}
First, rather than using an offline ERM oracle (as in \pref{eq:oracle}), we modify the algorithm to work
with the online regression oracle provided in VW by following the strategy of
\cite{foster2020beyond}. The protocol for the oracle is as follows: At each
round, we provide the context $x_t$, the oracle provides a predicted reward
$\hat{y}_t(a)$ for each action, then we select an action $a_t$ and
update the oracle with $(x_t, a_t, r_t(a_t))$. Instead than applying $\igw$ to
the empirical risk minimizer as in \pref{alg:main}, we simply apply it to $\hat{y}_t$. This
strategy is natural because the protocol above is exactly what is
implemented by the base learner in VW, and thus allows us to take advantage of VW's fast
online regression implementation.

For our second simplification, rather than computing $\cA_t(x;\cF_t)$
using the reductions from \pref{sec:oracle} (which also require an
offline oracle), we use a
sensitivity-based heuristic that takes advantage of the online
oracle. This is described in Section 7.1 of
\cite{krishnamurthy2017active}, and is also used by the VW implementation
of \regcb.

Finally, rather than using a data-dependent learning rate, we set
$\gamma_t=\gamma_0t^{\rho}$, where $\gamma_0\in\crl{10, 50, 100, 400,
  700, 10^{3}}$ and $\rho\in\crl*{.25, .5}$ are hyperparameters (the
same as for \squarecb). We
set the confidence radius as $\beta^2_t=c_0\log(Kt)$, where
$c_0\in\crl{10^{-1}, 10^{-2}, 10^{-3}}$ is another hyperparameter.

\begin{table}[ht]
  \centering
    \begin{tabular}{ | l | c | c | c | c | c | c | c | c | }
    \hline
    $\downarrow$ vs $\rightarrow$ & G & R & C & B & $\epsilon$G  & S & A \\ \hline
    Greedy & - & -48 & -51 & -19 & -6  & -55 & -64 \\ \hline
    RegCB & 48 & - & 6 & 31 & 40  & 5 & -21 \\ \hline
    Cover & 51 & -6 & - & 25 & 33 &  -8 & -27 \\ \hline
    Bagging & 19 & -31 & -25 & - & 9 &  -33 & -47 \\ \hline
    $\epsilon$-Greedy & 6 & -40 & -33 & -9 & -  & -45 & -58 \\ \hline
    SquareCB & 55 & -5 & 8 & 33 & 45 &  - & -23 \\ \hline
    \textbf{AdaCB} & \textbf{64} & \textbf{21} & \textbf{27} & \textbf{47} & \textbf{58} &  \textbf{23} & - \\ \hline
    \end{tabular}
      \caption{Head-to-head win/loss rates for datasets with $\K\geq{}3$ actions. Each (row, column) entry indicates the statistically
    significant win-loss
    difference between the row algorithm and the column algorithm.}
  \label{tab:exp1}
\end{table}

\begin{table}[ht]
  \centering
  \begin{tabular}{ | l | c | c | c | c | c | c | c | c | }
    \hline
    $\downarrow$ vs $\rightarrow$ & G & R & C & B & $\epsilon$G &  S & A \\ \hline
    Greedy & - & -122 & -127 & -50 & -4 &  -86 & -114 \\ \hline
    \textbf{RegCB} & \textbf{122} & - & \textbf{1} & \textbf{81} & \textbf{122} & \textbf{71} & \textbf{20} \\ \hline
    Cover & 127 & -1 & - & 63 & 122 &  52 & 3 \\ \hline
    Bagging & 50 & -81 & -63 & - & 63 &  -19 & -63 \\ \hline
    $\epsilon$G & 4 & -122 & -122 & -63 & - &  -69 & -111 \\ \hline
    SquareCB & 86 & -71 & -52 & 19 & 69 &  - & -44 \\ \hline
    AdaCB & 114 & -20 & -3 & 63 & 111 &  44 & - \\ \hline
  \end{tabular}
  \caption{Head-to-head win/loss rates across all datasets.}
  \label{tab:exp0}
\end{table}

\paragraph{Evaluation}
We evaluate the performance of each algorithm on a given dataset using
the \emph{progressive validation} (PV) loss
\citep{blum1999beating}. For each algorithm above, we choose the best
hyperparameter configuration for a given dataset based on the PV loss.

To compare each pair of algorithms on a given dataset, we use the notion
of a \emph{statistically significant win or loss} defined in
\cite{bietti2018contextual}, which is based on an approximate
Z-test. For each
algorithm pair, we count the total number of significant wins /
losses, using the best configuration for each dataset.

\paragraph{Results: Hard datasets}
To evaluate performance on hard exploration problems, we restricted
only to datasets with $\K\geq{}3$ actions. Head-to-head results are
displayed in \pref{tab:exp1}. For this subset of datasets, we find
that \mainalg has the best overall performance, and has a (large) positive win-loss
difference against all of the baselines. This suggests that \mainalg may
be a promising approach for solving challenging exploration
problems. 

\paragraph{Results: All datasets}
Head-to-head results across all datasets are displayed in \pref{tab:exp0}. We find that
for the full collection of datasets, \regcb
\citep{foster2018practical} has the best overall performance, with a
positive win-loss difference against every algorithm. \mainalg has
strong performance overall, but narrowly loses to \regcb and \cover,
with a positive win-loss difference against all other algorithms. The
strong performance of \regcb mirrors the findings of
\cite{bietti2018contextual}. They observed that many of the datasets in
the OpenML collection are ``easy'' for exploration in the sense that
1) they have few actions (in
fact, around 400 datasets have only 2 actions), and 2) even the simple
greedy strategy performs well; \regcb seems to be good at
exploiting this. It would be interesting to understand whether
we can make \mainalg eliminate actions even more aggressively for the
easy problems on which \regcb excels, and whether we can develop
theory to support this.

\section{Discussion}
\label{sec:discussion}

We have developed efficient, instance-dependent algorithms for
contextual bandits and reinforcement learning with function
approximation. We showed that disagreement coefficients and
related combinatorial parameters play a fundamental role in
determining the optimal instance-dependent rates, and that algorithms that adapt to these parameters can be simple and practically effective. Our results suggest many fruitful directions for future research.

For contextual bandits, there are a number of very interesting
and practically relevant questions:
\begin{itemize}
\item For adversarial contexts, can we develop instance-dependent
  algorithms that are efficient in terms of \emph{online} regression
  oracles, as in \cite{foster2020beyond}?
\item Can we reduce the total number of regression oracle calls used by
  \mainalg to $\bigoht(\polylog(T))$, thereby matching the runtime of
  \cite{simchi2020bypassing} for non-instance-dependent regret?
\item Can we extend our algorithms and complexity measures to
  optimally handle
  infinite actions? 
\end{itemize}
Beyond these questions, we hope to see the various gaps between our upper and
lower bounds closed.

For reinforcement learning, we are excited to see whether our analysis
techniques can be applied more broadly, and to develop more refined
lower bounds that better reflect the role of disagreement in
determining the difficulty of exploration. On the technical side,
there are many possible improvements to \pref{thm:block_mdp}. For example, is
it possible to develop best-of-both-worlds guarantees similar to those
attained by \mainalg, or even to efficiently attain the minimax rate
in the absence of bounded gap and disagreement coefficient? Can we
improve the dependence on $S$, $A$, and so forth to match the optimal
rates for the tabular setting?

\subsection*{Acknowledgements}
We thank Alekh Agarwal, Haipeng Luo, Akshay Krishnamurthy, Max
Simchowitz, and Yunbei Xu for helpful discussions. We thank Alberto Bietti for help
with replicating the setup from \cite{bietti2018contextual}. DF
acknowledges the support of NSF Tripods grant \#1740751. AR
acknowledges the support of ONR awards \#N00014-20-1-2336 and \#N00014-20-1-2394. DSL and YX
acknowledge the support of the MIT-IBM Watson AI Lab.

\bibliography{refs}
\vfill
\newpage
\appendix

\begin{comment} {
%
\hypersetup{linkcolor=black}
\renewcommand{\contentsname}{Contents of Appendix}
\tableofcontents
%
%
\addtocontents{toc}{\protect\setcounter{tocdepth}{3}} }
\end{comment}

\section*{Organization of Appendix}
This appendix is organized as follows. \pref{app:cg} contains details
for the computational results presented in
\pref{sec:oracle}. \pref{part:cb} contains proofs for our contextual
bandit results, and  \pref{part:rl} contain proofs for our
reinforcement learning results.
\section{Computational Guarantees}\label{app:cg}

In what follows, we establish computational guarantees for the subroutines \cboracle, \dcoracle, \csoracle and \cworacle. %

\paragraph{Computation of confidence bounds}

We first provide guarantees for \textsf{ConfBound} and \dcoracle, in \pref{lm:conf-bound} and \pref{lm:conf-bound-dis} respectively. \pref{lm:conf-bound} is adapted from known results in the literature. The proof of \pref{lm:conf-bound-dis} can be found in \pref{app:oracle}.

\begin{lemma}[Theorem 1 in \citealt{foster2018practical} \& Theorem 1 in \citealt{krishnamurthy2017active}]\label{lm:conf-bound}
Consider the \mainalg setting. 
Let $\cH_m=\{(x_t,a_t,r_t(a_t))\}_{t=1}^{t_{m-1}}$. If the function class $\cF$ is convex and closed under pointwise convergence, then for any $x\in\cX$, the computation procedures
$$\normalfont
\textsf{ConfBound}({\rm High}, x, a, \cH_m, \beta_m,\alpha)
\mathand
\normalfont
\textsf{ConfBound}({\rm Low}, x, a, \cH_m, \beta_m,\alpha)
$$
terminate after $\bigoh(\log(1/\alpha))$ calls to $\oracle_1$, and the returned values satisfy
$$\normalfont
\left|\sup_{f\in\cF_m}f(x,a)-\textsf{ConfBound}({\rm High}, x, a, \cH_m, \beta_m,\alpha)\right|\le \alpha,
$$
$$\normalfont
\left|\inf_{f\in\cF_m}f(x,a)-\textsf{ConfBound}({\rm Low}, x, a, \cH_m, \beta_m,\alpha)\right|\le \alpha.
$$
If the function class $\cF$ is non-convex, then the required number of oracle calls is $\bigoh(1/\alpha^2\log(1/\alpha))$.
\end{lemma}

\begin{lemma}\label{lm:conf-bound-dis}
Consider the \mainalg setting. 
Let $\cH_m=\{(x_t,a_t,r_t(a_t))\}_{t=1}^{t_{m-1}}$. If the function class $\cF$ is convex and closed under pointwise convergence, then for any $x\in\cX$ and $a_1,a_2\in\cA$, the computation procedure
$$\normalfont
\dcoracle(x, a_1,a_2, \cH_m, \beta_m,\alpha)
$$
terminates after $\bigoh(\log (1/\alpha))$ calls to $\oracle_{1/\alpha}$, and the returned values satisfy
$$\normalfont
\left|\inf_{f\in\cF_m}\left(f(x,a_1)-f(x,a_2)\right)-\dcoracle(x, a_1,a_2, \cH_m, \beta_m,\alpha)\right|\le 2\alpha.
$$
If the function class $\cF$ is non-convex, then the required number of oracle calls is $\bigoh(1/\alpha^2\log(1/\alpha))$.
\end{lemma}

\pref{lm:conf-bound} and \pref{lm:conf-bound-dis} show that \cboracle and \dcoracle compute the desired
confidence bounds up to a precision of $\alpha$ in $\bigoh(\log(1/\alpha))$ or $\bigoh(1/\alpha^2\log(1/\alpha))$ iterations. Note that both \pref{lm:conf-bound} and \pref{lm:conf-bound-dis}  are stated in the \mainalg setting. For the sake of brevity, we omit the guarantee of \cboracle for the \blockalg setting here, as it is essentially the same as \pref{lm:conf-bound}, with only notations changing.

\paragraph{Computation of candidate action sets and confidence width}

We now provide guarantees for \dcoracle and \cworacle. All the results in this part are stated in the \mainalg setting, as \csoracle and \cworacle are only required for \mainalg.

We discuss two cases: when $\cF$ is a product function class, it is possible for \csoracle to precisely compute $\cA(x;\cF_m)$; when $\cF$ is not a product function class, precisely computing $\cA(x;\cF_m)$ becomes difficult, yet it is possible for \csoracle to precisely determine whether $|\cA(x;\cF_m)|>1$, which is already sufficient for \cworacle to precisely compute $w(x;\cF_m)$ and for \mainalg to achieve all the statistical guarantees stated in \pref{sec:cb}.

\paragraph{Case 1: Product function class}
When $\cF$ is a product function class, i.e. $\cF=\cG^{\cA}$ for some function class $\cG$, it is possible for \mainalg to maintain each version space $\cF_m$ as a product function class, which makes it especially simple to compute $\cA(x;\cF_m)$. 
To ensure the product structure of $\cF_m$, we need to make  slight modifications to the definition of $\cF_m$ in \pref{line:version_space} of \mainalg.

\begin{definition}[Alternative definition of $\cF_m$ for \mainalg]\label{def:version} If $\cF$ is a product function class, i.e., $\cF=\cG^\cA$ for some $\cG$, then at \pref{line:version_space} of \pref{alg:main}, we define $\cF_m=\prod_{a\in\cA}\cG_{m,a}$, where
\begin{equation*}
  \cG_{m,a}=\left\{g\in\cG \;\Big{|}\; \sum_{t=1}^{t_{m-1}}(g(x_t)-r_t(a_t))^2\1\{a_t=a\}\le\inf_{g'\in\cG}\sum_{t=1}^{t_{m-1}}(g'(x_t)-r_t(a_t))^2\1\{a_t=a\}+\beta_m\right\}.
\end{equation*}
\end{definition}

It is not difficult to see that the above modifications do not  affect \mainalg's statistical guarantees in \pref{sec:cb},\footnote{In order to achieve the same regret bound, we need to adjust the configuration of $\beta_m$ from $16(M-m+1)\log(2|\cF|T^2/\delta)$ to $16(M-m+1)\log(2|\cG|AT^2/\delta)$.} or \cboracle's computational guarantee in \pref{lm:conf-bound}. The following lemma demonstrates the value of the product structure of $\cF_m$. %

\begin{lemma}\label{lm:dis-set-product}If $\cF'\subset\cF$ is a product function class, then for any $x\in\cX$,
$$\cA(x;\cF')=\left\{a\in\cA:\sup_{f\in\cF'}f(x,a)\ge\max_{a'\in\cA}\inf_{f\in\cF'}f(x,a')\right\},$$
$$
w(x;\cF')=\1\left\{|\cA(x;\cF')|>1\right\}\max_{a\in\cA}\left|\sup_{f\in\cF'}f(x,a)-\inf_{f\in\cF'}f(x,a)\right|.
$$
\end{lemma}

\pref{lm:dis-set-product} implies that if $\cF$ is a product function class (and thus $\cF_m$ is a product function class under \pref{def:version}), then both $\cA(x;\cF_m)$ and $w(x;\cF_m)$ can be explicitly expressed in terms of the upper and lower confidence bounds $\sup_{f\in\cF_m}f(x,a)$ and $\inf_{f\in\cF_m}f(x,a)$. Since \cboracle enables us to compute $\sup_{f\in\cF_m}f(x,a)$ and $\inf_{f\in\cF_m}f(x,a)$ for all $a\in\cA$ with high accuracy (see \pref{lm:conf-bound}), we can precisely compute $\cA(x;\cF_m)$ and $w(x;\cF_m)$ by calling $\csoracle(x,\cH_m,\beta_m,\alpha)$ and $\cworacle(x,\cH_m,\beta_m,\alpha)$ with sufficiently small $\alpha$, which can be efficiently implemented with $\oracle_{\beta/\alpha}$ (as \cboracle can be  efficiently implemented with $\oracle_{\beta/\alpha}$). %

%

%

\iffalse
\textbf{Inexact Computation.} Exactly computing

$$\widetilde{a}=\arg\max_{a\in\cA}\textsf{ConfBound}({\rm High},x,a,H,\beta,\alpha)$$
$$A_{\rm dis}=\{\tilde a\}\cup\left\{a\in\cA:\textsf{ConfBound}({\rm High},x,a,H,\beta,\alpha)\ge\max_{a'\in\cA}\textsf{ConfBound}({\rm Low},x,a',H,\beta,\alpha)+2\alpha\right\}.$$
\fi

\paragraph{Case 2: General function class}

If $\cF$ is not a product function class, then it is impossible to maintain $\cF_m$ as a product function class. In this case, our implementation of \csoracle uses both \textsf{ConfBound} and \dcoracle to approximately compute $\cA(x;\cF_m)$.%
We use the following observation.
\begin{lemma}\label{lm:dis-set-general}
For any $\cF'\subset\cF$, for any $x\in\cX$, define $\widetilde{a}=\arg\max_{a\in\cA}\sup_{f\in\cF'}f(x,a)$, and define
\[
\widehat{\cA}(x;\cF')\ldef \left\{a\in\cA: \inf_{f\in\cF'} (f(x,\widetilde{a})-f(x,a))\le0\right\}.
\]
It holds that
\begin{enumerate}
    \item $
\cA(x;\cF')\subset\widehat{\cA}(x;\cF').
$ 
\item If $|\cA(x;\cF')|=1$, then, $
\cA(\cF',x)=\widehat{\cA}(x;\cF')=\left\{\widetilde{a}\right\}.
$
\item $\1\{\abs{\widehat{\cA}(x;\cF')}>1\}=\1\{\abs{{\cA}(x;\cF')}>1\}$.
\item $w(x,\cF')=\1\crl[\big]{\abs{\widehat{\cA}(x;\cF')}>1}\max_{a\in\cA}\left|\sup_{f\in\cF'}f(x,a)-\inf_{f\in\cF'}f(x,a)\right|$.
\end{enumerate}
\end{lemma}

Parts 1 and 2 of \pref{lm:dis-set-general} imply that, for a general function class $\cF$, we are able to compute $\wh{\cA}(x;\cF_m)$  by calling $\dcoracle(x,\cH_m,\beta_m,\alpha)$ with sufficiently small $\alpha$, which serves as an approximation to the true \setname $\cA(x;\cF_m)$, with the following two properties:
\begin{itemize}
    \item The computed set $\wh{\cA}(x;\cF_m)$ always contain the true set ${\cA}(x;\cF_m)$.
    \item The computed set $\wh{\cA}(x;\cF_m)$ coincides with the true set $\cA(x;\cF_m)$ when $|\cA(x;\cF_m)|=1$.
\end{itemize}
Parts 3 and 4 state two consequences of the above properties:
\begin{itemize}
    \item Though we may not be able to directly compute the \setname $\cA(x;\cF_m)$, we can always precisely compute the disagreement indicator $\1\{\abs{{\cA}(x;\cF')}>1\}$ by computing $\1\{\abs{\wh{\cA}(x;\cF')}>1\}$.
    \item As a result, we can always precisely compute the confidence width $w(x;\cF_m)$ (which depends on $\1\{\abs{{\cA}(x;\cF')}>1\}$), irrespective of whether $\cA(x;\cF_m)$ is precisely computed or note.
\end{itemize}
Note that it is sufficient for \mainalg to use $\wh{\cA}(x;\cF_m)$ (rather than the true \setname ${\cA}(x;\cF_m)$) and $w(x;\cF_m)$ to obtain the near-optimal instance-dependent regret stated in \pref{thm:disagreement_ub} and \pref{thm:value_disagreement_ub}. In particular, in the proofs of \pref{thm:disagreement_ub} and \pref{thm:value_disagreement_ub}, if we replace ${\cA}(x;\cF_m)$ by its approximation $\wh{\cA}(x;\cF_m)$, then all the arguments hold as long as the following three conditions hold:
\begin{enumerate}
    \item $\pi^*(x)\in\wh{\cA}(x;\cF_m)$ with high probability;
    \item $\1\{\abs{\wh{\cA}(x;\cF')}>1\}=\1\{\abs{{\cA}(x;\cF')}>1\}$;
    \item $\wh{\cA}(x;\cF_m)={\cA}(x;\cF_m)$ when $|\cA(x;\cF_m)|=1$.
\end{enumerate}
The first condition holds because $\cA(x;\cF_m)\subset\widehat{\cA}(x;\cF_m)$ and \mainalg guarantees that $\pi^*(x)\in\cA(x;\cF_m)$ with high probability. The second and third conditions are exactly $\wh{\cA}(x;\cF_m)$'s properties. As a result, \csoracle and \cworacle work.

\subsection{Deferred Proofs}\label{app:oracle}

\begin{proof}[\pfref{lem:star_erm}]
  Recall that any element of $\starhull(\cF,\fhat)$ can be written as $t\cdot{}f + (1-t)\fhat$, where $f\in\cF$ and $t\in\brk*{0,1}$. Hence, for any $\veps>0$, if we set $N=\ceil{1/\veps}$ and take $\cI=\crl*{0, \veps, \ldots, (N-1)\veps}$ as an $\veps$-net for $\brk*{0,1}$ and define $\wt{\cF}=\crl{tf' + (1-t)\fhat\mid{}f\in\cF, t\in\cI}$, we have that $\wt{\cF}$ $\veps$-approximates $\starhull(\cF,\fhat)$ pointwise. We can implement ERM over $\wt{\cF}$ using $1/\veps$ oracle calls. In particular, for each $t\in\cI$, we simply solve
  \begin{align}
    &\argmin_{f\in\cF}\sum_{(w,x,a,y)\in\cH}w\prn*{tf(x,a) + (1-t)\fhat(x,a) - y}^{2}\label{eq:star_explicit}\\
    &=\argmin_{f\in\cF}\sum_{(w,x,a,y)\in\cH}wt^{1/2}\prn*{f(x,a) - (y - (1-t)\fhat(x,a))/t}^{2},\notag
  \end{align}
  which can be seen as an instance $\oracle_{b/\veps}$. Letting $f_t$ denote the solution above, we return $f=tf_t + (1-t)\fhat$ for the $(t,f_t)$ pair that minimizes \pref{eq:star_explicit}. Since all the arguments to the square loss in \pref{eq:star_explicit} are bounded by $B$ and $w\leq{}W$, it is clear that this leads to an $\bigoh(BW\veps)$-approximate minimizer.

\end{proof}

\begin{proof}[\pfref{lm:conf-bound}]
Since
$$
f(x,a_1)-f(x,a_2)=\frac{\alpha}{2}\left(f(x,1)+\frac{1}{\alpha}\right)^2+\frac{\alpha}{2}\left(f(x,2)-\frac{1}{\alpha}\right)^2-\frac{1}{\alpha}-\frac{\alpha}{2}\left({f(x,1)^2+f(x,2)^2}\right),
$$
we have
\begin{align*}
    \inf_{f\in\cF_m} (f(x,a_1)-f(x,a_2))\le&
{\inf_{f\in\cF_m}\left\{ \frac{\alpha}{2}\left(f(x,a_1)+\frac{1}{\alpha}\right)^2+\frac{\alpha}{2}\left(f(x,a_2)-\frac{1}{\alpha}\right)^2\right\}}-\frac{1}{\alpha},\\
    \inf_{f\in\cF_m} (f(x,a_1)-f(x,a_2))\ge&  {\inf_{f\in\cF_m}\left\{ \frac{\alpha}{2}\left(f(x,a_1)+\frac{1}{\alpha}\right)^2+\frac{\alpha}{2}\left(f(x,a_2)-\frac{1}{\alpha}\right)^2\right\}}-\frac{1}{\alpha}-\alpha.
\end{align*}
Therefore, to compute $\inf_{f\in\cF_m} (f(x,a_1)-f(x,a_2))$ with to an error up to $2\alpha$, we only need to compute 
$$
{\inf_{f\in\cF_m}\left\{ \frac{\alpha}{2}\left(f(x,a_1)+\frac{1}{\alpha}\right)^2+\frac{\alpha}{2}\left(f(x,a_2)-\frac{1}{\alpha}\right)^2\right\}}
$$
with an error up to $\alpha$. This can be efficiently accomplished through a binary search procedure  similar to the procedures  in \cite{foster2018practical} and \cite{krishnamurthy2017active}. This procedure requires access to $\oracle_{\beta/\alpha}$.
\end{proof}

\part{Proofs for Contextual Bandit Results}
\label{part:cb}
\section{Technical Tools}
\label{app:technical}
\subsection{Concentration}

\begin{lemma}[Freedman's inequality (e.g., \citet{agarwal2014taming})]
  \label{lem:freedman}
  Let $(Z_t)_{t\leq{T}}$ be a real-valued martingale difference
  sequence adapted to a filtration $\gfilt_{t}$, and let
  $\En_{t}\brk*{\cdot}\ldef\En\brk*{\cdot\mid\gfilt_{t}}$. If
  $\abs*{Z_t}\leq{}R$ almost surely, then for any $\eta\in(0,1/R)$
    it holds that with probability at least $1-\delta$,
    \[
      \sum_{t=1}^{T}Z_t \leq{} \eta\sum_{t=1}^{T}\En_{t-1}\brk*{Z_t^{2}} + \frac{\log(\delta^{-1})}{\eta}.
    \]
  \end{lemma}
      \begin{lemma}
      \label{lem:regret_freedman}
            Let $(X_t)_{t\leq{T}}$ be a real-valued sequence of random
      variables adapted to a filtration $\gfilt_{t}$. If
  $\abs*{X_t}\leq{}R$ almost surely, then with probability at least
  $1-\delta$,
  \begin{align*}
    &\sum_{t=1}^{T}X_t \leq{}
                        \frac{3}{2}\sum_{t=1}^{T}\En_{t-1}\brk*{X_t} +
                        4R\log(2\delta^{-1}),
    \intertext{and}
      &\sum_{t=1}^{T}\En_{t-1}\brk*{X_t} \leq{} 2\sum_{t=1}^{T}X_t + 8R\log(2\delta^{-1}).
  \end{align*}
    \end{lemma}
    \begin{proof}
      Immediate consequence of \pref{lem:freedman}.
    \end{proof}

\begin{lemma}[Bernstein's inequality for {$[0,1]$-valued} random variables]\label{lem:bernstein}
  Let $X_1,\dots,X_n$ be i.i.d. $[0,1]$-valued random variables with mean $\mu$ and variance $\sigma^2$. For any $\delta>0$, with probability at least $1-2\delta$,
  \[
  \left|\frac{1}{n}\sum_{i=1}^n X_i-\mu\right|\le\frac{1}{3n}\log(2\delta^{-1})+\sqrt{\frac{2\sigma^2\log(2\delta^{-1})}{n}}\le\frac{1}{3n}\log(2\delta^{-1})+\sqrt{\frac{2\mu\log(2\delta^{-1})}{n}}.
  \]
\end{lemma}
\begin{proof}
The first inequality follows from $|X|\le1$ and Bernstein's inequality for bounded random variables. The second inequality follows from $X\in[0,1]$ and
$
\sigma^2=\E[X^2]-(\E[X])^2\le \E[X]-(\E[X])^2\le\mu
$.
\end{proof}

    \subsection{Information Theory}
   For a pair of distributions $P\ll{}Q$ with densities $p$ and $q$, we
define
\[
\kl{P}{Q} = \int{}p(x)\log(p(x)/q(x))dx.
\]

\begin{lemma}[High-probability Pinsker (Lemma 2.6,
  \cite{tsybakov2008introduction})]
  \label{lem:pinsker}
  Let $P$ and $Q$ be probability measures over a measurable space
  $(\Omega,\cF)$. Then for any measurable subset $A\in\cF$,
  \begin{equation}
    \label{eq:pinsker}
    P(A) + Q(A^{c}) \geq{} \frac{1}{2}\exp\prn*{-\kl{P}{Q}}.
  \end{equation}
\end{lemma}

For $p\in\brk*{0,1}$, we let $\Ber(p)$ denote the Bernoulli
distribution over $\crl{0,1}$ with bias $\Pr(X=1)=p$.

\begin{lemma}
  \label{lem:kl_bernoulli}
  For any $p,q\in(0,1)$, $\kl{\Ber(p)}{\Ber(q)}\leq{}\frac{1}{\min\crl{p,q,1-p,1-q}}\abs*{p-q}^{2}$.
\end{lemma}
\begin{comment}
  \begin{proof}[\pfref{lem:kl_bernoulli}]
    Let $u=(p,1-p)$ and $v=(q,1-q)$. Then
    $\kl{\Ber(p)}{\Ber(q)}=D_{\cR}(u\dmid{}v)$, where
    $\cR(v)=\sum_{i}v_i\log{}v_i$ is the negative entropy, and
    $D_{\cR}(u\dmid{}v)=\cR(u)-\cR(v)-\tri*{\grad{}\cR(v),u-v}$ is
    denotes the Bregman divergence for $\cR$. Now, by Taylor's
    theorem, we have
    $\breg{u}{v}=\frac{1}{2}\grad^{2}\cR(\bar{v})[u-v,u-v]$, where
    $\bar{v}\in\conv\crl*{u,v}$. Now, since
    $\grad^{2}\cR(v)=\diag(\crl{1/v_i})$, we have
    $\grad^{2}\cR(\bar{v})[u-v,u-v]\leq{}\frac{1}{\min\crl*{p,q,1-p,1-q}}\prn*{\abs*{u_1-v_1}^2+\abs*{u_2-v_2}^2}$,
    leading to the result.
  \end{proof}
\end{comment}

%
%
%
%
 
\section{Proofs for Upper Bounds}
\label{app:cb_upper}
This section is dedicated to the proof of \pref{thm:disagreement_ub}
and \pref{thm:value_disagreement_ub}. For a brief overview of the proof
ideas, we refer the reader to \pref{sec:cb_prooftech}.

For simplicity, in this section, we assume that \pref{alg:main} exactly compute 
$\cA(x;\cF_m)$ and $w(x;\cF_m)$
for all $x\in\cX$ and all $m$ (rather than approximately,
using the oracle machinery of \pref{sec:oracle}). %
While  we specify $\delta=1/T$  in the pseudocode of \pref{alg:main},
throughout this section we deal with a general value
$\delta\in(0,1]$. Likewise, we analyze a slightly more general version
of the update in \pref{line:adacb_learning_rate}, which sets the
learning rate as
    \[\gamma_m=\lambda_m\cdot c\sqrt{\frac{\K
          n_{m-1}}{\log(2|\cF|T^2/\delta)}},\]
    where $c>0$ is an additional hyperparameter. We assume throughout the section that $\{\beta_m\}_{m=1}^M\in\mathbb{R}_{+}^M$, $\{\mu_m\}_{m=1}^M\in\mathbb{R}_{+}^M$.

    \subsection{Preliminaries}
    \newcommand{\cEdp}{\cE_{\mathrm{dp}}}
    \newcommand{\cEw}{\cE_{\mathrm{w}}}

For all $t\in[T]$, we let
$\gfilt_t=\sigma((x_1,a_1,\ls_1(a_1)),\ldots,(x_t,a_t,\ls_t(a_t)))$
denote the sigma-algebra generated by the history up to round $t$
(inclusive), and let $m(t)\ldef{}\min\{m\in\N:t\le \tau_m\}$ denote
the epoch that round $t$ belongs to; recall that epoch $m$ consists of
rounds $\tau_{m-1}+1,\ldots,\tau_m$. We define
\[
\Regbar\ldef{}\sum_{t=1}^{T}\En\brk*{r_t(\pistar(x_t))-r_t(a_t)\mid\gfilt_{t-1}}
\]
to
be the sum of conditional expectations of the instantaneous regret.

\paragraph{Algorithm-related definitions}

For all epoch $m\in[M]$, for all round $t$ in epoch $m$, we define the
following quantities, all of which are $\gfilt_{\tau_{m-1}}$-measurable.
First, we define the greedy policy for epoch $m$ by
\begin{flalign*}
  \forall x\in\cX, ~~\widehat{\pi}_m(x)\ldef{}\arg\max_{a\in
    \cA(x;\cF_m)}\widehat{f}_m(x,a).
\end{flalign*}
Next, we denote the algorithm's probability distribution for epoch $m$ by
\begin{flalign*}
  \forall x\in\cX, ~~p_m(a \mid x)\ldef{}\begin{cases}
    \frac{1}{|\cA(x;\cF_m)|+\gamma_m\left(\widehat{f}_m(x,\widehat{\pi}_m(x))-\widehat{f}_m(x,a)\right)},&\text{for all }a\in \cA(x;\cF_m)/\{\widehat{\pi}_m(x)\},\\
    1-\sum_{a\in \cA(x;\cF_m)/\{\widehat{\pi}_m(x)\}}p_m(a \mid x),&\text{for }a=\widehat{\pi}_m(x),\\
    0,&\text{for }a\notin \cA(x;\cF_m).
  \end{cases}
\end{flalign*}
Finally, we define the following disagreement-related quantities:
\begin{align*}
q_m&\ldef{}\E_{x\sim\cD}[\1\{|\cA(x,\cF_m)|>1\}]=\Prob_{\cD}(|\cA(x;\cF_m)|>1).\\
\widehat{q}_m&\ldef{}\E_{x\sim\cD_m}[\1\{|\cA(x,\cF_m)|>1\}]=\Prob_{\cD_m}(|\cA(x;\cF_m)|>1).\\
\extq{m}&\ldef{}q_m+\mu_m.\\
\exthq{m}&\ldef{}\widehat{q}_m+\mu_m.\\
w_m&\ldef{}\E_{x\sim\cD}[w(x;\cF_m)].\\
\widehat{w}_m&\ldef{}\E_{x\sim\cD_m}[w(x;\cF_m)].
\end{align*}

\paragraph{Policy-related definitions}
We define the \emph{universal policy space} \citep{simchi2020bypassing}
as $\Psi\ldef{}\cA^{\cX}$, and for all $\pi\in\Psi$ we define
\begin{align*}
  \cR(\pi)\ldef{}\E_{x\sim\cD}\left[{f}^*(x,\pi(x))\right]\mathand
\pReg(\pi)\ldef{}\cR({\pi}_{f^*})-\cR(\pi).
\end{align*}
For all rounds $t$ in epoch $m$, we define the following quantities
for all $\pi\in\Psi$:
\begin{align*}
\wcR_t(\pi)&\ldef{}\E_{x\sim\cD}\left[\widehat{f}_{m(t)}(x,\pi(x))\right].\\
\wReg_t(\pi)&\ldef{}\wcR_{t}(\widehat{\pi}_{m(t)})-\wcR_t(\pi).\\
\cR^{\rm Dis}_t(\pi)&\ldef{}\E_{x\sim\cD}[\1\{|\cA(x;\cF_m)|>1\}f^*(x,\pi(x))].\\
\wcR_t^{\rm Dis}(\pi)&\ldef{}\E_{x\sim\cD}\left[\1\{|\cA(x;\cF_m)|>1\}\widehat{f}_{m}(x,\pi(x))\right].
\end{align*}

\paragraph{High-probability events} The proofs in this section involve
three high-probability events $\cE$, $\cEdp$, and $\cEw$, which are defined in
\pref{lm:high-prob-event}, \pref{lm:sandwich} and
\pref{lm:width-approx} below. We also present \pref{lm:beta-value},
which is a consequence of the event $\cE$ stated in \pref{lm:high-prob-event}.

\begin{lemma}[Lemma 9, \cite{foster2018practical}]\label{lm:high-prob-event}Let $C_\delta\ldef{}16\log\left(\frac{2|\cF|T^2}{\delta}\right)$. 
Define $$
M_t(f)\ldef{}(f(x_t,a_t)-r_t(a_t))^2-(f^*(x_t,a_t)-r_t(a_t))^2.
$$
With probability at least $1-\delta/2$, it holds that
\begin{align*}
\sum_{t=\tau}^{\tau'}\E_{x_t,a_t}\left[(f(x_t,a_t)-f^*(x_t,a_t))^2\mid\gfilt_{t-1}\right]=\sum_{t=\tau}^{\tau'}\E_{x_t,a_t}\left[M_t(f)\mid\gfilt_{t-1}\right]\le 2\sum_{t=\tau}^{\tau'}M_t(f)+C_\delta
\end{align*}
for all $f\in\cF$ and $\tau,\tau'\in[T]$. 
\end{lemma}
We let $\cE$ denote the high-probability event from
\pref{lm:high-prob-event}. We have the following consequence.
\begin{lemma}\label{lm:beta-value}
Assume that $\cE$ holds. Then
\begin{enumerate}
\item For all $m\in[M]$, for all $\beta_m\ge0$,
\begin{align}\label{eq:estimation-g}
\sum_{t=\tau_{m-2}+1}^{\tau_{m-1}}\E_{x_t,a_t}\left[(\widehat{f}_m(x_t,a_t)-f^*(x_t,a_t))^2\mid\gfilt_{t-1}\right]\le C_\delta,
\end{align}
and
\begin{align*}
\forall f\in\cF_m, ~~\sum_{t=\tau_{m-2}+1}^{t_{m-1}}\E_{x_t,a_t}\left[(f(x_t,a_t)-f^*(x_t,a_t))^2\mid\gfilt_{t-1}\right]\le 2\beta_m+C_\delta.
\end{align*}
\item If $\beta_m\ge C_\delta/2$ for all $m\in[M]$, then $f^*\in\cF_m$ for all $m\in[M]$.
\item If $\beta_m={(M-m+1)C_\delta}$ for all $m\in[M]$, then $f^*\in\cF_M\subset \cF_{M-1}\subset\cdots\subset \cF_1$.
\end{enumerate}
\end{lemma}

\begin{proof}
Part 1 follows from \pref{lm:high-prob-event} and the fact that $\widehat{f}_m$ minimizes the empirical square loss.
Parts 2 and 3 are adapted from Lemma 10 of \cite{foster2018practical}.
\end{proof}

\begin{lemma}\label{lm:sandwich}
For any $\delta\in(0,1]$, if we set $\mu_m={64\log(4M/\delta)}/{n_{m-1}}$ for all $m\in[M]$, then with probability at least $1-\delta/2$, the following event holds:
\begin{equation}\label{eq:event-dp}
\cE_{\rm dp}\ldef{}\left\{\forall m\in[M],~~ \frac{2}{3}\extq{m}\le\exthq{m}\le\frac{4}{3}\extq{m}\right\}.
\end{equation}
\end{lemma}
\begin{proof}
For $m=1$, $\extq{m}=q_1+128\log(4M/\delta)$, $\exthq{m}=1+128\log(4M/\delta)$, so $\frac{2}{3}\extq{1}\le\exthq{1}\le\frac{4}{3}\extq{1}$ trivially holds.

Fix any $m\in[M]\setminus\{1\}$. Since $x_{t_{m-1}+1},\dots,x_{\tau_{m-1}}$ are independent of $\cF_m$, we have
$$
\E_{x_t\sim\cD}[\1\{|\cA(x_t;\cF_m)|>1\}]=\Prob_{x\sim\cD}(|\cA(x;\cF_m)|>1)=q_m
$$
for $t=t_{m-1}+1,\dots,\tau_{m-1}$. Thus given $\cF_m$, $\1\{|\cA(x_{t_{m-1}+1};\cF_m)|>1\},\dots,\1\{|\cA(x_{\tau_{m-1}};\cF_m)|>1\}$ are $n_{m-1}/2$ i.i.d. $[0,1]$-valued random variable with mean $q_m$. Since $\widehat{q}_m=\frac{1}{n_{m-1}/2}\sum_{t=t_{m-1}+1}^{\tau_{m-1}}\1\{|\cA(x_{t};\cF_m)|>1\}$, by \pref{lem:bernstein}, we know that
\begin{equation}\label{eq:bern}
\left|\widehat{q}_m-q_m\right|\le\frac{2\log (4M/\delta)}{3n_{m-1}}+2\sqrt{\frac{{q_m\log (4M/\delta)}}{n_{m-1}}}
\end{equation}
with probability at least $1-\delta/(2M)$. Note that to apply
\pref{lem:bernstein}, we have used that our sample splitting schedule
guarantees that the contexts $x_{t_{m-1}+1},\ldots,\tau_{m-1}$ used to
form $\qhat_m$ are independent of $\cF_m$. Continuing, we have
\begin{align*}
\left|\exthq{m}-\extq{m}\right|
&=\left|\widehat{q}_m-q_m\right|\\
&\le\frac{2\log (4M/\delta)}{3n_{m-1}}+2\sqrt{\frac{{q_m\log (4M/\delta)}}{n_{m-1}}}\\
&\le\frac{2\log (4M/\delta)}{3n_{m-1}}+2\sqrt{\frac{{\extq{m}\log (4M/\delta)}}{n_{m-1}}}\\
&\le\frac{2}{3}\frac{\extq{m}}{64}+2\frac{\extq{m}}{8}\\
&\le\frac{1}{3}\extq{m},
\end{align*}
where the first inequality follows from \pref{eq:bern}, the second inequality follows from $\extq{m}=q_m+\mu_m\ge q_m$, and the third inequality follows from $\log(4M/\delta)/n_{m-1}\le\mu_m/64\le\extq{m}/64$.

By a union bound over $m\in[M]$, this implies that with probability at least $1-\delta/2$,
\begin{align*}
\forall m\in[M], ~~\frac{2}{3}\extq{m}\le\exthq{m}\le\frac{4}{3}\extq{m}.
\end{align*}
\end{proof}

\begin{lemma}\label{lm:width-approx}
  Let $\delta\in(0,1]$ be fixed. 
For each $m\in\brk*{M}$, let $\cE_{\mathrm{w}}\ind{m}$ denote the
event that
\begin{equation}
\begin{cases}\widehat{w}_m\in[\frac{2}{3}w_m,\frac{4}{3}w_m],&\text{if }w_m\ge64\log(4M/\delta)/n_{m-1},\\
\widehat{w}_m<65{\log(4M/\delta)}/{n_{m-1}},&\text{if
}w_m<64\log(4M/\delta)/n_{m-1}.\end{cases}\label{eq:event-w}
\end{equation}
Let $\cE_{\mathrm{w}}=\bigcap_{m=1}^{M}\cE_{\mathrm{w}}\ind{m}$. Then
with probability at least $1-\delta/2$, $\cE_{\mathrm{w}}$ holds.
\end{lemma}

\begin{proof}
For $m=1$, $w_m\le1<64\log(4M/\delta)/n_{m-1}$, and we likewise do have $\widehat{w}_m\le1\le64\log(4M/\delta)/n_{m-1}$.

Fix any $m\in[M]\setminus\{1\}$. Since $x_{t_{m-1}+1},\dots,x_{\tau_{m-1}}$ are independent of $\cF_m$, we have
$$
\E_{x_t\sim\cD}[w(x_t;\cF_m)\}]=\E_{x\sim\cD}[w(x;\cF_m)]=w_m
$$
for $t=t_{m-1}+1,\dots,\tau_{m-1}$. Thus given $\cF_m$, $w(x_{t_{m-1}+1};\cF_m),\dots,w(x_{\tau_{m-1}};\cF_m)$ are i.i.d. $[0,1]$-valued random variables with mean $q_m$. Since $\widehat{q}_m=\frac{1}{n_{m-1}/2}\sum_{t=t_{m-1}+1}^{\tau_{m-1}}w(x_{t};\cF_m)$, by \pref{lem:bernstein}, we know that
\begin{equation}\label{eq:bern-w}
\left|\widehat{w}_m-w_m\right|\le\frac{2\log (4M/\delta)}{3n_{m-1}}+2\sqrt{\frac{{w_m\log (4M/\delta)}}{n_{m-1}}}
\end{equation}
with probability at least $1-\delta/(2M)$. If $w_m\ge64\log(4M/\delta)/n_{m-1}$, then
\begin{align*}
\left|\widehat{w}_{m}-w_{m}\right|
\le\frac{2\log (4M/\delta)}{3n_{m-1}}+2\sqrt{\frac{{w_m\log (4M/\delta)}}{n_{m-1}}}
\le\frac{2}{3}\frac{w_{m}}{64}+2\frac{w_{m}}{8}
\le\frac{1}{3}w_{m},
\end{align*}
where the first inequality follows from \pref{eq:bern-w} and the second inequality follows from $\log(4M/\delta)/n_{m-1}\le w_m/64$. In this case, $2/3w_m\le\widehat{w}_m\le4/3w_m$. 
If $w_m<64\log(4M/\delta)/n_{m-1}$, then
\begin{align*}
\left|\widehat{w}_{m}-w_{m}\right|
&\le\frac{2\log (4M/\delta)}{3n_{m-1}}+2\sqrt{\frac{{w_m\log (4M/\delta)}}{n_{m-1}}}\\
&\le\frac{2\log (4M/\delta)}{3n_{m-1}}+\frac{2\log (4M/\delta)}{8n_{m-1}}\\
&\le\frac{11}{12}\frac{\log (4M/\delta)}{n_{m-1}},
\end{align*}
where the first inequality follows from \pref{eq:bern-w} and the second inequality follows from $w_m\le\log(4M/\delta)/n_{m-1}$. In this case, $\widehat{w}_m\le w_m+\log(4M/\delta)/n_{m-1}< 65\log(4M/\delta)/n_{m-1}$.

By a union bound over $m\in[M]$, we know that with probability at least $1-\delta/2$,
\begin{align*}
\forall m\in[M],~~\begin{cases}\widehat{w}_m\in[\frac{2}{3}w_m,\frac{4}{3}w_m],&\text{if }w_m\ge64\log(4M/\delta)/n_{m-1},\\
\widehat{w}_m<65{\log(4M/\delta)}/{n_{m-1}},&\text{if }w_m<64\log(4M/\delta)/n_{m-1}.\end{cases}
\end{align*}
\end{proof}

\subsection{Analysis in Policy Space}\label{appsub:policy}
As mentioned in \pref{sec:cb_prooftech}, our regret analysis builds on
a framework established in \cite{simchi2020bypassing}, which analyzes
contextual bandit algorithms in the universal policy space $\Psi$. In
this section, we prove a number of structural properties for the
action distribution $p_m$ selected in \pref{alg:main} by focusing on a data-dependent subspace of $\Psi$ at each epoch (this is closely related to the elimination procedure of our algorithm and is essential to our instance-dependent analysis).

The analysis in this subsection deals with an arbitrary choice for the
scale factor schedule $\crl*{\lambda_m}_{m=1}^{M}$ rather than the explicit specification in
\pref{alg:main}, and serves as a foundation of the proofs of the main
theorems in \pref{appsub:cbub-p} and \pref{appsub:cbub-v} (where we
instantiate the learning rate using \optionone and \optiontwo).

For each epoch $m\in[M]$ and any round $t$ in epoch $m$, for any possible realization of $\gamma_m$, $\widehat{f}_m$ and $\cF_m$, we define a (data-dependent) subspace of $\Psi$:
$$
\Psi_m\ldef{}\prod_{x}\cA(x;\cF_m).
$$
Let $Q_m(\cdot)$ be the \emph{equivalent policy distribution} for $p_m(\cdot \mid \cdot)$, i.e.,
\[
Q_m(\pi)\ldef{}\prod_{x} p_m(\pi(x) \mid x),~~\forall \pi\in\Psi.
\]
Note that both $\Psi_m$ and $Q_m(\cdot)$ are
$\gfilt_{\tau_{m-1}}$-measurable. We refer to Section 3.2 of
\cite{simchi2020bypassing} for more detailed intuition for
$Q_m(\cdot)$ and proof of existence. By Lemma 4 of \cite{simchi2020bypassing}, we know that for all epoch $m\in[M]$ and all rounds $t$ in epoch $m$, \[\E\left[r_t(\pi_{f^*})-r_t(a_t)\mid\gfilt_{t-1}\right]=\E\left[r_t(\pi_{f^*})-r_t(a_t)\mid\gfilt_{\tau_{m-1}}\right]=\sum_{\pi\in\Psi}Q_m(\pi)\pReg(\pi).\]
Moreover, since the specification of $p_m(\cdot \mid \cdot)$ ensures
that $Q_m(\pi)=0$ for all $\pi\notin\Psi_m$, we know that
$\sum_{\pi\in\Psi}Q_m(\pi)\pReg(\pi)=\sum_{\pi\in\Psi_m}Q_m(\pi)\pReg(\pi)$,
thus it is sufficient for us to only focus on policies in $\Psi_m$ at
epoch $m$. The following lemma is an extension of Lemma 5 and 6 of
\cite{simchi2020bypassing}, which refines the policy space from $\Psi$ to $\Psi_m$.

\begin{lemma}[Implicit Optimization Problem]\label{lm:iop}
For all epoch $m\in[M]$ and all rounds $t$ in epoch $m$,  $Q_m(\cdot)$ is a feasible solution to the following \emph{Implicit Optimization Problem}:
\begin{align}
    \sum_{\pi\in\Psi_m}Q_m(\pi)\wReg_{t}(\pi)&\le (\E_{x\sim\cD}\left[|\cA(x;\cF_m)|\right]-1)/{\gamma_m},\label{eq:op11}\\
    \forall \pi\in\Psi_m,~~~\E_{x\sim\cD}\left[\frac{1}{p_m(\pi(x) \mid x)}\right]&\le \E_{x\sim\cD}\left[|\cA(x;\cF_m)|\right]+\gamma_m\wReg_{t}(\pi).\label{eq:op22}
\end{align}
\end{lemma}
\begin{proof}
Let $m$ and $t$ in epoch $m$ be fixed. We have
\begin{align*}
\sum_{\pi\in\Psi_m}Q_{m}(\pi)\wReg_{t}(\pi)&=\sum_{\pi\in\Psi_m}Q_m(\pi)\E_{x\sim\cD}\left[\widehat{f}_{m}(x,\widehat{\pi}_m(x))-\widehat{f}_m(x,\pi(x))\right]\\
    &=\E_{x\sim\cD}\left[\sum_{\pi\in\Psi_m}Q_m(\pi)\left(\widehat{f}_{m}(x,\widehat{\pi}_m(x))-\widehat{f}_m(x,\pi(x))\right)\right]\\
    &=\E_{x\sim\cD}\left[\sum_{a\in\cA(x;\cF_m)}\sum_{\pi\in\Psi_m}\1\{\pi(x)=a\}Q_m(\pi)\left(\widehat{f}_{m}(x,\widehat{\pi}_m(x))-\widehat{f}_m(x,a)\right)\right]\\
    &=\E_{x\sim\cD}\left[\sum_{a\in\cA(x;\cF_m)}p_m(a \mid x)\left(\widehat{f}_{m}(x,\widehat{\pi}_m(x))-\widehat{f}_m(x,a)\right)\right].
\end{align*}
Now, given any context $x\in\cX$, we have
\begin{align*}
\sum_{a\in\cA(x;\cF_m)}p_m(a \mid x)\left(\widehat{f}_{m}(x,\widehat{\pi}_m(x))-\widehat{f}_m(x,a)\right)&=\sum_{a\in\cA(x;\cF_m)\setminus\widehat{\pi}_m(x)}\frac{\widehat{f}_{m}(x,\widehat{\pi}_m(x))-\widehat{f}_m(x,a)}{|\cA(x,\cF_m)|+\gamma_m\left(\widehat{f}_{m}(x,\widehat{\pi}_m(x))-\widehat{f}_m(x,a)\right)}\\
&\le\frac{|\cA(x;\cF_m)|-1}{\gamma_m}.
\end{align*}
The result in \pref{eq:op11} follows immediately by taking an expectation over $x\sim\cD$.

For \pref{eq:op22}, we first observe that for any policy $\pi\in\Psi_m$, given any context $x\in\cX$,
\begin{align*}
\frac{1}{p_m(\pi(x) \mid x)}=|\cA(x;\cF_m)|+\gamma_m\left(\widehat{f}_{m}(x,\widehat{\pi}_m(x))-\widehat{f}_m(x,\pi(x))\right),\quad\text{if }\pi(x)\ne\widehat{\pi}_m(x),
\end{align*}
and
\begin{align*}
\frac{1}{p_m(\pi(x) \mid x)}  \le \frac{1}{1/|\cA(x;\cF_m)|}=\cA(x;\cF_m)+\gamma_m\left(\widehat{f}_{m}(x,\widehat{\pi}_m(x))-\widehat{f}_m(x,\pi(x))\right),\quad\text{if }\pi(x)=\widehat{\pi}_m(x).
\end{align*}
Thus
\begin{align*}
\E_{x\sim\cD}\left[\frac{1}{p_m(\pi(x) \mid x)}\right]&\le\E_{x\sim\cD}\left[|\cA(x;\cF_m)|\right]+\gamma_m\E_{x\sim\cD}\left[\widehat{f}_{m}(x,\widehat{\pi}_m(x))-\widehat{f}_m(x,\pi(x))\right]\\&=\E_{x\sim\cD}\left[|\cA(x;\cF_m)|\right]+\gamma_m\wReg_t(\pi).
\end{align*}
\end{proof}

We now formulate a more refined disagreement-based version of the implicit optimization problem. 
\begin{lemma}[Disagreement-based Implicit Optimization Problem]\label{lm:iop-i}
For all epoch $m\in[M]$, all rounds $t$ in epoch $m$, $Q_m(\cdot)$ is
a feasible solution to the following constraints:
\begin{align}
    \sum_{\pi\in\Psi_m}Q_m(\pi)\wReg_{t}(\pi)&\le q_m{\K}/{\gamma_m},\label{eq:op111}\\
    \forall \pi\in\Psi_m:~~~\E_{x\sim\cD}\left[\frac{\1\{|\cA(x;\cF_m)|>1\}}{p_m(\pi(x) \mid x)}\right]&\le q_m{\K}+\gamma_m\wReg_{t}(\pi).\label{eq:op222}
\end{align}
\end{lemma}
The constraints \pref{eq:op111} and \pref{eq:op222} involve several data-dependent quantities including the disagreement indicator $\1\{|\cA(x;\cF_m)|>1\}$ and the disagreement probability $q_m$, and will play important roles in our analysis.
\begin{proof}
Fix epoch $m$ and round $t$. \pref{eq:op111} directly follows from
\pref{eq:op11}. We now show that \pref{eq:op222} holds. For any
$\pi\in\Psi_m$, we have
\begin{align*}
\E_{x}\left[\frac{\1\{|\cA(x;\cF_m)|>1\}}{p_{m}(\pi(x) \mid x)}\right]&= \E_{x}\left[\frac{1-\1\{|\cA(x;\cF_m)|=1\}}{p_{m}(\pi(x) \mid x)}\right].
\intertext{Note that when the indicator above is $1$, we have $p_m(\pi(x)\mid{}x)=1$ for all $\pi\in\Psi_m$, and hence this is equal to}
& = \E_{x}\left[\frac{1}{p_{m}(\pi(x) \mid x)}\right]- \Prob\{|\cA(x;\cF_m)|=1\}\\
& \le \E_{x}[|\cA(x;\cF_{m})|]+\gamma_m\wReg_{t}(\pi)- (1-q_m)\\
&\le ((1-q_m)+q_m\K)+\gamma_m\wReg_{t}(\pi)-(1-q_m)\\
&= q_m \K+\gamma_m\wReg_{t}(\pi),
\end{align*}
where the first inequality follows from \pref{eq:op22}.
\end{proof}

\begin{lemma}\label{lm:reward-dis}
Assume that $\cE$ holds and $f^*\in\cF_m$ for all $m\in[M]$. For all
epoch $m\in[M]$, all rounds $t$ in epoch $m$, and all policies
$\pi\in\Psi_m$, we have
\begin{align*}
\pReg(\pi)=\cR^{\rm Dis}_t(\pi_{f^*})-\cR^{\rm Dis}_t(\pi),\mathand
\wReg_t(\pi)=\wcR^{\rm Dis}_t(\widehat{\pi}_m)-\wcR_t^{\rm Dis}(\pi).
\end{align*}
\end{lemma}
\begin{proof}
Since $\widehat{f}_m\in\cF_m$ and $f^*\in\cF_m$, for all
$\pi\in\Psi_m$, if $|\cA(x;\cF_m)|=1$, then
$\pi(x)=\pi_{f^*}(x)=\widehat{\pi}_m(x)$. The result follows
immediately from this observation.
\end{proof}

\begin{lemma}\label{lm:reward-estimates}
Assume that $\cE$ holds and $f^*\in\cF_M\subset\cdots\subset\cF_1$. For all epochs $m>1$, all rounds $t$ in epoch $m$, and all policies $\pi\in\Psi_m$, if $\gamma_m>0$, then
\begin{align*}
\left|\wcR_t^{\rm Dis}(\pi)-\cR_t^{\rm Dis}(\pi)\right|&\le 4{\crate}\sqrt{\E_{x}\left[\frac{\1\{|\cA(x;\cF_{m-1})|>1\}}{p_{m-1}(\pi(x) \mid x)}\right]}\frac{\lambda_m\sqrt{\K}}{\gamma_{m}}.
\end{align*}
\end{lemma}
\begin{proof}
Fix any epoch $m>1$, any round $t$ in epoch $m$, and any policy $\pi\in\Psi_m$. By the definitions of $\wcR_t^{\rm Dis}(\pi)$ and $\cR_t^{\rm Dis}(\pi)$, we have
$$\wcR_t^{\rm Dis}(\pi)-\cR_t^{\rm Dis}(\pi)=\E_{x\sim\cD}\left[\1\{|\cA(x;\cF_{m})|>1\}\left(\widehat{f}_{m}(x,\pi(x))-f^*(x,\pi(x))\right)\right].$$
Given a context $x$, define
$$
\Delta_x=\widehat{f}_{m}(x,\pi(x))-f^*(x,\pi(x))
$$
so that
$$
\wcR_t^{\rm Dis}(\pi)-\cR_t^{\rm Dis}(\pi)=\E_{x}[\1\{|\cA(x;\cF_{m})|>1\}\Delta_x].
$$
For all $s=t_{m-2}+1,\dots,t_{m-1}$, we have
\begin{align}
\E_{a_s \mid x_s}\left[\left(\widehat{f}_{m}(x_s,a_s)-f^*(x_s,a_s)\right)^2\mid\gfilt_{s-1}\right]&=\sum_{a\in\cA_{m-1}}p_{m-1}(a \mid x_s)\left(\widehat{f}_{m}(x_s,a)-f^*(x_s,a)\right)^2\notag\\&\ge p_{m-1}(\pi(x_s)|x_s)\left(\widehat{f}_{m}(x_s,\pi(x_s))-f^*(x_s,\pi(x_s))\right)^2\notag\\&=p_{m-1}(\pi(x_s)|x_s)\left(\Delta_{x_s}\right)^2.\label{eq:reward-estimation}
\end{align}
Thus, we have
\begin{align*}
E_{x}\left[\frac{\1\{|\cA(x;\cF_{m-1})|>1\}}{p_{m-1}(\pi(x) \mid x)}\right]\cdot{}C_\delta
&\ge\E_{x}\left[\frac{\1\{|\cA(x;\cF_{m})|>1\}}{p_{m-1}(\pi(x) \mid x)}\right]\cdot{}C_\delta\\
&\ge\E_{x}\left[\frac{\1\{|\cA(x;\cF_{m})|>1\}}{p_{m-1}(\pi(x) \mid x)}\right]\sum_{s=\tau_{m-2}+1}^{\tau_{m-1}}\E_{x_s,a_s}\left[(\widehat{f}_{m}(x_s,a_s)-f^*(x_s,a_s))^2\mid\gfilt_{s-1}\right]\\
    &=\E_{x}\left[\frac{\1\{|\cA(x;\cF_{m})|>1\}}{p_{m-1}(\pi(x) \mid x)}\right]\sum_{s=\tau_{m-2}+1}^{\tau_{m-1}}\E_{x_s}\E_{a_s \mid x_s}\left[\left(\widehat{f}_{m}(x_s,a_s)-f^*(x_s,a_s)\right)^2\mid\gfilt_{s-1}\right]\\
    &\ge \E_{x}\left[\frac{\1\{|\cA(x;\cF_{m})|>1\}}{p_{m-1}(\pi(x) \mid x)}\right]\sum_{s=\tau_{m-2}+1}^{\tau_{m-1}}\E_{x_s}\left[p_{m-1}(\pi(x_s)|x_s)\left(\Delta_{x_s}\right)^2\right]\\
    &=n_{m-1}\E_{x}\left[\frac{\1\{|\cA(x;\cF_{m})|>1\}}{p_{m-1}(\pi(x)
      \mid x)}\right]\E_{x}\left[p_{m-1}(\pi(x) \mid
      x)\left(\Delta_{x}\right)^2\right],
\end{align*}
where the first inequality follows from $\cF_m\subset \cF_{m-1}$, the
second inequality follows from \pref{eq:estimation-g}, and the third
inequality follows from \pref{eq:reward-estimation}. To proceed, note
that we have
\begin{align*}
\E_{x}\left[\frac{\1\{|\cA(x;\cF_{m})|>1\}}{p_{m-1}(\pi(x)
      \mid x)}\right]\E_{x}\left[p_{m-1}(\pi(x) \mid
      x)\left(\Delta_{x}\right)^2\right]    &\ge\left(\E_{x}\left[\sqrt{\frac{\1\{|\cA(x;\cF_{m})|>1\}}{p_{m-1}(\pi(x) \mid x)}p_{m-1}(\pi(x) \mid x)\left(\Delta_{x}\right)^2}~\right]\right)^2\\
    &=\left(\E_{x}\left[|\1\{|\cA(x;\cF_{m})|>1\}\Delta_{x}|\right]\right)^2\\&\ge n_{m-1}\left|\wcR_t^{\rm Dis}(\pi)-\cR_t^{\rm Dis}(\pi)\right|^2.
\end{align*}
where the first inequality follows from Cauchy-Schwarz and the second inequality follows uses convexity of the $L_1$ norm. Now, from the definition of $\gamma_{m}$, we have
$$\frac{{\crate}^2\lambda_m^2
  \K}{\gamma_{m}^2}=\frac{\log(2|\cF|T^2/\delta)}{n_{m-1}} = \frac{C_{\delta}}{n_{m-1}}.$$
Therefore,
\begin{align*}
\left|\wcR_t^{\rm Dis}(\pi)-\cR_t^{\rm Dis}(\pi)\right|^2&\le\E_{x}\left[\frac{\1\{|\cA(x;\cF_{m-1})|>1\}}{p_{m-1}(\pi(x) \mid x)}\right]\frac{C_\delta}{n_{m-1}}\\&=16{\crate}^2\E_{x}\left[\frac{\1\{|\cA(x;\cF_{m-1})|>1\}}{p_{m-1}(\pi(x) \mid x)}\right]\frac{\lambda_m^2 \K}{\gamma_{m}^2},
\end{align*}
which concludes the proof.
\end{proof}

\subsection{Proof of \cref*{thm:disagreement_ub}}\label{appsub:cbub-p}
We now prove \pref{thm:disagreement_ub}, which concerns \mainalg with
\optionone, where 
\[
\lambda_m=\frac{\exthq{m}}{\sqrt{\exthq{m-1}}}%
\]
for all $m\in[M]$ (for $m=1$, we have defined $\lambda_1=1$). Note that since
$\lambda_m>0$ for all $m\in[M]$, we have $\gamma_m>0$ for all
$m\in[M]$. We consider general values for $\crl*{\beta_m}_{m=1}^{M}$
unless explicitly specified. %

\begin{lemma}\label{lm:monotone}
Assume that $\cE_{\rm dp}$ holds and $\cF_M\subset\cdots\subset \cF_1$. Then $\gamma_m/\exthq{m}$ is monotonically non-decreasing in $m$, i.e., 
${\gamma_1}/{\exthq{1}}\le\cdots\le{\gamma_M}/{\exthq{M}}.$
\end{lemma}
\begin{proof}
We have
$\frac{\gamma_1}{\exthq{1}}={\crate}\frac{1}{\exthq{1}}\sqrt{\frac{A/2}{\log(2|\cF|T^2/\delta)}}$
(since $\lambda_1=1$)
and 
\[
\frac{\gamma_m}{\exthq{m}}={\crate}\frac{1}{\sqrt{\exthq{m-1}}}\sqrt{\frac{An_{m-1}}{\log(2|\cF|T^2/\delta)}}
\]
for all $m\in[M]\setminus\{1\}$. Since
$\exthq{1}=\widehat{q}_1+\mu_1\ge\widehat{q}_1=1$, we have
\[\frac{\gamma_1}{\exthq{1}}={\crate}\frac{1}{\exthq{1}}\sqrt{\frac{A/2}{\log(2|\cF|T^2/\delta)}}\le{\crate}\frac{1}{\sqrt{\exthq{1}}}\sqrt{\frac{A}{\log(2|\cF|T^2/\delta)}}=\frac{\gamma_2}{\exthq{2}}.\] For all $m\in[M]\setminus\{1,2\}$, we have
\[
\frac{\gamma_{m-1}}{\exthq{m-1}}/\frac{\gamma_m}{\exthq{m}}=\sqrt{\frac{\exthq{m-1}}{\exthq{m-2}}\frac{n_{m-2}}{n_{m-1}}}=\sqrt{\frac{1}{2}\frac{\exthq{m-1}}{\exthq{m-2}}}\le\sqrt{\frac{1}{2}\frac{\frac{4}{3}\extq{m-1}}{\frac{2}{3}\extq{m-2}}}=\sqrt{\frac{\extq{m-1}}{\extq{m-2}}}\le1,
\]
where the first inequality follows from \pref{eq:event-dp} and the
second inequality follows from $\extq{m-1}=q_{m-1}+\mu_{m-1}\le
q_{m-2}+\mu_{m-2}=\extq{m-2}$ (since $\cF_M\subset\cdots\subset\cF_1$).
\end{proof}

\begin{lemma}\label{lm:regret-estimates}
Assume that both $\cE$ and $\cE_{\rm dp}$ hold, and $f^*\in\cF_M\subset\cdots\subset\cF_1$. Let $c_1\ldef{}200{\crate}^2+3$. For all epochs $m\in[M]$, all rounds $t$ in epoch $m$, and all policies $\pi\in\Psi_m$,
$$
\pReg(\pi)\le 2\wReg_t(\pi)+c_1\exthq{m}\K/\gamma_{m},
$$
$$
\wReg_t(\pi)\le 2\pReg(\pi)+c_1\exthq{m}\K/\gamma_{m}.
$$
\end{lemma}

\begin{proof}
We prove \pref{lm:regret-estimates} via induction on $m$. We first consider the base case where $m=1$ and $1\le t\le \tau_1$. In this  case, since $\widehat{q}_1=1$ and $\gamma_1={\crate}\sqrt{\frac{A/2}{\log(2|\cF|T^2/\delta)}}$, we know that $\forall\pi\in\Psi_1$,
$$
\pReg(\pi)\le1= \widehat{q}_{1}\le\exthq{1}\le c_1 \exthq{1}\K/\gamma_1,
$$
$$
\wReg_t(\pi)\le1= \widehat{q}_1\le\exthq{1}\le c_1\exthq{1}\K/\gamma_1.
$$
Thus, the claim holds in the base case.

For the inductive step, fix some epoch $m>1$. Assume that for epoch $m-1$, all rounds $t'$ in epoch $m-1$, and all $\pi\in\Psi_{m-1}$,
\begin{equation}\label{eq:ind1-inf}
\pReg(\pi)\le 2\wReg_{t'}(\pi)+c_1\exthq{m-1}\K/\gamma_{m-1},
\end{equation}
\begin{equation}\label{eq:ind2-inf}
\wReg_{t'}(\pi)\le 2\pReg(\pi)+c_1\exthq{m-1}\K/\gamma_{m-1}.
\end{equation}

We first show that for all rounds $t$ in epoch $m$ and all $\pi\in\Psi_m$,
$$
\pReg(\pi)\le 2\wReg_{t}(\pi)+c_1\exthq{m}\K/\gamma_{m}.
$$
For any $t$ in epoch $m$, we have
\begin{align}\label{eq:ind-dif-inf}
&~~~~\pReg(\pi)-\wReg_t(\pi)\notag\\&\overset{\rm(i)}{=}\left(\cR_t^{\rm Dis}(\pi_{f^*})-\cR_t^{\rm Dis}(\pi)\right)-\left(\wcR_{t}^{\rm Dis}(\widehat{\pi}_{m(t)})-\wcR_t^{\rm Dis}(\pi)\right)\notag\\
&\overset{\rm(ii)}{\le}\left(\cR_t^{\rm Dis}(\pi_{f^*})-\cR_t^{\rm Dis}(\pi)\right)-\left(\wcR_t^{\rm Dis}(\pi_{f^*})-\wcR_t^{\rm Dis}(\pi)\right)\notag\\
&\overset{\rm(iii)}{\le}\left|\wcR_t^{\rm Dis}(\pi)-\cR_t^{\rm Dis}(\pi)\right|+\left|\wcR_t^{\rm Dis}(\pi_{f^*})-\cR^{\rm Dis}(\pi_{f^*})\right|\notag\\
&\overset{\rm(iv)}{\le}4{\crate}\sqrt{\E_{x}\left[\frac{\1\{|\cA(x;\cF_{m-1})|>1\}}{p_{m-1}(\pi(x) \mid x)}\right]}\frac{\lambda_m\sqrt{\K}}{\gamma_{m}}+4{\crate}\sqrt{\E_{x}\left[\frac{\1\{|\cA(x;\cF_{m-1})|>1\}}{p_{m-1}(\pi_{f^*}(x) \mid x)}\right]}\frac{\lambda_m\sqrt{\K}}{\gamma_{m}}\notag\\
&\overset{\rm(v)}{\le}\left(\frac{\E_{x}\left[\frac{\1\{|\cA(x;\cF_{m-1})|>1\}}{p_{m-1}(\pi(x) \mid x)}\right]}{5\exthq{m-1}\gamma_{m}/\exthq{m}}+\frac{20{\crate}^2\K}{\gamma_{m}/\exthq{m}}\right)+\left(\frac{\E_{x}\left[\frac{\1\{|\cA(x;\cF_{m-1})|>1\}}{p_{m-1}(\pi_{f^*}(x) \mid x)}\right]}{5\exthq{m-1}\gamma_{m}/\exthq{m}}+\frac{20{\crate}^2\K}{\gamma_{m}/\exthq{m}}\right)\notag\\
&=\frac{\E_{x}\left[\frac{\1\{|\cA(x;\cF_{m-1})|>1\}}{p_{m-1}(\pi(x) \mid x)}\right]}{5\exthq{m-1}\gamma_{m}/\exthq{m}}+\frac{\E_{x}\left[\frac{\1\{|\cA(x;\cF_{m-1})|>1\}}{p_{m-1}(\pi_{f^*}(x) \mid x)}\right]}{5\exthq{m-1}\gamma_{m}/\exthq{m}}+\frac{40{\crate}^2\K}{\gamma_{m}/\exthq{m}},
\end{align}
where (i) is by \pref{lm:reward-dis}, (ii) is by $\pi_{f^*}\in\Psi_m$ and the optimality of $\widehat{\pi}_m(\cdot)$ for $\wcR_t(\cdot)$ over $\Psi_m$, (iii) is by the triangle inequality, (iv) is by \pref{lm:reward-estimates}, and (v)  is by the AM-GM inequality.
By  \pref{eq:op222} and $\pi_{f^*}\in\Psi_{m-1}$,
$$
\E_{x\sim\cD}\left[\frac{\1\{|\cA(x;\cF_{m-1})|>1\}}{p_{m-1}(\pi(x) \mid x)}\right]\le q_{m-1}\K+\gamma_{m-1}\wReg_{t_{m-1}}(\pi),
$$
and
$$
\E_{x\sim\cD}\left[\frac{\1\{|\cA(x;\cF_{m-1})|>1\}}{p_{m-1}(\pi_{f^*}(x) \mid x)}\right]\le q_{m-1}\K+\gamma_{m-1}\wReg_{t_{m-1}}(\pi_{f^*}).
$$
Combining the above two inequalities with $q_{m-1}\le\extq{m-1}$ and \pref{eq:ind2-inf}, we have
\begin{align}\label{eq:ind3-inf}
\frac{\E_{x}\left[\frac{\1\{|\cA(x;\cF_{m-1})|>1\}}{p_{m-1}(\pi(x) \mid x)}\right]}{5\exthq{m-1}\gamma_{m}/\exthq{m}}&\le\frac{q_{m-1}\K+\gamma_{m-1}\wReg_{t_{m-1}}(\pi)}{5\exthq{m-1}\gamma_{m}/\exthq{m}}\notag\\
&\le \frac{\extq{m-1}\K+\gamma_{m-1}(2\pReg(\pi)+c_1\exthq{m-1}\K/\gamma_{m-1})}{5\exthq{m-1}\gamma_{m}/\exthq{m}}\notag\\
&=\frac{(\extq{m-1}+c_1\exthq{m-1})A}{5\exthq{m-1}\gamma_{m}/\exthq{m}}+\frac{2\gamma_{m-1}\pReg(\pi)}{5\exthq{m-1}\gamma_{m}/\exthq{m}}\notag\\
&=\left(\frac{\extq{m-1}}{\exthq{m-1}}+c_1\right)\frac{\exthq{m}A}{5\gamma_m}+\frac{2}{5}\pReg(\pi)\frac{\gamma_{m-1}/\exthq{m-1}}{\gamma_m/\exthq{m}}\notag\\
&\overset{\rm(i)}{\le}\frac{(1.5+c_1)\exthq{m}\K}{5\gamma_{m}}+\frac{2}{5}\pReg(\pi),
\end{align}
and
\begin{align}\label{eq:ind4-inf}
\frac{\E_{x}\left[\frac{\1\{|\cA(x;\cF_{m-1})|>1\}}{p_{m-1}(\pi_{f^*}(x) \mid x)}\right]}{5\exthq{m-1}\gamma_{m}/\exthq{m}}&{\le} \frac{q_{m-1}\K+\gamma_{m-1}\wReg_{t_{m-1}}(\pi_{f^*})}{5\exthq{m-1}\gamma_{m}/\exthq{m}}\notag\\
&\le
\frac{\extq{m-1}\K+\gamma_{m-1}(2\pReg(\pi_{f^*})+c_1\extq{m-1}\K/\gamma_{m-1})}{5\exthq{m-1}\gamma_{m}/\exthq{m}}\notag\\
&\overset{\rm(ii)}{=}\frac{(\extq{m-1}+c_1\exthq{m-1})A}{t\exthq{m-1}\gamma_{m}/\exthq{m}}\notag\\
&=\left(\frac{\extq{m-1}}{\exthq{m-1}}+c_1\right)\frac{\exthq{m}A}{5\gamma_m}\notag\\
&\overset{\rm(iii)}{\le}\frac{(1.5+c_1)\exthq{m}\K}{5\gamma_{m}},
\end{align}
where (i) in \pref{eq:ind3-inf} follows from
\pref{eq:event-dp} and \pref{lm:monotone}, (ii) in
\pref{eq:ind4-inf} follows from $\pReg(\pi_{f^*})=0$, and (iii) in \pref{eq:ind4-inf} follows from
\pref{eq:event-dp}. Combining \pref{eq:ind-dif-inf}, \pref{eq:ind3-inf} and \pref{eq:ind4-inf}, we have
\begin{equation}\label{eq:ind5-inf}
\pReg(\pi)\le\frac{7}{5}\wReg_t(\pi)+\left(\frac{200{\crate}^2+2c_1+3}{3}\right)\frac{q_m\K}{\gamma_{m}}\le 2\wReg_t(\pi)+\frac{c_1q_m\K}{\gamma_{m}}.
\end{equation}
We now show that for all rounds $t$ in epoch $m$ and all $\pi\in\Psi_m$,
$$
\wReg_t(\pi)\le 2\pReg_{t}(\pi)+c_1\exthq{m}\K/\gamma_{m}.
$$
Similar to \pref{eq:ind-dif-inf}, for any round $t$ in epoch $m$ we have
\begin{align}\label{eq:ind-dif'-inf}
&~~~~\wReg_t(\pi)-\pReg(\pi)\notag\\
&=\left(\wcR^{\rm Dis}_t(\widehat{\pi}_m)-\wcR^{\rm Dis}_t(\pi)\right)-\left(\cR^{\rm Dis}_t(\pi_{f^*})-\cR^{\rm Dis}_t(\pi)\right)\notag\\
&\le\left(\wcR^{\rm Dis}_t(\widehat{\pi}_m)-\wcR^{\rm Dis}_t(\pi)\right)-\left(\cR^{\rm Dis}_t({\widehat{\pi}_m})-\cR^{\rm Dis}_t(\pi)\right)\notag\\
&\le \left|\wcR^{\rm Dis}_t(\pi)-\cR^{\rm Dis}_t(\pi)\right|+\left|\wcR^{\rm Dis}_t(\widehat{\pi}_m)-\cR^{\rm Dis}_t(\widehat{\pi}_m)\right|\notag\\
&\le4{\crate}\sqrt{\E_{x}\left[\frac{\1\{|\cA(x;\cF_{m-1})|>1\}}{p_{m-1}(\pi(x) \mid x)}\right]}\frac{\lambda_m\sqrt{\K}}{\gamma_{m}}+4{\crate}\sqrt{\E_{x}\left[\frac{\1\{|\cA(x;\cF_{m-1})|>1\}}{p_{m-1}(\widehat{\pi}_m(x) \mid x)}\right]}\frac{\lambda_m\sqrt{\K}}{\gamma_{m}}\notag\\
&\le\left(\frac{\E_{x}\left[\frac{\1\{|\cA(x;\cF_{m-1})|>1\}}{p_{m-1}(\pi(x) \mid x)}\right]}{5\exthq{m-1}\gamma_m/\exthq{m-1}}+\frac{20{\crate}^2\K}{\gamma_{m}/\exthq{m}}\right)+\left(\frac{\E_{x}\left[\frac{\1\{|\cA(x;\cF_{m-1})|>1\}}{p_{m-1}(\widehat{\pi}_m(x) \mid x)}\right]}{5\exthq{m-1}\gamma_m/\exthq{m-1}}+\frac{20{\crate}^2\K}{\gamma_{m}/\exthq{m}}\right)\notag\\
&=\frac{\E_{x}\left[\frac{\1\{|\cA(x;\cF_{m-1})|>1\}}{p_{m-1}(\pi(x) \mid x)}\right]}{5\exthq{m-1}\gamma_m/\exthq{m-1}}+\frac{\E_{x}\left[\frac{\1\{|\cA(x;\cF_{m-1})|>1\}}{p_{m-1}(\widehat{\pi}_m(x) \mid x)}\right]}{5\exthq{m-1}\gamma_m/\exthq{m-1}}+\frac{40{\crate}^2\K}{\gamma_{m}/\exthq{m}}.
\end{align}
By \pref{eq:op222}, we have
$$
\E_{x\sim\cD}\left[\frac{\1\{|\cA(x;\cF_{m-1})|>1\}}{p_{m-1}({\widehat{\pi}_m(x)}|x)}\right]\le q_{m-1}\K+\gamma_{m-1}\wReg_{t_{m-1}}({\widehat{\pi}_m}).
$$
Hence, using \pref{eq:ind2-inf}, \pref{lm:monotone}, \pref{eq:ind5-inf}, and $\wReg_{t}(\widehat{\pi}_m)=0$, we have
\begin{align}\label{eq:ind6-inf}
\frac{\E_{x}\left[\frac{\1\{|\cA(x;\cF_{m-1})|>1\}}{p_{m-1}(\widehat{\pi}_m(x) \mid x)}\right]}{5\exthq{m-1}\gamma_m/\exthq{m-1}m}&\le \frac{q_{m-1}\K+\gamma_{m-1}\wReg_{t_{m-1}}(\widehat{\pi}_m)}{5\exthq{m-1}\gamma_m/\exthq{m-1}}\notag\\
&\le \frac{\extq{m-1}\K+\gamma_{m-1}(2\pReg(\widehat{\pi}_m)+c_1\exthq{m-1}\K/\gamma_{m-1})}{5\exthq{m-1}\gamma_m/\exthq{m-1}}\notag\\
&\le\frac{(1.5+c_1)\exthq{m}\K}{5\gamma_{m}}+\frac{2}{5}\pReg(\widehat{\pi}_m)\notag\\
&\le\frac{(1.5+c_1)\exthq{m}\K}{5\gamma_{m}}+\frac{2}{5}\left(2\wReg_t(\widehat{\pi}_m)+\frac{c_1\exthq{m}\K}{\gamma_{m}}\right)\notag\\
&=\frac{(1.5+3c_1)\extq{m}\K}{5\gamma_{m}}.
\end{align}
Combining \pref{eq:ind3-inf}, \pref{eq:ind-dif'-inf} and \pref{eq:ind6-inf}, we have
\begin{align*}
    \wReg_t(\pi)\le\frac{7}{5}\pReg(\pi)+\left(40{\crate}^2+\frac{3+4c_1}{5}\right)\frac{q_m\K}{\gamma_{m}}\le2\pReg(\pi)+\frac{c_1q_m\K}{\gamma_{m}}.
\end{align*}
This completes the inductive step, and we conclude that the claim is true for all $m\in\N$.
\end{proof}

\begin{lemma}\label{lm:policy-convergence}
Assume that both $\cE$ and $\cE_{\rm dp}$ hold, and $f^*\in\cF_M\subset\cdots\subset\cF_1$. For all epochs $m\in\N$, all rounds $t$ in epoch $m$, and all predictors $f\in\cF_m$,
$$
\pReg(\pi_f)\le (2\beta_m/C_\delta+1)c_1\exthq{m}\K/\gamma_{m}.
$$
\end{lemma}
\begin{proof}
We rewrite $\pReg(\pi_f)$ as $\E_{x}[\1\{\pi_f(x)\neq\pi_{f^*}(x)\}(f^*(x,\pi_{f^*}(x))-f^*(x,\pi_f(x))$], and we have
\begin{align}\label{eq:policy-convergence}
  (\pReg(\pi_f))^2 
      &\le \left(\E_{x}\left[\1\{\pi_f(x)\neq\pi_{f^*}(x)\}(|f(x,\pi_f(x))-f^*(x,\pi_{f}(x))|+|f^*(x,\pi_{f^*(x)})-f(x,\pi_{f^*}(x))|)\right]\right)^2\notag\\
    &\le\E_{x}\left[\frac{\1\{\pi_f(x)\neq\pi_{f^*}(x)\}}{p_{m-1}(\pi_f(x) \mid x)}+\frac{\1\{\pi_f(x)\neq\pi_{f^*}(x)\}}{p_{m-1}(\pi_{f^*}(x) \mid x)}\right]\E_{x\sim\cD,a \sim p_{m-1}(\cdot \mid x)}\left[(f^*(x,a)-f(x,a)^2\right]\notag\\
    &\le\E_{x}\left[\frac{\1\{|\cA(x;\cF_{m-1})|>1\}}{p_{m-1}(\pi_f(x) \mid x)}+\frac{\1\{|\cA(x;\cF_{m-1})|>1\}}{p_{m-1}(\pi_{f^*}(x) \mid x)}\right]\frac{(2\beta_m+C_{\delta})}{n_m/2}\notag\\
&\le\E_{x}\left[\frac{\1\{|\cA(x;\cF_{m-1})|>1\}}{p_{m-1}(\pi_f(x) \mid x)}+\frac{\1\{|\cA(x;\cF_{m-1})|>1\}}{p_{m-1}(\pi_{f^*}(x) \mid x)}\right]\frac{16(4\beta_m/C_\delta+2){\crate}^2\exthq{m}^2 \K}{\exthq{m-1}\gamma_{m}^2},
\end{align}
where in the third inequality we use the fact that {both $\pi_f$ and
  $\pi_{f^*}$ belongs to $\Psi_{m-1}$}, as well as \pref{lm:beta-value}.
By \pref{eq:event-dp}, \pref{lm:regret-estimates}, 
\pref{eq:op222}, and the fact that round $t_{m-1}$ belongs to epoch $m-1$, we have
\begin{align*}
    \E_{x}\left[\frac{\1\{|\cA(x;\cF_{m-1})|>1\}}{p_{m-1}(\pi_f(x) \mid x)}\right]&\le q_{m-1}{\K}+\gamma_{m-1}\wReg_{t_{m-1}}(\pi_f)\\
    &\le \extq{m-1}{\K}+2\gamma_{m-1}\pReg(\pi_f)+{c_1}\exthq{m-1}\K\\
    &\le 3/2\exthq{m-1}{\K}+2\gamma_{m-1}\pReg(\pi_f)+{c_1}\exthq{m-1}\K
\end{align*}
and
\begin{align*}
    \E_{x}\left[\frac{\1\{|\cA(x;\cF_{m-1})|>1\}}{p_{m-1}(\pi_{f^*}(x) \mid x)}\right]&\le q_{m-1}{\K}+\gamma_{m-1}\wReg_{t_{m-1}}(\pi_{f^*})\\
    &\le \extq{m-1}{\K}+2\gamma_{m-1}\pReg(\pi_{f^*})+c_1\exthq{m-1}A\\
    &= \extq{m-1}{\K}+{c_1}\exthq{m-1}\K\\
    &\le 3/2\exthq{m-1}{\K}+{c_1}\exthq{m-1}\K.
\end{align*}
Plugging the above two inequalities into \pref{eq:policy-convergence}, 
we have
\begin{align*}
    (\pReg(\pi_f))^2\le(2\gamma_{m-1}\pReg(\pi_f)+(2c_1+3)\exthq{m-1}\K)\frac{16(4\beta_m/C_\delta+2){\crate}^2\exthq{m}^2 \K}{\exthq{m-1}\gamma_{m}^2},
\end{align*}
which means that
\begin{align}\label{eq:quadratic}
\left(\frac{\gamma_m}{\exthq{m}}\pReg(\pi_f)\right)^2&\le64{\crate}^2(2\beta_m/C_\delta+1)\K\left(\frac{\gamma_{m-1}}{\exthq{m-1}}\pReg(\pi_f)\right)+64(c_1+1.5){\crate}^2(2\beta_m/C_\delta+1)\K^2.%
\end{align}
Solving \pref{eq:quadratic} for $\pReg(\pi_f)$, we have
$$
\pReg(\pi_f)\le (2\beta_m/C_\delta+1)\left(32{\crate}^2+8\sqrt{16{\crate}^4+(c_1+1.5){\crate}^2}\right)\frac{\extq{m}\K}{\gamma_m}\le\frac{(2\beta_m/C_\delta+1)c_1\extq{m}\K}{\gamma_m}.
$$
\end{proof}

\begin{corollary}\label{cor:policy-convergence}
Set $\beta_m=(M-m+1)C_\delta$ for all $m\in[M]$. Assume that both $\cE$ and $\cE_{\rm dp}$ hold. Define $c_2\ldef{} (2M+1)c_1$. For all epochs $m\in\N$, all rounds $t$ in epoch $m$, and all predictors $f\in\cF_m$,
$$
\pReg(\pi_f)\le c_2\exthq{m}\K/\gamma_{m}.
$$
\end{corollary}
\begin{proof}
This follows from \pref{lm:policy-convergence} and part 3 of \pref{lm:beta-value}.
\end{proof}

\textbf{Remark.} {\pref{lm:reward-estimates}, \pref{lm:regret-estimates}, \pref{lm:policy-convergence} require $\cF_M\subset\cdots\subset\cF_1$. This can be relaxed to the following condition: $\1\{|\cA(x;\cF_m)|>1\}$ is non-increasing in $m$ for all $x\in\cX$.}

\subsubsection{Incorporating the Policy Disagreement
  Coefficient}\label{appsub:relate}
At this point, can bound the regret within each epoch using
\pref{cor:policy-convergence}, which gives a bound in terms of the
empirical disagreement probability $\exthq{m}$. To proceed, we relate this
quantity to the \policydis.

Define $$
\eta_m\ldef{}\sup_{f\in\cF_m}\pReg(\pi_f),
$$
and let $\Dis(\Pi')=\crl*{x\mid{}\exists\pi\in\Pi' :
  \pi(x)\neq{}\pistar(x)}$ for $\Pi'\subseteq\Pi$. Of particular interest is
$$
{\rm Dis}(\Pi_{\eta_m}^{\csc})\ldef{}\{x \mid \exists \pi\in\Pi_{\eta_m}^{\csc}: \pi(x)\ne\pi^*(x)\}).
$$
We observe that for all $m\in[M]$, since $\pi_f\in\Pi_{\eta_m}^{\csc}$ for all $f\in\cF_m$, we have
$$
q_m=\Prob_{\cD}(|\cA(x;\cF_m)|>1)\le\Prob_{\cD}(x\in\text{Dis}(\Pi_{\eta_m}^{\csc})).
$$
The following lemma uses this result to upper bound regret in terms of
the disagreement coefficient.

\begin{lemma}\label{lm:dis1}
Set $\beta_m=(M-m+1)C_\delta$ and $\mu_m=64\log(4M/\delta)/n_{m-1}$ for all $m\in[M]$.
Assume that both $\cE$ and $\cE_{\rm dp}$ hold. Then for all
$m\in[M]\setminus\{1\}$, for all $f\in\cF_{m}$,
$$
\pReg(\pi_f)\le
\frac{\eta_{m-1}}{\eta_m}\CostDis{\eta_{m-1}}\frac{8c_2^2\K\log(2|\cF|T^2/\delta)}{3{\crate}^2n_{m-1}}
+\frac{128\log(4M/\delta)}{n_{m-1}}.
$$
\end{lemma}
\begin{proof}
Suppose  $q_{m-1}\ge\mu_{m-1}$. Then \pref{cor:policy-convergence},
Eq. \pref{eq:event-dp}, and the fact that
$q_{m-1}\le\Prob_{\cD}(x\in{\rm{Dis}}(\Pi_{\eta_{m-1}}^{\csc}))$, we have
\begin{align*}
\pReg(\pi_f)\cdot\eta_m&\le\left(\frac{c_2\exthq{m}\K}{\gamma_m}\right)\left(\frac{c_2\exthq{m}\K}{\gamma_m}\right)\\
&=\exthq{m-1}\frac{c_2^2\K\log(2|\cF|T^2/\delta)}{{\crate}^2n_{m-1}}\\
&\le\frac{4}{3}\extq{m-1}\frac{c_2^2\K\log(2|\cF|T^2/\delta)}{{\crate}^2n_{m-1}}\\
&={(q_{m-1}+\mu_{m-1})}\frac{4c_2^2A\log(2|\cF|T^2/\delta)}{3{\crate}^2n_{m-1}}\\
&\le\Prob_{\cD}(x\in{\rm{Dis}}(\Pi_{\eta_{m-1}}^{\csc}))\frac{8c_2^2A\log(2|\cF|T^2/\delta)}{3{\crate}^2n_{m-1}},
\end{align*}
thus
$$
\pReg(\pi_f)\le\frac{\Prob_{\cD}(x\in{\rm{Dis}}(\Pi_{\eta_{m-1}}^{\csc}))}{\eta_m}\frac{8c_2^2\K\log(2|\cF|T^2/\delta)}{3{\crate}^2n_{m-1}}\le\frac{\eta_{m-1}}{\eta_m}\CostDis{\eta_{m-1}}\frac{8c_2^2\K\log(2|\cF|T^2/\delta)}{3{\crate}^2n_{m-1}},
$$
where the second inequality invokes the definition of the disagreement coefficient.

On the other hand, suppose $q_{m-1}<\mu_{m-1}$. Since
$f^*\in\cF_m\subset\cF_{m-1}$ (by part 3 of \pref{lm:beta-value}) and
since $\abs*{f(x,a)}\leq{}1$ for all $f\in\cF$,  for all $f\in\cF_m$ we have
\[
\pReg(\pi_f)=\cR_{\tau_{m-1}+1}^{\rm Dis}(\pi_{f^*})-\cR_{\tau_{m-1}+1}^{\rm Dis}(\pi_f)\le q_m\le q_{m-1}<\mu_{m-1}=\frac{64\log(4M/\delta)}{n_{m-2}}=\frac{128\log(4M/\delta)}{n_{m-1}}.
\]
Summing both cases, we have
$$
\pReg(\pi_f)\le
\frac{\eta_{m-1}}{\eta_m}\CostDis{\eta_{m-1}}\frac{8c_2^2\K\log(2|\cF|T^2/\delta)}{3{\crate}^2n_{m-1}}
+\frac{128\log(4M/\delta)}{n_{m-1}}.
$$
\end{proof}

We now solve the recurrence in \pref{lm:dis1} to obtain an absolute
upper bound on $\eta_m$ for each round.

\begin{lemma}\label{lm:dis2} Set $\beta_m=(M-m+1)C_\delta$ and $\mu_m=64\log(4M/\delta)/n_{m-1}$ for all $m\in[M]$.
Assume that both $\cE$ and $\cE_{\rm dp}$ hold. Fix any $\veps>0$. For every epoch $m\in[M]$, if $\eta_m>\veps$, then
$$
\eta_m\le\CostDis{\veps}\frac{16c_2^2\K\log(2|\cF|T^2/\delta)}{3{\crate}^2n_{m-1}}
+\frac{128\log(4M/\delta)}{n_{m-1}}.
$$

\end{lemma}
\begin{proof}
We prove this result by induction. The hypothesis trivially holds for
$m=1$. Now assume that the hypothesis holds for $m-1$ where $m>1$. If $\eta_m\le\varepsilon$ then we are done. If $\eta_m>\varepsilon$, then by part 3 of \pref{lm:beta-value}, we have $\eta_{m-1}\ge\eta_{m}>\varepsilon$. We consider two cases.

\textbf{Case 1}: $\eta_m\ge \frac{1}{2}\eta_{m-1}$. In this case, by \pref{lm:dis1} we have
\begin{align*}\eta_m&\le
\frac{\eta_{m-1}}{\eta_m}\CostDis{\eta_{m-1}}\frac{8c_2^2\K\log(2|\cF|T^2/\delta)}{3{\crate}^2n_{m-1}}
+\frac{128\log(4M/\delta)}{n_{m-1}}\\
&\le\CostDis{\eta_{m-1}}\frac{16c_2^2\K\log(2|\cF|T^2/\delta)}{3{\crate}^2n_{m-1}}
+\frac{128\log(4M/\delta)}{n_{m-1}},
\end{align*}
and by  $\eta_{m-1}\ge\eta_{m}>\varepsilon$, we have
$$
\eta_m\le\CostDis{\veps}\frac{16c_2^2\K\log(2|\cF|T^2/\delta)}{3{\crate}^2n_{m-1}}
+\frac{128\log(4M/\delta)}{n_{m-1}}.
$$

\textbf{Case 2}: $\eta_m<\frac{1}{2}\eta_{m-1}$. Since $\eta_{m-1}\ge\eta_{m}>\varepsilon$, by the induction assumption we have
$$
\eta_{m-1}\le\CostDis{\veps}\frac{16c_2^2\K\log(2|\cF|T^2/\delta)}{3{\crate}^2n_{m-2}}
+\frac{128\log(4M/\delta)}{n_{m-2}},
$$
and by $\eta_m<\frac{1}{2}\eta_{m-1}$ we know that
\begin{align*}
\eta_m&<\frac{1}{2}\left(\CostDis{\veps}\frac{16c_2^2\K\log(2|\cF|T^2/\delta)}{3{\crate}^2n_{m-2}}
+\frac{128\log(4M/\delta)}{n_{m-2}}\right)\\
&=\CostDis{\veps}\frac{16c_2^2\K\log(2|\cF|T^2/\delta)}{3{\crate}^2n_{m-1}}
+\frac{128\log(4M/\delta)}{n_{m-1}},
\end{align*}
where we use that fact that $n_{m-1}=2n_{m-2}$.

Combining Case 1 and Case 2, we have that the hypothesis holds for
$m$, concluding the inductive proof.
\end{proof}

\begin{lemma}\label{lm:true-reg}
Assume that both $\cE$ and $\cE_{\rm dp}$ hold, and $f^*\in\cF_M\subset\cdots\subset\cF_1$. For every epoch $m\in[M]$,
$$
\sum_{\pi\in\Psi}Q_m(\pi)\pReg(\pi)\le(3+c_1)\exthq{m}\K/\gamma_m.
$$
\end{lemma}
\begin{proof}
Fix any epoch $m\in\N$. Since $\tau_{m-1}+1$ belongs to epoch $m$, we have
\begin{align*}
    \sum_{\pi\in\Psi}Q_m(\pi)\pReg(\pi)&\le\sum_{\pi\in\Psi_m}Q_m(\pi)\left(2{\wReg}_{\tau_{m-1}+1}(\pi)+\frac{c_1\exthq{m}\K}{\gamma_{m}}\right)\\
    &=2\sum_{\pi\in\Psi_m}Q_m(\pi){\wReg}_{\tau_{m-1}+1}(\pi)+\frac{c_1\exthq{m}0\K}{\gamma_m}\\
    &\le\frac{(2q_m+c_1\exthq{m})\K}{\gamma_m}\\
    &\le\frac{(3+c_1)\exthq{m}\K}{\gamma_m}
\end{align*}
where the first inequality follows from \pref{lm:regret-estimates},
the second inequality follows from \pref{eq:op111}, and the third
inequality follows from $q_m\le\extq{m}\le\frac{3}{2}\exthq{m}$ by \pref{lm:sandwich}.
\end{proof}

\begin{lemma}\label{lm:true-reg-dis}
Set $\beta_m=(M-m+1)C_\delta$ and $\mu_m=64\log(4M/\delta)/n_{m-1}$ for all $m\in[M]$.
Assume that both $\cE$ and $\cE_{\rm dp}$ hold. Fix any $\veps>0$. For every epoch $m\in[M]$,
$$
\sum_{\pi\in\Psi}Q_m(\pi)\pReg(\pi)\le\max\left\{\varepsilon, \CostDis{\veps}\frac{32c_2^2\K\log(2|\cF|T^2/\delta)}{3{\crate}^2n_{m-1}}
\right\}+\frac{256\log(4M/\delta)}{n_{m-1}}.
$$
\end{lemma}
\begin{proof} The result trivially holds for $m=1$.

For $m>1$, we consider several cases.

\textbf{Case 1: $q_{m-1}\ge\mu_{m-1}$.}
By \pref{lm:true-reg}, \pref{eq:event-dp}, and
$q_{m-1}\le\Prob_{\cD}(x\in{\rm{Dis}}(\Pi_{\eta_{m-1}}^{\csc}))$,
using the definition of $\gamma_m$ we have
\begin{align*}
\left(\sum_{\pi\in\Psi}Q_m(\pi)\pReg(\pi)\right)^2&\le\left(\frac{(c_1+3)\exthq{m}\K}{\gamma_m}\right)\left(\frac{(c_1+3)\exthq{m}\K}{\gamma_m}\right)\\
&=\exthq{m-1}\frac{(c_1+3)^2\K\log(2|\cF|T^2/\delta)}{{\crate}^2n_{m-1}}\\
&\le\frac{4}{3}\extq{m-1}\frac{(c_1+3)^2\K\log(2|\cF|T^2/\delta)}{{\crate}^2n_{m-1}}\\
&={(q_{m-1}+\mu_{m-1})}\frac{4(c_1+3)^2A\log(2|\cF|T^2/\delta)}{3{\crate}^2n_{m-1}}\\
&\le\Prob_{\cD}(x\in{\rm{Dis}}(\Pi_{\eta_{m-1}}^{\csc}))\frac{8(c_1+3)^2A\log(2|\cF|T^2/\delta)}{3{\crate}^2n_{m-1}},
\end{align*}
thus
\begin{align}\label{eq:true-reg-dis}
\sum_{\pi\in\Psi}Q_m(\pi)\pReg(\pi)\le\frac{\Prob_{\cD}({x\in\rm{Dis}}(\Pi_{\eta_{m-1}}^{\csc}))}{\sum_{\pi\in\Psi}Q_m(\pi)\pReg(\pi)}\frac{8(c_1+3)^2\K\log(2|\cF|T^2/\delta)}{3{\crate}^2n_{m-1}}.%
\end{align}

\textbf{Case 1.1: ${\sum_{\pi\in\Psi}Q_m(\pi)\pReg(\pi)}>\eta_{m-1}$.}
In this case, $\Pi^{\csc}_{\eta_{m-1}}\subset\Pi^{\csc}_{r_m}$, where
$r_m\ldef\sum_{\pi\in\Psi}Q_m(\pi)\pReg(\pi)$, and \pref{eq:true-reg-dis} implies that
$$
\sum_{\pi\in\Psi}Q_m(\pi)\pReg(\pi)\le\CostDisA\left(\Pi,r_m\right)\frac{8(c_1+3)^2\K\log(2|\cF|T^2/\delta)}{3{\crate}^2n_{m-1}}.
$$
If $\sum_{\pi\in\Psi}Q_m(\pi)\pReg(\pi)>\veps$, then
$$
\sum_{\pi\in\Psi}Q_m(\pi)\pReg(\pi)<\CostDis{\veps}\frac{8(c_1+3)^2\K\log(2|\cF|T^2/\delta)}{3{\crate}^2n_{m-1}}.
$$
Otherwise, $\sum_{\pi\in\Psi}Q_m(\pi)\pReg(\pi)\le\veps$.

\textbf{Case 1.2:
  ${\sum_{\pi\in\Psi}Q_m(\pi)\pReg(\pi)}\le\eta_{m-1}$.} In this case,
we apply \pref{lm:dis2}, which gives
$$
\sum_{\pi\in\Psi}Q_m(\pi)\pReg(\pi)\le \eta_{m-1}\le \max\left\{\varepsilon, \CostDis{\veps}\frac{16c_2^2\K\log(2|\cF|T^2/\delta)}{3{\crate}^2n_{m-2}}
+\frac{128\log(4M/\delta)}{n_{m-2}}\right\}.
$$

\textbf{Case 2: $q_{m-1}<\mu_{m-1}$.} By Part 3 of
\pref{lm:beta-value} and \pref{lm:reward-dis}, and using that
$\abs*{\fstar(x,a)}\leq{}1$, we have
\[
\sum_{\pi\in\Psi}Q_m(\pi)\pReg(\pi)=\sum_{\pi\in\Psi_m}Q_m(\pi)\pReg(\pi)\le q_m\le q_{m-1}<\mu_{m-1}=\frac{64\log(4M/\delta)}{n_{m-2}}.
\]

Combining all the cases above, we have
\begin{align*}
    \sum_{\pi\in\Psi}Q_m(\pi)\pReg(\pi)&\le \max\left\{\varepsilon, \CostDis{\veps}\frac{16c_2^2\K\log(2|\cF|T^2/\delta)}{3{\crate}^2n_{m-2}}
\right\}+\frac{128\log(4M/\delta)}{n_{m-2}}\\
&=\max\left\{\varepsilon, \CostDis{\veps}\frac{32c_2^2\K\log(2|\cF|T^2/\delta)}{3{\crate}^2n_{m-1}}
\right\}+\frac{256\log(4M/\delta)}{n_{m-1}}.
\end{align*}

\end{proof}

\begin{lemma}\label{lm:final}
For any $\delta\in(0,1]$, $\crate>0$, by setting $\beta_m=(M-m+1)C_{\delta}$ and $\mu_m=64\log(4M/\delta)/n_{m-1}$ for all $m\in[M]$, \pref{alg:main} with \optionone ensures that for every instance,
\[
\Regbar= \bigoht{({\crate}^2+{\crate}^{-2})}\cdot\min_{\veps>0}\max\left\{\varepsilon T, \CostDis{\veps}{\K\log(|\cF|/\delta)}\right\}+\bigoht(\log(1/\delta))
\]
with probability at least $1-\delta$.
\end{lemma}

\begin{proof}
By \pref{lm:high-prob-event}, \pref{lm:beta-value} and
\pref{lm:sandwich}, the choice of $\{\beta_m\}_{m=1}^M$ and
$\{\mu_m\}_{m=1}^{M}$ ensures that $\cE$ and $\cE_{\rm dp}$
simultaneously hold with probability at least $1-\delta$, and $f^*\in\cF_M\subset\cdots\subset\cF_1$. By
\pref{lm:true-reg-dis}, conditional on the occurrence of $\cE$ and
$\cE_{\rm dp}$, for any $\veps>0$, we have 
\begin{align*}
\Regbar&=\sum_{t=1}^{T}\En\brk*{r_t(\pistar(x_t))-r_t(a_t)\mid\gfilt_{t-1}}\\
&\le\sum_{m=1}^M\sum_{t=\tau_{m-1}+1}^{\tau_m}\En\brk*{r_t(\pistar(x_t))-r_t(a_t)\mid\gfilt_{\tau_{m-1}}}\\
&=\sum_{m=1}^M\sum_{t=\tau_{m-1}+1}^{\tau_m}\sum_{\pi\in\Psi}Q_m(\pi)\pReg(\pi)\\
&\le\sum_{m=1}^M n_m\cdot \left(\max\left\{\varepsilon, \CostDis{\veps}\frac{32c_2^2\K\log(2|\cF|T^2/\delta)}{3{\crate}^2n_{m-1}}
\right\}+\frac{256\log(4M/\delta)}{n_{m-1}}\right)\\
&=\sum_{m=1}^M\left(\max\left\{\varepsilon n_m, \CostDis{\veps}\frac{64c_2^2\K\log(2|\cF|T^2/\delta)}{3{\crate}^2}
\right\}+{512\log(4M/\delta)}\right)\\
&\le M\cdot\max\left\{2\varepsilon T, \CostDis{\veps}\frac{64c_2^2\K\log(2|\cF|T^2/\delta)}{3{\crate}^2}
\right\}+512M\log(4M/\delta)\\
&=\bigoht{(c_2^2/{\crate}^2)}\cdot\max\left\{\varepsilon T, \CostDis{\veps}{\K\log(|\cF|/\delta)}\right\}+\bigoht(\log(1/\delta))\\
&=\bigoht{({\crate}^2+{\crate}^{-2})}\cdot\max\left\{\varepsilon T, \CostDis{\veps}{\K\log(|\cF|/\delta)}\right\}+\bigoht(\log(1/\delta)).
\end{align*}
Thus
\[
\Regbar= \bigoht{({\crate}^2+{\crate}^{-2})}\cdot\min_{\veps>0}\max\left\{\varepsilon T, \CostDis{\veps}{\K\log(|\cF|/\delta)}\right\}+\bigoht(\log(1/\delta))
\]
with probability at least $1-\delta$.
\end{proof}

\begin{proof}[\pfref{thm:disagreement_ub}]
  The $\CostDisA$-based upper bound in \pref{thm:disagreement_ub} can
  be directly obtained from \pref{lm:final}: By taking $\delta=1/T$
  and ${\crate}=1$, we have
  \[
    \E\left[\Reg\right]=\E\left[\Regbar\right]=\bigoht{(1)}\cdot\min_{\veps>0}\max\left\{\varepsilon
      T, \CostDis{\veps}{\K\log|\cF|}\right\}+\bigoht(1).
  \]
  The $\PolicyDisA$-based upper bound (under the uniform gap
  assumption) in \pref{thm:disagreement_ub} is an immediate corollary of
  the $\CostDisA$-based upper bound.
\end{proof}

\subsection{Proof of \cref*{thm:value_disagreement_ub}}\label{appsub:cbub-v}
We now prove that 
\mainalg with \optiontwo, i.e. with
\[
\lambda_m=\1\left\{\widehat{w}_m\ge\frac{\sqrt{AT\log(|\cF|/\delta)}}{n_{m-1}}\right\}%
\]
for all $m\in[M]$, attains the regret bound in \pref{thm:value_disagreement_ub}.

\begin{lemma}\label{lm:delta-bounds-true}
Assume that $\cE$ holds and $f^*\in\cF_m$ for all $m\in[M]$. Then for
all epoch $m\in[M]$ and  all rounds $t$ in epoch $m$,
\[
\En\brk*{r_t(\pistar(x_t))-r_t(a_t)\mid\gfilt_{t-1}}\le2w_m.
\]
\end{lemma}
\begin{proof}
Consider any epoch $m\in[M]$, any round $t$ in epoch $m$. 
We have
\begin{align*}
\En\brk*{r_t(\pistar(x_t))-r_t(a_t)\mid\gfilt_{t-1}}=\En_{x\sim\cD,a\sim p_m(\cdot \mid x)}[f^*(x,\pi^*(x))-f^*(x,a)]
\end{align*}
For all $x$ such that $|\cA(x;\cF_m)|=1$, since $f^*\in\cF_m$, we have $p_m(\pi^*(x)|x)=1$ and $w(x;\cF_m)=0$, thus
$$
\En_{a\sim p_m(\cdot \mid x)}[f^*(x,\pi^*(x))-f^*(x,a)|x]=0=2w(x;\cF_m).
$$
For all $x$ such that $|\cA(x;\cF_m)|>1$, for all $a\in\cA(x;\cF_m)$, there exists $f^{\circ}\in\cF_m$ such that $\pi_{f^\circ}(x)=a$, thus
\begin{align*}
    f^*(x,\pi^*(x))-f^*(x,a)&= f^*(x,\pi_{f^*}(x))-f^*(x,\pi_{f^\circ}(x))\\
    &\le  f^*(x,\pi_{f^*}(x))-f^*(x,\pi_{f^\circ}(x)) + f^{\circ}(x,\pi_{f^\circ}(x))-f^\circ(x,\pi_{f^*}(x))\\
    &= |f^{\circ}(x,\pi_{f^\circ}(x))-f^*(x,\pi_{f^\circ}(x))|+|f^\circ(x,\pi_{f^*}(x))-f^*(x,\pi_{f^*}(x))|\\
    &\le 2\sup_{a\in \cA(x;\cF_m)}\sup_{f,f'\in\cF_m} \left|f(x,a)-f'(x,a)\right|\\
    &=2w(x,\cF_m).
\end{align*}
Hence for all $x$ such that $|\cA(x;\cF_m)|>1$, 
$$
\En_{a\sim p_m(\cdot \mid x)}[f^*(x,\pi^*(x))-f^*(x,a)\mid x]\le \sup_{a\in\cA(x;\cF_m)}[f^*(x,\pi^*(x))-f^*(x,a)]\le2w(x;\cF_m).
$$
Putting both cases together, we have
\[
\En_{a\sim p_m(\cdot \mid x)}[f^*(x,\pi^*(x))-f^*(x,a)\mid x]\le2w(x.\cF_m),~~\forall x\in\cX
\]
and
\begin{align*}
    \En\brk*{r_t(\pistar(x_t))-r_t(a_t)\mid\gfilt_{t-1}}= \En_{x\sim\cD,a\sim p_m(\cdot \mid x)}[f^*(x,\pi^*(x))-f^*(x,a)]\le2\E_{x\sim\cD}[w(x;\cF_m)]=2w_m.
\end{align*}
\end{proof}

\subsubsection{Minimax Regret}

Before we derive sharp instance-dependent guarantees for \mainalg with \optiontwo, we first show that \optiontwo will never degrade the algorithm's worst-case performance, i.e., \mainalg with \optiontwo always guarantees the minimax rate $\bigoht(\sqrt{AT\log|\cF|})$.

\begin{corollary}\label{cor:iop-s}
For all epoch $m\in[M]$, all round $t$ in epoch $m$,  $Q_m(\cdot)$ is a feasible solution to:
\begin{align}
    \sum_{\pi\in\Psi_m}Q_m(\pi)\wReg_{t}(\pi)&\le {\K}/{\gamma_m},\notag\\
    \forall \pi\in\Psi_m,~~~\E_{x\sim\cD}\left[\frac{1}{p_m(\pi(x) \mid x)}\right]&\le A+\gamma_m\wReg_{t}(\pi).\label{eq:op-s2}
\end{align}
\end{corollary}

\begin{proof}
This is a direct corollary of \pref{lm:iop}.
\end{proof}

\begin{lemma}\label{lm:largest}
For any epoch $m\in[M]$, if $\gamma_m>0$, then $\gamma_m\ge\max\{\gamma_1,\dots,\gamma_{m-1}\}$.
\end{lemma}
\begin{proof}
If $\gamma_m>0$, then $\lambda_m=1$, thus $\gamma_m=c\sqrt{(A n_{m-1})/\log(2|\cF|T^2)/\delta}$. For any $m'\in\{1,\dots,m-1\}$,
\[
\gamma_{m'}=\lambda_{m'}\cdot c\sqrt{\frac{An_{m'-1}}{\log(2|\cF|T^2/\delta)}}\le c\sqrt{\frac{An_{m'-1}}{\log(2|\cF|T^2/\delta)}}\le\gamma_m.
\]
\end{proof}

\begin{lemma}\label{lm:minimax-vf}
Assume that $\cE$ holds, and $f^*\in\cF_m$ for all $m\in[M]$. Let $c_1\ldef{}200\crate^2+3$. For all epochs $m\in[M]$ such that $\gamma_m>0$,
\[
\sum_{\pi\in\Psi}Q_m(\pi)\pReg(\pi)\le\frac{(2+c_1)\K}{\gamma_m}.
\]
\end{lemma}
\begin{proof}
The proof is based on slight modifications of Lemma 7, Lemma 8 and Lemma 9 of \cite{simchi2020bypassing}, which (essentially) proves the same result based on the following two conditions
\begin{enumerate}
    \item $Q_m(\cdot)$ is a feasible solution to
    \begin{align}
    \sum_{\pi\in\Psi}Q_m(\pi)\wReg_{t}(\pi)&\le {\K}/{\gamma_m},\notag\\
    \forall \pi\in\Psi,~~~\E_{x\sim\cD}\left[\frac{1}{p_m(\pi(x) \mid x)}\right]&\le A+\gamma_m\wReg_{t}(\pi).\label{eq:op-o2}
\end{align}
\item For any epoch $m\in[M]$, $\gamma_m\ge\max\{\gamma_1,\dots,\gamma_{m-1}\}$.
\end{enumerate}
While our \pref{cor:iop-s} is weaker than their first condition
(specifically, \pref{eq:op-s2} holds for $\forall \pi\in\Psi_m$ while
\pref{eq:op-o2} requires $\forall \pi\in\Psi$), their proof still
works under our condition, as we have assumed that $f^*\in\cF_m$ for
all $m\in[M]$. While our \pref{lm:largest} is weaker than  their
second condition, this difference does not affect
\pref{lm:minimax-vf}, as \pref{lm:minimax-vf} only considers epochs
$m$ such that $\gamma_m>0$. As a result, \pref{lm:minimax-vf} is
indeed implied by their result.
\end{proof}

\begin{lemma}\label{lm:minimax-v}
For any $\delta\in(T^{-2},1]$, $\crate>0$, by setting $\beta_m\ge C_{\delta}/2$ for all $m\in[M]$, \pref{alg:main} with \optiontwo ensures that for every instance,
\[
\Regbar= \bigoht{({\crate}+{\crate}^{-1})}\cdot\sqrt{AT\log(|\cF|/\delta)}
\]
with probability at least $1-\delta$.
\end{lemma}
\begin{proof}
Without loss of generality, we assume that $T$ is sufficiently large
such that $3/2\sqrt{AT\log(|\cF|/\delta)}\ge64\log(4M/\delta)$ for all $\delta\in(T^{-2},1]$ (otherwise, $T$ must be upper bounded by an absolute constant, and we know that $\Regbar$ is upper bounded by the same constant).

By \pref{lm:high-prob-event} and \pref{lm:width-approx}, $\cE$ and
$\cE_{\rm w}$ simultaneously hold with probability at least
$1-\delta$. In the rest of the proof, we assume that both $\cE$ and
$\cE_{\rm w}$ hold. By \pref{lm:beta-value}, the specification of
$\{\beta_m\}_{m=1}^T$ in \pref{alg:main} ensures that $f^*\in\cF_m$ for $m\in[M]$.

Consider any epoch $m\in[M]$, any round $t$ in epoch $m$. When $w_m<\frac{3\sqrt{AT\log(|\cF|/\delta)}}{2n_{m-1}}$, by \pref{lm:delta-bounds-true}, \begin{equation}\label{eq:minimax-v-1}
\En\brk*{r_t(\pistar(x_t))-r_t(a_t)\mid\gfilt_{t-1}}\le 2w_m=\frac{3\sqrt{AT\log(|\cF|/\delta)}}{n_{m-1}}.
\end{equation}
When $w_m\ge\frac{3\sqrt{AT\log(|\cF|/\delta)}}{2n_{m-1}}$, we have
$w_m\ge\frac{3\sqrt{AT\log(|\cF|/\delta)}}{2n_{m-1}}\ge\frac{64\log(4M/\delta)}{n_{m-1}}$,
and thus by \pref{eq:event-w},
\begin{equation*}
    \widehat{w}_m\ge\frac{2}{3}w_m\ge\frac{\sqrt{AT\log(|\cF|/\delta)}}{n_{m-1}},
\end{equation*}
and we have
\[
\lambda_m=\1\left\{\widehat{w}_m\ge\frac{\sqrt{AT\log(|\cF|/\delta)}}{n_{m-1}}\right\}=1,
\]
which implies $\gamma_m>0$. By \pref{lm:minimax-vf}, we have
\begin{align}
    \En\brk*{r_t(\pistar(x_t))-r_t(a_t)\mid\gfilt_{t-1}}&=\sum_{\pi\in\Psi}Q_m(\pi)\pReg(\pi)\notag\\
    &\le\frac{(2+c_1)\K}{\gamma_m}\notag\\
    &=(200\crate+3\crate^{-1})\sqrt{\frac{A\log(2|\cF|T^2/\delta)}{n_{m-1}}}.\label{eq:minimax-v-2}
\end{align}
Combining \pref{eq:minimax-v-1} and \pref{eq:minimax-v-2}, we have
\[
\En\brk*{r_t(\pistar(x_t))-r_t(a_t)\mid\gfilt_{t-1}}\le\frac{3\sqrt{AT\log(|\cF|/\delta)}}{n_{m-1}}+(200\crate+3\crate^{-1})\sqrt{\frac{A\log(2|\cF|T^2/\delta)}{n_{m-1}}}.
\]
Therefore,
\begin{align*}
\Regbar&=\sum_{t=1}^{T}\En\brk*{r_t(\pistar(x_t))-r_t(a_t)\mid\gfilt_{t-1}}\\
&\le\sum_{m=1}^M\sum_{t=\tau_{m-1}+1}^{\tau_m}\En\brk*{r_t(\pistar(x_t))-r_t(a_t)\mid\gfilt_{\tau_{m-1}}}\\
&\le\sum_{m=1}^Mn_m\cdot\left(\frac{3\sqrt{AT\log(|\cF|/\delta)}}{n_{m-1}}+(200\crate+3\crate^{-1})\sqrt{\frac{A\log(2|\cF|T^2/\delta)}{n_{m-1}}}\right)\\
&=6M\sqrt{AT\log(|\cF|/\delta)}+(400\crate+6\crate^{-1})\sqrt{A\log(2|\cF|T^2/\delta)}\sum_{m=1}^M\sqrt{n_{m-1}}\\
&\le6M\sqrt{AT\log(|\cF|/\delta)}+(400\crate+6\crate^{-1})M\sqrt{AT\log(2|\cF|T^2/\delta)}\\
&=\bigoht(\crate+\crate^{-1})\cdot\sqrt{AT\log(|\cF|/\delta)}.
\end{align*}
\end{proof}

\subsubsection{Instance-Dependent Regret}\label{appsub:id-s}
We now show that whenever the uniform gap condition
\begin{equation*}
  \fstar(x,\pistar(x)) -\fstar(x,a)\geq\Delta\quad\forall{}a\neq\pistar(x),~ \forall x\in\cX
\end{equation*}
holds, \mainalg with \optiontwo enjoys the
$\frac{\ValueDisA\log|\cF|}{\Delta}$-type instance-dependent rate in \pref{thm:value_disagreement_ub}.
\begin{lemma}\label{lm:expected-delta}
Assume $\cE$ holds and that $f^*\in\cF_M\subset\cdots\subset\cF_1$. For all
$m>1$, for all choices for $\beta_m\ge0$, it holds that
\[
w_m\le4(A+\gamma_{m-1})\cdot\frac{\ValueDis{\Delta/2}{\sqrt{(2\beta_m+C_\delta)/n_{m-1}}}}{\Delta}\cdot\frac{2\beta_m+C_\delta}{n_{m-1}}.
\]
\end{lemma}

\begin{proof}
Consider any epoch $m>1$. Define
$\veps_m=\sqrt{2\beta_m+C_\delta}$. We first observe that for all $x$ such that $|\cA(x;\cF_m)|>1$, there exists $f^{\circ}\in\cF_m$ such that $\pi_{f^\circ}(x)\ne\pi_{f^*}(x)$ and
\begin{align*}
f^*(x,\pi_{f^*}(x))-f^*(x,\pi_{f^\circ}(x))\ge\Delta.
\end{align*}
Thus
\begin{align*}
    \Delta&\le f^*(x,\pi_{f^*}(x))-f^*(x,\pi_{f^\circ}(x))\\
    &\le f^*(x,\pi_{f^*}(x))-f^*(x,\pi_{f^\circ}(x)) + f^{\circ}(x,\pi_{f^\circ}(x))-f^\circ(x,\pi_{f^*}(x))\\
    &= |f^{\circ}(x,\pi_{f^\circ}(x))-f^*(x,\pi_{f^\circ}(x))|+|f^\circ(x,\pi_{f^*}(x))-f^*(x,\pi_{f^*}(x))|\\
    &\le 2\sup_{a\in \cA(x;\cF_m)}\sup_{f\in\cF_m} \left|f(x,a)-f^*(x,a)\right|,
\end{align*}
where the last inequality utilizes $f^*\in\cF_m$. 

Therefore, $\sup_{a\in \cA(x;\cF_m)}\sup_{f'\in\cF_m} \left|f(x,a)-f^*(x,a)\right|\ge\Delta/2$ whenever $|\cA(x;\cF_m)|>1$. We then have
\begin{align}\label{eq:vf1}
w_m&=\E_{\cD}[w(x;\cF_m)]\notag\\
    &=\E_\cD\left[\mathbb{I}\{|\cA(x;\cF_m)|>1\}\sup_{a\in \cA(x;\cF_m)}\sup_{f,f'\in\cF_m} \left|f(x,a)-f'(x,a)\right|\right]\notag\\
    &\le\E_\cD\left[\mathbb{I}\{|A(x;\cF_m)|>1\}\sup_{a\in \cA(x;\cF_m)}\left(\sup_{f\in\cF_m} \left|f(x,a)-f^*(x,a)\right|+\sup_{f'\in\cF_m}\left|f'(x,a)-f^*(x,a)\right|\right)\right]\notag\\
    &=2\E_\cD\left[\mathbb{I}\{|\cA(x;\cF_m)|>1\}\sup_{a\in \cA(x;\cF_m)}\sup_{f\in\cF_m} \left|f(x,a)-f^*(x,a)\right|\right]\notag\\
    &\le2\int_{\Delta/2}^1\Prob_\cD\left(\sup_{a\in \cA(x;\cF_m)}\sup_{f\in\cF_m} \left|f(x,a)-f^*(x,a)\right|>\omega\right)d\omega\notag\\
    &\le2\int_{\Delta/2}^1\Prob_\cD\left(\sup_{a\in \cA(x;\cF_{m-1})}\sup_{f\in\cF_m} \left|f(x,a)-f^*(x,a)\right|>\omega\right)d\omega,%
\end{align}
where the last inequality utilizes $\cF_{m-1}\subset\cF_m$.

Let $\omega>0$ be fixed. For all $x\in\cX$, the definition of $p_{m-1}$ implies that
\begin{align*}
&~~~~{\sup_{a\in \cA(x;\cF_{m-1})}\sup_{f\in\cF_m}\1\left\{ \left|f(x,a)-f^*(x,a)\right|>\omega\right\}}\\
&\le \left(\sup_{a\in\cA(x;\cF_{m-1}))}\frac{1}{p_{m-1}(a \mid x)}\right){\E_{a\sim p_{m-1}(\cdot \mid x)}\left[\sup_{f\in\cF_m}\1\left\{ \left|f(x,a)-f^*(x,a)\right|>\omega\right\}\right]}\\
&\le(A+\gamma_{m-1}){\E_{a\sim p_{m-1}(\cdot \mid x)}\left[\sup_{f\in\cF_m}\1\left\{ \left|f(x,a)-f^*(x,a)\right|>\omega\right\}\right]}.
\end{align*}
It follows that
\begin{align}\label{eq:vf2}
&~~~~\Prob_\cD\left(\sup_{a\in \cA(x;\cF_{m-1})}\sup_{f\in\cF_m} \left|f(x,a)-f^*(x,a)\right|>\omega\right)\notag\\
&=\E_{x\sim\cD}\left[\sup_{a\in \cA(x;\cF_{m-1})}\sup_{f\in\cF_m}\1\left\{ \left|f(x,a)-f^*(x,a)\right|>\omega\right\}\right]\notag\\
&\le(A+\gamma_{m-1})\E_{x\sim\cD,a\sim p_{m-1}(\cdot \mid x)}\left[\sup_{f\in\cF_m}\1\left\{ \left|f(x,a)-f^*(x,a)\right|>\omega\right\}\right]\notag\\
&=(A+\gamma_{m-1})\Prob_{x\sim\cD,a\sim p_{m-1}(\cdot \mid x)}\left(\sup_{f\in\cF_m} \left|f(x,a)-f^*(x,a)\right|>\omega\right).
\end{align}

By part 1 of \pref{lm:beta-value}, we have
\[
\forall f\in\cF_m,~~\E_{x\sim\cD,a\sim p_{m-1}(\cdot \mid x)}(f(x,a)-f^*(x,a))^2\le \varepsilon_m^2.
\]
Hence, for all $\omega\in(0,1]$, by the definition of $\ValueDisA$, we have
\begin{align}\label{eq:vf3}
    &~~~~\Prob_{x\sim\cD,a\sim p_{m-1}(\cdot \mid x)}\left(\sup_{f\in\cF_m} \left|f(x,a)-f^*(x,a)\right|>\omega\right)\notag\\
    &\le\Prob_{x\sim\cD,a\sim p_{m-1}(\cdot \mid x)}\left(\exists f\in\cF: \left|f(x,a)-f^*(x,a)\right|>\omega, \E_{x\sim\cD,a\sim p_{m-1}(\cdot \mid x)}(f(x,a)-f^*(x,a))^2\le \varepsilon_m^2\right)\notag\\
    &\le\ValueDis{\omega}{\veps_m}\frac{\varepsilon_m^2}{\omega^2}.
\end{align}

Combining \pref{eq:vf1}, \pref{eq:vf2}, \pref{eq:vf3}, we have
\begin{align}\label{eq:width-bound}
   w_m&\le2\int_{\Delta/2}^1\Prob_\cD\left(\sup_{a\in \cA(x;\cF_{m-1})}\sup_{f\in\cF_m} \left|f(x,a)-f^*(x,a)\right|>\omega\right)d\omega\notag\\
   &\le2\left(A+\gamma_{m-1}\right)\int_{\Delta/2}^1\Prob_{x\sim\cD,a\sim p_{m-1}(\cdot \mid x)}\left(\sup_{f\in\cF_m} \left|f(x,a)-f^*(x,a)\right|>\omega\right)d\omega\notag\\
    &\le2\left(A+\gamma_{m-1}\right)\int_{\Delta/2}^1\ValueDis{\omega}{\veps_m}\frac{\varepsilon_m^2}{\omega^2}d\omega\notag\\
    &\le 2\left(A+\gamma_{m-1}\right)\ValueDis{\Delta/2}{\veps_m}\varepsilon_m^2\int_{\Delta/2}^1\frac{1}{\omega^2}d\omega\notag\\
    &=4\left(A+\gamma_{m-1}\right)\frac{\ValueDis{\Delta/2}{\veps_m}\varepsilon_m^2}{\Delta}.\notag
\end{align}
\end{proof}

We are now ready to state our instance-dependent regret bound, which is a best-of-both-worlds guarantee.

\begin{lemma}\label{lm:final-v}For any $\delta\in(T^{-2},1]$, $c>0$,  by setting $\beta_m=(M-m+1)C_\delta$ for all $m\in[M]$, \pref{alg:main} with \optiontwo
 ensures that for any instance with uniform gap $\Delta>0$,
\[
\Regbar\leq{}\bigoht(\crate+\crate^{-1})\cdot\max\left\{\sqrt{AT\log(|\cF|/\delta)},\frac{\ValueDis{\Delta/2}{\sqrt{3C_{\delta}/n_{M-1}}}A\log\abs{\cF}}{\Delta}\right\}
\]
with probability at least $1-\delta$.
\end{lemma}
\begin{proof}

Without loss of generality, we assume $T$ is large enough such that $\sqrt{AT\log(|\cF|/\delta)}\ge65\log(4M/\delta)$ for all $\delta\in(T^{-2},1]$ (otherwise, $T$ must be upper bounded by an absolute constant, and we know that $\Regbar$ is upper bounded by the same constant).

By \pref{lm:high-prob-event} and \pref{lm:width-approx}, $\cE$ and
$\cE_{\rm w}$ simultaneously hold with probability at least
$1-\delta$. In the rest of the proof, we assume that both $\cE$ and
$\cE_{\rm w}$ hold. By \pref{lm:beta-value}, the specification of
$\{\beta_m\}_{m=1}^T$ in \pref{alg:main} ensures that
$f^*\in\cF_M\subset\cdots\subset\cF_1$ whenever $\cE$ holds.

We consider two cases.

\textbf{Case 1.} First, consider the case where
\[
\frac{\ValueDis{\Delta/2}{\sqrt{3C_{\delta}/n_{M-1}}}A\log\abs{\cF}}{\Delta}< \frac{\sqrt{AT\log(|\cF|/\delta)}}{16MC_{\delta}}.
\]
Since $\beta_m=(M-m+1)C_{\delta}$ for all $m\in[M]$, we have
\[\frac{3MC_{\delta}}{n_{m-1}}\ge\frac{2\beta_m+C_{\delta}}{n_{m-1}}\ge\frac{3C_{\delta}}{n_{M-1}}\]
for all $m\in[M]$. In what follows, we prove that $\gamma_1=\cdots=\gamma_M=0$ via induction.
\begin{itemize}
    \item Base case: Since $\widehat{w}_1=1<\sqrt{AT\log(|\cF|/\delta)}/n_0$, we have $\lambda_1=0$, thus  $\gamma_1=0$.
    \item Assume that $\gamma_{m-1}=0$. Then by \pref{lm:expected-delta},
\begin{align*}
w_m&\le4(A+\gamma_{m-1})\cdot\frac{\ValueDis{\Delta/2}{\sqrt{(2\beta_m+C_{\delta})/n_{m-1}}}}{\Delta}\cdot\frac{2\beta_m+C_{\delta}}{n_{m-1}}\\
&\le 4A\cdot\frac{\ValueDis{\Delta/2}{\sqrt{3C_{\delta}/n_{M-1}}}}{\Delta}\cdot\frac{3MC_{\delta}}{n_{m-1}}\\
&<\frac{3\sqrt{AT\log(|\cF|/\delta)}}{4n_{m-1}}.
\end{align*}
If $w_m\ge 64\log(4M/\delta)/n_{m-1}$, then by \pref{eq:event-w}, 
\[
\widehat{w}_m\le\frac{4}{3}w_m<\frac{\sqrt{AT\log(|\cF|/\delta)}}{n_{m-1}},
\]
thus $\lambda_m=\1\{\widehat{w}_m\ge\sqrt{AT\log(|\cF|/\delta)}/n_{m-1}\}=0$, and $\gamma_m=0$. If $w_m<64\log(4M/\delta)/n_{m-1}$, then by \pref{eq:event-w},
\[
\widehat{w}_m<65\frac{\log(4M/\delta)}{n_{m-1}}\le\frac{\sqrt{AT\log(|\cF|/\delta)}}{n_{m-1}},
\]
thus $\lambda_m=\1\{\widehat{w}_m\ge\sqrt{AT\log(|\cF|/\delta)}/n_{m-1}\}=0$, and $\gamma_m=0$.
\item Therefore, $\gamma_1=\cdots=\gamma_M=0$.
\end{itemize}

By \pref{lm:expected-delta}, since $\gamma_1=\cdots=\gamma_M=0$, we
now have that for all $m\in\brk*{M}$,
\begin{equation*}
    w_m\le4A\cdot\frac{\ValueDis{\Delta/2}{\sqrt{3C_{\delta}/n_{M-1}}}}{\Delta}\cdot\frac{3MC_{\delta}}{n_{m-1}}.
\end{equation*}
By the definition of $w(x;\cF_m)$, we have
\begin{align*}
\Regbar&=\sum_{t=1}^{T}\En\brk*{r_t(\pistar(x_t))-r_t(a_t)\mid\gfilt_{t-1}}\\
&\le\sum_{m=1}^M\sum_{t=\tau_{m-1}+1}^{\tau_m}\En\brk*{r_t(\pistar(x_t))-r_t(a_t)\mid\gfilt_{\tau_{m-1}}}\\
&\le2\sum_{m=1}^M n_{m}w_m\\
&\le48MC_{\delta}\sum_{m=1}^M\frac{\ValueDis{\Delta/2}{\sqrt{3C_{1/T}/n_{M-1}}}A\log\abs{\cF}}{\Delta}\\
&=48M^2C_{\delta}\cdot\frac{\ValueDis{\Delta/2}{\sqrt{3C_{1/T}/n_{M-1}}}A\log\abs{\cF}}{\Delta}.
\end{align*}

\textbf{Case 2.} We now consider the case where
\[
\frac{\ValueDis{\Delta/2}{\sqrt{3C_{\delta}/n_{M-1}}}A\log\abs{\cF}}{\Delta}\ge \frac{\sqrt{AT\log(|\cF|/\delta)}}{16MC_{\delta}}.
\]
In this case, by \pref{lm:minimax-v}, we have
\begin{align*}
\Regbar&=\bigoht(\crate+\crate^{-1})\cdot\sqrt{AT\log(|\cF|/\delta)}\\
&\le\bigoht(\crate+\crate^{-1})\cdot\left(16MC_{\delta}\cdot\frac{\ValueDis{\Delta/2}{\sqrt{3C_{\delta}/n_{M-1}}}A\log\abs{\cF}}{\Delta}\right)\\
&=\bigoht(\crate+\crate^{-1})\cdot\frac{\ValueDis{\Delta/2}{\sqrt{3C_{\delta}/n_{M-1}}}A\log\abs{\cF}}{\Delta}.
\end{align*}

Combining the above two cases with \pref{lm:minimax-v}, we know that
\[
\Regbar\leq{}\bigoht(\crate+\crate^{-1})\cdot\max\left\{\sqrt{AT\log(|\cF|/\delta)},\frac{\ValueDis{\Delta/2}{\sqrt{3C_{\delta}/n_{M-1}}}A\log\abs{\cF}}{\Delta}\right\}.
\]

\end{proof}

\pref{thm:value_disagreement_ub} can be directly obtained from \pref{lm:final-v}: by taking $\delta=1/T$ and $c=1$, we have
\[
\E\left[\Reg\right]=\E\left[\Regbar\right]=\bigoht{(1)}\cdot\frac{\ValueDis{\Delta/2}{\sqrt{3C_{\delta}/n_{M-1}}}A\log\abs{\cF}}{\Delta}+\bigoh(1),
\]
where the additional $\bigoh(1)$ term comes from the fact $\Regbar$ may violate the high-probability bound in \pref{lm:final-v}
with probability at most $1/T$.

\section{Proofs for Lower Bounds}
\label{app:cb_lower}

\newcommand{\mstar}{m^{\star}}
\newcommand{\mhat}{\wh{m}}
\newcommand{\simiid}{\overset{\textrm{i.i.d.}}{\sim}}
\newcommand{\pbar}{\bar{p}}

For the proofs in this section, we let
$\gfilt_t=\sigma((x_1,a_1,\ls_1(a_1)),\ldots,(x_t,a_t,\ls_t(a_t)))$ be
the natural filtration, and define
\[
\Regbar\ldef{}\sum_{t=1}^{T}\En\brk*{r_t(\pistar(x_t))-r_t(a_t)\mid\gfilt_{t-1}}
\]
to
be the sum of conditional expectations of the instantaneous regret. We
define $p_t(x,a)$ to be the algorithm's action distribution at time
$t$ when $x_t=x$, i.e.
\begin{equation}
  \label{eq:action_dist}
  p_t(x,a) = \bbP(a_t=a\mid{}\gfilt_{t-1},x_t=x).
\end{equation}
We also define $\bar{p}=\frac{1}{T}\sum_{t=1}^{T}p_t$ to be the
average action distribution (or, the result of applying
online-to-batch-conversion to the algorithm).

For functions $p:\cX\times\cA\to\bbR$, we define $\nrm*{p}_{L_1(\cD)}=\En_{x\sim\cD}\brk*{\sum_{a\in\cA}\abs*{p(x,a)}}$. We overload this notation for policies $\pi:\cX\to\cA$ in the natural way by writing $\pi(x,a)=\indic\crl*{\pi(x)=a}$.

\subsection{Basic Technical Results}
The following lemmas are used in multiple proofs in this section.
\begin{lemma}
  \label{lem:regret_metric}
  Consider a fixed context distribution $\cD$ and Bayes reward
  function $\fstar$. Suppose that $\fstar$ has gap $\Delta$ almost
  surely. Let $p_t(x,a)$ denote the contextual bandit
  algorithm's action probability for action $a$ given context $x$ at
  time $t$ (i.e., $p_t(x,a)=\bbP(a_t=a\mid{}\gfilt_{t-1},x_t=x)$), and let
  $\bar{p}(x,a)=\frac{1}{T}\sum_{t=1}^{T}p_t(x,a)$. Then
  \[
    \Regbar \geq{}
    \frac{\Delta}{2}T\cdot\nrm*{\bar{p}-\pi_{\fstar}}_{L_1(\cD)}
    \mathand     \En\brk*{\Reg} \geq{}
    \frac{\Delta}{2}T\cdot\En\nrm*{\bar{p}-\pi_{\fstar}}_{L_1(\cD)}.
  \]
\end{lemma}
\begin{proof}[\pfref{lem:regret_metric}]
This result follows by using the uniform gap property, then repeatedly
invoking Jensen's inequality.
  \begin{align*}
    \Regbar &\geq{}
              \sum_{t=1}^{T}\Delta\En_{x\sim\cD}\En_{a\sim{}p_t(x)}\indic\crl*{a_t\neq\pi_{\fstar}(x)}\\
            &\geq{}
              \sum_{t=1}^{T}\frac{\Delta}{2}\En_{x\sim\cD}\En_{a\sim{}p_t(x)}\sum_{a\in\cA}\abs*{e_{a_t}-\pi_{\fstar}(x,a)}\\
            &\geq{}
              \sum_{t=1}^{T}\frac{\Delta}{2}\En_{x\sim\cD}\sum_{a\in\cA}\abs*{p_t(x,a)-\pi_{\fstar}(x,a)}\\
            &\geq{}
              \frac{\Delta}{2}T\En_{x\sim\cD}\sum_{a\in\cA}\abs*{\bar{p}(x,a)-\pi_{\fstar}(x,a)}\\
            &= \frac{\Delta}{2}T\nrm*{\bar{p}-\pi_{\fstar}}_{L_1(\cD)}.
  \end{align*}
\end{proof}

\begin{lemma}[Fano method with reverse KL-divergence \citep{raginsky2011lower}]
  \label{lem:fano}
Let \[\cH=(x_1,a_1,\ls_1(a_1)),\ldots,(x_T,a_T,\ls_T(a_T)),\] and let
$\crl*{\bbP\ind{i}}_{i\in\brk{M}}$ be a collection of measures over
$\cH$, where $M\geq{}2$. Let $\bbQ$ be any fixed reference measure over $\cH$, and let
$\bbP$ be the law of $(\mstar,\cH)$ under the following process:
\begin{itemize}
\item Sample $\mstar\sim\brk*{M}$ uniformly.
\item Sample $\cH\sim\bbP\ind{\mstar}$.
\end{itemize}
Then for any function $\mhat(\cH)$, if
$\bbP(\mhat=\mstar)\geq{}1-\delta$, then
\[
\frac{1}{2}\log(1/\delta) - \log{}2 \leq{} \frac{1}{M}\sum_{i=1}^{M}\kl{\bbQ}{\bbP\ind{i}}.
\]
  
\end{lemma}
\begin{proof}
This is established in \cite{raginsky2011lower}, but we re-prove the lemma in detail for completeness.

We begin with an intermediate result for general $f$-divergences. Let $\phi:[0,\infty)\to\bbR$ be any convex function for which
$\phi(1)=0$. For measures $\bbP$ and $\bbQ$
over any measurable space $\Omega$ which are dominated by a $\sigma$-finite
measure $\mu$, define
\[
  \Dphi{\bbP}{\bbQ} = \int\frac{d\bbQ}{d\mu}\phi\prn*{\frac{d\bbP/d\mu}{d\bbQ/d\mu}}d\mu,
\]
and let
\[
\dphi{p}{q} = q\phi(p/q) + (1-q)\phi((1-p)/(1-q)).
\]
Let $\bbP$ be as in the theorem statement, and let
$\bbP_0=\unif(\brk{M})\otimes{}\bbQ$. Then, with
$Z=\indic\crl{\mhat=\mstar}$,
\[
  \Dphi{\bbP}{\bbP_0} \overgeq{(i)
  } \Dphi{\bbP|_Z}{\bbP_0|_Z} \overeq{(ii)} \dphi{\bbP(Z=1)}{\bbP_0(Z=1)},
\]
where $\bbP|_Z$ is the law of $Z$ under $\bbP$. Here $(i)$ is the
standard data processing inequality for $f$-divergences (e.g.,
\cite{raginsky2011lower}) and $(ii)$ follows because $Z$ is binary.

We choose $\phi(u)=-\log(u)$. Then we have
\begin{align*}
  \dphi{\bbP(Z=1)}{\bbP_0(Z=1)} &= \dphi{\bbP(Z=1)}{1/M} \\
  &= \frac{1}{M}\phi(\bbP(Z=1)/(1/M)) +
    (1-1/M)\phi(\bbP(Z=0)/(1-1/M))\\
  &= (1-1/M)\log(1/\bbP(Z=0)) - \frac{1}{M}\log(\bbP(Z=1)\cdot{}(1/M))
    - (1-1/M)\log(1/(1-1/M))\\
                                &\overgeq{(i)}\frac{1}{2}\log(1/\bbP(Z=0)) -
                                  \frac{1}{M}\log((1/M)) -
                                  (1-1/M)\log(1/(1-1/M))\\
                                &\overgeq{(ii)}
                                  \frac{1}{2}\log(1/\bbP(Z=0)) -
                                  \log2\\
                                  &\geq{} \frac{1}{2}\log(1/\delta) - \log{}2,
\end{align*}
where $(i)$ uses that $M\geq{}2$ and that $\bbP(Z=1)\leq{}1$, and
  $(ii)$ uses that the entropy of a binary random variable is at most
  $\log{}2$.

Finally, we observe that since $\bbP(m,\cH) =
\frac{1}{M}\bbP\ind{m}(\cH)$ and $\bbP_0(m,\cH)=\frac{1}{M}\bbQ(\cH)$
\begin{align*}
  \Dphi{\bbP}{\bbP_0} = \En_{m\sim\unif(M)}\brk*{\Dphi{\bbP\ind{m}}{\bbQ}}=\frac{1}{M}\sum_{i=1}^{M}\kl{\bbQ}{\bbP\ind{i}}.
\end{align*}  
  
\end{proof}

\begin{lemma}
  \label{lem:generic_hypothesis_test}
  Let $\crl*{f\ind{i},\ldots,f\ind{m}}$ be a collection of
  regression functions, all with uniform gap at least $\Delta$, and let $\cD\in\Delta(\cX)$. Let $\bbP\ind{i}$ be the law of $\gfilt_{T}$ for an arbitrary
  contextual bandit instance in which $\cD$ is the context
  distribution and $f\ind{i}$ is the Bayes reward function, and let $\En_{i}\brk*{\cdot}$ denote the expectation under $\bbP\ind{i}$. Let $\bbQ$ be a fixed reference measure. Suppose the regression
  functions form a $2\veps$-packing with respect to $\cD$ in the sense
  that
  \begin{align}
    \nrm*{\pi_{f\ind{i}}-\pi_{f\ind{j}}}_{L_1(\cD)}\geq{}2\veps\;\;\forall{}i\neq{}j.\label{eq:fano_packing}
  \end{align}
Then any algorithm that ensures that
$\En_i\brk*{\Reg}\leq{}\frac{\veps{}\Delta{}T}{32}$ for all $i$ must
have
\begin{align}
  \log{}2 \leq{} \frac{1}{m}\sum_{i=1}^{m}\kl{\bbQ}{\bbP\ind{i}}.\label{eq:fano_kl_lb}
\end{align}

\end{lemma}
\begin{proof}
  Using \pref{lem:regret_metric}, for all choices $\fstar=f\ind{i}$,
  we have
  \begin{align*}
    \En_i\nrm*{\bar{p}-\pi_{\fstar}}_{L_1(\cD)} \leq{} \frac{2}{\Delta{}T}\En_{i}\brk*{\Reg}.
  \end{align*}
  In particular, by Markov's inequality, this implies that
  \[
\bbP_i\prn*{\nrm*{\bar{p}-\pi_{\fstar}}_{L_1(\cD)}>\veps} \leq{} \frac{2}{\Delta\veps{}T}\En_{i}\brk*{\Reg}
  \]
Hence, if $\En_{i}\brk*{\Reg}\leq{}\frac{\Delta\veps{}T}{32}$ for all
$i$, we have
$\bbP_i\prn*{\nrm*{\bar{p}-\pi_{\fstar}}_{L_1(\cD)}>\veps}\leq{}\frac{1}{2}$. Since
$\nrm[\big]{\pi_{f\ind{i}}-\pi_{f\ind{i'}}}_{L_1(\cD)}>2\veps$ for all $i\neq{}i'$,
  this implies that for any choice $\fstar=f\ind{i}$, with probability
  at least $1-1/16$ over the data generating process, 
  \[
    \hat{i}\ldef\argmin_{j}\nrm*{\bar{p}-\pi_{f\ind{j}}}
  \]
  has $\hat{i}=i$. Indeed, conditioned on the event above, we have
  $\nrm*{\bar{p}-\pi_{f\ind{i}}}\leq\veps$, and the packing property
  implies that
  \[
    \nrm[\big]{\bar{p}-\pi_{f\ind{i'}}}_{L_1(\cD)}\geq{}
    \nrm[\big]{\pi_{f\ind{i}}-\pi_{f\ind{i'}}}_{L_1(\cD)}
    -     \nrm*{\bar{p}-\pi_{f\ind{i}}}_{L_1(\cD)}
    >2\veps-    \nrm*{\bar{p}-\pi_{f\ind{i}}}_{L_1(\cD)}\geq{}\veps
  \]
  for all $i'\neq{}i$. Hence,
applying \pref{lem:fano} with $\delta=1/16$, we have that for any reference measure $\bbQ$,
\[
  \log{}2 = \frac{1}{2}\log(1/\delta)-\log{}2
  \leq{} \frac{1}{m}\sum_{i=1}^{m}\kl{\bbQ}{\bbP\ind{i}}.
\]

\end{proof}

\subsection{Proof of \cref*{thm:disagreement_lb} and
  \cref*{thm:value_disagreement_lb}}
The roadmap for this proof is as follows. First, we construct a family
of hard instances as a function of the parameters in
\pref{thm:disagreement_lb} and show that it leads to the lower bound
in terms of the \policydis. Then, at the end of the theorem, we show
how the same construction immediately implies
\pref{thm:value_disagreement_lb} using a different choice of problem parameters.
\paragraph{Construction}
Let us take $\cA=\crl*{\aone,\atwo,\ldots,a\ind{\K{}}}$ to be an
arbitrary set of discrete actions. Let $k\ldef{}\floor{\theta}\leq{}1/\veps$ and (recall that
$\theta\geq{}1$, so that $k\geq{}1$ as well) let $d$ be a parameter
of the construction to be chosen momentarily. We define $\cX\subseteq\bbN$
based on $d$ disjoint partitions
$\cX\ind{1},\ldots,\cX\ind{d}$. We set $\cX\ind{1}=\crl*{x\ind{1,0},x\ind{1,1},\ldots,x\ind{1,k}}$,
$\cX\ind{2}=\crl*{x\ind{2,0},x\ind{2,1},\ldots,x\ind{2,k}}$, and so forth,
where $\crl*{x\ind{i,j}}$ is an arbitrary collection of distinct
contexts of size $d\cdot{}(k+1)$. We set $\cX=\bigcup\cX\ind{i}$.

For each partition $\cX\ind{i}$, we take
$\Pi\ind{i}\subseteq(\cX\ind{i}\to\cA$) to be a collection of
policies $\crl*{\pi\ind{i,l,b}}$ where, for each
$l\in\crl*{1,\ldots,k}$ and
$b\in\cA_0\ldef{}\cA\setminus\crl*{\aone}$, we have $\pi\ind{i,l,b}(x\ind{i,0})=\aone$ and
\[
  \pi\ind{i,l,b}(x\ind{i,j}) = \left\{
    \begin{array}{ll}
      \aone,&\quad{}j\neq{}l,\\
      b,&\quad{}j=l,\\
    \end{array}
    \right.
\]
We also include a policy $\pi\ind{i,0}$ that always selects
$\aone$. We define $\Pi$ obtained by stitching together $\Pi\ind{1},\ldots,\Pi\ind{d}$
over their respective subsets of the domain. The resulting policy
class consists of all policies which deviate from $\aone$ on a subset
of contexts of size at most $d$, and for which this subset
intersects with each $\cX\ind{i}$ at most once.

We now choose a regression function class $\cF$ that induces $\Pi$. For each subset $\cX\ind{i}$ we define a class of regression functions
$\cF\ind{i}:\cX\ind{i}\to\brk*{0,1}$
as follows. First, we let
$f\ind{i,0}(x\ind{i,j},\cdot)=\mu_0\ldef{}(\nicefrac{1}{2}+\Delta,\nicefrac{1}{2},\ldots,\nicefrac{1}{2})$ for
all $j$. Next, for each $b\in\cA_0$ let
\[
  \mu_b = (\nicefrac{1}{2}+\Delta,\ldots,\underbrace{\nicefrac{1}{2}+2\Delta}_{\text{coordinate $b$}},
\ldots,\nicefrac{1}{2}),
\] with the $\nicefrac{1}{2}+2\Delta$ entry on
the $b$th coordinate. Next, for each $l\in\crl{1,\ldots,k}$ and
$b\in\cA_0$ we let
\[
  f\ind{i,l,b}(x\ind{i,j},\cdot) = \left\{
    \begin{array}{ll}
      (\nicefrac{1}{2}+\Delta,\nicefrac{1}{2},\ldots,\nicefrac{1}{2}),\quad{}&j\neq{}l,\\
     \mu_b,\quad{}&j=l.\\
    \end{array}
    \right.
  \]
As with $\Pi$, we obtain $\cF$ by stitching together $\cF\ind{1},\ldots,\cF\ind{d}$
over their respective subsets of the domain. It is easily verified
that $\Pi$ is precisely the set of argmax policies for $\cF$.

We define the context distribution $\cD$ as follows:
\begin{itemize}
\item Let $\cD\ind{i}$ be the distribution over $\cX\ind{i}$ which
  takes each of $x\ind{i,1},\ldots,x\ind{i,k}$ with probability
  $\veps$ and takes $x\ind{i,0}$ with probability $1-k\veps\geq{}0$.
\item Let $\cD=\frac{1}{d}\sum_{i=1}^{d}\cD\ind{i}$.
\end{itemize}

\paragraph{Verifying the problem parameters}
We now choose the parameter $d$ and verify that $\abs*{\cF}$,
$\PolicyDisA$, and $\ValueDisA$ are bounded appropriately. We first
observe that since each value function $f\in\cF$ deviates from the
vector $\mu_0$ on at
most $d$ contexts, and since it can switch to one of the $\K-1$
vectors $\crl{\mu_b}_{b\in\cA_0}$
each such context,
\[
  \abs*{\cF} \leq{} \sum_{i=0}^{d}{d(k+1)\choose i}\cdot{}(\K-1)^{i}
  \leq{}\K^{d}\cdot{}\prn*{\frac{ed(k+1)}{d}}^{d} = (e^{2}Ak)^{d}.
\]
We choose $d$ to be the largest possible value such that
$(e^{2}Ak)^{d}\leq{}F$; this is possible by the assumption that
$\theta\leq{}e^{-2}F/A$. Since $(e^{2}Ak)^{d+1}\geq{}F$, we have
\[
  d\geq{}\frac{\log{}F}{2\log(e^{2}A/\veps)}\geq{}\bigom(1)\cdot\frac{\log{}F}{\log(A/\veps)},
\]
where the last expression uses that $\veps\leq{}1$ and
$A\geq{}2$. Hence, going forward, we focus our attention to lower bounding the
regret in terms of $d$, which is equivalent to $\log{}F$ up to logarithmic factors.

To verify that the conditions of \pref{thm:disagreement_lb} hold, we show that the \policydis is
bounded by $\theta$ for our construction. Observe that for any fixed policy $\pistar\in\Pi$,
$\bbP_{\cD}(\exists{}\pi\in\Pi:\pi(x)\neq{}\pistar(x))\leq{}k\veps$,
since for any block $\cX\ind{i}$, all policies agree on
$x\ind{i,0}$. Hence, if we define $\Pi_{\veps}(\pistar)=\crl*{\pi\in\Pi:
  \bbP_{\cD}(\pi(x)\neq{}\pistar(x))\leq{}\veps}$, we have that for all $\veps'\geq{}\veps$,
\begin{align*}
  \bbP_{\cD}(\exists{}\pi\in\Pi_{\veps'}(\pistar):\pi(x)\neq{}\pistar(x))
  \leq{}\bbP_{\cD}(\exists{}\pi\in\Pi:\pi(x)\neq{}\pistar(x))
  \leq{}k\veps
  \leq{}k\veps'.
\end{align*}
It follows that $\PolicyDis{\veps}\leq{}k\leq\theta$ for any choice of
$\pistar$. Since $\theta\geq{}1$, we also have $k\geq{}\theta/2$.

\paragraph{The lower bound}

For each $i,l,b$, let $\rdist\ind{i,l,b}$ denote the reward distribution given by $\ls(a)\sim{}\Ber(f\ind{i,l,b}(x,a))$ conditioned on $x$ for
each $x\in\cX\ind{i}$. Note that $b$ is ignored if $l=0$, so in this
case we may abbreviate to $\rdist\ind{i,0}$.

For any sequence $\alpha=\alpha_1,\ldots,\alpha_d$, where
$\alpha_i=(v_i,b_i)$ for $v_i\in\crl*{0,1,\ldots,k}$ and
$b_i\in\crl*{2,\ldots,\K{}}$, we let $\bbP_{\alpha}$
denote the law of
$\cH\ldef{}(x_1,a_1,\ls_1(a_1)),\ldots,(x_T,a_T,\ls_T(a_T))$
when
the loss distribution for $\cX\ind{i}$ is given by $\rdist
\ind{i,v_i,b_i}$. We sample the problem instance $\alpha$ from a distribution $\nu$ defined as follows:
\begin{itemize}
\item For each $i$, set $v_i=0$ with probability $1/2$. Otherwise,
  select $v_i$ uniformly from $\crl*{1,\ldots,k}$. Select $b_i$ uniformly
  from $\crl*{2,\ldots,\K{}}$. 
\end{itemize}
Note that when $v_i=0$ we disregard the value of $b_i$.
Let $\pi_{\alpha}$ denote the
  optimal policy under $\alpha$, and let $\pi_{\alpha}\ind{i}$ denote its
  restriction to $\cX\ind{i}$.

Let $p_t(x,a)$ be the algorithm's action distribution at time $t$, as
in \pref{eq:action_dist}. Let $\En_{\alpha}\brk*{\cdot}$ denote the expectation under
$\bbP_{\alpha}$. Since the Bayes reward function has uniform gap
$\Delta$ for each choice of $\alpha$, \pref{lem:regret_metric} implies
that
\begin{align*}
  \En_{\alpha}\brk{\Reg} \geq{} \frac{\Delta}{2}T\cdot\En_{\alpha}\nrm*{\pbar-\pi_{\alpha}}_{L_1(\cD)},
\end{align*}
where $\bar{p}=\frac{1}{T}\sum_{t=1}^{T}p_t$. Moreover, we have
\[
  \nrm*{\pbar-\pi_{\alpha}}_{L_1(\cD)} =
  \frac{1}{d}\sum_{i=1}^{d}\nrm*{\pbar-\pi_{\alpha}\ind{i}}_{L_1(\cD\ind{i})},
\]
where $\pi_{\alpha}\ind{i}$ is the restriction of $\pi_{\alpha}$ to $\cX\ind{i}$.

Now, for each $(v,b)$ for $v\in\crl*{1,\ldots,k}$ and $b\in\cA$, let
$\bbP\ind{i,v,b}=\En_{\alpha\sim\nu}\brk*{\bbP_{\alpha}\indic\crl{v_i=v,b_i=b}}$,
and let
$\bbP\ind{i,0}\En_{\alpha\sim\nu}\brk*{\bbP_{\alpha}\indic\crl{v_i=0}}$. Then
for each $i$, the inequalities above imply that
\begin{align}
  \En_{\nu\sim\alpha}\En_{\alpha}\brk{\Reg}
&\geq{}\frac{\Delta}{2}\frac{T}{d}\cdot\sum_{i=1}^{d}\En_{\alpha\sim\nu}\En_{\alpha}\nrm*{\pbar-\pi_{\alpha}\ind{i}}_{L_1(\cD\ind{i})}\notag\\
  &=\frac{\Delta}{2}\frac{T}{d}\cdot\sum_{i=1}^{d}\prn*{\frac{1}{2}\En_{\bbP\ind{i,0}}\nrm*{\pbar-\pi\ind{i,0}}_{L_1(\cD\ind{i})}
+ \frac{1}{2k\K_0}\sum_{l=1}^{k}\sum_{b\in\cA_0}\En_{\bbP\ind{i,l,b}}\nrm*{\pbar-\pi\ind{i,l,b}}_{L_1(\cD\ind{i})}
    },\notag
\end{align}
where $\K_0\ldef{}\K-1$. In particular, we conclude that
\begin{align}
    &\En_{\nu\sim\alpha}\En_{\alpha}\brk{\Reg} \geq{}
\frac{\Delta}{4}\frac{T}{d}\cdot\sum_{i=1}^{d}\frac{1}{k\K_0}\sum_{l=1}^{k}\sum_{b\in\cA_0}\En_{\bbP\ind{i,l,b}}\nrm*{\pbar-\pi\ind{i,l,b}}_{L_1(\cD\ind{i})}.\label{eq:dis_lb_2}
\end{align}
Let $\cI\subseteq\brk*{d}$ denote the set of indices $i$ for which
\begin{equation}
  \frac{1}{k\K_0}\sum_{l=1}^{k}\sum_{b\in\cA_0}\En_{\bbP\ind{i,l,b}}\nrm*{\pbar-\pi\ind{i,l,b}}_{L_1(\cD\ind{i})}
  \leq{} \frac{\veps}{32}.\label{eq:hyp_test_condition}
\end{equation}
We consider two cases. First, if $\abs*{\cI}\leq{}d/2$, then
\pref{eq:dis_lb_2} implies that
\[
\En_{\nu\sim\alpha}\En_{\alpha}\brk{\Reg}\geq{}2^{-8}\cdot\veps{}\Delta{}T,
\]
so we are done. For the other case, we have $\abs*{\cI}\geq{}d/2$, and
we argue that the algorithm must solve a
hypothesis test for each index in this set. Let $i\in\cI$ be fixed. First, observe that for any
$(l,b)\neq(l',b')$, we have
\[
\nrm*{\pi\ind{i,l,b}-\pi\ind{i,l',b'}}_{L_1(\cD\ind{i})} \geq{} \veps
\]
Consider the Markov chain in which we draw
$(l,b)\sim\unif(\brk*{k}\times\brk*{\K_0})$, then take
$\cH\sim\bbP\ind{i,l,b}$. Letting $\bbP_i$ be the law under this
process, \pref{eq:hyp_test_condition} implies that
\[
  \bbP_i\prn*{
\nrm*{\pbar-\pi\ind{i,l,b}}_{L_1(\cD\ind{i})}>\veps/2
    } \leq{} \frac{1}{16}
  \]
  by Markov's inequality. Hence, if we define
  $(\hat{l},\hat{b})=\argmin_{l',b'}\nrm*{\pbar-\pi\ind{i,l',b'}}_{L_1(\cD\ind{i})}$,
  we have that \[\bbP_i\prn*{(\hat{i},\hat{b})=(i,b)}\geq{}1-1/16.\]
  Applying \pref{lem:fano}, this implies that that for any reference
  measure $\bbQ$ over $\cH$,
  \[
    \log{}2 \leq{} \frac{1}{k\K_0}\sum_{l=1}^{k}\sum_{b\in\K_0}\kl{\bbQ}{\bbP\ind{i,l,b}}.
  \]
  We choose $\bbQ=\bbP\ind{i,0}$. Now note that under the distributions
$\bbP\ind{i,l,b}$ and $\bbP\ind{i,0}$, the feedback the algorithm receives in a
given round is identical unless both a) $x_t=x\ind{i,l}$ and b) the
algorithm selects action $b$. Hence, using the usual likelihood
ratio argument, we have
\begin{align*}
  \kl{\bbP\ind{i,0}}{\bbP\ind{i,l,b}}
  &=
  \kl{\Ber(\nicefrac{1}{2})}{\Ber(\nicefrac{1}{2}+2\Delta)}\cdot{}\En_{\bbP\ind{i,0}}\brk*{
    \abs*{\crl*{t\mid{}x_i=l,a_t=b}}
    } \\  &\leq{}
            4\Delta^{2}\En_{\bbP\ind{i,0}}\brk*{
    \abs*{\crl*{t\mid{}x_i=l,a_t=b}}
    },
\end{align*}
where we have used \pref{lem:kl_bernoulli} and that
$\Delta\leq{}1/4$. Thus, taking the average, we have
\begin{align*}
  \log{}2 \leq{} 4\Delta^{2}\frac{1}{k\K{}_0}\sum_{l=1}^{k}\sum_{b\in\cA_0}
  \En_{\bbP\ind{i,0}}\brk*{
    \abs*{\crl*{t\mid{}x_t=x\ind{i,l},a_t=b}}
    } \\= 4\Delta^{2}\frac{1}{k\K{}_0}
  \En_{\bbP\ind{i,0}}\brk*{
    \abs*{\crl*{t\mid{}x_t\in\cX\ind{i}\setminus{}x\ind{i,0},a_t\neq\aone}}
    },
\end{align*}
or in other words,
\begin{equation}
  \label{eq:per_block_lb}
\En_{\bbP\ind{i,0}}\brk*{
    \abs*{\crl*{t\mid{}x_t\in\cX\ind{i}\setminus{}x\ind{i,0},a_t\neq\aone}}
    } \geq{} \frac{\log{}2}{4}\frac{k\K_0}{\Delta^{2}}.
  \end{equation}
  Clearly, we have
  \begin{align*}
    \En_{\alpha\sim\nu}\En_{\alpha}\brk*{\Reg}
    &\geq{} \Delta\En_{\alpha\sim\nu}\En_{\alpha}\brk*{\sum_{t=1}^{T}\sum_{i=1}^{d}\indic\crl{
    x_t\in\cX\ind{i}\setminus{}x\ind{i,0}, a_t\neq{}\aone, v_i=0
    }}\\
    &\geq{}\frac{\Delta}{2}\sum_{i=1}^{d}\En_{\bbP\ind{i,0}}\brk*{
    \abs*{\crl*{t\mid{}x_t\in\cX\ind{i}\setminus{}x\ind{i,0},a_t\neq\aone}}
    }.
  \end{align*}
  Since we have assumed that $\abs*{\cI}\geq{}\frac{d}{2}$, this
  expression combined with \pref{eq:per_block_lb} implies that
  \[
    \En_{\alpha\sim\nu}\En_{\alpha}\brk*{\Reg} \geq{}
    \frac{\log{}2}{8}\cdot{}\frac{\K_0kd}{\Delta} = \bigomt(1)\cdot{}\frac{\theta{}A\log{}F}{\Delta}.
  \]
  This proves \pref{thm:disagreement_lb}.

\paragraph{Deducing \pref{thm:value_disagreement_lb}}
To prove \pref{thm:value_disagreement_lb} with parameters $A$, $F$,
$\Delta$, $\veps$, and $\theta$,
we apply the construction
above with parameter $\veps_0\ldef{}\frac{\veps^{2}}{\Delta^{2}}$,
which is admissible for any choice of
$\theta\leq{}1/\veps_0\wedge{}e^{-2}A/F$. Since we have already shown
that this construction ensures that any algorithm has
\[
\En\brk*{\Reg}\geq{}\bigomt(1)\cdot\min\crl*{\veps_0\Delta{}T, \frac{\theta\K\log{}F}{\Delta}}
\]
for some instance, all that remains is to verify that
$\ValueDis{\Delta/2}{\veps}\leq\theta$.

For any fixed $\fstar\in\cF$, we have
\[
\bbP_{\cD,p}\prn*{
  \exists{}f\in\cF: \abs*{f(x,a)-\fstar(x,a)}>\Delta/2}\leq{}k\veps_0,
\]
since all of the value functions in $\cF$ agree on $x\ind{i,0}$ for
all $i$. Furthermore, since $\abs*{f(x,a)-\fstar(x,a)}\leq{}\Delta$
for all $f\in\cF$, we also have
\[
\bbP_{\cD,p}\prn*{
  \exists{}f\in\cF: \abs*{f(x,a)-\fstar(x,a)}>\Delta'}\leq{}0
\]
for all $\Delta'\geq{}\Delta$. It follows that
\begin{align*}
  &\sup_{\Delta'\geq{}\Delta/2,\veps'\geq{}\veps}\frac{\Delta'^{2}}{\veps'^{2}}\bbP_{\cD,p}\prn*{
  \exists{}f\in\cF:
  \abs*{f(x,a)-\fstar(x,a)}>\Delta',\;\nrm*{f-\fstar}_{\cD,p}\leq{}\veps'}\\
  &\leq
    \sup_{\Delta'\geq{}\Delta/2,\veps'\geq{}\veps}\frac{\Delta'^{2}}{\veps'^{2}}k\veps_0\indic\crl{\Delta'<\Delta}\\
  &\leq
      \frac{\Delta^{2}}{\veps^{2}}k\veps_0
  \leq{}k,
\end{align*}
so that $\ValueDis{\Delta/2}{\veps}\leq{}k\leq\theta$.

\qed
\subsection{Proof of \cref*{thm:star_lb_simple}}
\label{sec:star_lb}

Rather than proving \pref{thm:star_lb_simple} directly, in this
section we first state a more general theorem which implies it, then
prove this theorem. To state the stronger theorem, we recall the
definition of the \emph{graph dimension}, which is a multiclass
analogue of the VC dimension.
\begin{definition}[Graph dimension \citep{natarajan1989learning}]
  The graph dimension $\dgraph(\Pi,\pistar)$ is the largest number $d$
  such that there exists $S\subseteq\cX$ with $\abs*{S}=d$ such that
  for all $T\subseteq{}S$, there exists $\pi\in\Pi$ such that
  \begin{equation}
    \label{eq:graph_dim}
    \pi(x)=\pistar(x)\;\;\forall{}x\in{}T,\mathand\pi(x)\neq{}\pistar(x)\;\;\forall{}x\in{}S\setminus{}T.
  \end{equation}
\end{definition}

\begin{theorem}[Full version of \pref{thm:star_lb_simple}]
  \label{thm:star_lb}
  Let a policy class $\Pi$, $\pistar\in\Pi$, and $\Delta\in(0,1/8)$ be given. Then
there exist $\cF$, $\cD$, and $\fstar\in\cF$ with
$\pi_{\fstar}=\pistar$ such that the following properties hold:
\begin{itemize}
\item $\crl*{\pi_f\mid{}f\in\cF}\subseteq\Pi$.
  \item Any algorithm with $\En\brk*{\Reg}\leq{}
    \frac{\Delta{}T}{32\PolicyStarWL}$ for all instances realizable by $\cF$
    must have
  \begin{equation}
    \En\brk*{\Reg}\geq\frac{\PolicyStarL}{8\Delta}
  \end{equation}
  for some instance in which $\fstar$ is the Bayes reward function.
\item For any algorithm with $\En\brk*{\Reg}\leq{}
  \frac{\Delta{}T}{16}$ for all instances realizable by $\cF$, there
  exists a realizable instance for which
  \begin{equation}
    \En\brk*{\Reg}\geq\frac{\dgraph(\Pi,\pistar)}{64\Delta}.
  \end{equation}
\end{itemize}
\end{theorem}

\begin{proof}[\pfref{thm:star_lb}]
  Let $(x\ind{1},a\ind{1}),\ldots,(x\ind{N},a\ind{N})$ witness the strong star
  number, and let us abbreviate $z\ind{i}=(x\ind{i},a\ind{i})$. Let
  $x\ind{i_1},\ldots,x\ind{i_m}$ denote the unique contexts above, and
  note that $m\leq{}\PolicyStarWL$. Let
  $y\ind{1},\ldots,y\ind{d}\in\cX$ witness the graph number
  $\dgraph(\Pi,\pistar)$. 
                       We choose $\cD$ to select $x$ from
                       $\unif(x\ind{i_1},\ldots,x\ind{i_m})$ with
                       probability $1/2$ and
                       $\unif(y\ind{1},\ldots,y\ind{d})$ with
                       probability $1/2$.

                       Define
  \[
    f\ind{0}(x,a) = \left\{\begin{array}{ll}
                             \nicefrac{1}{2}+\Delta,\quad{}&a=\pistar(x),\\
                             \nicefrac{1}{2},\quad{}&a\neq\pistar(x).\\
                           \end{array}
                         \right.
                       \]
                       We choose $\fstar=f\ind{0}$ as the reference
                       regression function in the theorem statement.
                       \paragraph{Part 1: Star number}
                       Let $\pi\ind{1},\ldots,\pi\ind{N}$ be the
                       policies accompanying
                       $z\ind{1},\ldots,z\ind{N}$ that witness the strong star
                       number. For each $i$, define a regression
                       function $f\ind{i}$ to have
                       $f\ind{i}(x\ind{j},a)=f\ind{0}(x\ind{j},a)$ for all
                       $x\neq{}x\ind{i}$, and let
                       \[
                         f\ind{i}(x\ind{i},a) =
                         \left\{\begin{array}{ll}
                                  \nicefrac{1}{2}+\Delta,\quad{}&a=\pistar(x),\\
                                  \nicefrac{1}{2}+2\Delta,\quad{}&a=\pi\ind{i}(x),\\
                                  \nicefrac{1}{2},\quad{}&\text{otherwise.}\\
                                \end{array}
                              \right.
                            \]
                            For
                            $x\notin\crl*{x\ind{1},\ldots,x\ind{N}}$,
                              we simply define
                              \[
                                f\ind{i}(x,a) =
                                \left\{\begin{array}{ll}
                                         \nicefrac{1}{2}+\Delta,\quad{}&a=\pi\ind{i}(x),\\
                                         \nicefrac{1}{2},\quad{}&\text{otherwise.}\\
                                       \end{array}
                                     \right.
                            \]
                            \begin{lemma}
                              For all $i$, we have $\pi\ind{i}=\pi_{f\ind{i}}$.
                            \end{lemma}
                            \begin{proof}
                              Let $i$ be fixed. For
                              $x\notin\crl*{x\ind{1},\ldots,x\ind{N}}$
                              the result is immediate. For
                              $x\ind{j}\neq{}x\ind{i}$,
                              \pref{def:policy_star_strong} requires
                              that
                              $\pi\ind{i}(x\ind{j})=\pistar(x\ind{j})$,
                              and we have
                              $\pi_{f\ind{i}}(x\ind{j})=\pi_{\fstar}(x\ind{j})=\pistar(x\ind{j})$. Finally,
                              we have
                              $\pi_{f\ind{i}}(x\ind{i})=\pi\ind{i}(x\ind{i})$
                              by construction.
                            \end{proof}

                            Observe that all of the regression
                            functions have uniform gap $\Delta$.
                            Moreover, for all $i\neq{}j$ for
                            $i,j\geq{}1$, we have
                            \[
                              \nrm*{\pi\ind{i}-\pi\ind{j}}_{L_1(\cD)}
                              =\frac{2}{m}.
                            \]

                            For each $0\leq{}i\leq{}m$, we define an instance with
                            law $\bbP\ind{i}$ by taking $\cD$ as the
                            context distribution and
                            $\ls(a)\sim\Ber(f\ind{i}(x,a))\mid{}x$. Let
                            $\En_{i}\brk*{\cdot}$ denote the
                            expectation under this instance.
                            Then by \pref{lem:generic_hypothesis_test}, any
                            algorithm that has
                            $\En_i\brk*{\Reg}\leq{}\frac{\Delta{}T}{32m}$
                            for all $1\leq{}i\leq{}N$ must have
                            \[
                              \log{}2\leq{}\frac{1}{N}\sum_{i=1}^{N}\kl{\bbP\ind{0}}{\bbP\ind{i}}.
                            \]
                            We have
                            \begin{align*}
                              \kl{\bbP\ind{0}}{\bbP\ind{i}}
                              &\leq{} \En_{0}\brk*{\abs*{\crl{t :
                                (x_t,a_t)=(x\ind{i},a\ind{i})}}}\cdot{}\kl{\Ber(\nicefrac{1}{2}+\Delta)}{\Ber(\nicefrac{1}{2}+2\Delta)}\\
                              &\leq{}\En_{0}\brk*{\abs*{\crl{t :
                                (x_t,a_t)=(x\ind{i},a\ind{i})}}}\cdot{}4\Delta^{2},
                            \end{align*}
                            as long as $\Delta\leq{}1/8$ (by \pref{lem:kl_bernoulli}). Since the
                            tuples $(x\ind{i},a\ind{i})$ are distinct, this implies that
                            \[
                              \En_0\abs*{\crl{t:a_t\neq\pistar(x_t)}}
                              \geq{} N\frac{\log{}2}{4\Delta^{2}}.
                            \]
                            Finally, we use that since $f\ind{0}$ has
                            uniform gap $\Delta$, we have
                            \[
                              \En_0\brk*{\Reg}\geq{}\Delta
                              \En_0\abs*{\crl{t:a_t\neq\pistar(x_t)}}
                              \geq{}
                              N\frac{\log{}2}{4\Delta}\geq{}\frac{N}{8\Delta}.
                            \]

\paragraph{Part 2: Graph number}
For each $v\in\pmo^{d}$, let $\pi_v\in\Pi$ be such that
$\pi_v(y\ind{i})=\pistar(y\ind{i})$ if $v_i=1$ and
$\pi_v(y\ind{i})\neq{}\pistar(y\ind{i})$ if $v_i=-1$; these policies
are guaranteed to exist by the definition of the graph number. Let
$f\ind{0}$ be defined as before, and for each $v$ define
\[
  f_v(y\ind{i},a) = \left\{\begin{array}{ll}
                             \nicefrac{1}{2}+\Delta,\quad{}&a=\pistar(y\ind{i}),\\
                             \nicefrac{1}{2}+2\Delta,\quad{}&a=\pi_v(y\ind{i}),\quad{}v_i=-1,\\
                             \nicefrac{1}{2},\quad{}&\text{otherwise}.\\
                           \end{array}
                         \right.
                       \]
for each $i$. For
                            $x\notin\crl*{y\ind{1},\ldots,y\ind{d}}$,
                              define
                              \[
                                f_v(x,a) =
                                \left\{\begin{array}{ll}
                                         \nicefrac{1}{2}+\Delta,\quad{}&a=\pi_v(x),\\
                                         \nicefrac{1}{2},\quad{}&\text{otherwise.}\\
                                       \end{array}
                                     \right.
                            \]
For each $v$, let $\bbP_{v}$ denote the law of
                       instance with $\cD$ as the context distribution
                       and $r(a)\sim\Ber(f_v(x,a))\mid{}x$, and let
                       $\En_{v}\brk*{\cdot}$ denote the expectation under
                         this distribution.

                       Our
                       starting point is the following lemma.
                       \begin{lemma}
                         \label{lem:bandit_assouad}
                         With our choice of $\cD$, for any Bayes reward function $\fstar$ with gap
                         $\Delta$, we have
                         \[
                           \Regbar
                           \geq{}\frac{\Delta{}T}{8d}\sum_{i=1}^{d}\indic\crl*{\bar{p}(y\ind{i},\pi_{\fstar}(y\ind{i}))<1/2}.
                         \]
                       \end{lemma}
                       \begin{proof}
                         We have
                         \begin{align*}
                           \Regbar &\geq{}
\sum_{t=1}^{T}\sum_{i=1}^{d}\Pr_{\cD}(x_t=y\ind{i})\sum_{a}p_t(y\ind{i},a)(\fstar(y\ind{i},\pi_{\fstar}(y\ind{i}))-\fstar(y\ind{i},a))\\
                                   &\geq{}
                                     \frac{1}{2d}\sum_{t=1}^{T}\sum_{i=1}^{d}\sum_{a}p_t(y\ind{i},a)(\fstar(y\ind{i},\pi_{\fstar}(y\ind{i}))-\fstar(y\ind{i},a))\\
                                   &\geq{}
                                     \frac{\Delta}{4d}\sum_{t=1}^{T}\sum_{i=1}^{d}\nrm*{p_t(y\ind{i},\cdot)-\pi_{\fstar}(y\ind{i},\cdot)}_1\\
                                   &\geq{}\frac{\Delta{}T}{4d}\sum_{i=1}^{d}\nrm*{\bar{p}(y\ind{i},\cdot)-\pi_{\fstar}(y\ind{i},\cdot)}_1\\
                                   &\geq{}\frac{\Delta{}T}{8d}\sum_{i=1}^{d}\indic\crl*{\bar{p}(y\ind{i},\pi_{\fstar}(y\ind{i}))<1/2}.
                         \end{align*}
                       \end{proof}
                       Using \pref{lem:bandit_assouad}, for any $v$ we
                       have
                       \begin{align*}
                         \En_{v}\brk*{\Reg}&\geq{}
                                             \frac{\Delta{}T}{8d}\sum_{i=1}^{d}\bbP_{v}\prn*{\bar{p}_t(y\ind{i},\pi_{v}(y\ind{i}))<1/2}\\
                                           &\geq{}
                                             \frac{\Delta{}T}{8d}\sum_{i=1}^{d}\bbP_{v}\prn*{\bar{p}_t(y\ind{i},\pistar(y\ind{i}))<1/2\wedge{}v_i=+1}
                                             +\bbP_{v}\prn*{\bar{p}_t(y\ind{i},\pistar(y\ind{i}))>1/2\wedge{}v_i=-1}.
                       \end{align*}
                       In particular, suppose we sample
                       $v\sim\pmo^{d}$ uniformly at random. Then, if
                       we let
                       $\bbP_{+i} =
                       \frac{1}{2^{d}}\sum_{v\in\cV:v_i=1}\bbP_{v}$
                       and
                       $\bbP_{-i} =
                       \frac{1}{2^{d}}\sum_{v\in\cV:v_i=-1}\bbP_{v}$,
                       we have
                       \begin{align*}
                         \En\brk*{\Reg}
                         &\geq{}
                           \frac{\Delta{}T}{8d}\sum_{i=1}^{d}\bbP_{+i}\prn*{\bar{p}_t(y\ind{i},\pistar(y\ind{i}))<1/2}
                           +\bbP_{-i}\prn*{\bar{p}_t(y\ind{i},\pistar(y\ind{i}))>1/2}.
                       \end{align*}
                       By the definition of the total variation
                       distance, we can lower bound this by
                       \[
                         \frac{\Delta{}T}{8}\prn*{1-
                           \frac{1}{d}\sum_{i=1}^{d}\nrm*{\bbP_{+i}-\bbP_{-i}}_1}.
                       \]
                       In particular, suppose that
                       $\En\brk*{\Reg}\leq{}\frac{\Delta{}T}{16}$. Then,
                       rearranging, we have
                       \[
                         \frac{1}{d}\sum_{i=1}^{d}\nrm*{\bbP_{+i}-\bbP_{-i}}_1
                         \geq{}\frac{1}{2}.
                       \]
                       For each $v$, let $\bbP_{v,+i}$ and
                       $\bbP_{v,-i}$ denote the law when the
                       $i$th coordinate of $v$ is forced to $+1$ and
                       $-1$, respectively. Then we have
                       \begin{align*}
                         \frac{1}{d}\sum_{i=1}^{d}\nrm*{\bbP_{+i}-\bbP_{-i}}
                         \leq{}\prn*{\frac{1}{d}\sum_{i=1}^{d}\nrm*{\bbP_{+i}-\bbP_{-i}}^{2}_1}^{1/2}
                         &\leq{}\prn*{\frac{1}{d}\sum_{i=1}^{d}\frac{1}{2^{d}}\sum_{v\in\cV}\nrm*{\bbP_{v,+i}-\bbP_{v,-i}}^{2}_1}^{1/2}\\
                         &\leq{}2^{1/2}\prn*{\frac{1}{d}\sum_{i=1}^{d}\frac{1}{2^{d}}\sum_{v\in\cV}\kl{\bbP_{v,+i}}{\bbP_{v,-i}}}^{1/2},
                       \end{align*}
                       where the last inequality is Pinsker. Rearranging, we have
                       \[
                         \sum_{i=1}^{d}\frac{1}{2^{d}}\sum_{v\in\cV}\kl{\bbP_{v,+i}}{\bbP_{v,-i}}
                         \geq\frac{d}{8}.
                       \]
                       For any fixed $v$, we have
                       \begin{align*}
                         \kl{\bbP_{v,+i}}{\bbP_{v,-i}}
                         &\leq{} \En_{v,+i}\brk*{\abs*{\crl{t :
                           x_t = y\ind{i}, a_t\neq\pistar(y\ind{i})
                           }}}\cdot{}\kl{\Ber(\nicefrac{1}{2}+\Delta)}{\Ber(\nicefrac{1}{2}+2\Delta)}\\
                         &\leq{}\En_{v,+i}\brk*{\abs*{\crl{t :
                           x_t = y\ind{i}, a_t\neq\pistar(y\ind{i})
                           }}}\cdot{}4\Delta^{2},
                       \end{align*}
                       since $\Delta\leq1/8$ (by \pref{lem:kl_bernoulli}). As a result, we have
                       \begin{align*}
                         \sum_{i=1}^{d}\frac{1}{2^{d}}\sum_{v\in\cV}\kl{\bbP_{v,+i}}{\bbP_{v,-i}}
                         &  \leq{} \sum_{i=1}^{d}\frac{1}{2^{d}}\sum_{v\in\cV}\En_{v,+i}\brk*{\abs*{\crl{t :
                           x_t = y\ind{i}, a_t\neq\pistar(y\ind{i})
                           }}}\cdot{}4\Delta^{2}\\
                         &    \leq{} \sum_{i=1}^{d}\frac{1}{2^{d}}\sum_{v\in\cV,v_i=+1}\En_{v,+i}\brk*{\abs*{\crl{t :
                           x_t = y\ind{i}, a_t\neq\pistar(y\ind{i})
                           }}}\cdot{}8\Delta^{2}.
                       \end{align*}
                       Rearranging, this gives
                       \[
                         \sum_{i=1}^{d}\frac{1}{2^{d}}\sum_{v\in\cV,v_i=+1}\En_{v,+i}\brk*{\abs*{\crl{t
                               : x_t = y\ind{i},
                               a_t\neq\pistar(y\ind{i}) }}} \geq{}
                         \frac{d}{64\Delta^{2}}.
                       \]

                       Finally, observe that for any $v$, we have
                       \begin{align*}
                         \En_v\brk*{\Reg}\geq{}\Delta\En_v\brk*{\sum_{t=1}^{T}\indic\crl*{a_t\neq{}\pi_v(x_t)}}
                         &\geq{}\Delta\En_v\brk*{\sum_{i=1}^{d}\abs*{\crl*{t:x_t=y\ind{i},
                           a_t\neq{}\pi_v(y\ind{i})}}}\\
                         &\geq{}\Delta\En_v\brk*{\sum_{i=1}^{d}\abs*{\crl*{t:x_t=y\ind{i},
                           a_t\neq{}\pistar(y\ind{i}), v_i=+1}}}.
                       \end{align*}
                       Hence, under $v$ drawn from the uniform
                       distribution, we have
                       \[
                         \En\brk*{\Reg}= \En_{v\sim\unif}
                         \En_v\brk*{\Reg} \geq{}
                         \Delta\cdot\sum_{i=1}^{d}\frac{1}{2^{d}}\sum_{v:v_i=+1}\En_v\brk*{\sum_{i=1}^{d}\abs*{\crl*{t:
                               x_t=y\ind{i},
                               a_t\neq{}\pistar(y\ind{i})}}} \geq{}
                         \frac{d}{64\Delta}.
                       \]
                     \end{proof}

\subsection{Proof of \cref*{thm:ss_star_lb}}
  \XDeclarePairedDelimiter{\nrmld}{
  left=\lVert,
  right=\rVert,
  subscript=L_1(\cD)
}
Let $\Delta\in(0,1)$ and $\fstar\in\cF$ be given. Let $T$ be fixed and
recall that $\veps_T$ is
chosen as the largest value such that
\begin{equation}
\veps_T^{2}T\leq{}\ValueStarWL{\Delta/2}{\veps_T}.  
\end{equation}
If $\ValueStarWL{\Delta/2}{\veps_T}=0$ the theorem is trivial, so let
us consider the case where $m\equiv
\ValueStarWL{\Delta/2}{\veps_T}\geq{}1$. Let $x\ind{1},\ldots,x\ind{m}$
  and $f\ind{1},\ldots,f\ind{m}$ witness the star number, and let
  $f\ind{0}=\fstar$. Define a family of contextual bandit instances with law $\bbP\ind{i}$ for $0\leq{}i\leq{}m$
via:
\begin{itemize}
\item Take the context distribution $\cD$ to be uniform over
  $\crl*{x_1,\ldots,x_m}$.
\item Choose $\ls(a)\sim{}\cN(f\ind{i}(x,a),1)\mid{}x$.
\end{itemize}
We let $\En_i\brk*{\cdot}$ denote the expectation under
$\bbP\ind{i}$. We draw the true instance (which we denote by $0\leq{}v\leq{}m$) from a
distribution $\nu$ defined as follows: Choose $v=0$ with probability $1/2$ and $v$ uniform from $\crl*{1,\ldots,m}$
otherwise.

Let $\pi\ind{i}$ denote the
optimal policy for instance $i$. Since $\fstar$ has gap $\Delta$ for
every context, the conditions characterizing the star number ensure
the following.
\begin{lemma}
  \label{lem:weak_star_lb_conditions}
  For all $1\leq{}i\leq{}m$,
  \begin{enumerate}
  \item $f\ind{i}$ has gap $\frac{\Delta}{2}$ over $\crl*{x_1,\ldots,x_m}$.
  \item $\pi\ind{i}(x\ind{j})=\pi\ind{0}(x\ind{j})$ for all
    $i\neq{}j$.
  \item $\pi\ind{i}(x\ind{i})\neq\pi\ind{0}(x\ind{i})$.
  \end{enumerate}
\end{lemma}
\begin{proof}
The third item is immediate. For the first item, we have two
cases. First, for $x\ind{i}$, that $f\ind{i}$ has gap $\Delta/2$ is
immediate from \pref{item:weak_star1} of
\pref{def:value_star_weak}. For $j\neq{}i$, \pref{item:weak_star3}
ensures that
\begin{align*}
  f\ind{i}(x\ind{j},\pistar(x\ind{j})
  &\geq \fstar(x\ind{j},\pistar(x\ind{j}) - \veps\\
  &\geq{} \max_{a\neq{}\pistar(x\ind{j})}\fstar(x\ind{j},a) + \Delta -
    \veps
  \geq{} \max_{a\neq{}\pistar(x\ind{j})}f(x\ind{j},a) + \Delta - 2\veps.
\end{align*}
In particular, since $\veps\leq{}\Delta/4$, we have
\[
  f\ind{i}(x\ind{j},\pistar(x\ind{j})\geq{}\max_{a\neq{}\pistar(x\ind{j})}f(x\ind{j},a) + \frac{\Delta}{2}.
\]
This also implies the second item.
  
\end{proof}
Now to begin, observe that we have
\begin{equation}
\En_{v\sim\nu}\En_{v}\brk*{\Reg} \geq{}
\max\crl*{\frac{1}{2}\En_{0}\brk*{\Reg},
  \frac{1}{2m}\sum_{i=1}^{m}\En_{i}\brk*{\Reg}}.\label{eq:star_risk_lb}
\end{equation}
Let $c$ be a constant, and assume that
$\En_{v\sim\nu}\En_{v}\brk*{\Reg}\leq{}c\frac{\Delta{}T}{m}$. Then by \pref{lem:regret_metric} and
\pref{eq:star_risk_lb}, we have
\begin{align*}
\frac{\Delta}{8}T\cdot\frac{1}{m}\sum_{i=1}^{m}\En_{i}\nrmld*{\bar{p}_T-\pi\ind{i}}
  \leq{} c\frac{\Delta{}T}{m},
\end{align*}
or
\begin{align*}
\frac{1}{m}\sum_{i=1}^{m}\En_{i}\nrmld*{\bar{p}_T-\pi\ind{i}}
  \leq{} 8c\frac{1}{m}.
\end{align*}
In particular, choosing $c=\frac{1}{64}$, this implies that
\begin{align*}
\frac{1}{m}\sum_{i=1}^{m}\bbP_{i}\prn*{\nrmld*{\bar{p}_T-\pi\ind{i}}>2/m}
  \leq{} \frac{1}{16}.
\end{align*}
Since the policies form a $\frac{4}{m}$-packing in the sense that
$\nrmld*{\pi\ind{i}-\pi\ind{j}}\geq{}\frac{4}{m}$ for $i\neq{}j$, this
implies that we can identify the true policies with probability at
least $1-1/16$ over the random draw of $v$ (conditioned on $v\neq{}0$). Hence,
applying \pref{lem:fano}, we have that for any reference measure
$\bbQ$,
\[
  \log{}2 = \frac{1}{2}\log(16)-\log{}2
  \leq{} \frac{1}{m}\sum_{i=1}^{m}\kl{\bbQ}{\bbP\ind{i}}.
\]
We choose $\bbQ=\bbP\ind{0}$. Defining $N_i=\abs*{\crl*{t:
    x_t=x\ind{i},\l;a_t\neq{}\pi\ind{0}(x\ind{i})}}$ and $N=\abs*{\crl*{t:a_t\neq{}\pistar(x_t)}}$, we use
\pref{item:weak_star2} and \pref{item:weak_star3} of
\pref{def:value_star_weak} to compute
\begin{align*}
  \frac{1}{m}\sum_{i=1}^{m}\kl{\bbQ}{\bbP\ind{i}}
  \leq{}\frac{1}{m}\sum_{i=1}^{m}\frac{\Delta^{2}}{2}\En_0\brk*{N_i} +
  \frac{\veps^{2}}{2}\frac{T}{m} = \frac{\Delta^{2}}{2m}\En_0\brk*{N} + \frac{\veps^{2}}{2}\frac{T}{m}.
\end{align*}
In particular, the choice for $\veps_T$ in \pref{eq:star_fixed_point}
ensures that $\veps_T^{2}T/m\leq{}1$. Hence, rearranging, we have
\[
\En_0\brk*{N}\geq{} \frac{2m}{\Delta^2}(\log{}2-1/2)\geq\frac{m}{4\Delta^2}.
\]
Since $\En_0\brk*{\Reg}\geq{}\Delta\En_0\brk*{N_0}$, this
implies that $\En_0\brk*{\Reg}\geq{}\frac{m}{4\Delta}$.
\qed

\subsection{Proof of \cref*{thm:eluder_lb}}
Let $\Delta\in(0,1)$ and $\fstar\in\cF$ be given. We consider instances defined by a value function $f\in\cF$ and
sequence $x_1,\ldots,x_T$, in which the contexts in the sequence are presented
one-by-one non-adaptively and rewards are drawn as
$r_t(a)\sim{}\cN(f(x_t,a),1)$. For each such $(f,x_{1:T})$ pair, we let
\[
  \widebar{R}(f,x_{1:T}) = \En\brk*{\sum_{t=1}^{T}r_t(\pi_f(x_t)) -
  \sum_{t=1}^{T}r_t(a_t)\mid{}f,x_{1:T}}
\]
and
\[
  R(f,x_{1:T}) = \sum_{t=1}^{T}f(x_t,\pi_f(x_t)) -
  \sum_{t=1}^{T}f(x_t,a_t)
\]
be the sum of conditional-expected instantaneous regrets under this
process, which is a random variable.

For each $T$, we let $d_T=\ValueEluderWL{\Delta/2}{\veps_T}$, where we recall that $\veps_T$ is
chosen such that
\begin{equation}
  \label{eq:eluder_fixed_point2}
\veps_T^{2}T\leq{}\ValueEluderWL{\Delta/2}{\veps_T}.  
\end{equation}
In particular, it will be useful to note that $d_T$ is non-increasing
with $T$, so that $d_T \leq{}T$ for $T$ sufficiently large.

Fix $T$ sufficiently large such that $d\leq{}T$. Let $x\ind{1},\ldots,x\ind{d}$ and $f\ind{1},\ldots,f\ind{d}$ realize
the eluder dimension, and let $f\ind{0}=\fstar$. Let
$\pi\ind{0},\ldots,\pi\ind{d}$ be the induced policies. We have the following
result.
\begin{lemma}
  \label{lem:weak_eluder_lb_conditions}
  For all $1\leq{}i\leq{}d$,
  \begin{enumerate}
  \item $f\ind{i}$ has gap $\frac{\Delta}{2}$ over $\crl*{x_1,\ldots,x_i}$.
  \item $\pi\ind{i}(x\ind{j})=\pi\ind{0}(x\ind{j})$ for all
    $j<i$.
  \item $\pi\ind{i}(x\ind{i})\neq\pi\ind{0}(x\ind{i})$.
  \end{enumerate}
\end{lemma}
\begin{proof}
  See proof of \pref{lem:weak_star_lb_conditions}.
\end{proof}

Let $X_T=\crl*{x_1,\ldots,x_T}$ denote the sequence that plays
  $x\ind{1}$ for the first $\floor{T/d}$ rounds, $x\ind{2}$ for the second
  $\floor{T/d}$ rounds, and so forth, and choose an arbitrary fixed
  context to fill out the remaining rounds. Let $X_T\ind{i}$ denote the subsequence
  consisting of the first $i$ blocks of contexts. Let $\cI_i\subset\brk*{T}$
  denote the rounds within the $i$th block. Set $M=\floor{T/d}$.

  Let the index $i$ be fixed, and let
  $N_T\ind{i}=\abs*{\crl*{t\in\cI_i\mid{}a_t\neq{}\pistar(x\ind{i})}}$ be the
  number of times the algorithm deviates from $\pistar$ in block
  $i$ when the sequence is $X_T$. Then we have
  \[
    R(\fstar,X_T\ind{i}) \geq{} \Delta\frac{M}{2}\indic\crl*{N_T\ind{i}\geq{}M/2},
  \]
  which follows from the fact that $\fstar$ has gap $\Delta$, and
    \[
      R(f\ind{i},X_T\ind{i}) \geq{} \Delta\frac{M}{4}\indic\crl*{N_T\ind{i}<M/2},
  \]
which follows from the first part of \pref{lem:weak_eluder_lb_conditions} (i.e., $\pistar$ is
$\Delta/2$-suboptimal under $f\ind{i}$ on context $x\ind{i}$).

Let $\tau_i$ denote the last round in block $i$, and let $\bbP_{f}\ind{i}$ denote the law of
$(a_1,\ldots,a_{\tau_i}),(\ls_1(a_1),\ldots,\ls_{\tau_i}(a_{\tau_i}))$
when $f$ is the mean reward function and $X_t\ind{i}$ is the context
sequence. With $A=\crl*{N_T\ind{i}\geq{}M/2}$, the development above
implies that
\begin{align}
  \notag
  \wb{R}(\fstar,X_T\ind{i}) + \wb{R}(f\ind{i},X_T\ind{i})
  &  \geq{} \Delta\frac{M}{4}\prn*{\bbP_{\fstar}\ind{i}(A) +
    \bbP_{f\ind{i}}\ind{i}(A^{\comp})} \\
    &  \geq{} \Delta\frac{M}{8}\exp\prn*{-\kl{\bbP_{\fstar}\ind{i}}{\bbP_{f\ind{i}}\ind{i}}},\label{eq:eluder_pinsker}
\end{align}
where the last inequality is by \pref{lem:pinsker}. In particular, if
the algorithm has $\En\brk*{\Reg}\leq{}\frac{T}{64d}\Delta\leq{}\frac{M}{32}\Delta$ for all sequences
of length $T$ that are realizable by $\cF$, then this implies that
\[
\log{}2 \leq{} \kl{\bbP_{\fstar}\ind{i}}{\bbP_{f\ind{i}}\ind{i}}.
\]

Using \pref{item:weak_eluder1} and \pref{item:weak_eluder2} of \pref{def:value_eluder_weak}, as well as the fact
that rewards are Gaussian, we have
\begin{align*}
  \kl{\bbP_{\fstar}\ind{i}}{\bbP_{f\ind{i}}\ind{i}}
  &\leq{} \frac{\Delta^{2}}{2}\cdot\En_{\bbP_{\fstar}\ind{i}}\brk*{N_T\ind{i}} +
    \sum_{j<i}\frac{T}{d}\cdot\frac{1}{2}\max_{a}(f\ind{i}(x\ind{j},a)-\fstar(x\ind{j},a))^{2}\\
  &\leq{} \frac{\Delta^{2}}{2}\cdot\En_{\bbP_{\fstar}\ind{i}}\brk*{N_T\ind{i}} + \frac{T}{2d}\veps_T^{2}
\end{align*}
Our choice of $\veps_T$ ensures that $\veps_T^2T/d\leq{}1$.
It follows that
\begin{equation}
  \label{eq:eluder_kl_bound}
  \kl{\bbP_{\fstar}\ind{i}}{\bbP_{f\ind{i}}\ind{i}}
  \leq{} \frac{1}{2}\Delta^{2}\cdot\En_{\bbP_{\fstar}\ind{i}}\brk*{N_T\ind{i}} + \frac{1}{2}.
\end{equation}
Combining this inequality with \pref{eq:eluder_pinsker} and
rearranging, we get
\[
  \En_{\bbP_{\fstar}\ind{i}}\brk*{N_T\ind{i}}
  \geq{} \frac{2}{\Delta^{2}}\prn*{\log{}2-\frac{1}{2}}.
\]
Importantly, if we let $\bbP_{\fstar}$ denote the full law of
$(a_1,\ldots,a_{T}),(\ls_1(a_1),\ldots,\ls_{T}(a_{T}))$ under
$\fstar$ and $X_{T}$, this is equivalent to
\[
    \En_{\bbP_{\fstar}}\brk*{N_T\ind{i}}
  \geq{} \frac{2}{\Delta^{2}}\prn*{\log{}2-\frac{1}{2}} \geq{} \frac{1}{4\Delta^{2}}.
\]
since $N_T\ind{i}$ is a measurable function of the data up to and
including the $i$th block.
Finally, under $\fstar$, we have
\[
\wb{R}(\fstar,X_T) \geq{}
\sum_{i=1}^{d}\Delta\En_{\bbP_{\fstar}}\brk*{N_T\ind{i}} \geq{} \frac{d}{4\Delta^{2}}.
\]
\qed

\subsection{Proof of \cref*{thm:policy_eluder_lb}}
Fix $T\geq\PolicyEluderL$ and let $(x\ind{1},a\ind{1}),\ldots,(x\ind{m},a\ind{m})$ and
$\pi\ind{1},\ldots,\pi\ind{N}$ witness the \policyeluder. Define
  \[
    f\ind{0}(x,a) = \left\{\begin{array}{ll}
                             \nicefrac{1}{2}+\Delta,\quad{}&a=\pistar(x),\\
                             \nicefrac{1}{2},\quad{}&a\neq\pistar(x).\\
                           \end{array}
                         \right.
                       \]
                       We choose $\fstar=f\ind{0}$ as the reference
                       regression function in the theorem statement.
                       For each $1\leq{}i\leq{}N$, define a regression
                       function $f\ind{i}$ to have
                       $f\ind{i}(x\ind{j},a)=f\ind{0}(x\ind{j},a)$ for all
                       $j<i$ with $x\ind{j}\neq{}x\ind{i}$, and set
                       \[
                         f\ind{i}(x\ind{i},a) =
                         \left\{\begin{array}{ll}
                                  \nicefrac{1}{2}+\Delta,\quad{}&a=\pistar(x\ind{i}),\\
                                  \nicefrac{1}{2}+2\Delta,\quad{}&a=\pi\ind{i}(x\ind{i}),\\
                                  \nicefrac{1}{2},\quad{}&\text{otherwise}\\
                                \end{array}
                              \right.
                            \]
                            Finally, for all $x\notin\crl*{x\ind{1},\ldots,x\ind{i}}$, set
                              \[
                                f\ind{i}(x,a) =
                                \left\{\begin{array}{ll}
                                         \nicefrac{1}{2}+\Delta,\quad{}&a=\pi\ind{i}(x),\\
                                         \nicefrac{1}{2},\quad{}&\text{otherwise.}\\
                                       \end{array}
                                     \right.
                            \]
                            Clearly we have
                            $\pi\ind{i}=\pi_{f\ind{i}}$. We consider instances defined by a value function $f\in\cF$ and
sequence $x_1,\ldots,x_T$, in which the contexts in the sequence are presented
one-by-one non-adaptively and rewards are drawn as
$r_t(a)\sim{}\Ber(f(x_t,a))$. For each such $(f,x_{1:T})$ pair, we let
\[
  \widebar{R}(f,x_{1:T}) = \En\brk*{\sum_{t=1}^{T}r_t(\pi_f(x_t)) -
  \sum_{t=1}^{T}r_t(a_t)\mid{}f,x_{1:T}}
\]
and
\[
  R(f,x_{1:T}) = \sum_{t=1}^{T}f(x_t,\pi_f(x_t)) -
  \sum_{t=1}^{T}f(x_t,a_t),
\]
with the actions in the latter quantity tacitly understood to depend
on rewards drawn under $r_t(a)\sim{}\Ber(f(x_t,a))$.

Let $X_T=\crl*{x_1,\ldots,x_T}$ denote the sequence that plays
  $x\ind{1}$ for the first $\floor{T/m}$ rounds, $x\ind{2}$ for the second
  $\floor{T/m}$ rounds, and so forth, and choose an arbitrary fixed
  context to fill out the remaining rounds. Let $X_T\ind{i}$ denote the subsequence
  consisting of the first $i$ blocks of contexts. Let $\cI_i\subset\brk*{T}$
  denote the rounds within the $i$th block. Set $M=\floor{T/m}$.                            

                              Let the index $i$ be fixed, and let
  $N_T\ind{i}=\abs*{\crl*{t\in\cI_i\mid{}a_t\neq{}\pistar(x\ind{i})}}$ be the
  number of times the algorithm deviates from $\pistar$ in block
  $i$ when the sequence is $X_T$. Then we have
  \[
    R(\fstar,X_T\ind{i}) \geq{} \Delta\frac{M}{2}\indic\crl*{N_T\ind{i}\geq{}M/2},
  \]
and
    \[
      R(f\ind{i},X_T\ind{i}) \geq{} \Delta\frac{M}{2}\indic\crl*{N_T\ind{i}<M/2},
  \]
since both instances have uniform gap $\Delta$ over block
$i$.\footnote{Note that if $x\ind{j}=x\ind{i}$ for some $j\leq{}i$ the
  latter lower bound may be pessimistic, since the algorithm will incur regret by following $\pistar$ in block $j$ as well.}

Let $\tau_i$ denote the last round in block $i$, and let $\bbP_{f}\ind{i}$ denote the law of
$(a_1,r_1(a_1)),\ldots,(a_{\tau_i},\ls_{\tau_i}(a_{\tau_i}))$
when $f$ is the mean reward function and $X_t\ind{i}$ is the context
sequence. With $A=\crl*{N_T\ind{i}\geq{}M/2}$, the inequalities above
imply that
\begin{align}
  \notag
  \wb{R}(\fstar,X_T\ind{i}) + \wb{R}(f\ind{i},X_T\ind{i})
  &  \geq{} \Delta\frac{M}{2}\prn*{\bbP_{\fstar}\ind{i}(A) +
    \bbP_{f\ind{i}}\ind{i}(A^{\comp})} \\
    &  \geq{} \Delta\frac{M}{4}\exp\prn*{-\kl{\bbP_{\fstar}\ind{i}}{\bbP_{f\ind{i}}\ind{i}}},\notag
\end{align}
where the last inequality is by \pref{lem:pinsker}. In particular, if
the algorithm has $\En\brk*{\Reg}\leq{}\frac{\Delta{}T}{32m}\leq\frac{\Delta{}M}{16}$ for all sequences
of length $T$ that are realizable by $\cF$, then this implies that
\[
\log{}2 \leq{} \kl{\bbP_{\fstar}\ind{i}}{\bbP_{f\ind{i}}\ind{i}}.
\]
Now, since $\fstar$ and $f\ind{i}$ agree on $x\ind{j}$ for all $j<i$
with $x\ind{j}\neq{}x\ind{i}$, 
we have
\begin{align*}
  \kl{\bbP_{\fstar}\ind{i}}{\bbP_{f\ind{i}}\ind{i}}
   &\leq{}
     \kl{\Ber(\nicefrac{1}{2}+\Delta)}{\Ber(\nicefrac{1}{2}+2\Delta)}\cdot\En_{\bbP_{\fstar}\ind{i}}\brk*{\abs*{\crl*{
     t\leq{}\tau_i : x_t=x\ind{i}, a_t=a\ind{i}
     }}}\\
   &\leq{} 4\Delta^{2}\cdot
     \En_{\bbP_{\fstar}\ind{i}}\brk*{\abs*{\crl*{
     t\leq{}\tau_i : x_t=x\ind{i}, a_t=a\ind{i}
     }}}
\end{align*}
since $\Delta\leq{}1/8$.
Rearranging, we have
\[
\En_{\bbP_{\fstar}\ind{i}}\brk*{\abs*{\crl*{
     t\leq{}\tau_i : x_t=x\ind{i}, a_t=a\ind{i}
     }}}
  \geq{} \frac{1}{8\Delta^{2}}.
\]
If we let $\bbP_{\fstar}$ denote the full law of
$(a_1,\ldots,a_{T}),(\ls_1(a_1),\ldots,\ls_{T}(a_{T}))$ under
$\fstar$ and $X_{T}$, this is equivalent to
\[
  \En_{\bbP_{\fstar}}\brk*{\abs*{\crl*{
     t\leq{}\tau_i : x_t=x\ind{i}, a_t=a\ind{i}
     }}}
  \geq{} \frac{1}{8\Delta^{2}}
\]
since $\abs*{\crl*{
     t\leq{}\tau_i : x_t=x\ind{i}, a_t=a\ind{i}
     }}$ is a measurable function of the data up to and
including the $i$th block.
Finally, since this argument holds for all $1\leq{}i\leq{}m$, under
$\fstar$ we have
\[
\wb{R}(\fstar,X_T) \geq{}
\sum_{i=1}^{m}\Delta\En_{\bbP_{\fstar}}\brk*{\abs*{\crl*{
     t : x_t=x\ind{i}, a_t=a\ind{i}
     }} } \geq{} \frac{m}{8\Delta}.
\]
\qed

\section{Additional Proofs from \cref*{sec:cb}}
\subsection{Proofs for Examples}
\begin{proof}[\pfref{prop:value_disagreement_linear}]
  We first prove the bound for linear function classes
  Let $p:\cX\to\Delta(\cA)$ be fixed and define
  $\Sigma=\En_{\cD,p}\brk*{\phi(x,a)\phi(x,a)^{\trn}}$. Let
  $\cW^{\star}\ldef{}\cW-w^{\star}$, where
  $\fstar(x,a)=\tri*{w^{\star},\phi(x,a)}$. Then we have
  \begin{align*}
    &\bbP_{\cD,p}\prn*{
    \exists{}f\in\cF: \abs*{f(x,a)-\fstar(x,a)}>\Delta,\;\nrm*{f-\fstar}^{2}_{\cD,p}\leq\veps^{2}
    }\\
    &=\bbP_{\cD,p}\prn*{
    \exists{}w\in\cW^{\star}: \abs*{\tri*{w,\phi(x,a)}}>\Delta,\;\tri*{w,\Sigma{}w}\leq\veps^{2}
    }\\
    &\leq{}\bbP_{\cD,p}\prn[\bigg]{
    \sup_{w: \tri*{w,\Sigma{}w}\leq\veps^{2}}\abs*{\tri*{w,\phi(x,a)}}>\Delta
      }\\
    &=\bbP_{\cD,p}\prn[\bigg]{
      \tri*{\phi(x,a),\Sigma^{\dagger}\phi(x,a)}\geq{}\frac{\Delta^{2}}{\veps^{2}}
    },
  \end{align*}
  where we have used that $\phi(x,a)\in\mathrm{Im}(\Sigma)$ almost surely.
  By Markov's inequality, we can upper bound
  \begin{align*}
    \bbP_{\cD,p}\prn[\bigg]{
    \tri*{\phi(x,a),\Sigma^{\dagger}\phi(x,a)}\geq{}\frac{\Delta^{2}}{\veps^{2}}
    }\leq{}
     \frac{\veps^{2}}{\Delta^{2}}\cdot{}\En_{\cD,p}\brk*{\tri*{\phi(x,a),\Sigma^{\dagger}\phi(x,a)}}.
  \end{align*}
  Finally, we use that
  $\En_{\cD,p}\brk*{\tri*{\phi(x,a),\Sigma^{\dagger}\phi(x,a)}}=\mathrm{tr}(\Sigma\Sigma^{\dagger})\leq{}d$.

  To bound the \valuedis for the general case, we use that
  \begin{align*}
    &\bbP_{\cD,p}\prn*{
    \exists{}f\in\cF: \abs*{f(x,a)-\fstar(x,a)}>\Delta,\;\nrm*{f-\fstar}^{2}_{\cD,p}\leq\veps^{2}
    }\\
    &=\bbP_{\cD,p}\prn*{
    \exists{}w\in\cW: \abs*{\sigma(\tri{w,\phi(x,a)})-\sigma(\tri{w^{\star},\phi(x,a)})}>\Delta,\;\nrm*{\sigma(\tri{w,\phi(\cdot)})-\sigma(\tri{w^{\star},\phi(\cdot)})}^{2}_{\cD,p}\leq\veps^{2}
      }\\
    &\leq{}\bbP_{\cD,p}\prn*{
    \exists{}w\in\cW: c_u\abs*{\tri{w-w^{\star},\phi(x,a)}}>\Delta,\;c_{l}^{2}\nrm*{\tri*{w-w^{\star},\phi(\cdot)}}^{2}_{\cD,p}\leq\veps^{2}
      }\\
    &=\bbP_{\cD,p}\prn*{
    \exists{}w\in\cW^{\star}: \abs*{\tri*{w,\phi(x,a)}}>\Delta/c_u,\;\tri*{w,\Sigma{}w}\leq\veps^{2}/c_l^2
    }.
  \end{align*}
  From here, we proceed exactly as in the linear case to get the result.
\end{proof}

\begin{proof}[\pfref{prop:sparse_linear}]
Let $\cW^{\star}=\crl*{w\in\bbR^{d}\mid{}\nrm*{w}_0\leq{}2s}$. Then for
  any $\Delta,\veps>0$, we have
  \begin{align*}
&\bbP_{\cD,p}\prn*{
    \exists{}f\in\cF: \abs*{f(x,a)-\fstar(x,a)}>\Delta,\;\nrm*{f-\fstar}_{\cD,p}\leq\veps
                   }    \\
    &\leq{}\bbP_{\cD,p}\prn*{
    \exists{}w\in\cW^{\star}: \abs*{\tri*{w,\phi(x,a)}}>\Delta,\;\alpha\tri*{w,\Sigma^{\star}w}\leq\veps^{2}
      }    \\
    &\leq{}\bbP_{\cD,p}\prn*{
    \exists{}w\in\cW^{\star}: \abs*{\tri*{w,\phi(x,a)}}>\Delta,\;\alpha\lambda_{\mathrm{re}}\nrm*{w}_2^2\leq\veps^{2}
      }    \\
    &\leq{}\indic\crl*{
    \exists{}w\in\cW^{\star}: \nrm*{w}_1>\Delta,\;\alpha\lambda_{\mathrm{re}}\nrm*{w}_2^2\leq\veps^{2}
      }    \\
        &\leq{}\indic\crl*{
          \exists{}w\in\cW^{\star}: \sqrt{2s}\nrm*{w}_2>\Delta,\;\alpha\lambda_{\mathrm{re}}\nrm*{w}_2^2\leq\veps^{2}
          }    \\
    &\leq\indic\crl*{
      \frac{2s\veps^{2}}{\alpha\lambda_{\mathrm{re}}}>\Delta^{2}
      }\\
    &\leq{}\frac{2s}{\alpha\lambda_{\mathrm{re}}}\cdot\frac{\veps^{2}}{\Delta^{2}}.
  \end{align*}
  Since this holds for all choices of $\Delta$ and $\veps$, the result
  is established.
\end{proof}

\subsection{Proofs for Star Number Results}
\begin{proof}[\pfref{prop:value_policy_star_separation}]
Let $\Delta\in(0,2/3)$ be fixed. Let $\cX=\brk*{d}$ and $\cA=\crl*{0,1}$. Set
$\fstar(x,0)=\frac{1}{2}\Delta$ and $\fstar(x,1)=\Delta$ for
all $x$. For each $i$, define a function $f_i$ as follows.
\begin{itemize}
\item $f_i(x,1)=\Delta$ for all $x$.
\item $f_i(i,0)=\frac{3}{2}\Delta$.
\item $f_i(j,0)=0$ for all $j\neq{}i$.
\end{itemize}
Let $\cF=\crl*{\fstar,f_1,\ldots,f_d}$. Clearly we have
$\PolicyStarL=d$, since for each $i$, $\pi_{f_i}(i)\neq{}\pistar(i)$,
and $\pi_{f_i}(j)=\pistar(j)$ for all $j\neq{}i$.

Now, consider the \valuestar. Observe that for any $i$, and for any
set of points $\cI\subseteq\brk*{d}$, we have
\[
  \sum_{j\in\cI\setminus\crl*{i}}(f_i(j,0)-\fstar(j,0))^{2} \geq{} \frac{\Delta^{2}}{4}(\abs*{\cI}-1)
\]
Since any $f_i$ has $\abs*{f_i(x,a)-\fstar(x,a)}\geq\Delta$ only if $x=i$
and $a=0$, we conclude the following:
\begin{enumerate}
\item $\ValueStarCL{\Delta'}=0$ for all $\Delta'\geq{}\Delta$.
\item $\ValueStarCL{\Delta'}\leq{}5$ for all $\Delta'<\Delta$, since
  we must have $\frac{\Delta^{2}}{4}(\abs*{\cI}-1)\leq(\Delta')^{2}$
  for any set $\cI$ that witnesses the star number.
\end{enumerate}
It follows that $\ValueStarL{\Delta'}\leq{}5$ for all $\Delta'$.

\end{proof}
\subsubsection{Proof of \cref*{thm:disagreement_to_star}}
  \newcommand{\ValueStarG}[1]{\mathfrak{s}^{\val}(\cG,#1)}
  \newcommand{\ValueStarCG}[1]{\check{\mathfrak{s}}^{\val}(\cG,#1)}
  \newcommand{\ValueStarCGL}[1]{\check{\mathfrak{s}}^{\val}(\cG_{S}(\veps),#1)}
  \newcommand{\ValueDisG}[2]{\mb{\theta}_{\cP}^{\val}(\cG,#1,#2)}
    \newcommand{\dependent}{\mathsf{Dependent}}
  \newcommand{\true}{\textsc{True}}
    \newcommand{\false}{\textsc{False}}

We prove a slightly more general version of \pref{thm:disagreement_to_star}. Consider a setting in which we have a function class $\cG:\cZ\to\brk*{0,1}$ and distribution $\cP\in\Delta(\cZ)$. We introduce the following generalizations of the \valuedis and \valuestar. Define
\begin{equation}
      \ValueDisG{\Delta_0}{\veps_0} = \sup_{\Delta>\Delta_0,\veps>\veps_0}\frac{\Delta^{2}}{\veps^{2}}\bbP_{\cP}\prn*{
    \exists{}g\in\cG: g(z)>\Delta,\;\nrm*{g}^{2}_{\cP}\leq\veps^{2}
    },\label{eq:value_dis_general}
\end{equation}
where $\nrm*{g}^{2}_{\cP}=\En_{\cP}\brk{g^{2}}$.   Let $\ValueStarCG{\Delta}
  $ be the length of the longest sequence
  of points $z\ind{1},\ldots,z\ind{m}$ such that for all $i$, there exists $g\ind{i}\in\cG$
  such that
  \[
g\ind{i}(z\ind{i})>\Delta,\quad\text{and}\quad\sum_{j\neq{}i}(g\ind{i})^{2}(z\ind{j})\leq{}\Delta^{2}.
  \]
  The \valuestar is defined as $\ValueStarG{\Delta_0} =
  \sup_{\Delta>\Delta_0}\ValueStarCG{\Delta}$.
  
Our goal will be to prove the following result.
\begin{theorem}
  \label{thm:disagreement_to_star_general}
  For any \uniformgc class $\cG\subseteq(\cZ\to\brk*{0,1})$
\begin{equation}
  \label{eq:disagreement_to_star_general}
  \sup_{\cP}\sup_{\veps>0}\ValueDisG{\Delta}{\veps} \leq{} 4(\ValueStarG{\Delta})^{2},\quad\forall{}\Delta>0.
\end{equation}
\end{theorem}
This immediately implies \pref{thm:disagreement_to_star} by taking
$\cZ=\cX\times\cA$,
$\cG=\crl*{(x,a)\mapsto\abs*{f(x,a)-\fstar(x,a)}\mid{}f\in\cF}$, and
$\cP=\cD\otimes{}p$ for an arbitrary mapping
$p:\cX\to\Delta(\cA)$. Note that $\cG$ inherits the uniform
Glivenko-Cantelli property from $\cF$ by the contraction principle.

The key step toward proving \pref{thm:disagreement_to_star_general} is
to prove an analogue of the result that holds whenever $\cP$ is the
uniform distribution over a finite set of elements. For any sequence $S=\prn{z_1,\dots,z_n}$, define
\[
\cG_{S}(\varepsilon)\ldef{}\crl[\bigg]{g\in\cG : \sum_{j=1}^n g^{2}(z_j)\leq{}\veps^{2}}.
\]
Define $w_{\cG_{S}(\varepsilon)}(x) = \sup_{g\in\cG_{S}(\varepsilon)}
g(x)$. The finite-support analogue of
\pref{thm:disagreement_to_star_general} is as follows.
\begin{lemma}\label{lem:star}
  For any sequence $S=\prn{z_1,\dots,z_n}$, for any $\zeta>0$, $\veps>0$,
  \begin{align*}
    \sum_{j=1}^{n}\indic\left\{w_{\cG_{S}(\varepsilon)}(x_j)>\zeta\right\} &\leq{} \frac{\veps^2}{\zeta^2}\ValueStarCG{\zeta}^2+\frac{\veps^2}{\zeta^2}\ValueStarCG{\zeta}+\ValueStarCG{\zeta}+1\\
    &\leq{} 4\prn*{\frac{\veps^{2}}{\zeta^{2}}\vee{}1}\cdot\ValueStarG{\zeta}^{2}
  \end{align*}
\end{lemma}
Before proving this result, we show how it implies
\pref{thm:disagreement_to_star_general}.
\begin{proof}[\pfref{thm:disagreement_to_star_general}]
  Let $1>\Delta>\veps$ be fixed; the result is trivial for all other
  parameter values. We first appeal to the following lemma.
  \begin{lemma}
    \label{lem:finite_support}
    For any $\gamma\leq{}1/4$, there exists a finitely supported
    distribution $\wh{\cP}=\unif(\prn*{z_1,\ldots,z_n})$ such that
    \begin{align*}
      \bbP_{\cP}\prn*{\exists{}g\in\cG:
  g(z)>\Delta,\;\nrm*{g}^{2}_{\cP}\leq\veps^{2}}
  \leq{}       \bbP_{\wh{\cP}}\prn*{\exists{}g\in\cG:
  g(z)>\Delta,\;\nrm*{g}^{2}_{\wh{\cP}}\leq\veps^{2}+\gamma} + \gamma.
    \end{align*}
  \end{lemma}
    \begin{proof}%
    Let $\gamma>0$ be fixed. Let $\cP_n$ denote the empirical distribution formed from $n$
    independent samples from $\cP$. By Hoeffding's inequality, we are guaranteed that for $n$ 
    sufficiently large, with probability at least $1-\gamma$
    \[
      \bbP_{\cP}\prn*{\exists{}g\in\cG:
        g(z)>\Delta,\;\nrm*{g}^{2}_{\cP}\leq\veps^{2}}
      \leq{}       \bbP_{\cP_n}\prn*{\exists{}g\in\cG:
        g(z)>\Delta,\;\nrm*{g}^{2}_{\cP}\leq\veps^{2}} + \gamma.
    \]
Next, we observe that since $\cG$ has the uniform Glivenko-Cantelli
property, the class $\crl*{z\mapsto{}g^{2}(z)\mid{}g\in\cG}$ does as
well (by the contraction principle, since $\abs*{g}\leq{}1$). This
implies that for $n$ sufficiently large,
\[
\bbP\prn*{\sup_{g\in\cG}\abs*{\nrm*{g}^{2}_{\cP_n}-\nrm*{g}^{2}_{\cP}}>\gamma}\leq{}\gamma.
\]
If we take $n$ large enough so that both claims hold and take a union
bound, we are guaranteed that with probability at least $1-2\gamma$,
\begin{align*}
  \bbP_{\cP}\prn*{\exists{}g\in\cG:
  g(z)>\Delta,\;\nrm*{g}^{2}_{\cP}\leq\veps^{2}}
  \leq{}       \bbP_{\cP_n}\prn*{\exists{}g\in\cG:
  g(z)>\Delta,\;\nrm*{g}^{2}_{\cP_n}\leq\veps^{2}+\gamma} + \gamma.
\end{align*}
Since $\gamma\leq{}1/4$, this event occurs with probability at least
$1/2$. This establishes the existence of the distribution claimed in
the lemma statement
\end{proof}
Write $S=\prn*{z_1,\ldots,z_n}$, so that $\cPhat=\unif(S)$. Then we
have
\begin{align*}
\bbP_{\wh{\cP}}\prn*{\exists{}g\in\cG:
  g(z)>\Delta,\;\nrm*{g}^{2}_{\wh{\cP}}\leq\veps^{2}+\gamma}
  = \frac{1}{n}\sum_{i=1}^{n}\indic\crl*{w_{\cG_S(\veps')}(z_i)>\Delta},
\end{align*}
where $\veps'^{2}\ldef{}n(\veps^{2}+\gamma)$. Applying
\pref{lem:star}, we have
\begin{align*}
  \frac{1}{n}\sum_{i=1}^{n}\indic\crl*{w_{\cG_S(\veps')}>\Delta}
  &\leq{}
  \frac{4}{n}\prn*{\frac{\veps'^{2}}{\Delta^{2}}\vee{}1}\cdot(\ValueStarG{\Delta})^{2}\\
  &=
    \frac{4}{n}\frac{\veps'^{2}}{\Delta^{2}}\cdot(\ValueStarG{\Delta})^{2}\\
  &= 4\frac{\veps^{2}+\gamma}{\Delta^{2}}\cdot(\ValueStarG{\Delta})^{2},
\end{align*}
where we have used that $\veps'\geq{}\Delta$ by
assumption. Altogether, this implies that
\[
      \bbP_{\cP}\prn*{\exists{}g\in\cG:
  g(z)>\Delta,\;\nrm*{g}^{2}_{\cP}\leq\veps^{2}} \leq{}
4\frac{\veps^{2}+\gamma}{\Delta^{2}}\cdot(\ValueStarG{\Delta})^{2} + \gamma.
\]
Since both sides are continuous functions of $\gamma$ (in fact, the
left-hand side does not depend on $\gamma$ at all), we may take
$\gamma\to{}0$ to conclude that
\[
      \bbP_{\cP}\prn*{\exists{}g\in\cG:
  g(z)>\Delta,\;\nrm*{g}^{2}_{\cP}\leq\veps^{2}} \leq{}
4\frac{\veps^{2}}{\Delta^{2}}\cdot(\ValueStarG{\Delta})^{2}.
\]
Since this holds for all $1>\Delta>\veps$, the result is established.
\end{proof}

\begin{proof}[\pfref{lem:star}]

  We begin with a definition.
\begin{definition}
Consider a point $z$ and a sequence $A$ such that $z\notin
A$. We say $z$  is \emph{$\zeta$-star-dependent on $A$} with respect to $
\cG$ if for all $g\in\cG$ such that $\sum_{z'\in A}
g^{2}(z')\leq{}\zeta^{2}$, we have $g(z)\le \zeta$. We say that $z$ is \emph{$\zeta$-star-independent} of $A$ w.r.t. $\cG$ if $z$ is not $\zeta$-star-dependent on $A$.
\end{definition}
  
   We first claim that for any $i\in[n]$, if
  $w_{\cG_S}(z_i)>\zeta$, then $z_i$ is $\zeta$-star-dependent on
  at most $\veps^2/\zeta^2$ disjoint subsequences of
  $S$ (with respect to $\cG_S(\veps)$). Indeed, let $g$ be a function in $\cG_S(\veps)$ such that
  $g(z_i)>\zeta$. If $z_i$ is
  $\zeta$-star-dependent on a particular subsequence
  $\prn{z_{i_1},\ldots,z_{i_k}}\subset S$ but $g(z_i)>\zeta$, we must have
  \[
    \sum_{j=1}^{k}g^{2}(z_{i_j})>{}\zeta^{2}.
  \]
  If there are $N$ such disjoint sequences, we have
  \[
    N\zeta^{2}<{}\sum_{j=1}^n g^{2}(z_j)\leq{}\veps^{2},
  \]
  so $N< {\veps^{2}}/{\zeta^{2}}$.

Now we claim that for any sequence $\prn{z_1,\ldots,z_{\tau}}$, there is
some $j\in[\tau]$ such that $z_j$ is $\zeta$-star-dependent on at least
$\floor{(\tau-1)/d}/(d+1)$ disjoint subsequences of
$\prn{z_1,\ldots,z_{\tau}}$ (with respect to $\cG_S(\veps)$), where
$d\equiv\ValueStarCGL{\zeta}$. This is a straightforward
corollary of the following lemma, which is purely combinatorial.
\begin{lemma}
  \label{lem:comb}
Let $A$ be a finite set with $\tau$ elements, and let $d<\tau$ be a
positive integer. Consider any function
$\dependent:2^{A}\times{}A\to\crl*{\true,\false}$, and let us say that
$x$ is \emph{dependent} on $A'\subseteq{}A$ if $\dependent(A',x)=\true$. Suppose
$\dependent$ has the property that for every subset $A'\subseteq A$
with $|A'|>d$, there exists $x\in A'$ such that $x$ is
dependent on $A'\setminus\{x\}$ (i.e.,
$\dependent(A'\setminus\crl*{x},x)=\true$). Then there must exist an
element of $A$ that is dependent on at least $\floor{(\tau-1)/d}/(d+1)$ disjoint subsets of $A$.
\end{lemma}
Let $\prn{z_{t_1},\ldots,z_{t_{\tau}}}$ consist of all elements
of $\prn*{z_1,\ldots,z_n}$ for
which $w_{\cG_S(\veps)}(z)>\zeta$. Each element of
$\prn{z_{t_1},\ldots,z_{t_{\tau}}}$ is $\zeta$-star-dependent on at most $\veps^{2}/\zeta^{2}$ disjoint subsets of
$(z_{t_1},\ldots,z_{t_{\tau}})$, and we claim that by \pref{lem:comb}, one
element is dependent on at least $\floor{(\tau-1)/d}/(d+1)$ disjoint
subsets. This implies that $\floor{(\tau-1)/d}/(d+1)\le{}
\veps^{2}/\zeta^{2}$, so that $\tau\leq{} (\veps^{2}/\zeta^{2})(d^2+d)+d+1$.

Let us carefully verify that we can indeed apply \pref{lem:comb} here. Take
$A=\crl*{1,\ldots,\tau}$ to be the index set of
$\prn*{z_{t_1},\ldots,z_{t_{\tau}}}$, and define
$\dependent(A',i)=\true$ if $z_{t_i}$ is $\zeta$-star-dependent on $\prn*{z_{i_k}}_{k\in{}A'}$.
With $d=\ValueStarCGL{\zeta}$, any sequence of more than $d$ (potentially
non-unique) elements of $\cZ$ cannot witness the \valuestar, so for
any $A'\subseteq{}A$ with $\abs*{A'}\geq{}d$, there must at least one $i\in{}A'$ such that for all
$g\in\cG(\veps)$,
$\sum_{k\in{}A'\setminus{}\crl*{i}}g^{2}(z_{t_k})\leq{}\zeta^{2}$ implies
that $g(z_{t_i})\leq{}\zeta$. Such a $z_{t_i}$ is
$\zeta$-star-dependent on
$\prn*{z_{t_k}}_{k\in{}A'\setminus\crl*{i}}$, so $\dependent$
satisfies the condition of the lemma. Since subsequences of
$(z_{t_1},\ldots,z_{t_\tau})$ are in one-to-one correspondence with
subsets of $A$, \pref{lem:comb} grants the desired result.

\end{proof}

\begin{proof}[\pfref{lem:comb}]
We provide two different proofs: One based on the probabilistic
method, and one based on a direct counting argument. The first proof
leads to slightly worse constants.
  \paragraph{Probabilistic proof}
  Consider $A=\crl*{1,\ldots,\tau}$. Suppose we sample $A'\subseteq{}A$ with $\abs*{A'}=d+1$ uniformly at random. Then we have
\begin{align*}
  1 & \leq{}
      \En\brk*{\max_{x\in{}A'}\indic\crl*{\dependent(A'\setminus\crl{x},x)=\true}}\\
    & \leq{}
      \En\brk*{\sum_{x}\indic\crl*{x\in{}A',\dependent(A'\setminus\crl{x},x)=\true}}\\
      & =
        \sum_{x}\En\brk*{\indic\crl*{x\in{}A',\dependent(A'\setminus\crl{x},x)=\true}}\\
        & =
          \sum_{x}\bbP(x\in{}A')\bbP(\dependent(A'\setminus\crl*{x},x)=\true\mid{}x\in{}A').
\intertext{Observe that $\bbP(x\in{}A')\leq\frac{d+1}{\tau-d}$. Now,
          consider a set $A''$ with $\abs*{A''}=d$ chosen from $A\setminus\crl*{x}$ uniformly
          at random. Then we have
          $\bbP(\dependent(A'\setminus\crl*{x},x)=\true\mid{}x\in{}A')=\bbP(\dependent(A'',x)=\true)$. This
          allows us to upper bound by}
    & \leq{}
      \sum_{x}\frac{d+1}{\tau-d}\bbP(\dependent(A'',x)=\true)\\
      & \leq{}
        (d+1)\frac{\tau}{\tau-d}\cdot\max_{x}\bbP(\dependent(A'',x)=\true).
\end{align*}
It follows that there exists $x$ such that
$\bbP(\dependent(A'',x)=\true)\geq{}(1-d/\tau)\frac{1}{d+1}$. Now, let
$N=\floor{\frac{\tau-1}{d}}$, and let $X_1,\ldots,X_N$ be a collection
of disjoint subsets of $A\setminus\crl*{x}$ with $\abs*{X_i}=d$, sampled uniformly at
random. Then by symmetry, we have
\begin{align*}
  \En\brk*{\sum_{i=1}^{N}\indic\crl*{\dependent(X_i,x)=\true}} &= N\cdot{}\bbP(\dependent(A'',x)=\true)\\&\geq{}\prn*{1-\frac{d}{\tau}}\frac{N}{d+1}=\prn*{1-\frac{d}{\tau}}\frac{1}{d+1}\floor*{\frac{\tau-1}{d}}.
\end{align*}
We conclude that there exists some $x\in{}A$ and a collection
$\crl*{X_i}$ of at
least $\prn*{1-\frac{d}{\tau}}\frac{1}{d+1}\floor*{\frac{\tau-1}{d}}$ disjoint subsets of $A\setminus\crl*{x}$ such that $\dependent(X_i,x)=\true$ for
all $i$.

\paragraph{Combinatorial proof}  
This is a constructive proof  based on a counting argument, which enables us to directly finds an element in $A$ that is dependent on at least $\floor{(\tau-1)/d}/(d+1)$ disjoint subsets of $A$.

  To simplify notation, let us assign the
  elements of $A$ an arbitrary order and represent $A$ as
  $\{1,\dots,\tau\}$. 
  For any $(d+1)$-size subset $A'$ of $A$, by the property of $\dependent$, we know that $\exists x\in
A'$ such that $x$ is dependent on $A'\backslash\{x\}$, and we
define \[
  \mu(A')\ldef\min\{x\in A': \dependent(A'\setminus\{x\},x)=\true\}.
\]
Note that $\mu(A')$ is always well-defined as long as $|A'|\ge d+1$.
  
Let $N>0$ and $c\in\{0,\dots,d-1\}$ be integers
  such that $\tau-1=Nd+c$. 
  We define a \underline{$(1,N\times d,c)$-partition} of $A$ as a set-valued sequence
\[
(X_0,X_1,\dots,X_N,X_{N+1})
\]
with the property that:
\begin{enumerate}
\item $X_0,\dots,X_{N+1}\text{ are disjoint}$.
\item $\bigcup_{n=0}^{N+1}X_n=A$.
\item $|X_0|=1$.
\item $|X_1|=\cdots=|X_N|=d$.
\item $|X_{N+1}|=c$.
\end{enumerate}
Let $\textsc{Par}(A;d)$ denote the set of all possible $(1,N\times
d,c)$-partitions of $A$. Note that in particular that different
permutations of $(X_1,\dots,X_N)$ may yield different $(1,N\times d,c)$-partitions.

Now let us consider an \emph{optimal} $(1,N\times
d,c)$-partition  of $A$ such that $\sum_{n=1}^N\indic\{\mu(X_0\cup X_n)\in X_0\}$ is maximized. 
Let $(X_0^*,X_1^*,\dots,X_N^*,X_{N+1}^*)$ denote this partition (if there are multiple optimal partitions, then we just pick one of them). By the optimality of $(X_0^*,X_1^*,\dots,X_N^*,X_{N+1}^*)$, we have
\begin{equation}\label{eq:optimal-partition}
\sum_{n=1}^N\indic\{\mu(X_0^*\cup X_n^*)\in X_0^*\}\ge
\frac{1}{|\textsc{Par}(A;d)|}\sum_{(X_0,\dots,X_{N+1})\in\textsc{Par}(A;d)}\sum_{n=1}^N\indic\{\mu(X_0\cup X_n)\in X_0\},
\end{equation}
and we use $\Lambda$ to denote the right-hand side of \pref{eq:optimal-partition}. Consider the single element in $X^*_0$, denoted as $x^*$. Since $X_1^*,\dots,X_N^*$ are disjoint and
\[
\sum_{n=1}^N\1\left\{\dependent(X_n,x^*)=\true\right\}\ge\sum_{n=1}^N\indic\{\mu(X_0^*\cup X_n^*)\in X_0^*\}\ge\Lambda,
\]
we know that $x^*$ is dependent on at least $\lceil \Lambda\rceil$ disjoint subsets of $A$.

In what follows, we calculate the value of $\Lambda$.

\textbf{Step 1.} We calculate the value of $|\textsc{Par}(A;d)|$ using the following identity:
\begin{align}
  \label{eq:par_identity}
|\textsc{Par}(A;d)|={\tau\choose{1}}\cdot\prod_{n=0}^{N-1}{\tau-1-nd\choose{d}}=\tau\prod_{n=0}^{N-1}{(N-n)d+c\choose{d}}.
\end{align}
This holds by a direct counting argument: there are
${\tau\choose1}=\tau$ choices for $X_0$, ${\tau-1-d\choose d}$ choices
for $X_1$ given each such choice, ${\tau-1-2d\choose d}$ choices for
$X_2$ given the preceding two choices, all the way on to
${\tau-1-(N-1)d\choose d}$ choices for $X_{N}$; the remaining elements
must be assigned to $X_{N+1}$.

\textbf{Step 2.} We have
\begin{align}
&\sum_{(X_0,\dots,X_{N+1})\in\textsc{Par}(A;d)}\sum_{n=1}^N\indic\{\mu(X_0\cup X_n)\in X_0\}\notag\\
=&\sum_{n=1}^N\sum_{(X_0,\dots,X_{N+1})\in\textsc{Par}(A;d)}\indic\{\mu(X_0\cup X_n)\in X_0\}\notag\\
\overeq{(i)}&\sum_{n=1}^N\sum_{x\in A}\sum_{X_n\subset A:|X_n|=d,x\notin X_n}\sum_{\substack{X_1,\dots,X_{n-1},X_{n+1},\dots, X_{N+1}\\\text{s.t.} (\{x\},X_1,\dots,X_{n-1},X_n,X_{n+1},\dots,X_{N+1})\in\textsc{Par}(A;d)}}\indic\{\mu(\{x\}\cup X_n)= x\}\notag\\
\overeq{(ii)}&\sum_{n=1}^N\sum_{A'\subset A: |A'|=d+1}\sum_{x\in A'}\sum_{\substack{X_1,\dots,X_{n-1},X_{n+1},\dots, X_{N+1}\\\text{s.t.} (\{x\},X_1,\dots,X_{n-1},A'\backslash\{x\},X_{n+1},\dots,X_{N+1})\in\textsc{Par}(A;d)}}\indic\{\mu(A')=x\}\notag\\
\overeq{(iii)}&\sum_{n=1}^N\sum_{A'\subset A: |A'|=d+1}\sum_{\substack{X_1,\dots,X_{n-1},X_{n+1},\dots, X_{N+1}\\\text{s.t.} (\{\mu(A')\},X_1,\dots,X_{n-1},A'\backslash\{\mu(A')\},X_{n+1},\dots,X_{N+1})\in\textsc{Par}(A;d)}}1\notag\\
\overeq{(iv)}&N\cdot{{\tau}\choose{d+1}}\cdot\prod_{n=0}^{N-2}{\tau-d-1-nd\choose d}\notag\\
=&N\cdot{{Nd+c+1}\choose{d+1}}\cdot\prod_{n=0}^{N-2}{(N-n-1)d+c\choose d}.\label{eq:par_lower_bound}
\end{align}
Here $(i)$ rewrites the sum over partitions to make
the choices for $X_0$ and $X_n$ explicit,  $(ii)$ rewrites this
once more by considering the choice of the $X_n$ and $X_0$ as
equivalent to the choice of a set $A'$ with $\abs*{A'}=d+1$ and an
element $x\in{}A'$ (so that $X_0=\crl*{x}$ and
$X_n=A'\setminus\crl*{x}$), and $(iii)$ restricts only to elements for
which $\mu(A')=x$. They key step above is $(iv)$, which can be seen to
hold as
follows. For each choice of $n$ in the outermost sum in line $(iv)$:
\begin{itemize}
  \item There are ${\tau\choose d+1}$ choices for $A'$ in the middle sum.
  \item For each choice of $n$ and $A'$, the only constraint on
    $X_1,\ldots,X_{n-1},X_{n+1},\ldots,X_{N+1}$ is that they form a
    partition of $A\setminus{}A'$. There are ${\tau-(d+1)\choose d}$
    choices for $X_1$, ${\tau-(d+1)-d\choose d}$ choices for $X_2$
    given each such choice, and eventually ${\tau-(d+1)-(N-2)d\choose d}$
      choices for $X_{N}$ given all the preceding choices; all
      elements left over from $X_0,\ldots,X_{N}$ are assigned to $X_{N+1}$.
\end{itemize}
The final equality above simply substitutes in $\tau=Nd+c+1$.

\textbf{Step 3.} Combining \pref{eq:par_identity} and \pref{eq:par_lower_bound}, we conclude that
\begin{align*}
    \Lambda=&\frac{N\cdot{{Nd+c+1}\choose{d+1}}\cdot\prod_{n=0}^{N-2}{(N-n-1)d+c\choose d}}{|\textsc{Par}(A;d)|}\\
    =&\frac{N\cdot{{Nd+c+1}\choose{d+1}}\cdot\prod_{n=0}^{N-2}{(N-n-1)d+c\choose
       d}}{\tau\prod_{n=0}^{N-1}{(N-n)d+c\choose{d}}}\\
  =&\frac{N\cdot{{Nd+c+1}\choose{d+1}}\cdot\prod_{n=1}^{N-1}{(N-n)d+c\choose
       d}}{\tau\prod_{n=0}^{N-1}{(N-n)d+c\choose{d}}}\\
    =&\frac{N\cdot{{Nd+c+1}\choose{d+1}}}{\tau\cdot{Nd+c\choose{d}}}=\frac{N(Nd+c+1)}{\tau(d+1)}=\frac{N}{d+1}=\frac{\floor{\frac{\tau-1}{d}}}{d+1}.
\end{align*}
This implies that $x^*\in A$ is dependent on at least $\floor{(\tau-1)/d}/(d+1)$ disjoint subsets of $A$.
\end{proof}

\subsection{Proofs for Eluder Dimension Results}
\begin{proof}[\pfref{prop:eluder_star_separation}]
  Let $d\in\bbN$ and $\Delta\in(0,1)$ be fixed. Let $\cX=\brk*{d}$ and $\cA=\crl*{0,1}$. Set
$\fstar(x,0)=0$ and $\fstar(x,1)=\Delta$ for
all $x$. For each $i$, define a function $f_i$ as follows.
\begin{itemize}
\item $f_i(x,1)=\Delta$ for all $x$.
\item $f_i(j,0)=0$ for all $j<i$
\item $f_i(j,0)=\Delta$ for all $j\geq{}i$.
\end{itemize}
Let $\cF=\crl*{\fstar,f_1,\ldots,f_d}$. We have $\ValueEluderL{\Delta/2}\geq{}d$ by taking $(1,0),\ldots,(d,0)$ and $f_1,\ldots,f_d$ as witnesses, since for each $i$, $\abs*{f_i(i,0)-\fstar(i,0)}=\Delta>\Delta/2$ and
\[
\sum_{j<i}\prn*{f_i(j,0) - \fstar(j,0)}^{2} = 0.
\]
We now upper bound the \valuestar. Clearly $\ValueStarCL{\Delta'}=0$ for any $\Delta'\geq{}\Delta$, so consider a fixed scale parameter $\Delta'<\Delta$. Suppose we have a set of points $(i_1,0),\ldots,(i_m,0)$ and functions $f_{j_1},\ldots,f_{j_m}$ that witness $\ValueStarCL{\Delta'}$, with $i_1<i_2<\ldots,i_m$ (we must have $0$ as the action for each witness, since all functions agree on the value for action $1$). Since $\abs*{f_{j_1}(i_1,0)-\fstar(i_1,0)}>\Delta'$, we must have $j_m\leq{}i_m$. But on the other hand, we have
\[
\sum_{l>1}^{m}\prn*{f_{j_1}(i_{l},0)-\fstar(i_l,0)}^{2}=\Delta^{2}(m-1),
\]
since $f_{j_1}(i_l,0)=\Delta$ for all $l>2$. Since we need $\Delta^{2}(m-1)\leq{}(\Delta')^2$, we must have $m\leq{}2$, so we conclude that $\ValueStarCL{\Delta'}\leq{}2$ for all $\Delta'<\Delta$.

\end{proof}

\subsubsection{Proof of \pref{prop:eluder_refined}}
  First, we recall the definition of the general function class UCB algorithm. Let $z_t=(x_t,a_t)$ and $\cZ_t=\crl*{z_1,\ldots,z_t}$. Define $\nrm*{f}^{2}_{\cZ}=\sum_{z\in\cZ}f^{2}(z)$. Then the algorithm is defined as follows. At round $t$:
  \begin{itemize}
  \item Set $\fhat_t=\argmin_{f\in\cF}\sum_{i<t}\prn*{f(x_i,a_i)-r_i(a_i)}^{2}$.
  \item Define $\cF_t=\crl[\big]{f\in\cF:\nrm{f-\fhat_t}_{\cZ_{t-1}}\leq{}\beta_t}$.
  \item Choose $a_t=\argmax_{a\in\cA}\max_{f\in\cF_t}f(x_t,a)$.
  \end{itemize}
From \cite{russo2013eluder}, Proposition 2, we are guaranteed that if $\beta_t=\beta\ldef{}\sqrt{C_1\cdot{}\log(\abs{\cF}/\delta)}$ for all $t$ for some absolute constant $C_1$, then $\fstar\in\cF_t$ for all $t$ with probability at least $1-\delta$. Let $\brk*{a}_{\Delta}=a\indic\crl*{a>\Delta}$. Conditioned on this event, and using that $\fstar$ has gap $\Delta$, we have
\begin{align*}
  \sum_{t=1}^{T}\fstar(x_t,\pistar(x_t)) - \fstar(x_t,a_t)
  & =   \sum_{t=1}^{T}\brk*{\fstar(x_t,\pistar(x_t)) - \fstar(x_t,a_t)}_{\Delta/2}\\
  & \leq{}   \sum_{t=1}^{T}\sup_{f\in\cF_t}\brk*{f(x_t,\pistar(x_t)) - \fstar(x_t,a_t)}_{\Delta/2}\\
  & \leq{}   \sum_{t=1}^{T}\sup_{f\in\cF_t}\brk*{f(x_t,a_t) - \fstar(x_t,a_t)}_{\Delta/2}.
\end{align*}
In particular, let us define $\cFbar_t=\crl*{f\in\cF: \nrm*{f-\fstar}_{\cZ_{t-1}}\leq{}2\beta_t}$. Then by triangle inequality, $\cF_t\subseteq\cFbar_t$, so if we define $w_t(z)=\sup_{f\in\cFbar_t}\brk{f(z)-\fstar(z)}$, then
\begin{align*}
  \sum_{t=1}^{T}\fstar(x_t,\pistar(x_t)) - \fstar(x_t,a_t)
  \leq{}   \sum_{t=1}^{T}w_t(z_t)\indic\crl*{w_t(z_t)>\Delta/2}.
\end{align*}
We now appeal to the following lemma.
\begin{lemma}[Variant of \cite{russo2013eluder}, Lemma 3]
  \label{lem:eluder_indicator_bound}
  For any $\zeta>0$,
  \[
    \sum_{t=1}^{T}\indic\crl{w_{t}(z_t)>\zeta} \leq{} \prn*{\frac{4\beta^2}{\zeta^2}+1}\ValueEluderCL{\zeta}.
  \]
\end{lemma}
To apply this result, let us order the indices such that $w_{i_1}(z_{i_1})\geq{}w_{i_2}(z_{i_2})\geq\ldots\geq{}w_{i_T}(z_{i_T})$. Consider any index $t$ for which $w_{i_t}(z_{i_t})>\Delta/2$. For any particular $\zeta>\Delta/2$, if we have $w_{i_t}(z_{i_t})>\zeta$, then \pref{lem:eluder_indicator_bound} (since $\zeta\leq{}1\leq\beta$) implies that
\begin{equation}
  \label{eq:wt_eluder_bound}
t \leq{} \sum_{t=1}^{T}\indic\crl{w_{t}(z_t)>\zeta} \leq{} \frac{5\beta^2}{\zeta^2}\ValueEluderCL{\zeta}.
\end{equation}
Since we have restricted to $\zeta\geq\Delta/2$, rearranging yields
\[
  w_{i_t}(z_{i_t}) \leq{} \sqrt{\frac{5\beta^{2}\ValueEluderL{\Delta/2}}{t}}.
\]
Now, let $T_0$ be the greatest index $t$ such that $w_{i_t}(z_{i_t})>\Delta/2$. Then we have
\begin{align*}
  \sum_{t=1}^{T}w_t(z_t)\indic\crl*{w_t(z_t)>\Delta/2} &=  \sum_{t=1}^{T_0}w_t(z_t)\indic\crl*{w_t(z_t)>\Delta/2}\\
                                                       &\leq{}  \sum_{t=1}^{T_0}\sqrt{\frac{5\beta^{2}\ValueEluderL{\Delta/2}}{t}}\\
                                                       &\leq{}  \sqrt{5\beta^{2}\ValueEluderL{\Delta/2}T_0}.
\end{align*}
We know that from \pref{eq:wt_eluder_bound} that $T_0\leq{}\frac{20\beta^2}{\Delta^2}\ValueEluderCL{\Delta/2}$, so altogether we have
\[
\sum_{t=1}^{T}w_t(z_t)\indic\crl*{w_t(z_t)>\Delta/2}\leq{}100\frac{\beta^2}{\Delta}\ValueEluderL{\Delta/2}.
\]

To conclude, we set $\delta=1/T$, and the final result follows from the law of total expectation.
\qed

\begin{proof}[\pfref{lem:eluder_indicator_bound}]
Let us adopt the shorthand $d=\ValueEluderCL{\zeta}$. We begin with a definition. We say $z$ is $\zeta$-independent of $z_1,\ldots,z_t$
  if there exists $f\in\cF$ such that
  $\abs*{f(z)-\fstar(z)}>\zeta$ and
  $\sum_{i=1}^{t}\prn*{f(z_i)-\fstar(z_i)}^{2}\leq\zeta^{2}$. We say
  $z$ is $\zeta$-dependent on $z_1,\ldots,z_t$ if for all $f\in\cF$ with
  $\sum_{i=1}^{t}\prn*{f(z_i)-\fstar(z_i)}^{2}\leq\zeta^{2}$, $\abs*{f(z)-\fstar(z)}\leq{}\zeta$.

  We first claim that for any $t$, if
  $w_{t}(z_t)>\zeta$, then $z_t$ is $\zeta$-dependent on
  at most $4\beta^2/\zeta^2$ disjoint subsequences of
  $z_1,\ldots,z_{t-1}$. Indeed, let $f$ be such that
  $\abs*{f(z_t)-\fstar(z_t)}>\zeta$. If $z_t$ is
  $\zeta$-dependent on a particular subsequence
  $z_{i_1},\ldots,z_{i_k}$ but $w_{t}(z_t)>\zeta$, we must have
  \[
    \sum_{j=1}^{k}(f(z_{i_j})-\fstar(z_{i_j}))^{2}\geq{}\zeta^{2}.
  \]
  If there are $M$ such disjoint sequences, we have
  \[
    M\zeta^{2}\leq{}\nrm*{f-\fstar}_{\cZ_{t-1}}^{2}\leq{}4\beta^2,
  \]
  so $M\leq{} \frac{4\beta^{2}}{\zeta^{2}}$.

Next we claim that for $\tau$ and any sequence $(z_1,\ldots,z_{\tau})$, there is
some $j$ such that $z_j$ is $\zeta$ dependent on at least
$\floor{\tau/d}$ disjoint subsequences of $z_1,\ldots,z_{j-1}$. Let
$N=\floor{\tau/d}$, and let $B_1,\ldots,B_N$ be subsequences of
$z_1,\ldots,z_{\tau}$. We initialize with $B_i = (z_i)$. If $z_{N+1}$ is
$\zeta$-dependent on $B_i=\prn*{z_i}$ for all $1\leq{}i\leq{}N$ we are done. Otherwise,
choose $i$ such that $z_{N+1}$ is $\zeta$-independent of
$B_i$, and add it to $B_i$. Repeat this process until we reach $j$
such that either $z_j$ is $\zeta$-dependent on all $B_i$ or 
$j=\tau$. In the first case we are done, while in the second case, we
have $\sum_{i=1}^{N}\abs*{B_i}\geq{}\tau\geq{}dN$. Moreover,
$\abs*{B_i}\leq{}d$, since each $z_j\in{}B_i$ is
$\zeta$-independent of its prefix. We conclude that
$\abs*{B_i}=d$ for all $i$, so in this case $z_{\tau}$ is $\zeta$-dependent
on all $B_i$.

Finally, let $(z_{t_1},\ldots,z_{t_{\tau}})$ be the subsequence $z_1,\ldots,z_T$ consisting of all elements for
which $w_{i_i}(z_{t_i})>\zeta$. Each element of the sequence is
dependent on at most $4\beta^{2}/\zeta^{2}$ disjoint subsequences of
$(z_{t_1},\ldots,z_{t_{\tau}})$, and by the argument above, one
element is dependent on at least $\floor{\tau/d}$ disjoint
subsequences, so we must have $\floor{\tau/d}\leq{}
4\beta^{2}/\zeta^{2}$, and in particular $\tau\leq{} (4\beta^{2}/\zeta^{2}+1)d$.
\end{proof}

\subsubsection{Proof of \pref{thm:disagreement_to_eluder}}
  \newcommand{\ValueEluderG}[1]{\mathfrak{e}^{\val}(\cG,#1)}
  \newcommand{\ValueEluderCG}[1]{\check{\mathfrak{e}}^{\val}(\cG,#1)}
This proof closely follows that of \pref{thm:disagreement_to_star}. As
with that theorem, we prove a slightly more general result. Let
$\cG\subseteq(\cZ\to\brk*{0,1})$ be a function class, and let
$\ValueDisG{\Delta}{\veps}$ be defined as in
\pref{eq:value_dis_general}. Let $\ValueEluderCG{\Delta}
  $ be the length of the longest sequence
  of points $z\ind{1},\ldots,z\ind{m}$ such that for all $i$, there exists $g\ind{i}\in\cG$
  such that
  \[
g\ind{i}(z\ind{i})>\Delta,\quad\text{and}\quad\sum_{j<i}(g\ind{i})^{2}(z\ind{j})\leq{}\Delta^{2}.
  \]
  The \valueeluder for $\cG$ is defined as $\ValueEluderG{\Delta_0} =
  \sup_{\Delta>\Delta_0}\ValueEluderCG{\Delta}$.
  
We will prove the following result.
\begin{theorem}
  \label{thm:disagreement_to_value_general}
  For any \uniformgc class $\cG\subseteq(\cZ\to\brk*{0,1})$
\begin{equation}
  \label{eq:disagreement_to_value_general}
  \sup_{\cP}\sup_{\veps>0}\ValueDisG{\Delta}{\veps} \leq{} 4\ValueEluderG{\Delta},\quad\forall{}\Delta>0.
\end{equation}
\end{theorem}
Let $\Delta,\veps>0$ be given. Let $\gamma\leq{}1/4$ be fixed. Then by
\pref{lem:finite_support}. There exists a distribution
$\cPhat=\unif(z_1,\ldots,z_n)$ such that
\[
      \bbP_{\cP}\prn*{\exists{}g\in\cG:
  g(z)>\Delta,\;\nrm*{g}^{2}_{\cP}\leq\veps^{2}}
  \leq{}       \bbP_{\wh{\cP}}\prn*{\exists{}g\in\cG:
  g(z)>\Delta,\;\nrm*{g}^{2}_{\wh{\cP}}\leq\veps^{2}+\gamma} + \gamma.
\]
Write $S=\prn*{z_1,\ldots,z_n}$, so that $\cPhat=\unif(S)$. Define
$\cG_{S}(\varepsilon)\ldef{}\crl*{g\in\cG : \sum_{j=1}^n
  g^{2}(z_j)\leq{}\veps^{2}}$. Then we
have
\begin{align*}
\bbP_{\wh{\cP}}\prn*{\exists{}g\in\cG:
  g(z)>\Delta,\;\nrm*{g}^{2}_{\wh{\cP}}\leq\veps^{2}+\gamma}
  = \frac{1}{n}\sum_{i=1}^{n}\indic\crl*{w_{\cG_S(\veps')}(z_i)>\Delta},
\end{align*}
where $\veps'^{2}\ldef{}n(\veps^{2}+\gamma)$. By
\pref{lem:eluder_indicator_bound}, we can bound
\begin{align*}
  \frac{1}{n}\sum_{i=1}^{n}\indic\crl*{w_{\cG_S(\veps')}(z_i)>\Delta}
  \leq{}
  \frac{1}{n}\prn*{\frac{4(\veps')^2}{\Delta^2}+1}\ValueEluderCG{\Delta}.
\end{align*}
To conclude the result, we take $\gamma\to{}0$ (and consequently
$n\to\infty$), so that the bound above yields
\[
      \bbP_{\cP}\prn*{\exists{}g\in\cG:
  g(z)>\Delta,\;\nrm*{g}^{2}_{\cP}\leq\veps^{2}} \leq{} \frac{4\veps^{2}}{\Delta^{2}}\cdot\ValueEluderCG{\Delta},
\]
as desired.

\qed

\part{Proofs for Reinforcement Learning Results}
\label{part:rl}
\newcommand{\thetaval}{\btheta^{\val}}
\newcommand{\thetacheck}{\check{\btheta}}

\section{Proof of  \cref*{thm:block_mdp}}
\label{app:rl_main}
We let $\tau\ind{k,h}$ denote the $h$th trajectory gathered by the
algorithm during iteration $k$ (i.e., the trajectory obtained by
rolling in to layer $h$ with $\pi\ind{k}$, then switching to uniform
exploration. Throughout the proof, we let $s_1\ind{k,h},\ldots,s_H\ind{k,h}$ denote
the latent states encountered during $\tau\ind{k,h}$, which emphasize are not
observed.

We define
$\cL_h\ind{k}=\crl*{s_h\ind{1,h},\ldots,s_h\ind{k-1,h}}$. For any
collection $\cL\subseteq\cS$, we define
\[
\nrm*{f-f'}_{\cL}^{2} = \sum_{s\in\cL}\En_{x\sim\emi(s),a\sim\piunif}\prn*{f(x,a)-f'(x,a)}^{2}.
\]

We also let $\cH\ind{k}=\crl*{(s_1\ind{k,h},x_1\ind{k,h},
  a_1\ind{k,h}, r_1\ind{k,h}), \ldots, (s_H\ind{k,h},x_H\ind{k,h},
  a_H\ind{k,h}, r_H\ind{k,h})}_{h=1}^{H}$ denote the entire history
for iteration $k$.

\begin{comment}
  Let us define an intermediate quantity which is closely related to
  the \valuedis, which we will work with throughout the proof. For
  each $s\in\cS_h$, define
  \begin{align*}
    &\thetacheck_s(\cF_h,\Delta_0,\veps_0) \\&=
\sup_{\fstar\in\cF_h}\sup_{\Delta>\Delta_0,\veps\geq{}\veps_0}\frac{\En_{x\sim\emi(s)}\En_{a\sim\piunif}\sup\crl*{\brk{f(x,a)-\fstar(x,a)}_{\Delta}^{2}
    : f\in\starhull(\cF_h,\fstar), \nrm*{f-\fstar}_s\leq\veps}}{\veps^{2}},
  \end{align*}
  where $\brk*{x}_{\Delta}\ldef{}x\indic\crl*{\abs*{x}>\Delta}$. We
  define
  $\thetacheck_{h}(\cF_h,\Delta,\veps)=\sum_{s\in\cS_h}\thetacheck_s(\cF_h,\Delta,\veps)$
  analogously. We will pass from this quantity to $\thetaval$ at the
  end of the proof.
\end{comment}
  Let us define an intermediate quantity which is closely related to
  the \valuedis, which we will work with throughout the proof. For
  each $s\in\cS_h$, define
  \begin{align*}
    &\thetacheck_s(\cF_h,\veps_0) \\&=
    1\vee\sup_{\fstar\in\cF_h}\sup_{\veps\geq{}\veps_0}\frac{\En_{x\sim\emi(s)}\En_{a\sim\piunif}\sup\crl*{\abs{f(x,a)-\fstar(x,a)}^{2}
    : f\in\starhull(\cF_h,\fstar), \nrm*{f-\fstar}_s\leq\veps}}{\veps^{2}},
  \end{align*}
  We
  define
  $\thetacheck_{h}(\cF_h\veps)=\sum_{s\in\cS_h}\thetacheck_s(\Delta,\veps)$
  analogously. Lastly, we abbreviate
  $\thetacheck_s(\veps)\equiv\thetacheck_s(\cF_h,\veps)$ and
  $\thetacheck_h(\veps)\equiv\thetacheck_h(\cF_h,\veps)$. We will pass from this quantity to $\thetaval$ at the
  end of the proof.

\subsection{Confidence Sets}
Let \[\cFhat_h\ind{k}\ldef{}\crl[\bigg]{f\in\starhull(\cF_h,\fhat_h\ind{k}),
  \nrm[\big]{f-\fhat_h\ind{k}}_{\cZ_h\ind{k}}\leq\beta_h}\] be the set used
to compute the upper confidence function $\Qbar_h\ind{k}$ in iteration
$k$. Let the Bayes predictor
for this round be defined as
\[
  \fbar_h\ind{k}(x,a) =\fbayes(x,a) + \brk*{\Pstar_h\Vbar_{h+1}\ind{k}}(x,a).
\]
Recall that the optimistic completeness assumption implies that $\fbar_h\ind{k}\in\cF_h$.
\begin{theorem}
  \label{thm:confidence_radius}
  For any $\delta\in(0,1)$, if we choose $\beta_1,\ldots,\beta_H$ in
  \pref{alg:blockalg} such that
  \begin{align*}
  &\beta^2_{H}=400H^{2}\log\prn{\Fmax{}KH\delta^{-1}}\\
  &\beta^{2}_{h} = \frac{1}{2}\beta_{h+1}^{2}+
  150^2H^{2}A^{2}\thetacheck_{h+1}(\beta_{h+1}K^{-1/2})^2\log(2\Fmax{}KH\delta^{-1}) +
  700H^{2}S\log(2eK),
  \end{align*}
then with probability at least $1-3\delta$,
$\fbar_h\ind{k}\in\cFhat_h\ind{k}$
for all $k$, $h$.
\end{theorem}
Let $\Econf$ denote the event from \pref{thm:confidence_radius}. This
event has the following consequence.
\begin{lemma}
  \label{lem:optimistic}
  Whenever $\Econf$ holds, we have
  \begin{align}
    \label{eq:optimistic}
    \Qbar_h\ind{k}(x,a)\geq{}\Qstar_h(x,a)
  \end{align}
  for all $h, k$.
\end{lemma}
\begin{proof}[\pfref{lem:optimistic}]
  Let $k$ be fixed. We prove the result by induction on $h$. First,
  the property holds trivially for round $H+1$, since $\Vbar_{H+1}\ind{K}=\Vstar_{H+1}=0$.

Now, consider a fixed timestep $h$, and suppose inductively that the
property holds for round $h+1$. Then we have
\begin{align*}
\Qbar_h\ind{k}(x,a) =
\sup_{f\in\cFhat_h\ind{k}}f(x,a)\geq{}\fbar_h\ind{k}(x,a)
&= \fbayes(x,a) +\brk*{\Pstar{}\Vbar_{h+1}\ind{k}}(x,a)\\
  &\geq{} \fbayes(x,a) +\brk*{\Pstar{}\Vstar_{h+1}}(x,a)
    = \Qstar_{h}(x,a),
\end{align*}
where we have used that $\Vbar_{h+1}\ind{k}(x)\geq{}\Vstar_{h+1}(x)$
whenever $\Qbar_{h+1}\ind{k}(x,a)\geq{}\Qstar_{h+1}(x,a)$ for all $x,a$.
\end{proof}
\subsection{Bounding Regret}
We now use the concentration guarantees established above to bound the
regret of the policies $\pi\ind{1},\ldots,\pi\ind{K}$. That is, we
wish to bound
\[
\sum_{k=1}^{K}\Vstar-\Vf^{\pi\ind{k}}
\]
Note that since
our algorithm executes policies besides these ones, this is not a
bound on the true regret of the algorithm, but rather for an
intermediate quantity which is only used for the analysis.

We first state a regret decomposition for $Q$-functions and induced policies that are
optimistic in the following sense.
\begin{definition}
  A $Q$-function $\Qbar$ and policy $\pi$ are said to be optimistic if
  for all $h\in\brk*{H}$,
  \[
    \Qbar_h(x,a)\geq{}\Qstar_h(x,a)\;\;\forall{}x,a\mathand\pi(x)\in\argmax_{a\in\cA}\Qbar_h(x,a)\;\;\forall{}x\in\cX_h.
  \]
  We define $\Vbar_h(x)=\max_a\Qbar_h(x,a)$ as the induced value function.
\end{definition}
Define
$\clip\brk*{x\mid{}\veps}=x\indic\crl*{x\geq{}\veps}$. The following
lemma bounds the regret of any optimistic pair $(\Qbar,\pi)$ in terms of
clipped surpluses for $\Qbar$.
\begin{lemma}[\cite{simchowitz2019non}, Theorem B.3]
  \label{lem:gap_decomp_simple}
Let $(\Qbar,\pi)$ be optimistic, and let $\Ebar_h(x,a)=\Qbar_h(x,a) -
\prn*{\fbayes(x,a) + \brk{\Pstar_h\Vbar_{h+1}}(x,a)}$ be the optimistic
surplus. Then
\begin{align}
  \Vstar - \Vf^{\pi}\leq{}2e\sum_{h=1}^{H}\En_{\pi}\brk*{\clip\brk*{
  \Ebar_h(x_h,a_h)\mid{}\gapcheck(x_h,a_h)}
  },
\end{align}
where $\gapcheck(x,a)\ldef{}\gap(x,a)/4H$.
\end{lemma}
Define
\[
\Ebar_h\ind{k}(x,a)=\Qbar_h\ind{k}(x,a) - \prn*{\fbayes(x,a) + \brk*{\Pstar_h\Vbar_{h+1}\ind{k}}(x,a)}.
\]
Then \pref{lem:optimistic} and \pref{lem:gap_decomp_simple} imply that
conditioned on $\Econf$, we have
\[
\Vstar-\Vf^{\pi\ind{k}}\leq{}2e\sum_{h=1}^{H}\En_{\pi\ind{k}}\brk*{\clip\brk*{
  \Ebar_h\ind{k}(x_h,a_h)\mid{}\gapcheck(x_h,a_h)}
  }
\]
for all $k$.

Since $\clip\brk*{x\mid\veps}\leq{}\frac{x^{2}}{\veps}$, we can
further upper bound as
\begin{align*}
\Vstar-\Vf^{\pi\ind{k}}\leq{}2e\sum_{h=1}^{H}\En_{\pi\ind{k}}\brk*{\frac{\prn*{\Ebar_h\ind{k}(x_h,a_h)}^{2}}{\gapcheck(x_h,a_h)}}
  =8eH\sum_{h=1}^{H}\En_{\pi\ind{k}}\brk*{\frac{\prn*{\Ebar_h\ind{k}(x_h,a_h)}^{2}}{\gap(x_h,a_h)}}.
\end{align*}
Let $p_h\ind{k}(s)=\bbP_{\pi\ind{k}}(s_h=s)$. Then we can write
\begin{align*}
  \En_{\pi\ind{k}}\brk*{\frac{\prn*{\Ebar_h\ind{k}(x_h,a_h)}^{2}}{\gap(x_h,a_h)}}
  &=
    \sum_{s\in\cS_h}p_h\ind{k}(s)\En_{x_h\sim\emi(s)}\brk*{\frac{\prn*{\Ebar_h\ind{k}(x_h,\pi\ind{k}(x_h))}^{2}}{\gap(x_h,\pi\ind{k}(x_h))}}\\
  &\leq{}
    A\sum_{s\in\cS_h}p_h\ind{k}(s)\En_{x_h\sim\emi(s)}\En_{a_h\sim\piunif}\brk*{\frac{\prn*{\Ebar_h\ind{k}(x_h,a_h)}^{2}}{\gap(x_h,a_h)}}\\
  &\leq{} A\sum_{s\in\cS_h}\frac{p_h\ind{k}(s)}{\gap(s)}\En_{x_h\sim\emi(s)}\En_{a_h\sim\piunif}\brk*{\prn*{\Ebar_h\ind{k}(x_h,a_h)}^2}.
\end{align*}

Altogether, we have
\begin{align}
\Vstar-\Vf^{\pi\ind{k}} \leq{} 8eHA\sum_{h=1}^{H}\sum_{s\in\cS_h}\frac{1}{\gap(s)}p_h\ind{k}(s) \En_{x_h\sim\emi(s)}\En_{a_h\sim\piunif}\brk*{\prn*{\Ebar_h\ind{k}(x_h,a_h)}^2}.\label{eq:per_episode_bound}
\end{align}

\subsection{Bounding the Surplus}
We now focus on bounding the surplus terms
\[
\En_{x_h\sim\emi(s)}\En_{a_h\sim\piunif}\brk*{\prn*{\Ebar_h\ind{k}(x_h,a_h)}^2}.
\]
Let $k$, $h$, and $s\in\cS_h$ be fixed. Then we have
\begin{align}
  &\En_{x_h\sim\emi(s)}\En_{a_h\sim\piunif}\brk*{\prn*{\Ebar_h\ind{k}(x_h,a_h)}^2}\notag\\
  &\leq
\En_{x_h\sim\emi(s)}\En_{a_h\sim\piunif}\sup\crl*{\prn*{f(x_h,a_h)-\fbar_h\ind{k}(x_h,a_h)}^2\mid{}f\in\cFhat_h\ind{k}}\notag\\
    &\leq
2\En_{x_h\sim\emi(s)}\En_{a_h\sim\piunif}\sup\crl*{\prn*{f(x_h,a_h)-\fhat_h\ind{k}(x_h,a_h)}^2\mid{}f\in\cFhat_h\ind{k}}\label{eq:ebound_1}\\
  &~~~~+2\En_{x_h\sim\emi(s)}\En_{a_h\sim\piunif}\brk*{\prn*{\fhat\ind{k}(x_h,a_h)-\fbar_h\ind{k}(x_h,a_h)}^2}.\label{eq:ebound_2}
\end{align}
Now, for each $k$, define
\[
\nrm*{f-f'}_{k,h}^{2} = \sum_{j<k}\En_{x_h\sim\pi\ind{j}}\En_{a_h\sim\piunif}\prn*{f(x_h,a_h)-f'(x_h,a_h)}^{2}.
\]
We appeal to the following uniform concentration guarantee.
\begin{lemma}
  \label{lem:uniform_conc2}
  With probability at least $1-\delta$, for all $k,h$, we have
  \[
    \nrm*{f-f'}^2_{k,h} \leq{}
2\nrm*{f-f'}_{\cZ_h\ind{k}}^{2}+8H^{2}\log(2KH\abs{\cF_h}\delta^{-1}),\quad\text{for
      all } f'\in\cF_h, f\in\starhull(\cF_h,f').
  \]
\end{lemma}
\newcommand{\Econc}{\cE_{\mathrm{conc}}}
Let $\Econc$ denote the event in \pref{lem:uniform_conc2}. Since
$\beta^{2}_h\geq{}8H^{2}\log(2KH\abs{\cF_h}\delta^{-1})$ for all $h$, we
have
\[
\sup_{f\in\cFhat_h\ind{k}}\nrm[\big]{f-\fhat_h\ind{k}}^{2}_{k,h}\leq{}3\beta_h^{2}
\]
conditioned on $\Econc$.

Now, we can write
\begin{align*}
  \nrm[\big]{f-\fhat_h\ind{k}}^{2}_{k,h}
  &=
  \sum_{j<k}\En_{x_h\sim\pi\ind{j}}\En_{a_h\sim\piunif}\prn*{f(x_h,a_h)-\fhat_h\ind{k}(x_h,a_h)}^{2}\\
  &= \sum_{s\in\cS_h}\prn[\Bigg]{\sum_{j<k}p\ind{j}_h(s)}\En_{x_h\sim\emi(s)}\En_{a_h\sim\piunif}\prn*{f(x_h,a_h)-\fhat_h\ind{k}(x_h,a_h)}^{2}.
\end{align*}
In particular, if we define $q\ind{k}_h(s) = \sum_{j<k}p_h\ind{j}(s)$,
we are guaranteed that for all $f\in\cFhat_h\ind{k}$,
\[
\En_{x_h\sim\emi(s)}\En_{a_h\sim\piunif}\prn*{f(x_h,a_h)-\fhat_h\ind{k}(x_h,a_h)}^{2} \leq\frac{3\beta_h^{2}}{q_h\ind{k}(s)}\quad\forall{}s\in\cS_h.
\]
Since $\fbar_h\ind{k}\in\cFhat_h\ind{k}$, this immediately allows us
to bound \pref{eq:ebound_2} by
$\frac{6\beta_h^{2}}{q_h\ind{k}(s)}$. To handle \pref{eq:ebound_1} we
use the definition of the disagreement coefficient, which gives us
that
\begin{align*}
  \En_{x_h\sim\emi(s)}\En_{a_h\sim\piunif}\sup\crl*{\prn*{f(x_h,a_h)-\fhat_h\ind{k}(x_h,a_h)}^2\mid{}f\in\cFhat_h\ind{k}}
  &\leq{}\thetacheck_s\prn*{3\beta_h^2/(q_h\ind{k}(s)}^{1/2})\cdot\frac{3\beta_h^{2}}{q_h\ind{k}(s)}.
  \\
  &\leq{}\thetacheck_s(\beta_hK^{-1/2})\cdot\frac{3\beta_h^{2}}{q_h\ind{k}(s)},
\end{align*}
where we have used that $\thetacheck_s(\cdot)$ is non-increasing and $q_h\ind{k}\leq{}K$. Since
$\theta\geq{}1$, we conclude that for all $s$,
\begin{align}
  \En_{x_h\sim\emi(s)}\En_{a_h\sim\piunif}\brk*{\prn*{\Ebar_h\ind{k}(x_h,a_h)}^2}
  \leq{} 12 \thetacheck_s(\beta_hK^{-1/2})\cdot\frac{\beta_h^{2}}{q_h\ind{k}(s)}.\label{eq:surplus_bound}
\end{align}
\subsection{Final Regret Bound}
Condition on $\Econf$ and $\Econc$, which occur together with
probability at least $1-4\delta$. To bound the total regret across all
episodes, we first apply \pref{eq:per_episode_bound} to give
\begin{align*}
\sum_{k=1}^{K}\Vstar-\Vf^{\pi\ind{k}} \leq{}
  8eHA\sum_{h=1}^{H}\sum_{s\in\cS_h}\frac{1}{\gap(s)}\sum_{k=1}^{K} p_h\ind{k}(s)
  \En_{x_h\sim\emi(s)}\En_{a_h\sim\piunif}\brk*{\prn*{\Ebar_h\ind{k}(x_h,a_h)}^2}.
\end{align*}
Now, consider a fixed state $s\in\cS_h$ and let
$k_s=\min\crl{k:q_h\ind{k}(s)\geq{}1}$. Then we have
$q_h\ind{k}(s)\leq{}2$, so we can bound
\begin{align*}
  \sum_{k=1}^{K} p_h\ind{k}(s)
  \En_{x_h\sim\emi(s)}\En_{a_h\sim\piunif}\brk*{\prn*{\Ebar_h\ind{k}(x_h,a_h)}^2}\leq
   \sum_{k=k_s}^{K} p_h\ind{k}(s)
  \En_{x_h\sim\emi(s)}\En_{a_h\sim\piunif}\brk*{\prn*{\Ebar_h\ind{k}(x_h,a_h)}^2}
  + 2H^{2}.
\end{align*}
We apply \pref{eq:surplus_bound} to each term in the sum to bound by
\begin{align*}
   \sum_{k=k_s}^{K} p_h\ind{k}(s)
  \En_{x_h\sim\emi(s)}\En_{a_h\sim\piunif}\brk*{\prn*{\Ebar_h\ind{k}(x_h,a_h)}^2}
  \leq{}    12\beta_h^{2}\thetacheck_s(\beta_hK^{-1/2})\sum_{k=k_s}^{K} \frac{p_h\ind{k}(s)}{q_h\ind{k}(s)}
\end{align*}
Observe that $\sum_{k=k_s}^{K}
\frac{p_h\ind{k}(s)}{q_h\ind{k}(s)}=\sum_{k=k_s}^{K}
\frac{q_h\ind{k+1}(s)-q_h\ind{k}(s)}{q_h\ind{k}(s)}$ and
$q_h\ind{k_s}(s)\geq{}1$. We appeal to the following lemma.
\begin{lemma}
  \label{lem:integral_bound}
  For any sequence $1\leq{}x_1,\leq{},\ldots,\leq{}x_{N+1}$ with
  $\abs*{x_i-x_{i+1}}\leq{}1$,  $\sum_{i=1}^{N}\frac{x_{i+1}-x_i}{x_i}\leq{}2\log(x_{N+1}/x_1)$.
\end{lemma}
\pref{lem:integral_bound} grants that $\sum_{k=k_s}^{K}
\frac{p_h\ind{k}(s)}{q_h\ind{k}(s)}\leq{}2\log(K)$. Altogether, since
$\beta_h\geq{}H^{2}$ and $\theta\geq{}1$, we have
\[
  \sum_{k=1}^{K} p_h\ind{k}(s)
  \En_{x_h\sim\emi(s)}\En_{a_h\sim\piunif}\brk*{\prn*{\Ebar_h\ind{k}(x_h,a_h)}^2}
  \leq{} 14\beta_h^{2}\thetacheck_s(\beta_hK^{-1/2})\log{}(K).
\]
Summing across all states, we have
\begin{align*}
  \sum_{k=1}^{K}\Vstar-\Vf^{\pi\ind{k}} &\leq{}310HA\log{}K\sum_{h=1}^{H}\beta_h^{2}\sum_{s\in\cS_h}\frac{\thetacheck_s(\beta_hK^{-1/2})}{\gap(s)}\\
  &\leq{}
  310HA\log{}K\max_{h}\beta_h^{2}\sum_{h=1}^{H}\sum_{s\in\cS}\frac{\thetacheck_s(\beta_hK^{-1/2})}{\gap(s)}.
\end{align*}
To simplify further, we use that for all $h$,
\begin{align*}
  \beta^{2}_{h}
  &= \bigoh\prn*{H^{2}A^{2}\max_{h'}\thetacheck_{h'}(\beta_{h'}K^{-1/2})^2\log(\Fmax{}HK\delta^{-1}) +
    H^{2}S\log(K)
    }.
\end{align*}
This gives
\begin{align*}
  \sum_{k=1}^{K}\Vstar-\Vf^{\pi\ind{k}}   &\leq{}
                                            \bigoht\prn*{
                                            H^{3}A^{3}\max_{h}\thetacheck_h(\beta_hK^{-1/2})^2\log(\Fmax{})
                                            +
                                            H^{3}AS}\cdot{}\sum_{h=1}^{H}\sum_{s\in\cS_h}\frac{\thetacheck_s(\beta_hK^{-1/2})}{\gap(s)}\\
                                          &=
                                            \bigoht\prn*{
                                            H^{2}A^{3}\max_{h}\thetacheck_h(\beta_hK^{-1/2})^2\log\abs*{\cF}
                                            + H^{3}AS}\cdot{}\sum_{h=1}^{H}\sum_{s\in\cS_h}\frac{\thetacheck_s(\beta_hK^{-1/2})}{\gap(s)},
\end{align*}
where have we used that $\log\abs*{\cF}=H\log(\Fmax)$. To deduce the
in-expectation error bound, we set $\delta=1/KH$, so that
\begin{align*}
\En\brk*{  \sum_{k=1}^{K}\Vstar-\Vf^{\pi\ind{k}}}=
                                            \bigoht\prn*{
                                            H^{2}A^{3}\max_{h}\thetacheck_h(\beta_hK^{-1/2})^2\log\abs*{\cF}
                                            + H^{3}AS}\cdot{}\sum_{h=1}^{H}\sum_{s\in\cS_h}\frac{\thetacheck_s(\beta_hK^{-1/2})}{\gap(s)},
\end{align*}
where we have used that the regret in each episode is bounded by
$H$. Finally, we observe that by dividing both sides above by $K$,
this is equivalent to
\begin{align*}
\En\brk*{\Vstar-\Vf^{\pi}}=
                                            \bigoht\prn*{
  \frac{H^{2}A^{3}\max_{h}\thetacheck_h(\beta_hK^{-1/2})^2\log\abs*{\cF}
                                            + H^{3}AS}{K}}\cdot{}\sum_{h=1}^{H}\sum_{s\in\cS_h}\frac{\thetacheck_s(\beta_hK^{-1/2})}{\gap(s)}.
\end{align*}
We now move to the disagreement coefficient defined in \pref{eq:rl_disagreement}.
\begin{lemma}
  \label{lem:disagreement_star}
  Let $\cG:\cZ\to\brk*{0,R}$ be any function class and
  $\cP\in\Delta(\cZ)$. Let $\nrm*{g}_{\cP}^{2}=\En_{\cP}\brk{g^{2}}$,
  and define
  \[
    \mb{\theta}_{\cD;\gstar}(\cG,\veps) =1\vee\sup_{\veps\geq{}\veps_0}\frac{\En_{\cP}\sup\crl*{(g(z)-\gstar(z))^{2}:g\in\cG,\En_{\cP}\brk{(g(z)-\gstar(z))^{2}}\leq\veps^{2}}}{\veps^{2}}.
  \]
  Then for all $\veps\in(0,R]$, we have
  $\mb{\theta}_{\cD;\gstar}(\starhull(\cG,\gstar),\veps)\leq{}\mb{\theta}_{\cD;\gstar}(\cG,\veps)(16\log(R/\veps)
  + 8)$.
\end{lemma}
\pref{lem:disagreement_star} implies that we have
$\thetacheck_s(\cF_h,\veps)\leq{}\thetaval_s(\cF_h,\veps)(16\log(H/\veps)+8)$. We
conclude the proof by noting that the value of $\beta_h$ in
\pref{alg:blockalg} is simply the recursion in
\pref{thm:confidence_radius} with this upper bound substituted in, using the upper bound $16\log(HK^{1/2}/\beta_h)+8\leq{}16\log(HK^{1/2}e)+8\leq{}24\log(HKe)$ recursively to simplify.

    \subsection{Deferred Proofs}

    \begin{proof}[\pfref{lem:uniform_conc2}]
  Let $k$ and $h$ be fixed. Let $\cH\ind{j}$ denote the entire history for episode $j$. Define a
  filtration
    \[
      \gfilt_{j-1}=\sigma(\cH\ind{1},\ldots,\cH\ind{j-1}),
    \]
    with the convention $\gfilt_{0}=\emptyset$. This filtration
    guarantees that $(x_j\ind{j,h},a_j\ind{j,h})$ is
    $\gfilt_j$-measurable and $\gfilt_{j-1}\subseteq\gfilt_j$.

    Let
    $f,f'\in\cF_h$ be fixed, and let
    $X_j=(f(x_j\ind{j,h},a_j\ind{j,h})-f'(x_j\ind{j,h},a_j\ind{j,h}))^{2}$. Observe
    that $\abs*{Z_j}\leq{}H^{2}$. Applying \pref{lem:regret_freedman} to the process $\prn*{X_j}$, we are guaranteed that 
    $1-\delta$,
    \begin{align*}
      \sum_{j=1}^{k-1}\En_{j-1}\brk*{X_j} &\leq{}
2\sum_{j=1}^{k-1}\En_{j-1}\brk*{X_j} +
                                            8H^2\log(2\delta^{-1}).
    \end{align*}
    We have
    \[
      \sum_{j=1}^{k-1}\En_{j-1}\brk*{X_j}
      =
      \sum_{j=1}^{k-1}\En_{x_h\sim{}\pi\ind{j},a_h\sim\piunif}\brk*{\prn*{f(x_h,a_h)-f'(x_h,a_h)}^{2}}
      = \nrm*{f-f'}_{k,h}^{2}.
    \]
    By taking a union bound, we are guaranteed that 
    \begin{align}
      \label{eq:uniform_conc3}
      \nrm*{f-f'}_{k,h}^2 \leq{}
      2\nrm*{f-f'}_{\cZ_h\ind{k}}^{2} +
      8H^2\log(2\abs*{\cF_h}\delta^{-1})
  \end{align}
for all $f,f'\in\cF_h$. We now deduce the result for the star hull by
homogeneity. Consider $f''\in\starhull(\cF_h,f)$, where $f\in\cF_h$. We can
write $f''=t(f'-f) + f$ for some $f'\in\cF_h$ and $t\in\brk*{0,1}$, so
that $f''-f=t(f'-f)$. \pref{eq:uniform_conc3} then implies that
\begin{align*}
  \nrm*{f-f''}_{k,h}^2=t^2\nrm*{f-f'}_{k,h}^2
  &\leq{} 2t^2\nrm*{f-f'}_{\cZ_h\ind{k}}^{2} +
  t^2 8H^2\log(2\abs{\cF_h}\delta^{-1}) \\&\leq{}
  2\nrm*{f-f''}_{\cZ_h\ind{k}}^{2} +
 8H^2\log(2\abs{\cF_h}\delta^{-1}),
\end{align*}
since $t\leq{}1$. This gives the the result for $k,h$ fixed. We union bound over all $(k,h)$ pairs to get the final result.
\end{proof}

\begin{proof}[\pfref{lem:integral_bound}]
  For each $i$, we have
  \[
    \log(x_{i+1}/x_i)=\log\prn*{1+\prn*{\frac{x_{i+1}}{x_i}-1}}\geq{}\prn*{\frac{x_{i+1}}{x_i}-1}\cdot\frac{x_i}{x_{i+1}}
    \geq{}\frac{1}{2}\prn*{\frac{x_{i+1}}{x_i}-1},
  \]
  where we have used the fact that $\log(1+y)\geq{}\frac{y}{y+1}$ for
  $y\geq{}0$, and that
  $\frac{x_i}{x_{i+1}}\geq{}\frac{x_i}{x_{i}+1}\geq{}\frac{1}{2}$. It
  follows that
  \[
    \sum_{i=1}^{N}\frac{x_{i+1}-x_i}{x_i}
    \leq{}2 \sum_{i=1}^{N}\log(x_{i+1}/x_i)\leq{}2\log(x_{N+1}/x_1).
  \]
\end{proof}

\begin{proof}[\pfref{lem:disagreement_star}]
  Assume $R/\veps\geq{}e$, otherwise
  $\mb{\theta}_{\cD;\gstar}(\starhull(\cG,\gstar),\veps)\leq{}e^{2}$. Observe that any $g'\in\starhull(\cG,\gstar)$ can be written as
  $t(g-\gstar)+\gstar)$ for $g\in\cG$. Consequently, we can write
  \begin{align*}
    &\En_{\cP}\sup\crl*{(g(z)-\gstar(z))^{2}:g\in\starhull(\cG,\gstar),\En_{\cP}\brk{(g(z)-\gstar(z))^{2}}\leq\veps^{2}}\\
&=\En_{\cP}\sup\crl*{t^{2}(g(z)-\gstar(z))^{2}:t\in\brk*{0,1}, g\in\cG,t^{2}\cdot{}\En_{\cP}\brk{(g(z)-\gstar(z))^{2}}\leq\veps^{2}}
  \end{align*}
  Let $N=\ceil*{\log(R/\veps)}$, and let $a_0=1$ and $a_i=e^{-i}$ for
  $i\in\brk*{N}$. Then we can bound
\begin{align*}
  &\En_{\cP}\sup\crl*{t^{2}(g(z)-\gstar(z))^{2}:t\in\brk*{0,1},
    g\in\cG,t^{2}\cdot{}\En_{\cP}\brk{(g(z)-\gstar(z))^{2}}\leq\veps^{2}}\\
  &\leq{}\En_{\cP}\sup_{i\in\brk*{N}}\sup\crl*{t^{2}(g(z)-\gstar(z))^{2}:t\in\brk*{a_i,a_{i-1}},
    g\in\cG,t^{2}\cdot{}\En_{\cP}\brk{(g(z)-\gstar(z))^{2}}\leq\veps^{2}}
    + \veps^{2},
\end{align*}
since any $t\notin\brk*{a_{N},1}$ has
$\abs*{t}\leq{}\veps/R$. We further upper bound
\begin{align*}
  &\En_{\cP}\sup_{i\in\brk*{N}}\sup\crl*{t^{2}(g(z)-\gstar(z))^{2}:t\in\brk*{a_i,a_{i-1}},
  g\in\cG,t^{2}\cdot{}\En_{\cP}\brk{(g(z)-\gstar(z))^{2}}\leq\veps^{2}}\\
  &\leq{}\sum_{i=1}^{N}\En_{\cP}\sup\crl*{t^{2}(g(z)-\gstar(z))^{2}:t\in\brk*{a_i,a_{i-1}},
    g\in\cG,t^{2}\cdot{}\En_{\cP}\brk{(g(z)-\gstar(z))^{2}}\leq\veps^{2}}\\
  &\leq{}\sum_{i=1}^{N}a_{i-1}^{2}\En_{\cP}\sup\crl*{(g(z)-\gstar(z))^{2}:g\in\cG,\cdot{}\En_{\cP}\brk{(g(z)-\gstar(z))^{2}}\leq\veps^{2}/a_{i}^{2}}\\
  &\leq{}\sum_{i=1}^{N}\frac{a_{i-1}^{2}}{a_i^{2}}\veps^{2}\mb{\theta}_{\cD;\gstar}(\cG,\veps)\\
  &\leq{}e^{2}N\cdot{}\veps^{2}\cdot{}\mb{\theta}_{\cD;\gstar}(\cG,\veps).
\end{align*}
where we have used that $a_{i-1}/a_i\leq{}e$ for all $i$. Finally,
since $R/\veps\geq{}1$, we have $N\leq{}2\log(R/\veps)$. Since
$\btheta_{\cD;\gstar}(\cG,\veps)\geq{}1$, by combining both cases we have $\mb{\theta}_{\cD;\gstar}(\starhull(\cG,\gstar),\veps)\leq{}e^2(2\log(R/\veps)+1) \btheta_{\cD;\gstar}(\cG,\veps)$.
\end{proof}

\section{Proof of \cref*{thm:confidence_radius}}
\label{app:confidence_radius}

\subsection{Preliminaries}
Recall that at round $k$, for each layer $h$, we solve
\[
\fhat_h\ind{k}=\argmin_{f\in\cF_h}\sum_{j<k}\prn*{f(x_h\ind{j,h},a_h\ind{j,h})
  - \prn*{r_h\ind{j,h}+\Vbar_{h+1}\ind{k}(x_{h+1}\ind{j,h})}}^{2},
\]
where for each round $j$, $x_h\ind{j,h},a_h\ind{j,h},x_{h+1}\ind{j,h}$ are
obtained by rolling in until time $h$ with $\pi\ind{j}$, then
sampling $a_h$ uniformly at random.

This proof is inductive. Let $1\leq{}h\leq{}H-1$ be fixed. Suppose we can guarantee that for
layer $h+1$, with probability at least $1-\delta_{h+1}$, we have
\[
\nrm[\big]{\fhat_{h+1}\ind{k}-\fbar\ind{k}_{h+1}}_{\cZ_{h+1}\ind{k}}\leq\betaconf[h+1],
\]
so that  $\fbar_h\ind{k}\in\cFhat_{h+1}\ind{k}=\crl*{f\in\starhull(\cF_{h+1},\fhat_{h+1}\ind{k}):
  \nrm[\big]{f-\fhat_{h+1}\ind{k}}_{\cZ_{h+1}\ind{k}}\leq\betaconf[h+1]}$.
Our goal is to figure out, based on this inductive hypothesis, what
are admissible values for $\beta_h$ and $\delta_h$ for layer $h$. We
handle the base case for layer $H$, which is straightforward, at the
end of the proof.
\subsection{Initial Bound on Least Squares Error}
Define
\[
L_h\ind{k}(f) = \sum_{j<k}\prn*{f(x_h\ind{j,h},a_h\ind{j,h})
  - \prn*{r_h\ind{j,h} + \Vbar_{h+1}\ind{k}(x_{h+1}\ind{j,h})}^{2}}.
\]
Then using strong convexity of the square loss (following the usual
basic inequality argument for least squares),
we have
\begin{align*}
  0\geq{}L_h\ind{k}(\fhat_h\ind{k})-L_h\ind{k}(\fbar\ind{k}_h)
  &\geq{} D_{f}L_h\ind{k}(\fbar_h\ind{k})[\fhat_h\ind{k}-\fbar\ind{k}_h] +
    \nrm[\big]{\fhat_h\ind{k}-\fbar\ind{k}_h}_{\cZ_{h}\ind{k}}^{2}.
\end{align*}
Now note that we have
\begin{align*}
  &D_{f}L_h\ind{k}(\fbar\ind{k}_h)[\fhat_h\ind{k}-\fbar\ind{k}_h]\\
  &=2\sum_{j<k}\prn*{\fbar_h\ind{k}(x_h\ind{j,h},a_h\ind{j,h})
    - \prn*{r_h\ind{j,h} + \Vbar_{h+1}\ind{k}(x_{h+1}\ind{j,h})}}\prn*{\fhat_h\ind{k}(x_h\ind{j,h},a_h\ind{j,h})-\fbar_h\ind{k}(x_h\ind{j,h},a_h\ind{j,h})}.
\end{align*}
Abbreviating $z_h\ind{j,h}=(x_h\ind{j,h},a_h\ind{j,h})$, we can use
this to rewrite the previous expression as
\begin{align*}
&  \nrm[\big]{\fhat_h\ind{k}-\fbar\ind{k}_h}_{\cZ_{h}\ind{k}}^{2}\\
  &\leq{} 4\sum_{j<k}\prn*{\fbar_h\ind{k}(z_h\ind{j,h})
    - \prn*{r_h\ind{j,h} + \Vbar_{h+1}\ind{k}(x_{h+1}\ind{j,h})}}\prn*{\fhat_h\ind{k}(z_h\ind{j,h})-\fbar_h\ind{k}(z_h\ind{j,h})}
  -   \nrm[\big]{\fhat_h\ind{k}-\fbar\ind{k}_h}_{\cZ_{h}\ind{k}}^{2}.
\end{align*}
Since
$\fhat_h\ind{k},\fbar_h\ind{k}\in\cF_h$,
if we define $\cG_h=\cF_h-\cF_h$, we can further upper bound as
\begin{align}
  &\nrm[\big]{\fhat_h\ind{k}-\fbar\ind{k}_h}_{\cZ_{h}\ind{k}}^{2}
  \\
  &\leq{} \sup_{g\in\cG_h}\crl*{4\sum_{j<k}\prn*{\fbar_h\ind{k}(z_h\ind{j,h})
    - \prn*{r_h\ind{j,h} + \Vbar_{h+1}\ind{k}(x_{h+1}\ind{j,h})}}g(z_h\ind{j,h})
  -   \nrm*{g}_{\cZ_{h}\ind{k}}^{2}}\notag\\
  &= \sup_{g\in\cG_h}\crl*{4\sum_{j<k}\prn*{\brk*{\Pstar_h\Vbar_{h+1}\ind{k}}(z_h\ind{j,h})
    - \Vbar_{h+1}\ind{k}(x_{h+1}\ind{j,h}) + \zeta_h\ind{j,h}}g(z_h\ind{j,h})
  -   \sum_{j<k}g^{2}(z_h\ind{j,h})},\label{eq:offset_rad_basic}
\end{align}
where $\zeta_h\ind{j,h}\ldef{}\fbayes(z_h\ind{j,h})-r_h\ind{j,h}$.

Now, recall that
\[
\Vbar_{h+1}\ind{k}(x) = \max_{a}\sup_{f\in\cFhat_{h+1}\ind{k}}f(x,a).
\]
Let us define another set,
\[
\cFbar_h\ind{k}(\veps)=\crl*{f\in\starhull(\cF_h,\fhat_h\ind{k}) : \nrm[\big]{f-\fhat_h\ind{k}}_{\cL_h\ind{k}}^2\leq\veps}.
\]
Let $\betatil_{h+1}\geq{}\beta_{h+1}$ be a free parameter, and define
\[
\Qtil_{h+1}\ind{k}(x,a) = \sup_{f\in\cFbar_{h+1}\ind{k}(\betatil^2_{h+1})}f(x,a)
\]
and $\Vtil_{h+1}\ind{k}(x)=\max_{a}\Qtil_{h+1}\ind{k}(x,a)$. Then we
can write
\[
\Vbar_{h+1}\ind{k}(x) = \Vtil_{h+1}\ind{k}(x) + \xi_{h+1}\ind{k}(x),
\]
where
\[
  \abs*{\xi_{h+1}\ind{k}(x)}\leq{}\max_{a}\abs*{
    \Qtil_{h+1}\ind{k}(x,a)
    - \Qbar_{h+1}\ind{k}(x,a)
    }.
\]
In particular, returning to \pref{eq:offset_rad_basic}, we have
\begin{align*}
  &\nrm[\big]{\fhat_h\ind{k}-\fbar\ind{k}_h}_{\cZ_{h}\ind{k}}^{2}
  \\
  &\leq \sup_{g\in\cG_h}\crl*{4\sum_{j<k}\prn*{\brk[\big]{\Pstar_h\Vbar_{h+1}\ind{k}}(z_h\ind{j,h})
    - \Vbar_{h+1}\ind{k}(x_{h+1}\ind{j,h})+\zeta_h\ind{j,h}}g(z_h\ind{j,h})
    -   \sum_{j<k}g^{2}(z_h\ind{j,h})}.\\
    &\leq \sup_{g\in\cG_h}\crl*{4\sum_{j<k}\prn*{\brk[\big]{\Pstar_h\Vtil_{h+1}\ind{k}}(z_h\ind{j,h})
      - \Vtil_{h+1}\ind{k}(x_{h+1}\ind{j,h})+\zeta_h\ind{j,h}
      + \brk*{\Pstar_h\xi_{h+1}\ind{k}}(z_h\ind{j,h})
      - \xi_{h+1}\ind{k}(x_{h+1}\ind{j,h})}g(z_h\ind{j,h})
  -   \sum_{j<k}g^{2}(z_h\ind{j,h})}.
\end{align*}
Using the AM-GM inequality, we have
\[
\xi_{h+1}\ind{k}(x_{h+1}\ind{j,h})\cdot{}g(z_h\ind{j,h})\leq{} \frac{3}{2}(\xi_{h+1}\ind{k}(x_{h+1}\ind{j,h}))^{2}+\frac{1}{3}g^{2}(z_h\ind{j,h}),
\]
and likewise
$\brk*{\Pstar_{h}\xi_{h+1}\ind{k}}(z_{h}\ind{j,h})\cdot{}g(z_h\ind{j,h})\leq{}
\frac{3}{2}\prn*{\brk*{\Pstar_{h}\xi_{h+1}\ind{k}}(z_{h}\ind{j,h})}^{2}+\frac{1}{3}g^{2}(z_h\ind{j,h}),$
so we can further upper bound the process above by
\begin{align*}
&\underbrace{\sup_{g\in\cG_h}\crl*{4\sum_{j<k}\prn*{\brk[\big]{\Pstar_h\Vtil_{h+1}\ind{k}}(z_h\ind{j,h})
      - \Vtil_{h+1}\ind{k}(x_{h+1}\ind{j,h})+\zeta_h\ind{j,h}}\cdot{}g(z_h\ind{j,h})
  -   \frac{1}{3}\sum_{j<k}g^{2}(z_h\ind{j,h})}}_{\rdef\OP_h}\\
  &~~~~+
\underbrace{\frac{3}{2}\sum_{j<k}(\xi_{h+1}\ind{k}(x_{h+1}\ind{j,h}))^{2}+\prn*{\brk*{\Pstar_{h}\xi_{h+1}\ind{k}}(z_h\ind{j,h})}^{2}}_{\rdef\ET_h},
\end{align*}
where $\OP_h$ is an offset process and $\ET_h$ is an error
term.
\subsection{Bounding the Offset Process}
Recall that
\[
\Vtil_{h+1}\ind{k}(x) = \max_{a}\sup\crl*{f(x,a)\mid{}f\in\starhull(\cF_{h+1},\fhat_{h+1}\ind{k}), \nrm[\big]{f-\fhat_{h+1}\ind{k}}^2_{\cL_{h+1}\ind{k}}\leq\betatil^2_{h+1}}.
\]
Hence, if we define a set
\[
  \cV_{h+1}\ind{k}=\crl*{
x\mapsto{}\max_{a}\sup_{f\in\starhull(\cF_{h+1},f'):\nrm*{f-f'}^2_{\cL}\leq\betatil^2_{h+1}}f(x,a)\mid{}f'\in\cF_{h+1},\cL\subseteq\cS_{h+1}^{k-1}
  },
\]
we have $\Vtil_{h+1}\ind{k}\in\cV_{h+1}\ind{k}$. Note that
$\cV_{h+1}\ind{k}$ is not a random variable, which is critical for
the analysis. We can bound the size of $\cV_{h+1}\ind{k}$ as
follows. First, observe that each function in $\cV_{h+1}\ind{k}$ is
uniquely defined by the choice of the center $f'$ and the set
$\cL\subseteq\cS_{h+1}^{k-1}$. Moreover, for any two sets $\cL$,
$\cL'$ that are equivalent up to permutation, we have
$\nrm*{f}_{\cL}=\nrm*{f}_{\cL'}$ for all $f$, meaning that the norm in
the constraint is determined only by the multiset of states in
$\cL$. By the usual stars-and-bars counting argument, there are only
$\multiset{k}{S-1}={k+S-2\choose
  S-1}\leq{}\prn*{\frac{e(K+S-2)}{S-1}}^{S}\leq\prn*{2eK}^{S}$ possible
such choices (assuming $S>1$, if not there is clearly at most $1$ such
choice). Hence, altogether, we have
\begin{align}
  \label{eq:cv_bound}
\abs*{\cV_{h+1}\ind{k}}\leq{}\prn*{2eK}^{S}\abs*{\cF_{h+1}}.
\end{align}
We move to the upper bound
\[
\OP_h \leq{}\sup_{g\in\cG_h}\sup_{V\in\cV_{h+1}\ind{k}}\crl*{4\sum_{j<k}\prn*{\brk*{\Pstar_hV}(x_h\ind{j,h},a_h\ind{j,h})
      - V(x_{h+1}\ind{j,h}) +\zeta_h\ind{j,h}}\cdot{}g(z_h\ind{j,h})
  -   \frac{1}{3}\sum_{j<k}g^{2}(z_h\ind{j,h})}.
\]
We now appeal to the following lemma.
\begin{lemma}
  \label{lem:op_bound}
  Let $k$ and $h$ be fixed. For any function classes
  $\cG\subseteq((\cX_h\times\cA)\to\brk*{-H,+H})$ and
  $\cV\subseteq(\cX_{h+1}\to\brk*{0,H})$ and any constant $c\in(0,1]$, with probability at least $1-\delta$,
  \begin{align}
    \sup_{g\in\cG}\sup_{V\in\cV}\crl*{\sum_{j<k}\prn*{\brk*{\Pstar_hV}(x_h\ind{j,h},a_h\ind{j,h})
      - V(x_{h+1}\ind{j,h})+\zeta_h\ind{j,h}}\cdot{}g(z_h\ind{j,h})
  -   c\sum_{j<k}g^{2}(z_h\ind{j,h})} \leq{} \frac{4H^{2}\log(\abs*{\cG}\abs*{\cV}\delta^{-1})}{c}.
  \end{align}
\end{lemma}
Using \pref{lem:op_bound} along with the bound from
\pref{eq:cv_bound}, we have that with probability at least $1-\delta$,
\begin{align}
  \OP_h &\leq{}
        192H^{2}\log(\abs*{\cG_h}\abs*{\cV_{h+1}\ind{k}}\delta^{-1})\notag\\
      &\leq{}
        192H^{2}\log(\abs*{\cF_h}^{2}\abs*{\cF_{h+1}}(2eK)^{S}\delta^{-1})\notag\\
        &\leq{}
          576H^{2}\prn*{S\log(2eK)+\log\prn*{\Fmax
          \delta^{-1}}}.
    \label{eq:op_bound_final}
\end{align}

\subsection{Bounding the Approximation Error}
By Jensen's inequality, we have
\begin{align*}
\ET_h &=
\frac{3}{2}\sum_{j<k}(\xi_{h+1}\ind{k}(x_{h+1}\ind{j,h}))^{2}+\prn*{\brk*{\Pstar_{h}\xi_{h+1}\ind{k}}(x_{h}\ind{j,h},a_h\ind{j,h})}^{2}\\
  &\leq{}
    \frac{3}{2}\sum_{j<k}(\xi_{h+1}\ind{k}(x_{h+1}\ind{j,h}))^{2}
    + \En\brk*{(\xi_{h+1}\ind{k}(x_{h+1}))^{2}\mid{}x_h=x_{h}\ind{j,h},a_h=a_{h}\ind{j,h}}.
\end{align*}
Now, recall that for any $x$, we have
\begin{align*}
  \abs*{\xi_{h+1}\ind{k}(x)}
  &\leq{}
    \max_{a}\abs*{
    \Qtil_{h+1}\ind{k}(x,a)
    - \Qbar_{h+1}\ind{k}(x,a)
    } \\
    &=
    \max_{a}\abs*{
     \sup_{f\in\cFbar_{h+1}\ind{k}(\betatil^2_{h+1})}f(x,a)
    - \sup_{f\in\cFhat_{h+1}\ind{k}}f(x,a)
    }.
\end{align*}
To relate the two terms in the absolute value, we appeal to a uniform concentration lemma. %
\begin{lemma}
  \label{lem:uniform_conc}
  Let $k$, $h$, and $\veps\in(0,1)$ be fixed.  With probability at least $1-\delta$,
  we have that for all $f\in\cF_h$ and $f'\in\starhull(\cF_h,f)$, 
  \begin{align}
    (1-\veps)\nrm*{f-f'}_{\cL_h\ind{k}}^{2} - c_{\veps,\delta;h}
    \leq{} \nrm*{f-f'}_{\cZ_h\ind{k}}^2 \leq{} (1+\veps)\nrm*{f-f'}_{\cL_h\ind{k}}^{2} +
    c_{\veps,\delta;h}
  \end{align}
  where $c_{\veps,\delta;h}\leq{}    \frac{2H^{2}\log(2\abs*{\cF_h}\delta^{-1})}{\veps}$.
\end{lemma}
Fix $\veps$, $\delta$ to be chosen later. \pref{lem:uniform_conc} implies that
with probability at least $1-\delta$,
\begin{equation}
  \label{eq:uniform_conc}
\cFbar_{h+1}\ind{k}\prn*{(1-\veps)\betaconf[h+1]^2-c_{\veps,\delta;h+1}}\subseteq{}\cFhat_{h+1}\ind{k}\subseteq{}\cFbar_{h+1}\ind{k}\prn*{(1+\veps)\betaconf[h+1]^2+c_{\veps,\delta;h+1}},
\end{equation}
where we assume for now that $\veps$ is chosen such that
$(1-\veps)\beta_{h+1}^{2}-c_{\veps,\delta;h+1}>0$.

Let us set
$\betatil_{h+1}^{2}=(1+\veps)\betaconf[h+1]^2+c_{\veps,\delta;h+1}$. Then,
conditioned on the event above, we have
\begin{align*}
    \abs*{\xi_{h+1}\ind{k}(x)}
  &\leq{}
     \max_{a}\abs[\Bigg]{
     \sup_{f\in\cFbar_{h+1}\ind{k}(\betatil^2_{h+1})}f(x,a)
    - \sup_{f\in\cFhat_{h+1}\ind{k}}f(x,a)
    }\\
   &=
     \max_{a}\crl[\Bigg]{
     \sup_{f\in\cFbar_{h+1}\ind{k}(\betatil^2_{h+1})}f(x,a)
    - \sup_{f\in\cFhat_{h+1}\ind{k}}f(x,a)
    },
\end{align*}
since the nesting property in \pref{eq:uniform_conc} implies that
$\sup_{f\in\cFbar_{h+1}\ind{k}(\betatil^2_{h+1})}f(x,a)$ is always the
larger of the two terms. Using the other direction of the nesting
property, we can further upper bound by
\begin{align*}
      \abs*{\xi_{h+1}\ind{k}(x)}
  &\leq{}
     \max_{a}\crl[\Bigg]{
    \sup_{f\in\cFbar_{h+1}\ind{k}((1+\veps)\beta^2_{h+1}+c_{\veps,\delta;h+1})}f(x,a)
    - \sup_{f\in\cFbar_{h+1}\ind{k}((1-\veps)\beta^2_{h+1}-c_{\veps,\delta;h+1})}f(x,a)
    }.
\end{align*}

Let $x$ and $a$ be fixed, and consider a function $f^{\star}$ that
achieves the value
\[
\sup_{f\in\cFbar_{h+1}\ind{k}((1+\veps)\beta^2_{h+1}+c_{\veps,\delta;h+1})}f(x,a).
\]
If the maximum is not achieved, we can simply consider a sequence of
functions approaching the supremum, but we omit the details.

Now, observe that since $\starhull(\cF_{h+1},\fhat_{h+1}\ind{k})$ contains
any function of the form $t(f-\fhat_{h+1}\ind{k}) +
\fhat_{h+1}\ind{k}$ for $f\in\cF_{h+1}$, $t\in\brk*{0,1}$ by definition, we have
\[
\tilde{f}\ldef{}\prn*{\frac{(1-\veps)\beta_{h+1}^{2}-c_{\veps,\delta;h+1}}{(1+\veps)\beta_{h+1}^{2}+c_{\veps,\delta;h+1}}}^{1/2}\cdot(\fstar-\fhat_{h+1}\ind{k})
+\fhat_{h+1}\ind{k}\in\cFbar_{h+1}\ind{k}((1-\veps)\beta_{h+1}^{2}-c_{\veps,\delta;h+1}),
\]
where we have used that $(1-\veps)\beta_{h+1}^2-\cveps\geq{}0$. This means that
\begin{align*}
  &\sup_{f\in\cFbar_{h+1}\ind{k}((1+\veps)\beta^2_{h+1}+c_{\veps,\delta;h+1})}f(x,a)
    -
    \sup_{f\in\cFbar_{h+1}\ind{k}((1-\veps)\beta^2_{h+1}-c_{\veps,\delta;h+1})}f(x,a)\\
  &\leq{}\fstar(x,a)
    - \ftil(x,a)\\
  &=
    \prn*{1-\prn*{\frac{(1-\veps)\beta_{h+1}^{2}-c_{\veps,\delta;h+1}}{(1+\veps)\beta_{h+1}^{2}+c_{\veps,\delta;h+1}}}^{1/2}}\prn*{\fstar(x,a)-\fhat_{h+1}\ind{k}(x,a)}\\
  &= \prn*{1-\prn*{\frac{(1-\veps)\beta_{h+1}^{2}-c_{\veps,\delta;h+1}}{(1+\veps)\beta_{h+1}^{2}+c_{\veps,\delta;h+1}}}^{1/2}}\sup_{f\in\cFbar_{h+1}\ind{k}((1+\veps)\beta^2_{h+1}+c_{\veps,\delta;h+1})}\crl*{f(x,a)-\fhat_{h+1}\ind{k}(x,a)}
\end{align*}
We also have
\begin{align*}
\prn*{1-\prn*{\frac{(1-\veps)\beta_{h+1}^{2}-c_{\veps,\delta;h+1}}{(1+\veps)\beta_{h+1}^{2}+c_{\veps,\delta;h+1}}}^{1/2}}^{2}
  \leq{}1-\frac{(1-\veps)\beta_{h+1}^{2}-c_{\veps,\delta;h+1}}{(1+\veps)\beta_{h+1}^{2}+c_{\veps,\delta;h+1}}
  \leq{} 2\prn*{\veps + \frac{\cveps}{\beta_{h+1}^{2}}}.
\end{align*}
Altogether, this allows us to bound
\begin{align*}
  \sum_{j<k}(\xi_{h+1}\ind{k}(x_{h+1}\ind{j,h}))^{2}
  \leq{} 2\prn*{\veps + \frac{\cveps}{\beta_{h+1}^{2}}}
\sum_{j<k}\max_{a}\sup_{f\in\cFbar_{h+1}\ind{k}(\betatil_{h+1}^{2})}\prn*{f(x_{h+1}\ind{j,h},a)-\fhat_{h+1}\ind{k}(x_{h+1}\ind{j,h},a)}^{2}
\end{align*}
Now, consider the function
\[
w_{h+1}\ind{k}(x)\ldef{}\max_{a}\sup_{f\in\cFbar_{h+1}\ind{k}(\betatil_{h+1}^{2})}\prn*{f(x,a)-\fhat_{h+1}\ind{k}(x,a)}^{2}.
\]
As in the analysis for $\OP_h$, we argue that this function belongs to a
relatively small class. In particular, we have
\begin{align*}
w_{h+1}\ind{k}\in  \cW_{h+1}\ind{k}\ldef\crl*{
x\mapsto{}\max_{a}\max_{f\in\starhull(\cF_{h+1},f'):\nrm*{f-f'}^2_{\cL}\leq\betatil^2_{h+1}}\prn*{f(x,a)-f'(x,a)}^2\mid{}f'\in\cF_{h+1},\cL\subseteq\cS_{h+1}^{k-1}
  }.
\end{align*}
Through the same counting argument as in the analysis of $\OP_h$, we have
$\abs*{\cW_{h+1}\ind{k}}\leq{}(2eK)^{S}\abs*{\cF_{h+1}}$. We use this to
relate
\[
\sum_{j<k}(\xi_{h+1}\ind{k}(x_{h+1}\ind{j,h}))^{2} = \sum_{j<k}w_{h+1}\ind{k}(x_{h+1}\ind{j,h})
\]
to a conditional-expected variant of the same quantity via a uniform
concentration bound.
\begin{lemma}
  \label{lem:w_unif_conc}
  With probability at least $1-\delta$, for all
  $w\in\cW_{h+1}\ind{k}$, we have
  \begin{align*}
    \sum_{j<k}w(x_{h+1}\ind{j,h})\leq{}\frac{3}{2}\sum_{j<k}\En_{x_{h+1}\sim\emi(s_{h+1}\ind{j,h})}\brk*{w(x_{h+1})}
    + 4H^{2}\prn*{S\log(2eK)+\log(2\abs*{\cF_{h+1}}\delta^{-1})}
  \end{align*}
  and
  \begin{align*}
\sum_{j<k}\En\brk*{w(x_{h+1})\mid{}x_h=x_{h}\ind{j,h},a_h=a_{h}\ind{j,h}}\leq{}3\sum_{j<k}\En_{x_{h+1}\sim\emi(s_{h+1}\ind{j,h})}\brk*{w(x_{h+1})}
    + 16H^{2}\prn*{S\log(2eK)+\log(2\abs*{\cF_{h+1}}\delta^{-1})}.
  \end{align*}
\end{lemma}
Conditioned on the event in \pref{lem:w_unif_conc}, we have
\begin{align*}
  \ET_h \leq{}
  28\prn*{\veps + \frac{\cveps}{\beta_{h+1}^{2}}}\sum_{j<k}\En_{x_{h+1}\sim\emi(s_{h+1}\ind{j,h})}\brk*{w_{h+1}\ind{k}(x_{h+1})}
  + 60\prn*{\veps + \frac{\cveps}{\beta_{h+1}^{2}}}H^2\prn*{S\log(2eK)+\log(2\abs*{\cF_{h+1}}\delta^{-1})}.
\end{align*}
We can further upper bound
\begin{align*}
  \sum_{j<k}\En_{x_{h+1}\sim\emi(s_{h+1}\ind{j,h})}\brk*{w_{h+1}\ind{k}(x_{h+1})}
  &=
  \sum_{j<k}\En_{x_{h+1}\sim\emi(s_{h+1}\ind{j,h})}\max_{a}\sup_{f\in\cFbar_{h+1}\ind{k}(\betatil_{h+1}^{2})}(f(x_{h+1},a)-\fhat_{h+1}\ind{k}(x_{h+1},a_{h+1})^{2}\\
  &\leq{} A\sum_{j<k}\En_{x_{h+1}\sim\emi(s_{h+1}\ind{j,h})}\En_{a_{h+1}\sim\piunif}\sup_{f\in\cFbar_{h+1}\ind{k}(\betatil_{h+1}^{2})}(f(x_{h+1},a_{h+1})-\fhat_{h+1}\ind{k}(x_{h+1},a_{h+1}))^{2}.
\end{align*}
Now, recall that each $f\in\cFbar_{h+1}\ind{k}(\betatil_{h+1}^{2})$
has
\begin{align*}
  \nrm[\big]{f-\fhat_{h+1}\ind{k}}_{\cL_{h+1}\ind{k}}^{2}
  &=
\sum_{j<k}\En_{x_{h+1}\sim\emi(s_{h+1}\ind{j,h})}\En_{a_{h+1}\sim\piunif}(f(x_{h+1},a_{h+1})-\fhat_{h+1}\ind{k}(x_{h+1},a_{h+1}))^{2}
    \leq{} \betatil_{h+1}^{2}
\end{align*}
by definition. Letting $n_{h}\ind{k}(s) = \abs*{\crl*{j<k : s\ind{j,h}=s}}$ for each
$s\in\cS_h$, this implies
\[
\sum_{s\in\cS_{h+1}}n_{h+1}\ind{k}(s)\En_{x_{h+1}\sim\emi(s)}\En_{a_{h+1}\sim\piunif}(f(x_{h+1},a_{h+1})-\fhat_{h+1}\ind{k}(x_{h+1},a_{h+1}))^{2}
\leq{} \betatil_{h+1}^{2},
\]
so, in particular,
\[
\En_{x_{h+1}\sim\emi(s)}\En_{a_{h+1}\sim\piunif}(f(x_{h+1},a_{h+1})-\fhat_{h+1}\ind{k}(x_{h+1},a_{h+1}))^{2}
\leq{} \frac{\betatil_{h+1}^{2}}{n_{h+1}\ind{k}(s)}
\]
for all $s\in\cS_{h+1}$.

We can similarly upper bound
\begin{align*}
&\sum_{j<k}\En_{x_{h+1}\sim\emi(s_{h+1}\ind{j,h})}\En_{a_{h+1}\sim\piunif}\sup_{f\in\cFbar_{h+1}\ind{k}(\betatil_{h+1}^{2})}(f(x_{h+1},a_{h+1})-\fhat_{h+1}\ind{k}(x_{h+1},a_{h+1}))^{2}
  \\
& = \sum_{s\in\cS_{h+1}}n_{h+1}\ind{k}(s)\En_{x_{h+1}\sim\emi(s)}\En_{a_{h+1}\sim\piunif}\sup_{f\in\cFbar_{h+1}\ind{k}(\betatil_{h+1}^{2})}(f(x_{h+1},a_{h+1})-\fhat_{h+1}\ind{k}(x_{h+1},a_{h+1}))^{2}.
\end{align*}
Recall that for each latent state $s\in\cS_h$ and $\fstar\in\cF_{h}$,
we define the disagreement coefficient as
\[
\thetacheck_s(\veps_0) =
\sup_{\fstar\in\cF_h}\sup_{\veps>\veps_0}\frac{\En_{x\sim\emi(s)}\En_{a\sim\piunif}\sup\crl*{(f(x,a)-\fstar(x,a))^{2}
  : f\in\starhull(\cF_h,\fstar), \nrm*{f-\fstar}_s^{2}\leq\veps^{2}}}{\veps^{2}}
\]
It follows from the discussion above that we have
\begin{align*}
&\sum_{s\in\cS_{h+1}}n_{h+1}\ind{k}(s)\En_{x_{h+1}\sim\emi(s)}\En_{a_{h+1}\sim\piunif}\sup_{f\in\cFbar_{h+1}\ind{k}(\betatil_{h+1}^{2})}(f(x_{h+1},a_{h+1})-\fhat_{h+1}\ind{k}(x_{h+1},a_{h+1}))^{2}
  \\
  &\leq{}\sum_{s\in\cS_{h+1}}n_{h+1}\ind{k}(s)\cdot{}\thetacheck_s(\betatil_{h+1}/\prn{n_{h+1}\ind{k}(s)}^{1/2})\frac{\betatil_{h+1}^{2}}{n_{h+1}\ind{k}(s)}\\
&\leq{}\betatil_{h+1}^{2}\sum_{s\in\cS_{h+1}}\thetacheck_s(\beta_{h+1}/K^{1/2}),
\end{align*}
where we have used that $\thetacheck$ is decreasing in $\veps$. Recalling
the definition
$\thetacheck_{h+1}(\veps)=\sum_{s\in\cS_{h+1}}\thetacheck_s(\veps)$, we are
guaranteed that
\begin{align*}
    \ET_h &\leq{}
  28A\thetacheck_{h+1}(\beta_{h+1}K^{-1/2})\prn*{\veps + \frac{\cveps}{\beta_{h+1}^{2}}}\betatil_{h+1}^{2}
  + 60\prn*{\veps +
          \frac{\cveps}{\beta_{h+1}^{2}}}H^2\prn*{S\log(2eK)+\log(2\abs*{\cF_{h+1}}\delta^{-1})}\\
  &\leq{}
    28A\thetacheck_{h+1}(\beta_{h+1}K^{-1/2})\prn*{2\veps\beta_{h+1}^{2}+3c_{\veps,\delta;h+1} + \frac{c_{\veps,\delta;h+1}^{2}}{\beta_{h+1}^{2}}
    }
  + 60\prn*{\veps + \frac{\cveps}{\beta_{h+1}^{2}}}H^2\prn*{S\log(2eK)+\log(2\abs*{\cF_{h+1}}\delta^{-1})}.
\end{align*}

\subsection{Putting Everything Together}
Combining the bounds on $\OP_h$ and $\ET_h$ and taking a union bound, we
are guaranteed that with probability at least
$1-3\delta-\delta_{h+1}$,
\begin{align*}
  \nrm*{\fhat_h\ind{k}-\fbar\ind{k}_h}_{\cZ_{h}\ind{k}}^{2}
 &\leq{}28A\thetacheck_{h+1}(\betahat_{h+1})\prn*{2\veps\beta_{h+1}^{2}+3c_{\veps,\delta;h+1} + \frac{c_{\veps,\delta;h+1}^{2}}{\beta_{h+1}^{2}}
    }\\
  &~~~~+ 60\prn*{\veps +
    \frac{\cveps}{\beta_{h+1}^{2}}}H^2\prn*{S\log(2eK)+\log(2\abs*{\cF_{h+1}}\delta^{-1})}\\
  &~~~~+ 576H^{2}\prn*{S\log(2eK)+\log\prn*{\Fmax\delta^{-1}}},
\end{align*}
where we define $\betahat_{h+1}=\beta_{h+1}K^{-1/2}$.
This bound holds for any $\veps\in(0,1)$ chosen a-priori, so long as
$(1-\veps)\beta_{h+1}^{2}-c_{\veps,\delta;h+1}>0$. We now make an
appropriate choice. Recall that
\[
c_{\veps,\delta;h+1}\leq{}    \frac{2H^{2}\log(2\abs*{\cF_{h+1}}\delta^{-1})}{\veps}.
\]
We choose
\[
\veps = \sqrt{\frac{2H^{2}\log(2\Fmax\delta^{-1})}{\beta^{2}_{h+1}}}.
\]
We have $\veps\leq{}1$ as long as
$\beta_{h+1}^2\geq{}2H^{2}\log(2\Fmax\delta^{-1})$. Moreover,
\begin{align*}
(1-\veps)\beta_{h+1}^{2}-c_{\veps,\delta;h+1} \geq{}\beta_{h+1}^{2}-2\beta_{h+1}\sqrt{2H^{2}\log(2\Fmax\delta^{-1})},
\end{align*}
so this quantity is non-negative as long $\beta_{h+1}^{2}\geq{}8 H^{2}\log(2\Fmax\delta^{-1})$.

Assume for now that this constraint holds. Then we can bound
\begin{align*}
  \nrm[\big]{\fhat_h\ind{k}-\fbar\ind{k}_h}_{\cZ_{h}\ind{k}}^{2}
&\leq{}28A\thetacheck_{h+1}(\betahat_{h+1})\prn*{5\beta_{h+1}\sqrt{2H^{2}\log(2\Fmax\delta^{-1})}+2H^{2}\log(2\Fmax\delta^{-1})
    }\\
  &~~~~+ 120H^2\prn*{S\log(2eK)+\log(2\Fmax\delta^{-1})}\\
  &~~~~+ 576H^{2}\prn*{S\log(2eK)+\log\prn*{\Fmax\delta^{-1}}}.
\end{align*}
Furthermore, using the AM-GM inequality, we have
\begin{align*}
  28A\thetacheck_{h+1}(\betahat_{h+1})\prn*{5\beta_{h+1}\sqrt{2H^{2}\log(2\Fmax\delta^{-1})}}
  \leq{}\frac{1}{2}\beta_{h+1}^{2} +140^{2}A^2\thetacheck^2_{h+1}(\betahat_{h+1})H^{2}\log(2\Fmax\delta^{-1}).
\end{align*}
Using this, we can (rather coarsely) simplify the error bound to
\begin{align*}
  \nrm*{\fhat_h\ind{k}-\fbar\ind{k}_h}_{\cZ_{h}\ind{k}}^{2}
  \leq{} \frac{1}{2}\beta_{h+1}^{2}+ 
  150^{2}H^{2}A^{2}\thetacheck_{h+1}^{2}(\betahat_{h+1})\log(2\Fmax\delta^{-1}) + 700H^{2}S\log(2eK).
\end{align*}
The critical detail here is that the constant in front of
$\beta_{h+1}^{2}$ is no larger than $1$, so there is no exponential blowup as we
propagate this constraint backward.

Overall, we get that any sequences $(\delta_h)$, $(\beta_h)$ are
admissible as long for all $1\leq{}h\leq{}H-1$, 
\begin{align*}
  &\beta_h^{2}\geq{}\frac{1}{2}\beta_{h+1}^{2}+
  150^2H^{2}A^{2}\thetacheck_{h+1}^{2}(\betahat_{h+1})\log(2\Fmax\delta^{-1}) +
  700H^{2}S\log(2eK),\\
  &\beta_h^{2}\geq{}8H^2\log(2\Fmax\delta^{-1}),\\
  &\delta_{h}\geq{}3\delta+\delta_{h+1},
\end{align*}
for any $\delta\in(0,1)$ chosen a-priori.

For the base case, we observe that since
$\Vbar_{H+1}\ind{k}(x)=\Vstar_{H+1}(x)=0$ for all $k$, we have that for layer
$H$,
\[
    \nrm[\big]{\fhat_{H}\ind{k}-\fbar\ind{k}_{H}}_{\cZ_{H}\ind{k}}^{2}
\leq{}\sup_{g\in\cG_{H}}\crl*{4\sum_{j<k}\Delta_{H}\ind{j,H}\cdot{}g(z_{H}\ind{j,H})
  -   \frac{1}{3}\sum_{j<k}g^{2}(z_{H}\ind{j,{H}})}.
\]
By \pref{lem:op_bound}, this implies
\[
  \nrm[\big]{\fhat_{H}\ind{k}-\fbar\ind{k}_{H}}_{\cZ_{H}\ind{k}}^{2}\leq{}
  400H^{2}\log\prn*{\Fmax\delta^{-1}}
\]
with probability at least $1-\delta$.

We conclude that the following sequence is admissible:
\begin{align*}
  &\delta_{H-1}=\delta\\
  &\beta^2_{H}=400H^{2}\log\prn{\Fmax\delta^{-1}}\\
  &\delta_{h}=3\delta+\delta_{h+1}\\
  &\beta^{2}_{h} = \frac{1}{2}\beta_{h+1}^{2}+
  150^2H^{2}A^{2}\thetacheck_{h+1}^{2}(\beta_{h+1}K^{-1/2})\log(2\Fmax\delta^{-1}) +
  700H^{2}S\log(2eK).
\end{align*}
In particular, this guarantees that with probability at least
$1-\delta_1\geq{}1-3H\delta$, we have
\[
\nrm[\big]{\fhat_{h}\ind{k}-\fbar\ind{k}_{h}}_{\cZ_{h}\ind{k}}^{2}\leq\beta_h^{2}
\]
for all $h$ within round $k$. By union bound, the same holds for all
$k$ with probability at least $1-3HK\delta$.

\qed

\subsection{Deferred Proofs}

\begin{proof}[\pfref{lem:op_bound}]
  Define a filtration
\[
    \gfilt_{j-1}=\sigma(\cH\ind{1},\ldots,\cH\ind{j-1},
      (s_1\ind{j,h},x\ind{j,h}_1,a_1\ind{j,h}, r_1\ind{j,h}),\ldots, (s_{h-1}\ind{j,h},x\ind{j,h}_{h-1},a_{h-1}\ind{j,h},r_{h-1}\ind{j,h}),
      (s_{h}\ind{j,h},x\ind{j,h}_{h},a_{h}\ind{j,h})
     ).
   \]
   Let $g\in\cG$, $V\in\cV$ be fixed, and let $Z_j=\prn*{\brk*{\Pstar_hV}(x_h\ind{j,h},a_h\ind{j,h})
      - V(x_{h+1}\ind{j,h})+\zeta_h\ind{j,h}}\cdot{}g(z_h\ind{j,h})$. Observe that
    $Z_j$ is $\gfilt_j$-measurable and is a martingale difference
    sequence, since
    \begin{align*}
      \En\brk*{Z_j\mid{}\gfilt_{j-1}}&=\En\brk*{\prn*{\brk*{\Pstar_hV}(x_h\ind{j,h},a_h\ind{j,h})
      -
V(x_{h+1}\ind{j,h})+\fbayes(x_h\ind{j,h},a_h\ind{j,h})-r_h\ind{j,h}}\cdot{}g(z_h\ind{j,h})\mid{}x_h\ind{j,h},a_h\ind{j,h}}\\
      &=\prn*{\brk*{\Pstar_hV}(x_h\ind{j,h},a_h\ind{j,h})
      -
\brk*{\Pstar_hV}(x_h\ind{j,h},a_h\ind{j,h})+\fbayes(x_h\ind{j,h},a_h\ind{j,h})-\fbayes(x_h\ind{j,h},a_h\ind{j,h})}\cdot{}g(z_h\ind{j,h})\\
      &=0.
    \end{align*}
Since $\abs*{Z_j}\leq{}H(H+1)$, we have by \pref{lem:freedman} that for
any $\eta\leq{}1/H(H+1)$, with probability at least $1-\delta$,
\[
  \sum_{j<k}Z_j\leq{}\eta\sum_{j<k}\En_{j-1}\brk*{Z_j^2} + \frac{\log(\delta^{-1})}{\eta}.
\]
Moreover
\[
  \sum_{j<k}\En_{j-1}\brk*{Z_j^2}\leq{}(H+1)^{2}\sum_{j<k}g^{2}(z\ind{j,h}).
\]
Hence, by choosing $\eta=\frac{c}{(H+1)^{2}}$, we are guaranteed that with
probability at least $1-\delta$,
\[
\sum_{j<k}\prn*{\brk*{\Pstar_hV}(x_h\ind{j,h},a_h\ind{j,h})
      - V(x_{h+1}\ind{j,h})+\zeta_h\ind{j,h}}\cdot{}g(z_h\ind{j,h})
  -   c\sum_{j<k}g^{2}(z_h\ind{j,h}) \leq{} \frac{(H+1)^{2}\log(\delta^{-1})}{c}.
\]
The result now follows by taking a union bound over all $g\in\cG$ and $V\in\cV$.
  
\end{proof}

\begin{proof}[\pfref{lem:uniform_conc}]
  Let $k$ and $h$ be fixed. Let $\cH\ind{j}$ denote the entire history for episode $j$. Define a
  filtration
    \[
      \gfilt_{j-1}=\sigma(\cH\ind{1},\ldots,\cH\ind{j-1},
      (s_1\ind{j,h},x\ind{j,h}_1,a_1\ind{j,h}),\ldots,
      (s_{h-1}\ind{j,h},x\ind{j,h}_{h-1},a_{h-1}\ind{j,h}), s_{h}\ind{j,h}),
    \]
    with the convention $\gfilt_{0}=\sigma((s_1\ind{j,h},x\ind{j,h}_1,a_1\ind{j,h}),\ldots,
      (s_{h-1}\ind{j,h},x\ind{j,h}_{h-1},a_{h-1}\ind{j,h}), s_{h}\ind{j,h})$. This filtration
    guarantees that $(x_j\ind{j,h},a_j\ind{j,h})$ is
    $\gfilt_j$-measurable and $\gfilt_{j-1}\subseteq\gfilt_j$.

    Let
    $f,f'\in\cF_h$ be fixed, and let
    $X_j=(f(x_j\ind{j,h},a_j\ind{j,h})-f'(x_j\ind{j,h},a_j\ind{j,h}))^{2}$
    and $Z_j=X_j-\En_{j-1}\brk*{X_j}$, where $\En_{j-1}\brk*{\cdot}=\En\brk*{\cdot\mid{}\gfilt_{j-1}}$. Observe
    that $\abs*{Z_j}\leq{}H^{2}$. Applying \pref{lem:freedman} to both
    the process $\prn*{Z_j}$ and $\prn*{-Z_j}$, we are guaranteed that for
    any $\lambda\in\brk*{0,1/H^{2}}$, with probability at least
    $1-\delta$,
    \begin{align*}
      \abs*{\sum_{j=1}^{k-1}X_j-\En_{j-1}\brk*{X_j}} &\leq{}
      \lambda\sum_{j=1}^{k-1}\En_{j-1}\brk*{X_j^{2}} +
                                                       \frac{\log(2\delta^{-1})}{\lambda}\\
      &\leq{}
        \lambda{}H^{2}\sum_{j=1}^{k-1}\En_{j-1}\brk*{X_j} + \frac{\log(2\delta^{-1})}{\lambda}.
    \end{align*}
    We have
    \[
      \sum_{j=1}^{k-1}\En_{j-1}\brk*{X_j}
      =
      \sum_{j=1}^{k-1}\En_{x_h\sim{}\emi(s_h\ind{j,h}),a_h\sim\piunif}\brk*{\prn*{f(x_h,a_h)-f'(x_h,a_h)}^{2}}
      = \nrm*{f-f'}_{\cL_h\ind{k}}^{2}.
    \]
    By taking a union bound over all $f,f'\in\cF_h$, and by choosing
    $\lambda=\veps/H^{2}$, we are guaranteed that 
    \begin{align}
      \label{eq:uniform_conc2}
    (1-\veps)\nrm*{f-f'}_{\cL_h\ind{k}}^{2} - c_{\veps,\delta;h}
    \leq{} \nrm*{f-f'}_{\cZ_h\ind{k}}^2 \leq{} (1+\veps)\nrm*{f-f'}_{\cL_h\ind{k}}^{2} +
    c_{\veps,\delta;h}
  \end{align}
for all $f,f'\in\cF_h$. We now deduce the result for the star hull by homogeneity. Consider $f''\in\starhull(\cF_h,f)$. We can
write $f''=t(f'-f) + f$ for some $f'\in\cF_h$ and $t\in\brk*{0,1}$, so
that $f''-f=t(f'-f)$. \pref{eq:uniform_conc2} then implies that
\[
  \nrm*{f-f''}_{\cZ_h\ind{k}}^2=t^2\nrm*{f-f'}_{\cZ_h\ind{k}}^2
  \leq{} (1+\veps)t^2\nrm*{f-f'}_{\cL_h\ind{k}}^{2} +
  t^2c_{\veps,\delta;h} \leq{} (1+\veps)\nrm*{f-f''}_{\cL_h\ind{k}}^{2} +
  c_{\veps,\delta;h},
\]
since $t\leq{}1$. Similarly, we have
\[
  \nrm*{f-f''}_{\cZ_h\ind{k}}^2=t^2\nrm*{f-f'}_{\cZ_h\ind{k}}^2
  \geq{} (1-\veps)t^2\nrm*{f-f'}_{\cL_h\ind{k}}^{2} -
  t^2c_{\veps,\delta;h} \geq{} (1-\veps)\nrm*{f-f''}_{\cL_h\ind{k}}^{2} -
  c_{\veps,\delta;h},
\]
leading to the result.
  \end{proof}

  \begin{proof}[\pfref{lem:w_unif_conc}]
  Let $w\in\cW_{h+1}\ind{k}$ be fixed, and recall that
  $\abs*{w}\leq{}H^{2}$. Define $Z_j=w(x_{h+1}\ind{j,h})$ and
  \[
    \gfilt_{j-1}=\sigma(\cH\ind{1},\ldots,\cH\ind{j-1},
      (s_1\ind{j,h},x\ind{j,h}_1,a_1\ind{j,h}),\ldots,
      (s_{h}\ind{j,h},x\ind{j,h}_{h},a_{h}\ind{j,h}),
      s_{h+1}\ind{j,h}),
    \]
    with the convention $\gfilt_{0}=\sigma((s_1\ind{j,h},x\ind{j,h}_1,a_1\ind{j,h}),\ldots,
      (s_{h}\ind{j,h},x\ind{j,h}_{h},a_{h}\ind{j,h}), s_{h+1}\ind{j,h})$,
    where $\cH\ind{j}$ denotes the entire history for episode
    $j$. \pref{lem:regret_freedman} implies that with probability at
    least $1-\delta$,
    \begin{align*}
\sum_{j<k}w(x_{h+1}\ind{j,h}) = \sum_{j<k}Z_j &\leq{}
                        \frac{3}{2}\sum_{j<k}\En\brk*{Z_j\mid\gfilt_{j-1}} +
                        4H^{2}\log(2\delta^{-1})\\
                                             &=
                                               \frac{3}{2}\sum_{j<k}\En_{x_{h+1}\sim\emi(s_{h+1}\ind{j,h})}\brk*{w(x_{h+1})} +
                        4H^{2}\log(2\delta^{-1}).
    \end{align*}
This proves the first statement for this choice of $w$.  To prove the second statement, we define a new filtration
\[
    \hfilt_{j-1}=\sigma(\cH\ind{1},\ldots,\cH\ind{j-1},
      (s_1\ind{j,h},x\ind{j,h}_1,a_1\ind{j,h}),\ldots,
      (s_{h}\ind{j,h},x\ind{j,h}_{h},a_{h}\ind{j,h})
     ).
   \]
   \pref{lem:regret_freedman} implies that with probability at least
   $1-\delta$,
   \begin{align*}
     \sum_{j<k}\En\brk*{w(x_{h+1})\mid{}x_h=x_{h}\ind{j,h},a_h=a_{h}\ind{j,h}}
     = \sum_{j<k}\En\brk*{Z_j\mid\hfilt_{j-1}}
     &\leq{}
       2\sum_{j<k}Z_j +      8H^{2}\log(2\delta^{-1})\\
     &=
       2\sum_{j<k}w(x_{h+1}\ind{j,h}) + 8H^{2}\log(2\delta^{-1}).
   \end{align*}
The final result now follows from a union bound over all
$(2eK)^{S}\abs*{\cF_{h+1}}$ possible functions in $\cW_{h+1}^{\ind{k}}$.
\end{proof}

\end{document}